\documentclass[a4paper,11pt,twoside]{ThesisStyle}

\usepackage{amsfonts}
\usepackage{amsmath,amssymb}             
\usepackage[english]{babel}
\usepackage[latin1]{inputenc}
\usepackage[T1]{fontenc}
\usepackage[left=1.5in,right=1.3in,top=1.1in,bottom=1.1in,includefoot,includehead,headheight=13.6pt]{geometry}

\usepackage[nottoc, notlof, notlot]{tocbibind}
\usepackage{minitoc}
\usepackage{aecompl}

\usepackage[intoc]{nomencl}

\makenomenclature

\usepackage{ifpdf}
\usepackage{subfigure}

\ifpdf
  \usepackage[pdftex]{graphicx}
  \DeclareGraphicsExtensions{.jpg}
  \usepackage[a4paper,pagebackref,hyperindex=true]{hyperref}
\else
  \usepackage{graphicx}
  \DeclareGraphicsExtensions{.ps,.eps}
  \usepackage[a4paper,dvipdfm,pagebackref,hyperindex=true]{hyperref}
\fi

\graphicspath{{.}{images/}}

\usepackage{color}
\definecolor{linkcol}{rgb}{0,0,0.4}
\definecolor{citecol}{rgb}{0.5,0,0}

\usepackage{rotating}                    
\usepackage{fancyhdr}                    

\let\headruleORIG\headrule
\renewcommand{\headrule}{\color{black} \headruleORIG}

\usepackage{colortbl}
\usepackage{epsfig}
\arrayrulecolor{black}

\fancypagestyle{plain}{
  \fancyhead{}
  \fancyfoot{}
  
}

\usepackage{algorithm}
\usepackage[noend]{algorithmic}

\makeatletter

\def\cleardoublepage{\clearpage\if@twoside \ifodd\c@page\else%
  \hbox{}%
  \thispagestyle{empty}
  \newpage%
  \if@twocolumn\hbox{}\newpage\fi\fi\fi}

\makeatother

{%

\hrulefill
\vspace*{0.5cm}%
\end{minipage}
}

\let\minitocORIG\minitoc
\renewcommand{\minitoc}{\minitocORIG \vspace{1.5em}}

\usepackage{multirow}
\usepackage{slashbox}

{ \begin{list}%
	{$\bullet$}%
	{\setlength{\labelwidth}{25pt}%
	 \setlength{\leftmargin}{30pt}%
	 \setlength{\itemsep}{\parsep}}}%
{ \end{list} }

\newtheorem{theorem}{Theorem}[section]
\newtheorem{lemma}[theorem]{Lemma}
\newtheorem{proposition}[theorem]{Proposition}

\newenvironment{proof}[1][Proof]{\begin{trivlist}
\item[\hskip \labelsep {\bfseries #1}]}{\end{trivlist}}
\newenvironment{definition}[1][Definition]{\begin{trivlist}
\item[\hskip \labelsep {\bfseries #1}]}{\end{trivlist}}

\newcommand{\qed}{\nobreak \ifvmode \relax \else
      \ifdim\lastskip<1.5em \hskip-\lastskip
      \hskip1.5em plus0em minus0.5em \fi \nobreak
      \vrule height0.75em width0.5em depth0.25em\fi}

\renewcommand{\epsilon}{\varepsilon}

\newenvironment{vcenterpage}
{\newpage\vspace*{\fill}\thispagestyle{empty}}
{\vspace*{\fill}}

\newcommand{\ignore}[1]{}

\definecolor{green}{rgb}{0,0.70,0}

\newcommand{\codice}[1]{\mbox{\textsc{#1}}}

\usepackage{fancybox,Modificatopseudocode}

\begin{document}

\begin{titlepage}
\begin{center}
\begin{figure}[htbp]
    \begin{center}
     \epsfig{file=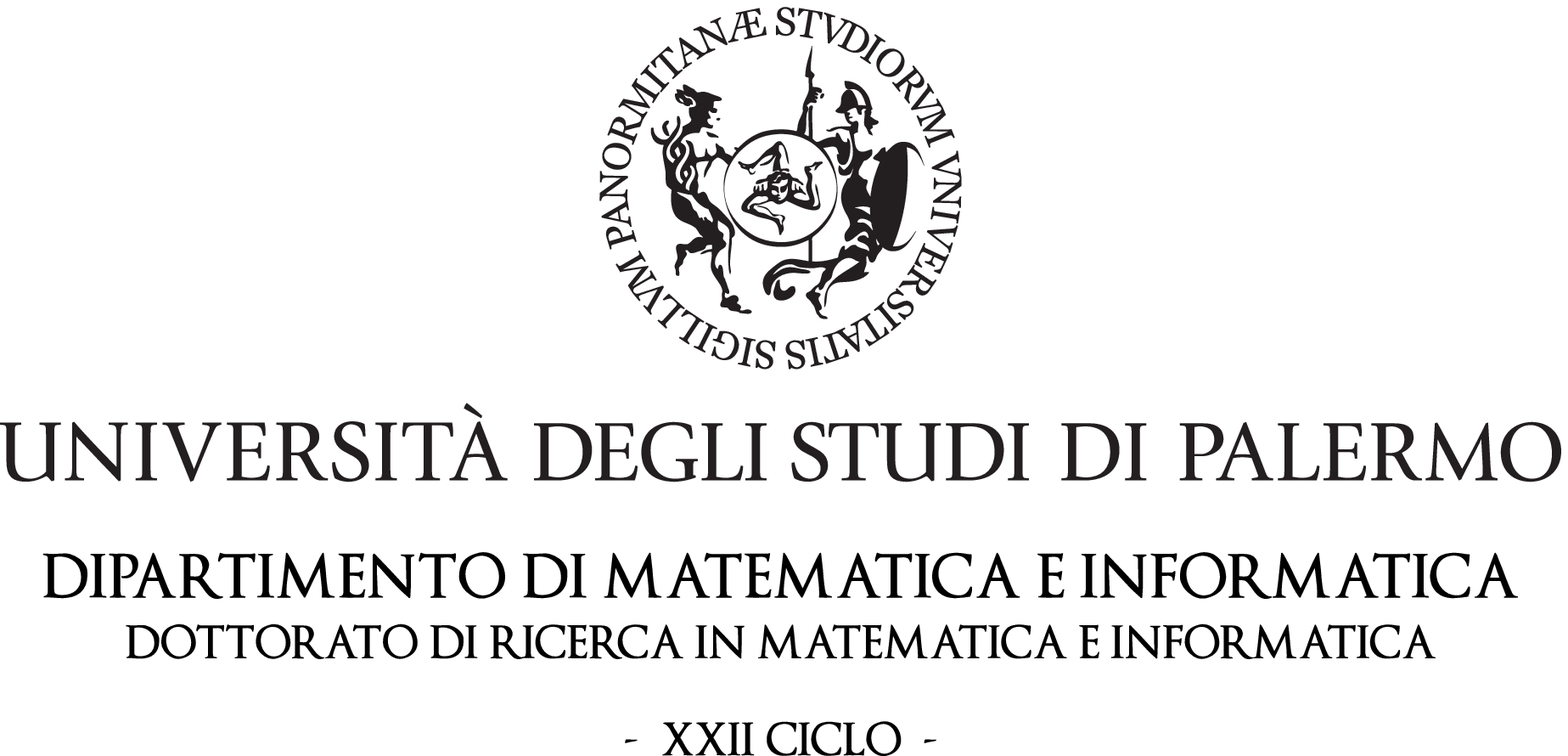,scale=0.7}
    \end{center}
  \end{figure}
\vspace*{0.5cm}
\vspace*{0.5cm}
\vspace*{0.3cm}
\vspace*{0.3cm}
\vspace*{0.3cm}
\vspace*{0.4cm}
\vspace*{0.6cm}
\noindent {\Huge \textbf{\textsc{Multi Layer Analysis}}} \\
\vspace*{5.0cm}

\begin{minipage}[t]{5cm}
\center Author\\ \vspace{-0.2cm}
\center \textsc{Luca Pinello}\\
\end{minipage}
\hfill
\begin{minipage}[t]{5cm}
\center Coordinator\\ \vspace{-0.2cm}
\center \emph{Prof. }\textsc{Camillo Trapani}\\ 
\vskip 0.3cm
\end{minipage}
\vspace{1cm}
\begin{center}
\center Thesis Advisor\\ \vspace{-0.22cm}
\center \textsc{Prof. Domenico Tegolo}\\
\center Co-Advisor\\ \vspace{-0.22cm}
\center \textsc{Dott. Giosuè Lo Bosco}\\
\null\vspace{-0.2cm}\hrulefill\\
Settore Scientifico Disciplinare INF/01
\end{center}

\end{center}
\end{titlepage}
\sloppy

\titlepage


\clearpage
\begin{vcenterpage}
\noindent\rule[2pt]{\textwidth}{0.5pt}
\begin{center}
{\large\textbf{Multi Layer Analysis\\}}
\end{center}
{\large\textbf{Abstract:}}
This thesis presents a new methodology to analyze one-dimensional signals trough a new approach called Multi Layer Analysis, for short MLA. It also provides some new insights on the relationship between one-dimensional signals processed by MLA and tree kernels, test of randomness and signal processing techniques.
\\The MLA approach has a wide range of application to the fields of pattern discovery and matching, computational biology and many other areas of computer science and signal processing. This thesis includes also some applications of this approach to real problems in biology and sismology.\\\\
{\large\textbf{Keywords:}}
multi layer analysis, machine learning, pattern discovery, classification, clustering,  tree kernel, test of randomness.\\
\noindent\rule[2pt]{\textwidth}{0.5pt}
\end{vcenterpage}

\dominitoc

\pagenumbering{roman}

\clearpage

\section*{Acknowledgments}

I owe a great deal of thanks to many people for making this thesis possible. First of all I dedicate this dissertation to Prof. qVito Di Gesù, who has been leading and supporting me and my research to be fruitful in his patience and I'm very sad that unfortunately he can no longer follow me at this important step. I would like to express my gratitude for my current advisor Prof. Domenico Tegolo, who has continued to leading and supporting me in this last year of research.

\medskip

\noindent A huge thanks goes to Giosuè Lo Bosco for his fundamental and precious collaboration in all aspects of my work. Without his skillful and infinite support my projects would not have been possible.

\medskip

\noindent I would specially like to thank Guocheng Yuan for his extremely valuable experience, support, insights and, most important his friendship.

\medskip

\noindent Thanks to my fellow PhD friends, in particular Filippo Utro, Fabio Bellavia, Marco Cipolla, Filippo Millonzi for our broad-ranging discussions and for sharing the joys and worries of the academic research.

\medskip

\noindent Furthermore, I am deeply indebted to my colleagues at Department of Mathematics and Computer Science that have provided the environment for sharing their experiences about the problem issues involved as well as participated in stimulating team exercises developing solutions to the identified problems.

\medskip

\noindent Finally, I wish to express my gratitude to my family and friends who provided continuous understanding, patience, love and energy. In particular, I would like to express a heartfelt thanks to my parents and my girlfriend Valeria for their infinite support in my research endeavors.

\medskip

\medskip

Thanks to all of you.

\clearpage

\section*{Originality Declaration}
This work contains no material which has been accepted for the award of any other degree or diploma in any university or other tertiary institution and, to the best of my knowledge and belief, contains no material previously published or written by another person, except where due reference has been made in the test. I give consent to this copy of my thesis, when deposited in the University Library, begin available for loan and photocopying.

\vspace{2cm}

Signed \ldots\ldots\ldots\ldots\ldots\ldots\ldots\ldots \hspace{150pt} January 2011

\tableofcontents
\listoffigures
\listoftables

\mainmatter
\chapter*{Introduction}
\label{chap:intro}
\addcontentsline{toc}{chapter}{Introduction} \pagenumbering{arabic}\markboth{Introduction}{Contributions and Thesis Outline}

\section*{What this thesis is about}
\addcontentsline{toc}{section}{What this thesis is about} 
This thesis presents a new methodology called Multi Layer Analysis (MLA) that acts a transformation from the space of one-dimensional signals to a new space called space of intervals. The main idea of this approach, shared by several other ones, is the decomposition of the input signal into basic features that allows to better extract its useful information.\\

The main motivation of this study was to develop a new high scalable methodology in order to extract shape information from one-dimensional signals. This because a lot of real problems fall in this context. In fact, several application domains such as Geology, Biomedicine and Biology require the analysis of one-dimensional signals in which their features are encoded in the shapes of whole signals or on the shapes of their sub-fragments (e.g seismic signals, ECG tracks or chip-chip or chip-seq tracks). The kind of analysis obviously depends on the application domains but usually involves  Pattern Discovery, Clustering or Classification methodologies. The main advantages of the MLA compared to other similar methods, are its scalability and the possibility to represent a one-dimensional signal in terms of a tree  of intervals, and this permits to express or characterize explicitly any kind of shape. Consequently, this has strong implications since it establishes a connection between the class of algorithms that process one dimensional signals, such as digital signal processing techniques, and algorithms on trees and graphs.

\section*{Contributions and Thesis Outline}
\addcontentsline{toc}{section}{Our Contributions} 
The MLA methodology can be used as preprocessing step in different fields of application e.g.: Classification,
Clustering, Pattern Discovery and Test of Randomness. Thus, it can be used as tool in the field of data analysis.
More in details:

 \begin{itemize}
   \item This method has been applied to the biological problem of nucleosome positioning providing similar performances to the state of the art method, but better scalability and computational time. This is a fundamental point because it allows to analyze more complex organisms. It is also able to recover the positions of fuzzy nucleosomes.
   \item A new nonparametric test of randomness based on MLA, that exploits shape features that are rare in a random signal, was developed.
   \item It allows to map a one-dimensional signal in a tree of intervals. Consequently some tree kernels, used in different contexts, have been adapted to this representation, providing new kernels that explicitly encode the shape information of a one-dimensional signal expressed as a tree of intervals.
   \item The mapping of a one-dimensional signal in a tree of intervals creates a new  and important connection between two  fundamental classes of algorithms: signal processing algorithms and algorithms on trees and graphs.\\\\
 \end{itemize}

Chapter 1 presents the motivations of MLA, focusing on different methodologies that exploit and share the same idea. Some approaches, at first sight disjointed, but actually exploiting the same idea of multi-resolution or multi-views analysis,  are presented. Some aspects of these methods are related to the MLA analysis; in particular similarities or advantages of one method with respect to the others are highlighted. In addition, all the basic definitions of the problems where the MLA can be productively applied are briefly given.\\

Chapter 2 provides a detailed and formal description of the MLA, explaining step by step the MLA transformation and highlighting its limits and properties. Finally,  some general guidelines on how to use the MLA as a preprocessing step for several problems are provided.\\

Chapter 3 explains how MLA can be integrated in the context of Pattern Discovery and Classification. In addition, a case study that regards a particular biological problem in which the  MLA was successfully used is introduced: the nucleosome spacing. Moreover, an alternative approach for the same problem based on Hidden Markov Model and a comparison of the two methods are presented. Finally, the last section is devoted to the description of a new one-class classifier that was used as new classifier module of the MLA.\\

Chapter 4 presents a new nonparametric test of randomness applicable to a set of one-dimensional signals that takes advantage of MLA preprocessing step. In particular, this procedure is based on the probability density function of the symmetrized Kullback-Leibler distance, estimated via a Monte Carlo simulation on the intervals lengths obtained by MLA. The main advantage of this new approach is to perform an exploratory analysis in order to directly verify the presence of several kinds of structures in an input signal. In particular, this test differs from the other approaches since it exploits shape features that are rare in a random signal.\\

Chapter 5 presents how the MLA can help on designing new kernel functions that explicitly take into account the shape information contained in a one-dimensional signal. The main idea of Kernel Methods is presented, giving more details on a particular subclass of kernel functions applicable to structured data, in particular trees. The MLA is used to define a mapping from the set of one-dimensional signal to the set of trees. Two new kernels that use the MLA representation are finally defined and a case study that regards sismographic signals is presented.

\chapter{Multi-resolution or multi-scale methodologies} \label{chap:1}

The proposed methodology is essentially a multi-level decomposition of a one-dimensional signal. The key point of this method is the multi-level analysis. The idea of ``multi-level'' or ``multi-resolution'' is shared by several apparently disjointed methodologies.

\section{Motivation of Multi Layer Analysis}
Recently the multi-scale or multi-resolution models have been research topics in rapid evolution, with great impact on Computer Science, Applied Mathematic, Image Analysis and Signal Processing. The key idea of the MLA is to obtain several ``views'' or ``features'' of the same input data (at different scale, resolution or in a different domain) in order to perform a better and maybe more understandable analysis. Using this approach it is possible to focus on the regions of interest with a finer resolution, having as a consequence an increase on the precision. The regions of interest can be detected by views or features at lower resolution; in this way it is possible to both obtain better results and an improvement in computational time. The idea of multi-scale analysis comes from the fact that many real systems have different behaviors at different scales. For example in physics there are different laws to describe a phenomenon at different scale or resolution, e.g. classical mechanics for describing the motion of macroscopic objects in opposition to quantum mechanics that describes atoms and molecules. It is not an exaggeration to say that many real problems can be handled using different scales or resolutions. For example the human being organizes his time using seconds, hours, days, weeks, months, years reflecting the multi-scale dynamic of the solar system, using scale depending on the problem he is handling. The folding of a protein can require a time in the scale of seconds, while the scale of vibration of covalent bonds is in the order of $10^{-15}$ seconds. In general, the more details of a system we want to model, the more complex the required laws to describe it becomes.

\subsection{Multi-resolution or multi-scale methodologies}
In the following sections some approaches will be presented, at first sight disjointed, but actually exploiting the same idea of multi-resolution or multi-views analysis. In fact, the shared motivation of all these approaches is that in some cases it is easier, given an input signal, to extract and analyze a set of features or views that represents different information contained in it that analyzes the original signal. This is done by each methodology in different ways but the main idea that connects them is to decompose a signals into simpler parts (in frequency, time domain or in another scale or resolution) and perform the analysis combining the results information on each part. The MLA as well as the other methods exploits the same idea, in which the analysis is performed on several ``parts'' of the original signal obtained, as it will be explained in the next chapter, by a simple operation called threshold. Some aspects of these methods will be related to the MLA analysis in particular where there are strongly similarities or advantages of one method respect to the others.

\subsection{Discrete Fourier Transform}
One of the well-known methods that firstly exploited this idea is the Fourier Transform and in particular its variant for discrete signals called Discrete Fourier Transform (DFT). This transformation is mainly adopted when the information of interest are encoded in the frequency domain of a signal. In fact, the Fourier Transform and its discrete version i.e.  DFT is an operation able to transform a discrete signal from the time domain into the frequency domain. This is done by decomposing it as a linear combination of  sinusoidal components. Here the parts  of the original signal are the pure sinusoids at different frequencies and phases. In more details the DFT decomposes a signal into a discrete spectra composed by its frequency components, while the inverse transform synthesizes the original signal from the frequency components into its spectra\cite{Steven_2007}. More formally: \\

\begin{definition}\emph{DFT}\\
Given a discrete signal $x(n)$ of $N$ samples its $DFT$, and its inverse $DFT$ are defined by these equations:\\

\begin{itemize}
  \item Synthesis equation:
  \begin{equation}
    x(n)=\sum\limits_{k=0}^{N-1}c_ke^\frac{2 \pi j kn}{N}
  \end{equation}\\
  \item Analysis equation:
  \begin{equation}
    c_k=\frac{1}{N} \sum\limits_{n=0}^{N-1}x(n)e^{-\frac{j2 \pi kn}{N}}
 \end{equation}
\end{itemize}

\end{definition}

In more details, the DFT allows to extract frequency, phase and amplitude information of the sinusoids coming from the decomposition of a signal.
In addition, with the DFT, it is possible to find the frequency response of a system from its impulse response and viceversa. In this way it is possible to analyze a system in the frequency domain as it is possible to use the convolution to analyze a signal in the time domain. This approach in some sense extracts several views of the same input signal correspondent to the frequency components that it contains. However one of the main limitation of this approach is that it not perform well for non-stationary signals, and in addition it cannot characterize directly the shapes contained in a signal as it is possible instead to do with the MLA analysis.

\subsection{Wavelet Analysis}
A method that overcome some limitations of the Fourier Analysis is the Wavelet Analysis. A wavelet is a mathematical function and it is used to decompose a signal in components with different frequencies, resolutions and positions\cite{Addison_2009_Wavelet}. The position component is particulary useful when the input signal is not stationary i.e. it has been generated by a stochastic process whose joint probability distribution does not change when shifted in time or space. For this reason wavelets are become popular and nowadays are widely used in multi-resolution analysis. The wavelets transform is the representation of a signal in term of scaled and translated copies of the same function called \emph{mother} wavelet. More in detail, the wavelet transform is obtained by the convolution between a signal and a wavelet function, as illustrated in figure \ref{fig:wavelet2}. It is possible to see in figure \ref{fig:wavelet} an example of scaling and translating a mother wavelet. A mother wavelet needs to satisfy some properties such as finite length and zero mean value. These properties make wavelet analysis more powerful than Fourier analysis since a signal can be decomposed as a sum of the same wavelet properly translated and scaled, instead of using smooth and continuous function like sinusoids. This leads to a good decomposition also in the case of signal that shows discontinuities or in the case of non stationary processes. Figures \ref{fig:haar_wavelet},\ref{fig:mex_hat_wavelet},\ref{fig:morlet_wavelet} show some possible \emph{mother} wavelets.

\begin{figure}[!htb]
\centering
\includegraphics[width=0.7\textwidth]{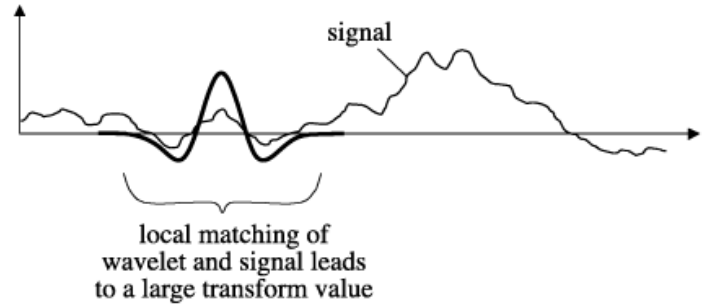}
\caption{Convolution of a signal with a wavelet function. (Part of) this figure  is taken from \cite{Addison_2009_Wavelet} }
\label{fig:wavelet2}
\end{figure}

\begin{figure}[!htb]
\centering
\includegraphics[width=0.7\textwidth]{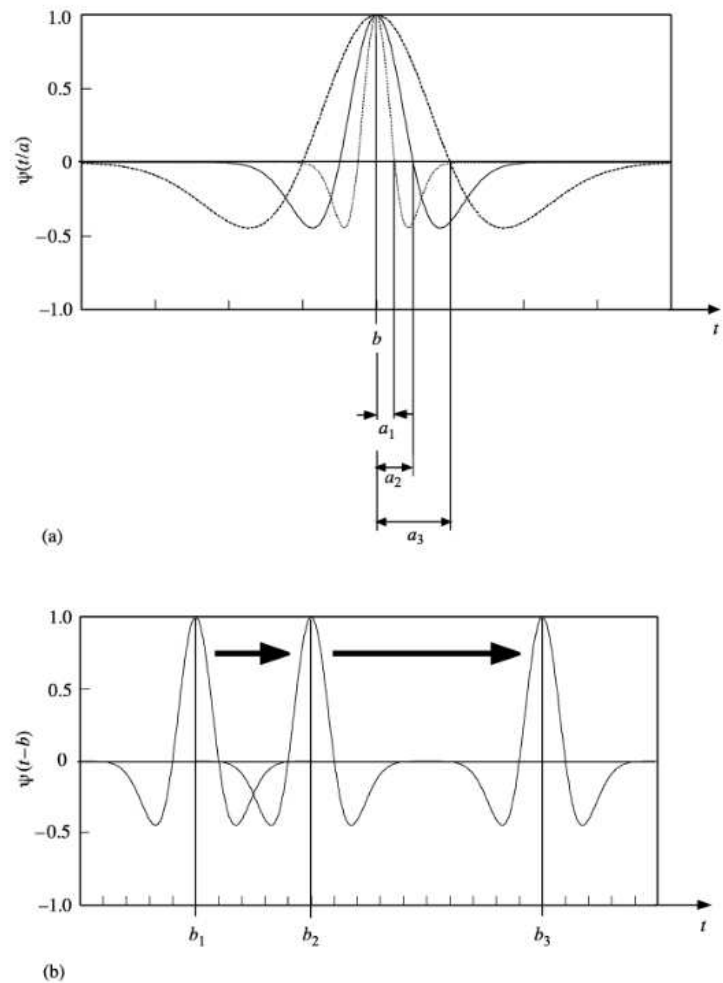}
\caption{Scaling and translation of a mother wavelet. (Part of) this figure  is taken from \cite{Addison_2009_Wavelet} }
\label{fig:wavelet}
\end{figure}

Now it will be formally introduced  the Continuous Wavelet Transform and the Inverse Continuous Wavelet Transform.
\newpage
\begin{definition} \emph{Continuous Wavelet Transform}\\
The continuous wavelet transform or CWT of a continuous signal $x(t)$, considering the mother wavelet $\psi(a,b)$ is defined as:
\begin{equation}
T(a,b) = w(a) \int \limits _{-\infty} ^{\infty} x(t) \psi^* \left ( \frac{t-b}{a} \right) dt
\end{equation}

\noindent where $\psi^*$ is the complex conjugate of the function $\psi$, $w(a)$ is a weighting function usually equal to $\frac{1}{\sqrt a}$ or $\frac{1}{a}$, $a$ control the location of $\psi$ and $b$ its scale.
\end{definition}

\begin{definition}\emph{Inverse Continuous  Wavelet Transform}\\
The continuous inverse wavelet transform or ICWT of the wavelet transform $T(a,b)$ of continuous signal $x(t)$ with respect to the mother wavelet $\psi(a,b)$ is defined as:

\begin{equation}
    x(t) = \frac {1}{C_g} \int \limits _{-\infty} ^{\infty} \int \limits _{0} ^{\infty} T(a,b) \psi_{a,b}(t) \frac{da db}{a^2}
\end{equation}
\noindent where $a$ control the location of $\psi$ used and $b$ its scale.
\end{definition}

\begin{figure}[!htb]
\centering
\includegraphics[scale=.52]{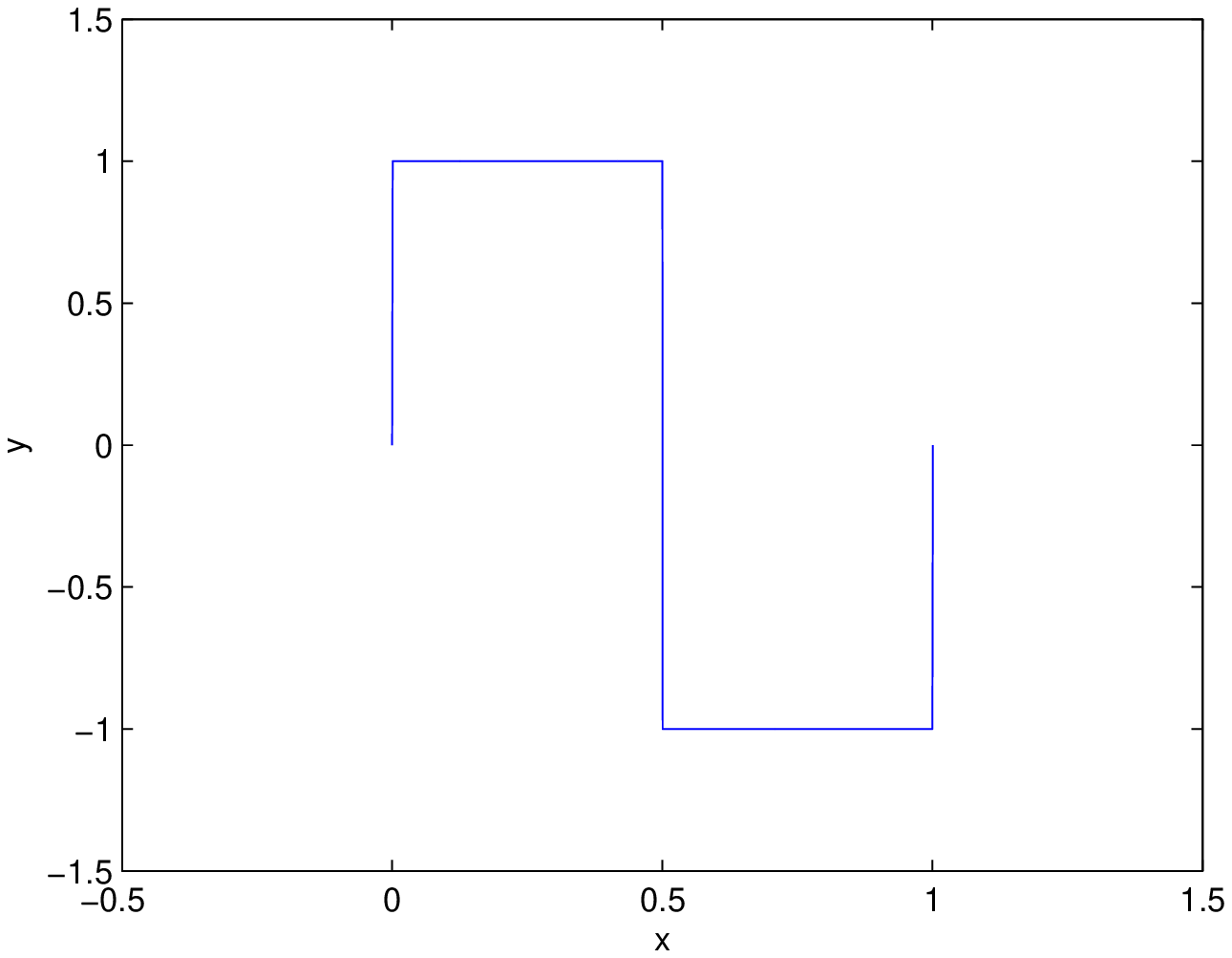}
\caption{Haar wavelet.}
\label{fig:haar_wavelet}
\end{figure}

\begin{figure}[!htb]
\centering
\includegraphics[scale=.52]{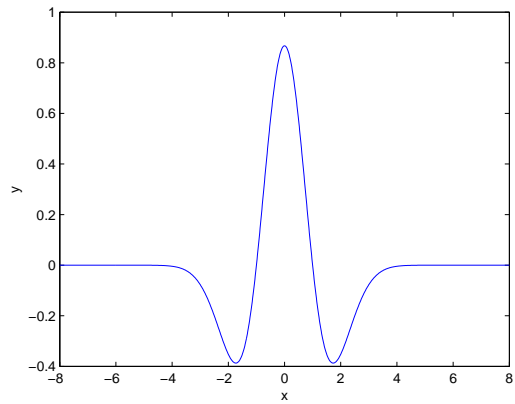}
\caption{Mexican hat wavelet.}
\label{fig:mex_hat_wavelet}
\end{figure}

\begin{figure}[!htb]
\centering
\includegraphics[scale=.52]{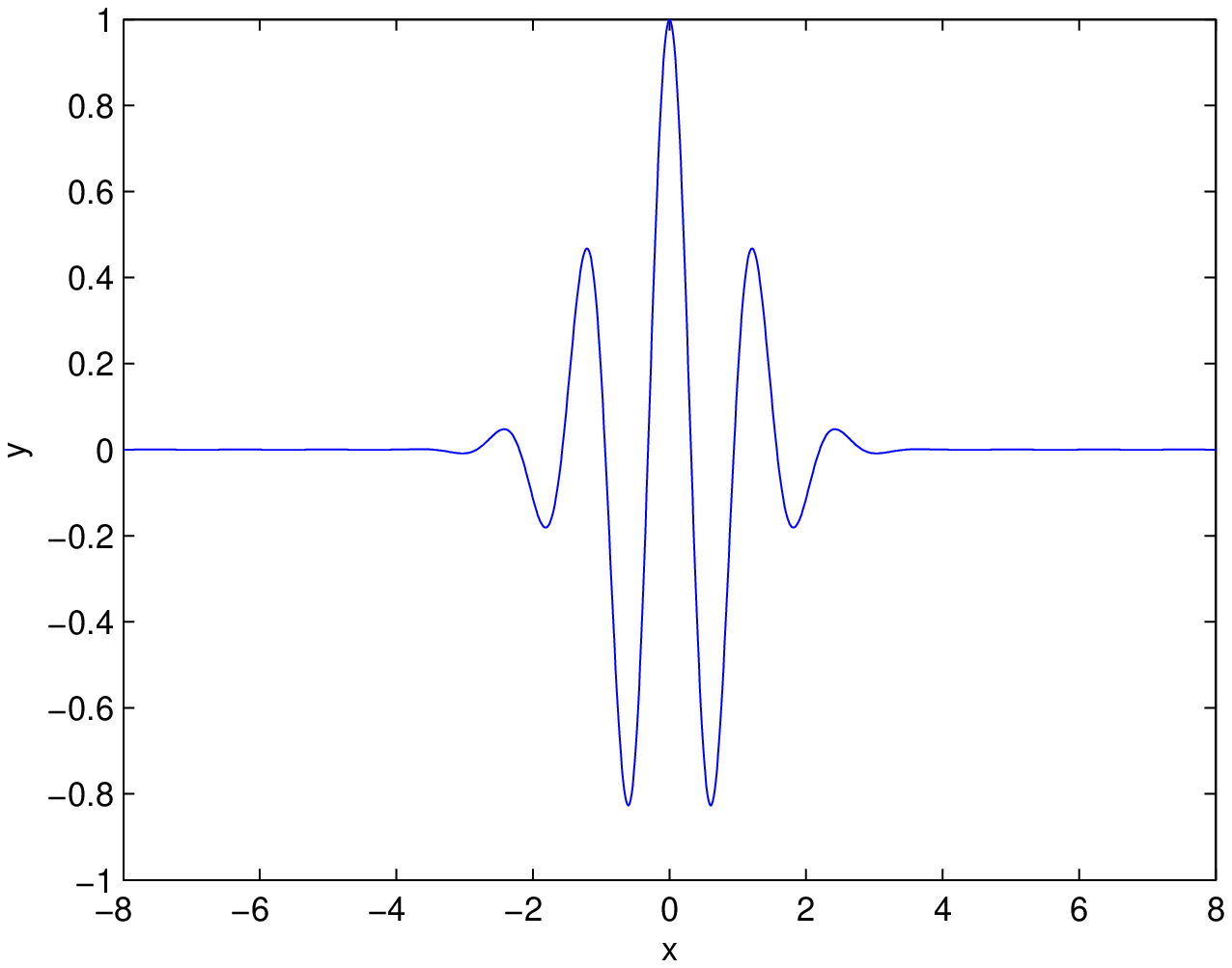}
\caption{Morlet wavelet.}
\label{fig:morlet_wavelet}
\end{figure}

\subsection{Scale Space Theory}
Another methodology that exploits the idea of decomposition of a signal in simpler ``parts'' is the Scale Space Theory that is a framework for a multi-scale representation of signals developed in the fields of computer vision, image processing and signal processing \cite{Lindeberg_1990_Scalespace}. It is a formal theory applied to manipulate signals of one or more dimensions at different scales. Here the ``parts'' of a signal are structures or features at different scales contained in it and as in the wavelet approach the parts are obtained by a convolution of a base signal at different scales. The main difference is how the convolution is performed and how the information of the parts are combined. The concept of scale space is general and  it can be used in  an arbitrary number of dimensions. For simplicity, here the most used framework, that is the case of linear scale space in two dimensions, will be described.

\begin{definition}\emph{Linear Scale Space}\\
Given a two-dimensional signal $f (x,y)$ (e.g. an image), its linear scale space is a family of derived signals $L(x,y,t)$ defined by the convolution of signal $f(x,y)$ with a Gaussian kernel $g$:

  \begin{equation}
    g(x,y,t)=\frac{1}{2\pi t}e^{-\frac{x^2 +y^2}{2t}}
  \end{equation}\\
such that:
 \begin{equation}
    L(x,y,t)=g(x,y,t)*f(x,y)
 \end{equation}

Where $t = \sigma^2$ is the variance of the Gaussian.
\end{definition}
The reason for generating a scale space representation of an image, for example, derives from the consideration that real world objects consist of different structures at different scales. This implies that the real-world objects are different from those of the idealized mathematical entities, such as points or lines, and may appear differently depending on  the scale we use to observe them. For example, the concept of tree is appropriate if we think in the scale of meters, while the concept of leaf requires a finer scale. For example, a machine vision system  that has to analyze an unknown scene, cannot know in advance which scales are appropriate to describe the data in the scene. For this reason, a reasonable approach is to consider descriptions of the scene at different scales simultaneously. An example of this approach is illustrated in figure \ref{fig:scalespace}.

\begin{figure}[!hb]
  \centering
  \subfigure[ $L(x,y,t)$ at scale $t = 0$ (original image)]{\label{fig:s0}\includegraphics[width=0.3\textwidth, bb=0 0 286 205]{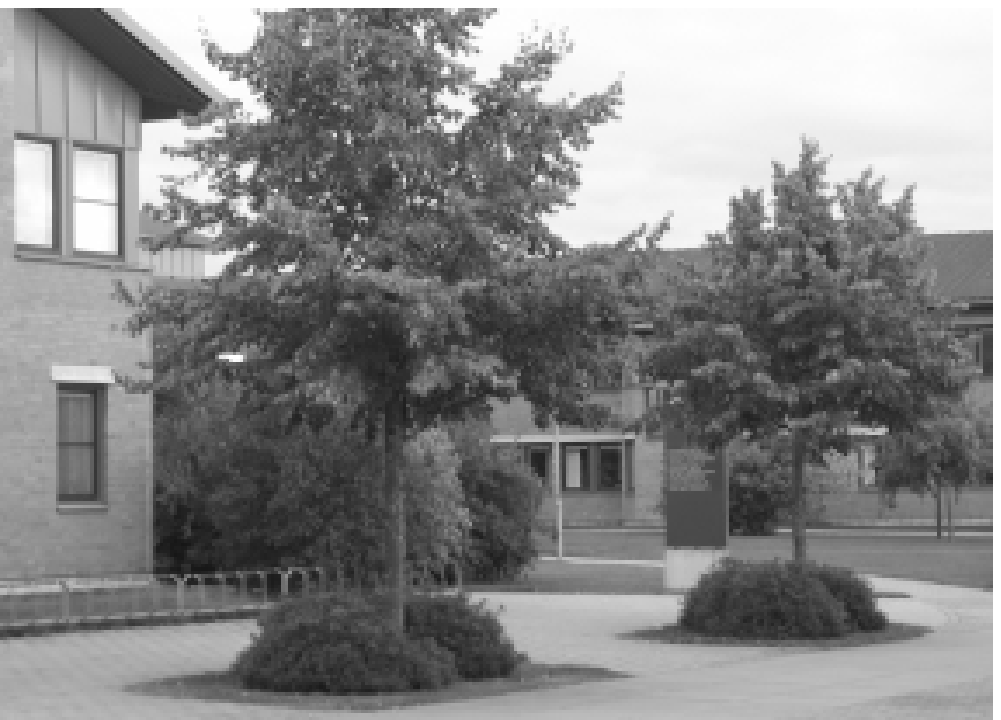}}
  \subfigure[$L(x,y,t)$ at scale $t = 1$]{\label{fig:s1}\includegraphics[width=0.3\textwidth, bb=0 0 286 205]{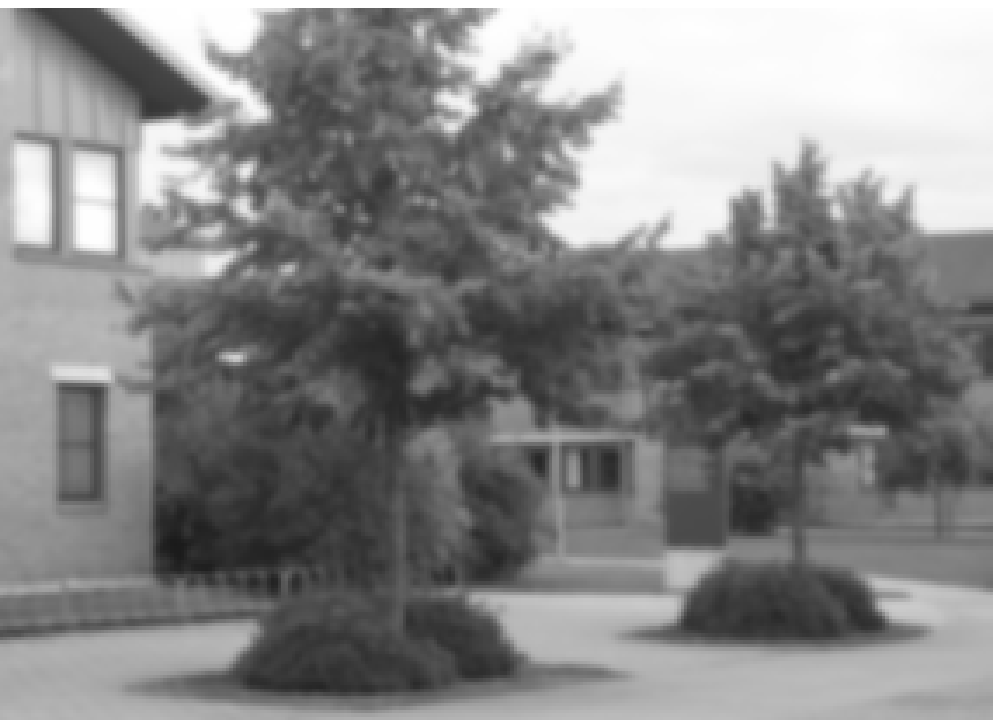}}
  \subfigure[$L(x,y,t)$ at scale $t = 4$]{\label{fig:s2}\includegraphics[width=0.3\textwidth,bb=0 0 286 205]{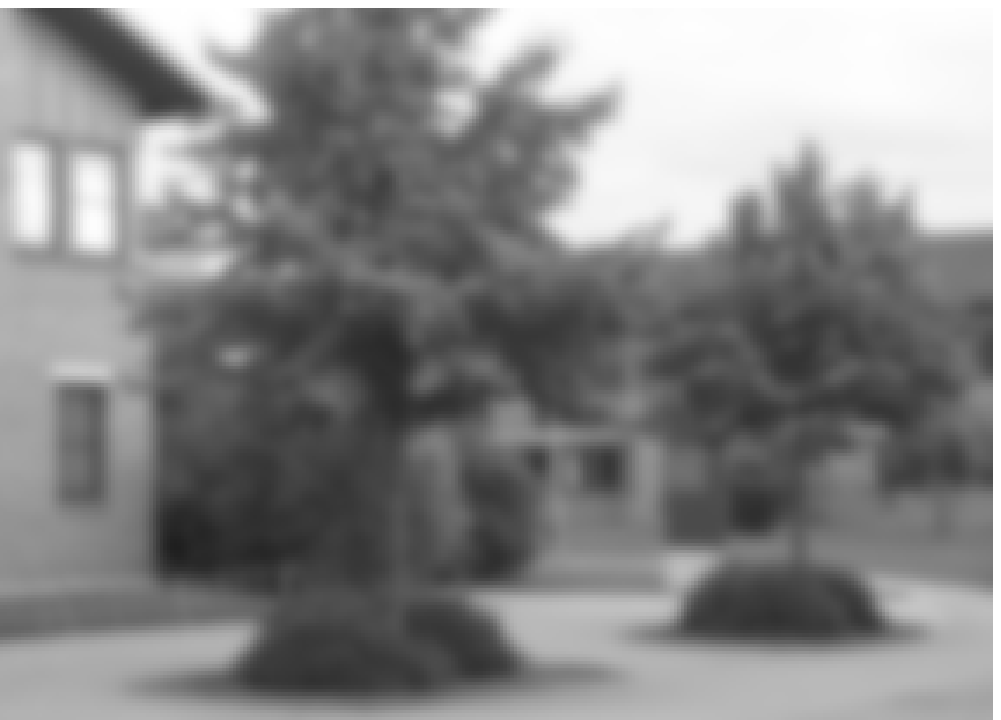}}
  \subfigure[$L(x,y,t)$ at scale $t = 16$]{\label{fig:s3}\includegraphics[width=0.3\textwidth,bb=0 0 286 205]{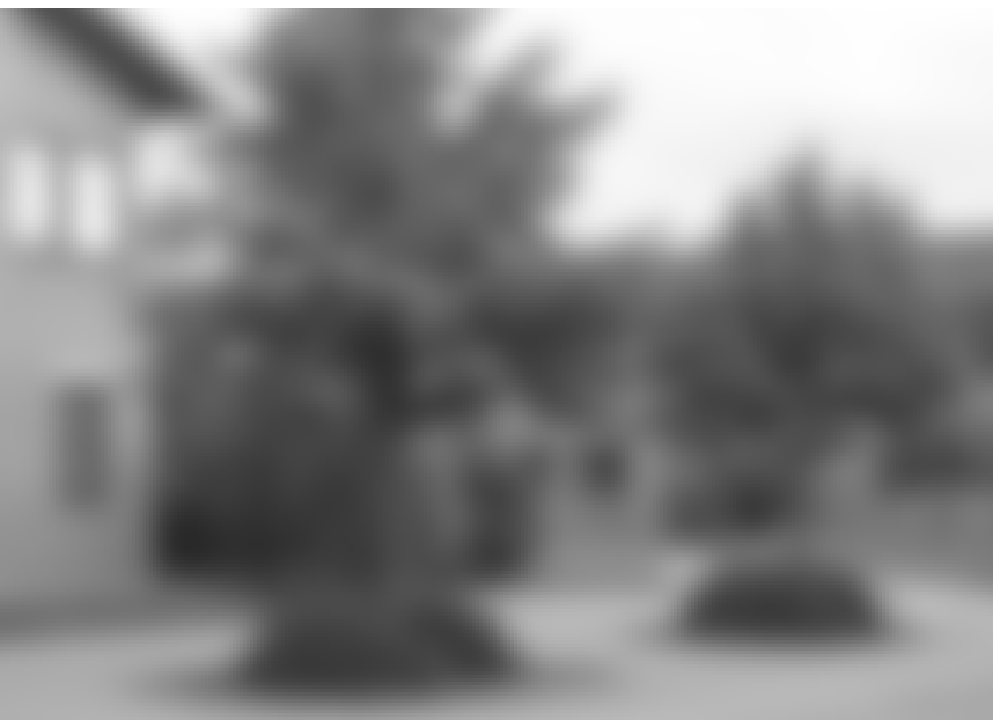}}
  \subfigure[$L(x,y,t)$ at scale $t = 64$]{\label{fig:s4}\includegraphics[width=0.3\textwidth,bb=0 0 286 205]{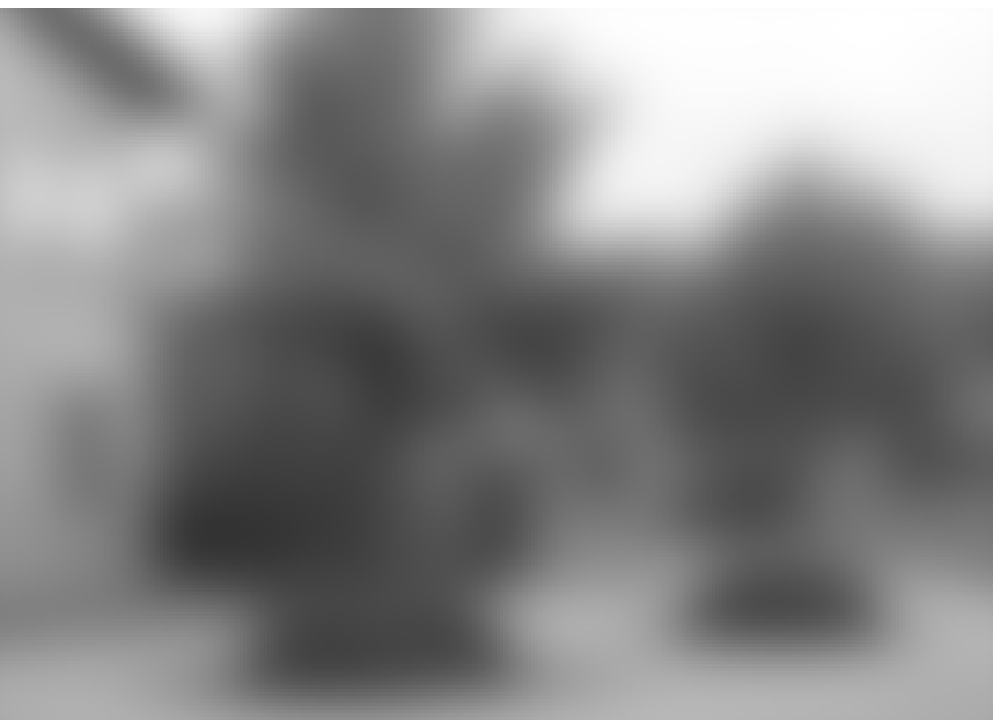}}
  \subfigure[$L(x,y,t)$ at scale $t = 256$]{\label{fig:s5}\includegraphics[width=0.3\textwidth,bb=0 0 286 205]{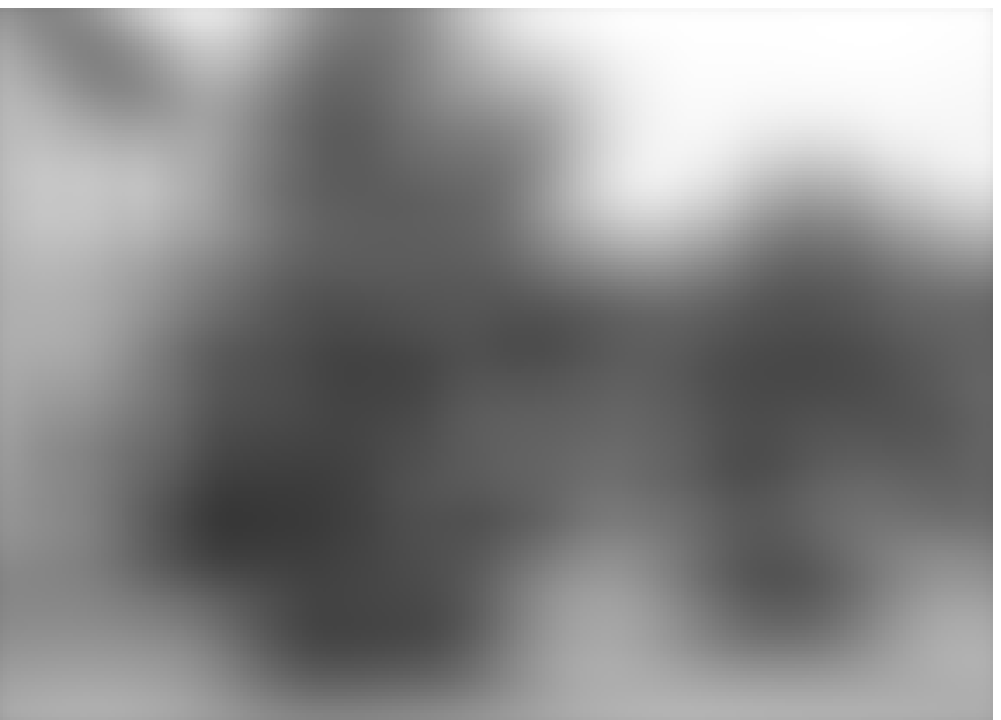}}
  \caption{Scale Space representation}
  \label{fig:scalespace}
\end{figure}

\subsection{Quadtree Analysis}
Quadtree Analysis is another image analysis technique that consists in  iteratively splitting an image into blocks that are more homogeneous than the image itself by using a particular data structure called \emph{quadtree} \cite{De_Berg_2000}. This technique, examining the image at different resolutions, allows to obtain information about its structure. It is also used as the first step in adaptive algorithms for image compression. The technique consists in dividing a square image into four blocks of equal size, and then test whether each block meets some homogeneity criterion (for example, if the gray levels of all pixels belonging into a block have a specific range of values). If the block meets the criterion it  will not be further splitted, otherwise it will be again divided into four blocks that will be tested again according to some homogeneity criterion. This process is iterated until each block meets the criteria. The entire process obviously   will split the image into blocks of different sizes. An example of quadtree analysis, used to detect salient objects in an image, is shown in figure \ref{fig:quadtree}.

\begin{figure}[!htb]
   \centering
   \includegraphics[width=0.7\textwidth]{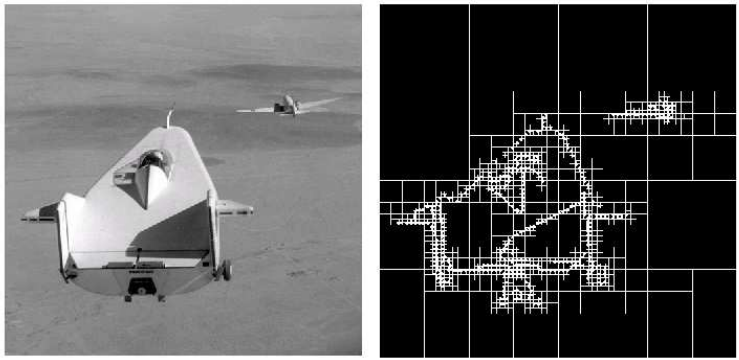}
 \caption{Quadtree image segmentation}
  \label{fig:quadtree}
\end{figure}


\subsection{String methods}
In a lot of discipline the input data comes in natural form as string: bio-sequences, graphs and text documents. In this scenario there are several methodologies that exploit the ``multi views'' approach in terms of subsequences or substrings of the input string. For example there are several similarity measures between string objects in which the more similar, the greater it is the number of the factors they share \cite{Mantaci_2008_DistMeasuresBioSeq}. Another example that will be presented in detail in chapter 5 is the family of convolution kernels\cite{Haussler_1999_ConvolutionKernels}. The basic idea of a convolution kernel is to decompose a data object into simpler parts and then define a kernel function in terms of such parts. A very common kernel for string classification (especially protein sequences) that exploits this idea is the spectrum kernel. The main idea behind it, is that the more substrings with a fixed length are shared by two string, the more similar they are (see \cite{Leslie_2002_SpectrumKernel} for details). More formally let's consider the following definition:
\begin{definition} \emph{Spectrum Kernel}\\
\noindent Let $\Sigma$ be a finite alphabet, $\Sigma^*$ denote all possible string over $\Sigma$ and $\Sigma^k$ all the string over $\Sigma$ of length $k$.
Let $\#x[w]$ denote the number of occurrences of $w$ in $x$ e.g.  $\#x[w]=|\{y|x =y\cdot w\cdot z \wedge y,z \in \Sigma^* \}|$ and
$G_k[x]$ the k-gram vector of $x$ over all the string in $\Sigma^k$ e.g. $G_k[x]=(\#x[w])_{w \in \Sigma^k}.$
Given a $k \in \mathbb{N}$ the spectrum kernel can be defined as:
\begin{equation}
    S_k(s_1,s_2)=\sum\limits_{w\in\Sigma^k}^{} \# s_1[w] \cdot \# s_2[w] = \langle G_k[s_1],G_k[s_2] \rangle
  \end{equation}
\end{definition}

\subsection{Level Set}
Another approach that decompose a signal in parts and that is very close to the MLA is the Level Set method that is a numerical technique for the recognition of shapes in a signal \cite{Sethian:level_set}. This methods is based on the fact that usually, it is easier to characterize a shape using a particular set of auxiliary functions called \emph{Level Sets} than using the shape directly. In fact the level sets allow to characterize a shape considering several of its levels or subviews. In figure \ref{fig:level_set} it is possible to see a pictorial representation of this approach on a function of $2$ variables. Now it will be provided the formal definition of Level Set:

\begin{definition} \emph{Level Set of a function }\\
Starting from a function $f:\mathbb{R}^n->\mathbb{R}$ a \emph{level set} is a set of the form:
\begin{equation}
\{(x_1,\ldots , x_n) | f (x_1,\ldots , x_n)=k\}
\end{equation}

\noindent If $n=2$ this set is called \emph{level curve}, if $n=3$ the set is called \emph{level surface} or more in general if  $n>3$ it is called \emph{level hypersurface}. In particular using a level set  it is possible to express a closed curve $\Gamma$ indirectly using the the function $f$ and considering the level set: $\Gamma=\{(x_1,\ldots , x_n) | f (x_1,\ldots , x_n)=0\}$
\end{definition}

As it will be possible to see later, the MLA idea in some sense is very close to this approach since the information that characterize the signal are similar. The main difference is the way the information are organized, in fact with MLA it is possible to characterize any shape in a natural and elegant way using a particular structure to store these information.

\begin{figure}[htb]
   \centering
   \includegraphics[width=.7 \textwidth]{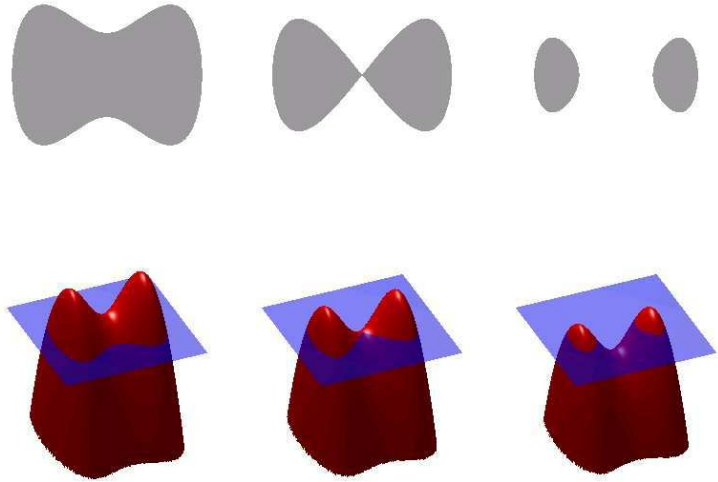}
 \caption{Level Set representation for a function depending on 2 variables.}
  \label{fig:level_set}
\end{figure}

\section{Pattern Discovery and Classification}
The next section presents two machine learning techniques in which the MLA can be promiscuously integrated. For this reason here will be introduced the general problems of Pattern Discovery and Classification, while in chapter 3 will be cover in  detail how to integrate the MLA in these contexts.

\subsection{Pattern Discovery}\label{pattern_discovery}
Pattern discovery is a general discipline in which the main goal is to process a large amounts of data in order to efficiently extract unknown useful knowledge  \cite{Andrew_pattern_discovery}. In other words a pattern discovery method discovers subsets of input data that are meaningful accordingly to a formal criteria. More in general, the pattern discovery is a research area that provides efficient methods to uncover, without using ``a priory'' knowledge on the data, patterns that are repetitive, unexpected or interesting, using a formal criteria.

In order to better understand pattern discovery, it is first necessary to define the meaning of pattern.  Informally a pattern is any relation in the data that is of our interest and that is not casual or random. In other words it is necessary to answer to the question: \emph{how meaningful is a pattern?} This is because the human mind has the tendency to see patterns everywhere. For this reason,  it is necessary to understand if a pattern is significative in a rigorous way. More formally, a pattern is a data vector serving to describe an anomalously high local density of data points \cite{Hand_pattern_discovery}. This means that particular points have a different behavior than the points in other regions usually called ``background'' and that are not interesting since in those regions they have a behavior not related to the true process that has generated the ``anomalies''.

During the last years a lot of attention was paid to this problem so that it is possible to find several tools in the realm of Statistic and in the Computer Science to address this problem. In particular these techniques can be fruitfully applied to several unconnected application domains such as: speech recognition, biology, finance and econometric, biomedicine, text analysis, statistics. As a matter of fact the data involved in the pattern discovery methods are of different kinds such as sequence, image, sound and structured data such as tree and graphs \cite{Andrew_pattern_discovery,Cooley_webmining,Bolton_stat,Park88unsupervisedpattern,Lee_patterndiscovery,Vilo02patterndiscovery}.

\subsection{General schema of a Pattern Discovery method}
A general pattern discovery method can be subdivided in three main parts\cite{Vilo02patterndiscovery} as it is possible to see in figure \ref{fig:pattern_discovery}.

\begin{figure}[htb!]
    \centering
    \includegraphics [scale=.5]{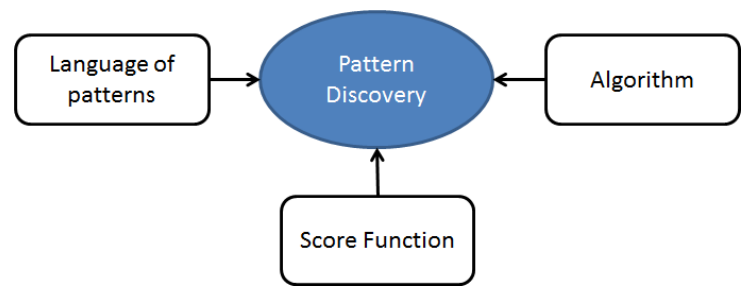}
    \caption {Pattern Discovery parts}
    \label{fig:pattern_discovery}
\end{figure}

\begin{itemize}
  \item A \textbf{language} to describe the pattern;
  \item a \textbf{score function} to assesses the interestingness of a pattern;
  \item an efficient \textbf{algorithm} that identifies the most interesting patterns using the score function.
\end{itemize}

\noindent Obviously, these three parts depend strongly on the particular application domains taken into consideration.
In particular, this is true for the language used to describe the patterns, in fact the data are not always
in the form of  feature vectors or in term of some formal languages (or grammars). In this sense languages can be thought as a transformation that encodes the information present in the data in a suitable form for a particular score function. Another important point is the choice of the most suitable score function for the particular process that has generated the data, in order to discover the ``anomalies''. The last but not least important point is the scalability of the algorithm that is fundamental in many practical application domains. In particular, this last point usually depends on the complexity of the language used to express the patterns and on the computational efficiency of the score function. For this reason it is necessary to consider a compromise between the expressivity and the computational efficiency of languages and score functions.

\subsection{Classification}\label{classification}
In recent years, several algorithms have been developed for classification, but all allow, albeit with different techniques, to match a set of elements defined over a space of features, with a set of labels corresponding to different groups or classes \cite{DUDA}. This is equivalent to partition the space of features into regions, assigning to each region a specific label. In general, classification refers to the class of methodologies of machine learning that given in input a set of data assign subparts of the input data to a given class taken from a finite number of categories. More formally, let's consider a set of observations $X \in \mathbb{R}^n$, a set of elements $Y={y_1, \ldots, y_M}$ called labels and a function $f:X \rightarrow Y$ that defines the true mapping from the set $X$ of observations to the set of labels. A classification algorithms considering a set $D={(x_1,y_1), \ldots ,(x_n,y_n)}$ called \emph{training set}  produce in output a function $g:X \rightarrow Y$ that approximate as close as possible the function $f$.
The classification can also be seen as a problem of parameter estimation, where the goal is to estimate a set of functions of the form:
\begin{equation}
P\left( {class|x } \right) = f\left( { x ;\overrightarrow \theta  } \right)
\end{equation}

\noindent where $x \in X$ represents the vector of input features for each item to be classified, and $f$ is a function depending on a vector of parameters denoted by $\overrightarrow \theta  $ related to the specific classification problem. This function represents the probability that the element represented by the vector of characteristics, belongs to a particular class.
In any case, the classification process generally follows the following steps:

\begin{enumerate}
  \item Selection of the classes of interest;
  \item Selection of the set of training;
  \item Statistical analysis of the set of training in order to assess whether they represent well the problem being tackled;
  \item Algorithm Selection for classification;
  \item Classification of data using the chosen  algorithm;
  \item Validation of the results and their interpretation.
\end{enumerate}

The most common algorithms to perform classification are: Bayesian Classifier, K-Nearest Neighbors, Support Vector Machine, Decision Tree and Neural Networks. The interested readers can found a good survey of the principal classification algorithms here \cite{DUDA}.

\chapter{Multi Layer Analysis}
\label{chap:2}

In this chapter  a detailed and formal description of the Multi Layer Analysis (MLA) will be presented.
The MLA is a general feature extraction method that can be adapted to discover patterns on one-dimensional signals
or as a preprocessing step to classification, clustering and other data analysis techniques.

\section{The Multi Layer Analysis}
The MLA is a feature extraction method in which the processed input data  can be used by a classifier or a clustering method in order to distinguish between several kinds of patterns. It is based on the generation of several sub-samples of the input signal, each one carried out by a particular threshold operation, chosen by respecting cut-set optimal conditions, within respect to the input data. In figure \ref{fig:MA_Schema}, it is shown a flowchart of the whole methodology. As it is possible to see in that figure, the method starting from the input signal and applying a set of simple operations, called thresholds, extracts a set of intervals. These intervals opportunely aggregated can encode the shape information of the input signal that can be used to characterize it or to discover structures contained in it. In the following, the formal definition of the threshold operation will be given, together with some some generic application of this transformation.

\begin{figure}[htb]
    \centering
    \includegraphics[width=0.8 \textwidth]{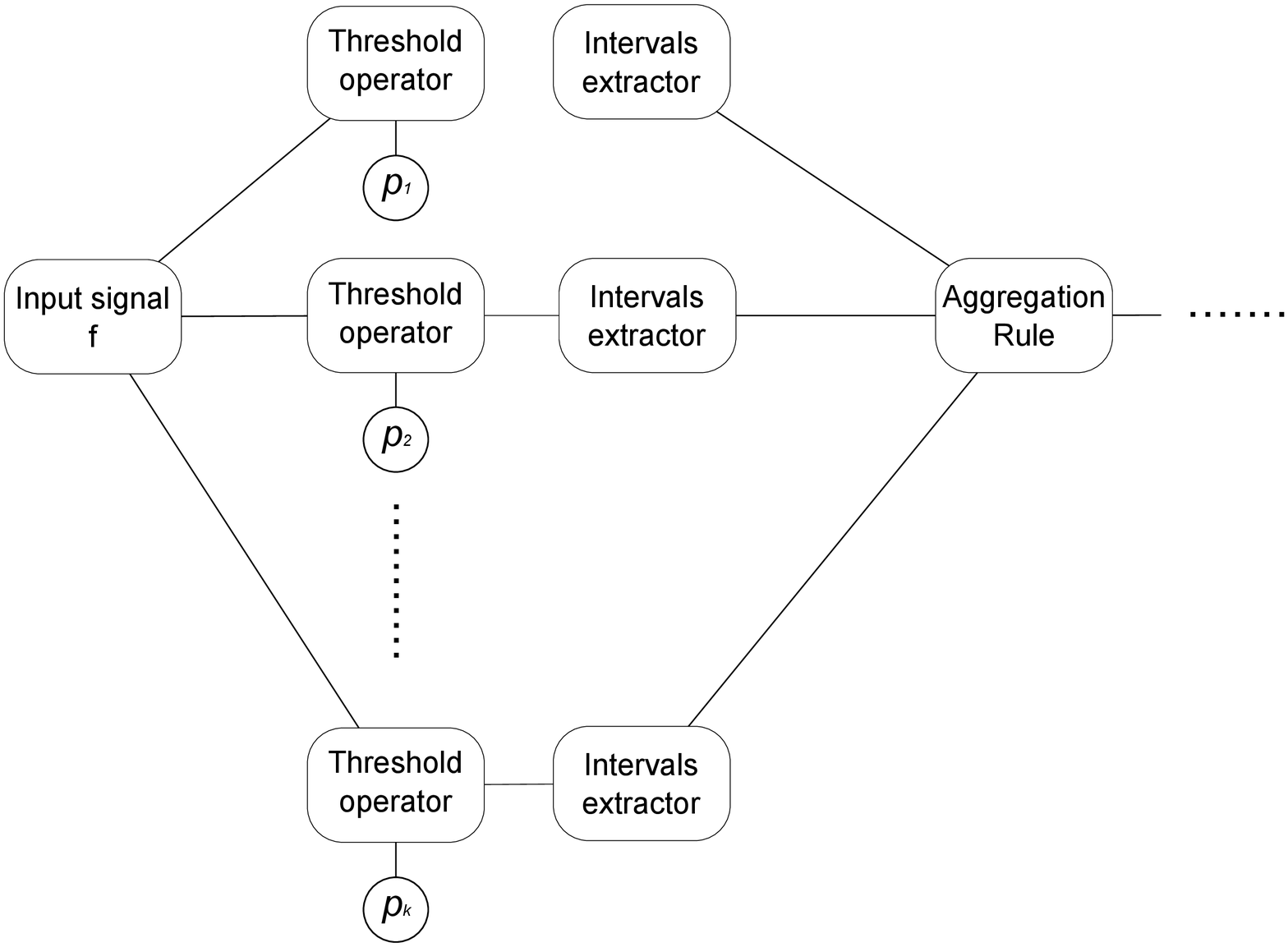}
    \caption{Schema of  MLA processing}\label{fig:MA_Schema}
\end{figure}

\subsection{The threshold operation}
\begin{definition}\emph{Threshold operation}\\
Given an input signal $f$ the threshold operation $\sigma_k$ is defined as follows:

\begin{displaymath}
\sigma_k(x)=
\left\{\begin{array}{lr}
    f(x)    & \mbox{if } p(f(x)) \mbox{ is true}\\
    k       & \mbox{otherwise}\\
\end{array}
\right.
\end{displaymath}
\\
\noindent where $p$ is a generic condition defined on the elements of $f$.
\end{definition}
\noindent In the simplest case $f$ can be defined in $\mathbb{R}$ and it is possible to set:

\begin{equation}
p(f(x))=
\begin{cases}
\text{true} & \text{if $f(x) \leq \phi$}\\
\text{false} & \text{otherwise}
\end{cases}
\label{eq:simple_condition}
\end{equation}

\noindent This approach detects sub-samples deriving from threshold operations that satisfy structural or shape properties. An example of a simple threshold operation with condition expressed in equation \ref{eq:simple_condition} is depicted in figure \ref{fig:thresholds}.\\

\begin{figure}[!htb]
   \centering
   \includegraphics[width=1\textwidth]{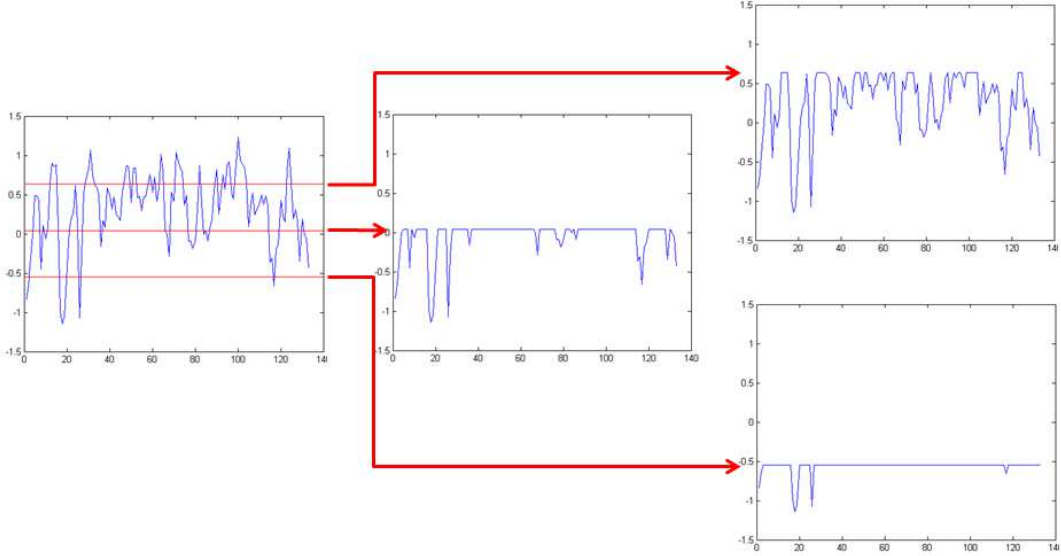}
 \caption{Thresold operation for three different values of $\phi$}
  \label{fig:thresholds}
\end{figure}

The key idea behind the MLA is to explore the input signal at different threshold levels that corresponds to its decomposition into several sub-signals, in order to discover the hidden pattern of interest.

\begin{definition}
\emph{General MLA}\\
The MLA can be defined as a set of sub-samples of a one-dimensional signal $f$
  \begin{equation}
    MLA(f)=\{\sigma_1(x),\sigma_2(x),\cdots,\sigma_K(x)\}
  \end{equation}

\noindent where each threshold operation indicated by the subscript of $\sigma$ could be characterized by a specific condition.
\end{definition}

\noindent The MLA is more accurate and robust in comparison to a naive methodology that, using a single threshold operation could give inaccurate results especially in the real case when the input data is affected by noise. The accuracy and robustness are due the fact that MLA uses more conditions $p$ in order to validate the same hypothesis or conditions on the multiple sub-samples extracted from the input signals $f$. For this reason this technique  introduces a sort of ``flexibility'' to the analysis of a signal. After the multiple threshold operations called \emph{horizontal sampling} it is possible to extract a set of intervals from the original signal and define its \emph{interval representation}; it is also possible to organize these intervals using a particular rule called \emph{aggregation rule}. A summary of the overall process is shown in figure \ref{fig:MA_Schema}. The next two subparagraph explain in details the  \emph{horizontal sampling}, the \emph{interval representation} and the \emph{aggregation rule} of a signal.

\subsection{The Horizontal Sampling, the Intervals Representation and the Aggregation Rule}
The core of the MLA is the interval identification obtained through the horizontal sampling procedure.

\begin{definition}
\emph{Horizontal sampling}\\
Given a bounded signal $f:[\alpha,\beta] \rightarrow \mathbb{R^+}$ and $K \in \mathbb{N}$ threshold operations $\sigma_k$ ($k=1,...,K$) for each $k$  it is possible to build a set of intervals:

\begin{equation}
I_k=\left \{ i^1_k,i^2_k,\cdots, i^{n_k}_k \right \}
\end{equation}
\
\noindent where $i^t_k = [a^{t}_{k},b^{t}_{k}]$ with $t=1,\cdots,n_k$, and $a^{t}_{k},b^{t}_{k}\in \mathbf{R}$
\end{definition}
In the simple case in which the condition $p$ of the generic threshold operation $\sigma_k$ is that expressed in equation \ref{eq:simple_condition} it is easy to prove that $f(a^t_k)=f(b^t_k)=t_k$. After the horizontal sampling process, a different representation of the input signal, called \emph{Interval representation} of $f$ is drawn and it will be denoted with $\Upsilon(f)$.\\

\begin{definition}
\emph{Disambiguation operation}\\
To avoid ambiguities in the case $f$ is discrete i.e. $f:\{1,2,\cdots,L\}\rightarrow \mathbb{R^+}$, and $f(1) \neq min(f)$ or $f(L) \neq min(f)$,  $f$ is transformed into a new signal $f':[a,b] \rightarrow \mathbb{R^+}$:

\begin{displaymath}
f'(x)=
\left\{\begin{array}{lr}
    min(f) & \mbox{if } x=a \bigvee x=b\\
    f(x)   & \mbox{if } 1 \le x \le L\\
\end{array}
\right.
\end{displaymath}

\noindent where

\begin{displaymath}
a=
\left\{\begin{array}{lr}
    0 & \mbox{if } f(1) \neq min(f)\\
    1  & \mbox{otherwise}\\
\end{array}
\right.
\end{displaymath}

\noindent and

\begin{displaymath}
b=
\left\{\begin{array}{lr}
    L+1 & \mbox{if } f(L) \neq min(f)\\
    L  & \mbox{otherwise}\\
\end{array}
\right.
\end{displaymath}
\end{definition}

\begin{definition}
\emph{Interval Representation}\\
Given a signal $f$ and $K$ threshold operations $\sigma_k$ ($k=1,...,K$), and let $I_k=\left \{ i^1_k,i^2_k,\cdots, i^{n_k}_k \right \}$ the set of intervals corresponding to  $\sigma_k$,
then the interval representation of $f$ indicated as $\Upsilon(f)$ is:

\begin{equation}
 \Upsilon(f) =\left \{ I_1,I_2,\cdots, I_K \right \}
\end{equation}

\end{definition}

\begin{definition}
\emph{Aggregation Rule}\\
Given a signal $f$ and its interval representation $\Upsilon(f) =\left \{ I_1,I_2,\cdots, I_K \right \}$ an aggregation rule is a rule that constructs sets of intervals taken from $\Upsilon(f)$ in order to characterize or represent ``interesting'' subparts of $f$. In general it is possible to define several aggregation rules to express different shape properties present in a signal. In the next chapters it will be presented several examples of aggregation rules applied to different application domains.
\end{definition}

\begin{definition}
\label{def:simple_mla}
\emph{Equally spaced simple MLA}
\\Without loss of generality, let assume that $f: \mathbb{R}\rightarrow[0,1]$ and  $K\geq2$. The equally spaced simple MLA is carried out by considering the thresholds $\sigma_k$ with $1 \leq k \leq K$ defined as follow:

\begin{displaymath}
\sigma_k(x,\phi_k)=
\left\{\begin{array}{lr}
    f(x)    & \mbox{if } f(x) \leq \phi_k\\
    \phi_k       & \mbox{otherwise}\\
\end{array}
\right.
\end{displaymath}

\noindent with $\phi_k=\frac{1}{K} \times (k-1)$\\

\noindent As convention the first threshold operation corresponds to $\sigma_1(x,0)$ and the last to $\sigma_K(x,1)$.
Note that all the intervals extracted by the last threshold operation $\sigma_K$ by convention
encompass a single point corresponding to the intersection of the signal with the straight line of equation: $y=1$.
In other words, these intervals $I_K$ have the property that $a_K^t=b_K^t$, $\forall 1 \le t \le K$.
In addition, by definition the first threshold
operation collects only one intervals $[1,L]$ where $L=\beta+1$. An example of
equally spaced simple MLA is depicted in figure \ref{fig:equally_spacing_MLA}.
\end{definition}

\begin{figure}[!htb]
\centering
\includegraphics[width=.6 \textwidth]{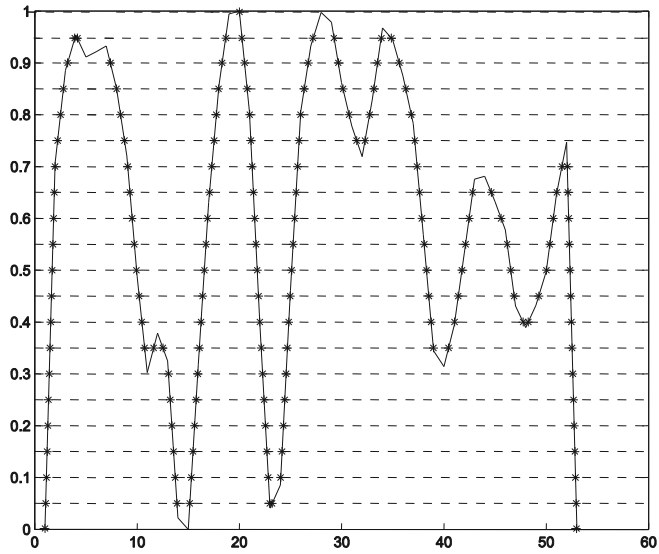}
\caption{Equally spaced simple MLA}
\label{fig:equally_spacing_MLA}
\end{figure}

\noindent In general the interval representation is lossy because it can only keep a subset of points of $f$ that form the intervals in $\Upsilon(f)$ (see figure \ref{fig:MA_intervals}).\\

\begin{figure}[htb!]
    \centering
    \includegraphics[width=0.9 \textwidth]{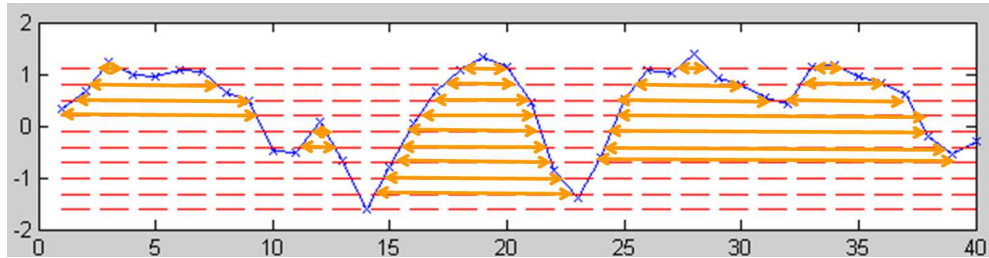}
    \caption {Interval representation of a signal}
    \label{fig:MA_intervals}
\end{figure}

Notice that as many other transformations presented in chapter \ref{chap:1}, by using the MLA it is always possible to reconstruct a lossless version of the input signal if some conditions arise, and this will be discussed later. Obviously, the information loss in this representation decreases as the number $K$ of threshold operations increases. Of course, it is always possible to reconstruct a lossy version of the original signal using an interpolation algorithm and using only the points of its interval representation. Given a generic signal $f$ it is also obvious that it is always possible to obtain a lossless reconstruction of $f$ from its representation $\Upsilon(f)$ as $k\rightarrow\infty$. If $f$ is a discrete signal, it is easy to prove that is always possible to obtain a lossless representation imposing that at least one of the threshold levels intersect each point of $f$, in particular the following theorem gives a way to calculate the minimum number of thresholds operations $K$  to use in order to build a lossless representation using equally spaced thresholds.

\begin{theorem}
\label{cut}
Let $\epsilon_{min}$ be the precision required, and let $f:[\alpha,\beta] \rightarrow [0,1]$ be a discrete time signal of length $L$ ($|[\alpha,\beta]|=L$). Then the lower bound of threshold operations $K$ allowing a lossless representation $h$ of $f$ using the equally spaced simple MLA (i.e. for each pair of adjacent point in $h$,  $d_n=|h(n+1)-h(n)|>=c$ with $c\in\mathbb{R}$) is:
\begin{equation}
K= \frac{1}{g} \sum_{n=1}^{L-1} \left [{\frac{d_n}{\epsilon_{min}}}\right ] \approx \left [ \frac{1}{g \times \epsilon_{min}} \right ]
\end{equation}\\
\noindent with $g$ the $GCD$ (Greatest Common Divisor) between all the integers: $F=\left \{\left [\frac{d_n}{\epsilon_{min}}\right]  ,\quad  n=1,2,\cdots,L \right \}$.
\end{theorem}\label{th:min_eq_K}

\begin{proof}
Using a precision of $\epsilon_{min}$ it is possible to map the set of the absolute differences $D=\{ d_n=|f(n+1)-f(n)|,\quad n=1,2,\cdots,L \}$ in the set of natural numbers $F= \left\{ \left[{ \frac {d_n} {\epsilon_{min}}} \right] ,\quad  n=1,2,\cdots,L \right \} $ and let $g =GCD(F)$. By definition of $g$ it results that \[\left[{\frac{d_n}{\epsilon_{min}}}\right]=g \times m_n\]

\noindent with $m_n \in \mathbb{N}$, and

$$K=\sum_{n=1}^{L-1} m_n = \sum_{n=1}^{L-1} \frac{1}{g} \left [{\frac{d_n}{\epsilon_{min}}}\right ] $$

\end{proof}

\begin{lemma}
Let $\epsilon_{min}$ the precision required, and let $f$ a discrete signal of length $L$ and without loss of generality let us assume that $f$ as values in $[0,1]$.
Then
\begin{equation}
K= \sum_{n=1}^{L-1} \left [{\frac{d_n}{\epsilon_{min}}}\right ]
\end{equation}
is the upper bound on the number of threshold operations $K$ to obtain a lossless representation of $f$ using an equally spaced subdivision of $f$.
\end{lemma}

\begin{proof}
The proof is straightforward, it is possible to obtain the largest $K$ when the $GCD$ $g$ assume its minimum value, this value is 1 because one property of $GCD$ is that $g\geq1$.
\end{proof}

Although the previous theorem and lemma show a lower and upper bound on $K$ allowing a lossless representation of a discrete signal $f$, it is usually convenient for several reasons to optimize the search for the best smallest $K$ allowing a reasonable lossy representation of $f$. It is obvious that the number of threshold operations strongly depends on the signal shape. For this reason, this representation is suggested when the information of the signal is encoded in the time space because it well characterizes  the shape information (as a solution to this problem it could be possible to use the Fourier Transform and apply this methodology on the spectra of the signal). In figure \ref{fig:degradation} it is shown the progressive degradation of a signal as the number of threshold operations decreases, and in table \ref{tab:degradation} the number of points required to represent a signal giving a fixed level of $K$ thresholds, and the correlation coefficient between the original and the reconstructed signal. In the subsection \ref{choose_K} a calibration procedure to select the proper value of $K$ will be described.


\begin{figure}[!hb]
  \centering
  \includegraphics[scale=.4]{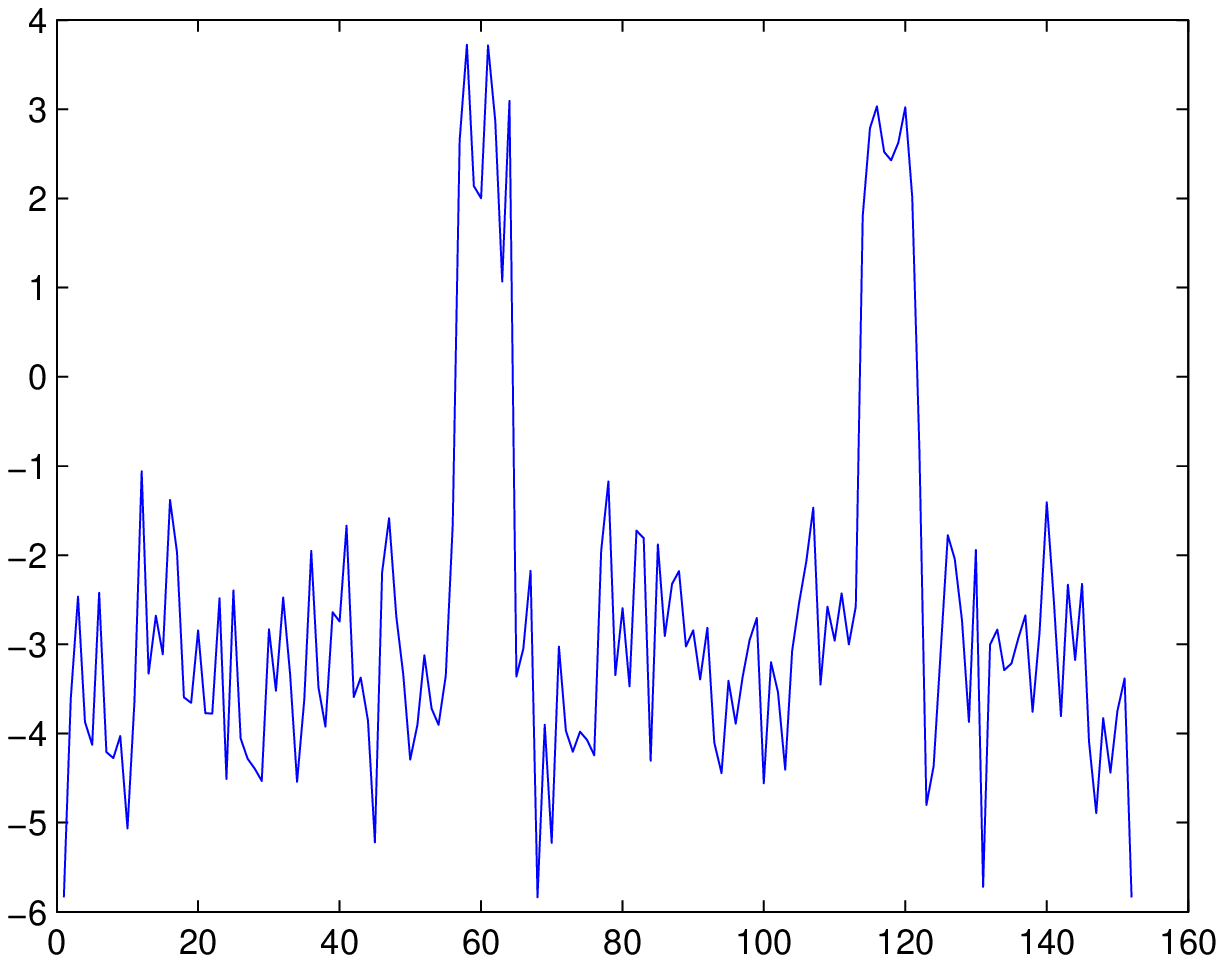}
  \caption{Original signal}

  \subfigure[Signal reconstructed with $K = 3$]{\label{fig:s1}\includegraphics[scale=.3]{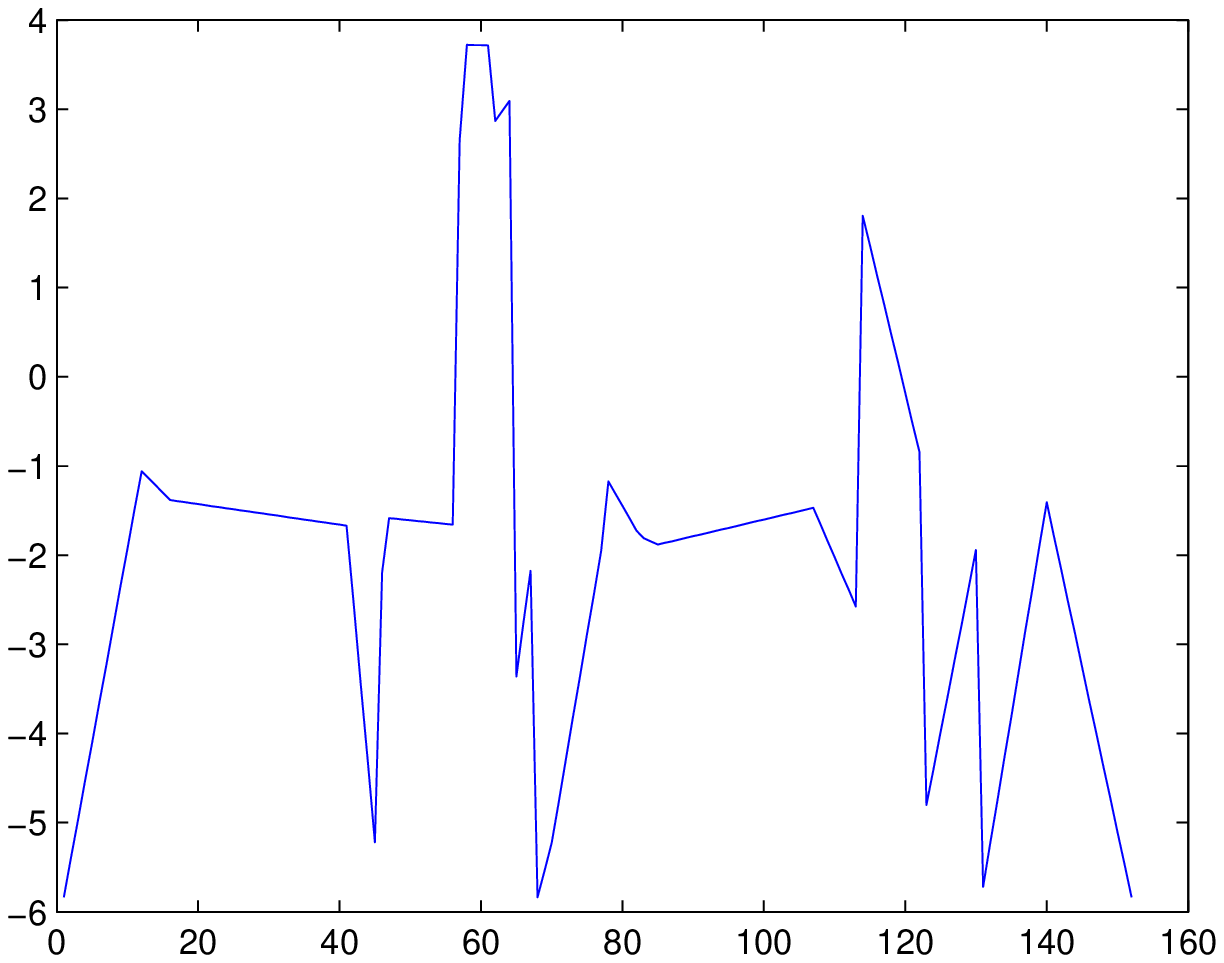}}
  \subfigure[Signal reconstructed with $K = 4$]{\label{fig:s2}\includegraphics[scale=.3]{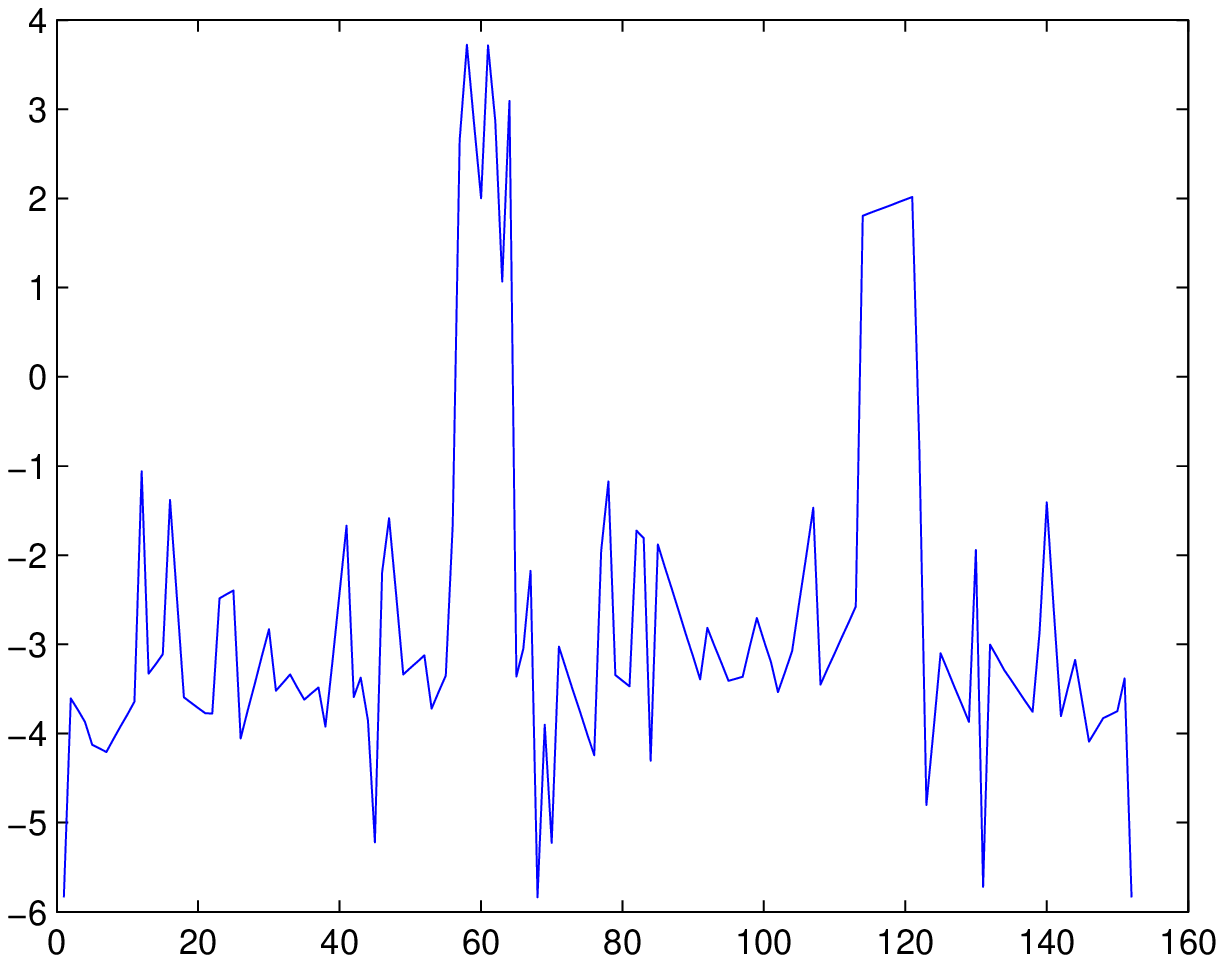}}
  \subfigure[Signal reconstructed with $K = 8$]{\label{fig:s3}\includegraphics[scale=.3]{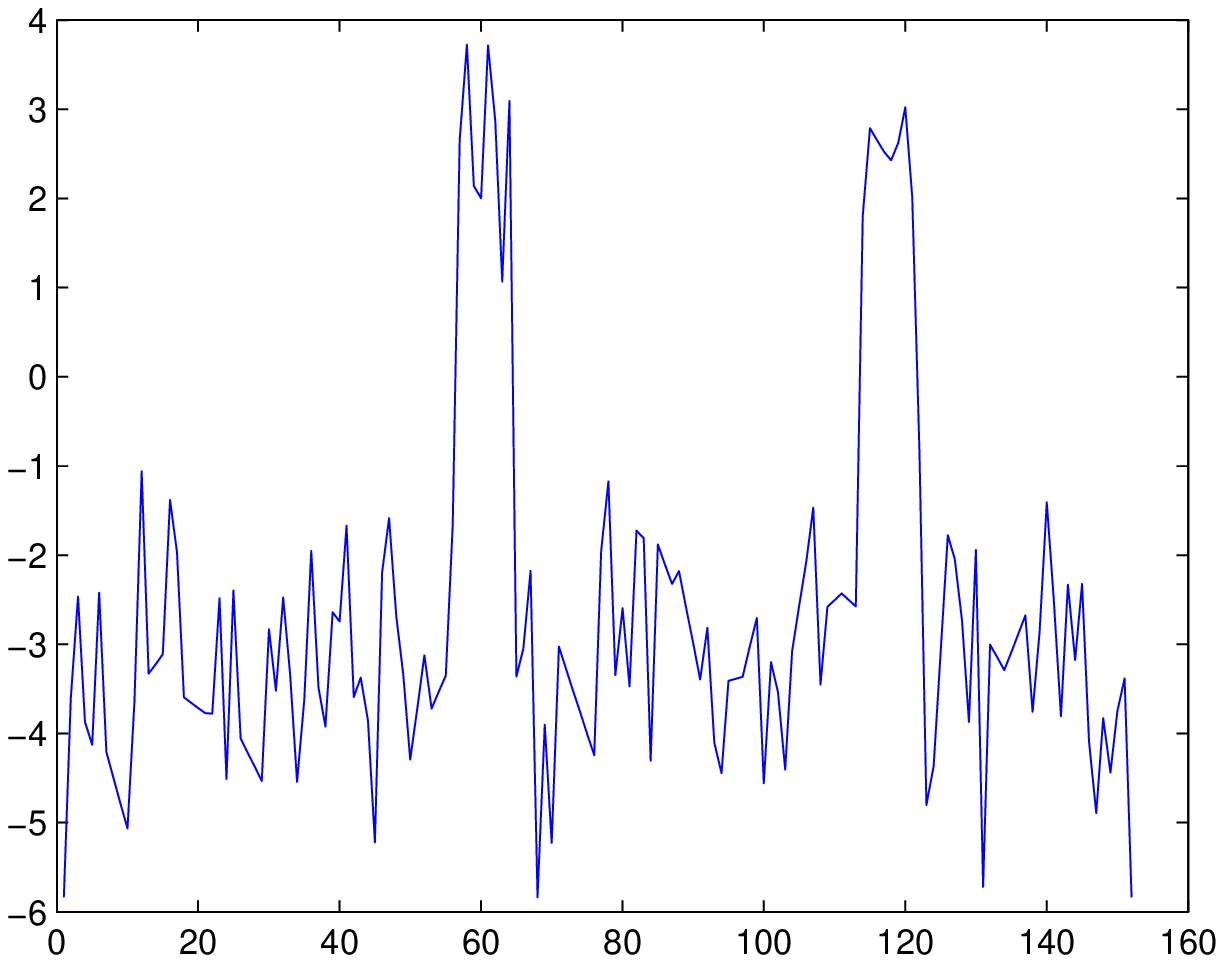}}
  \subfigure[Signal reconstructed with $K = 16$]{\label{fig:s4}\includegraphics[scale=.3]{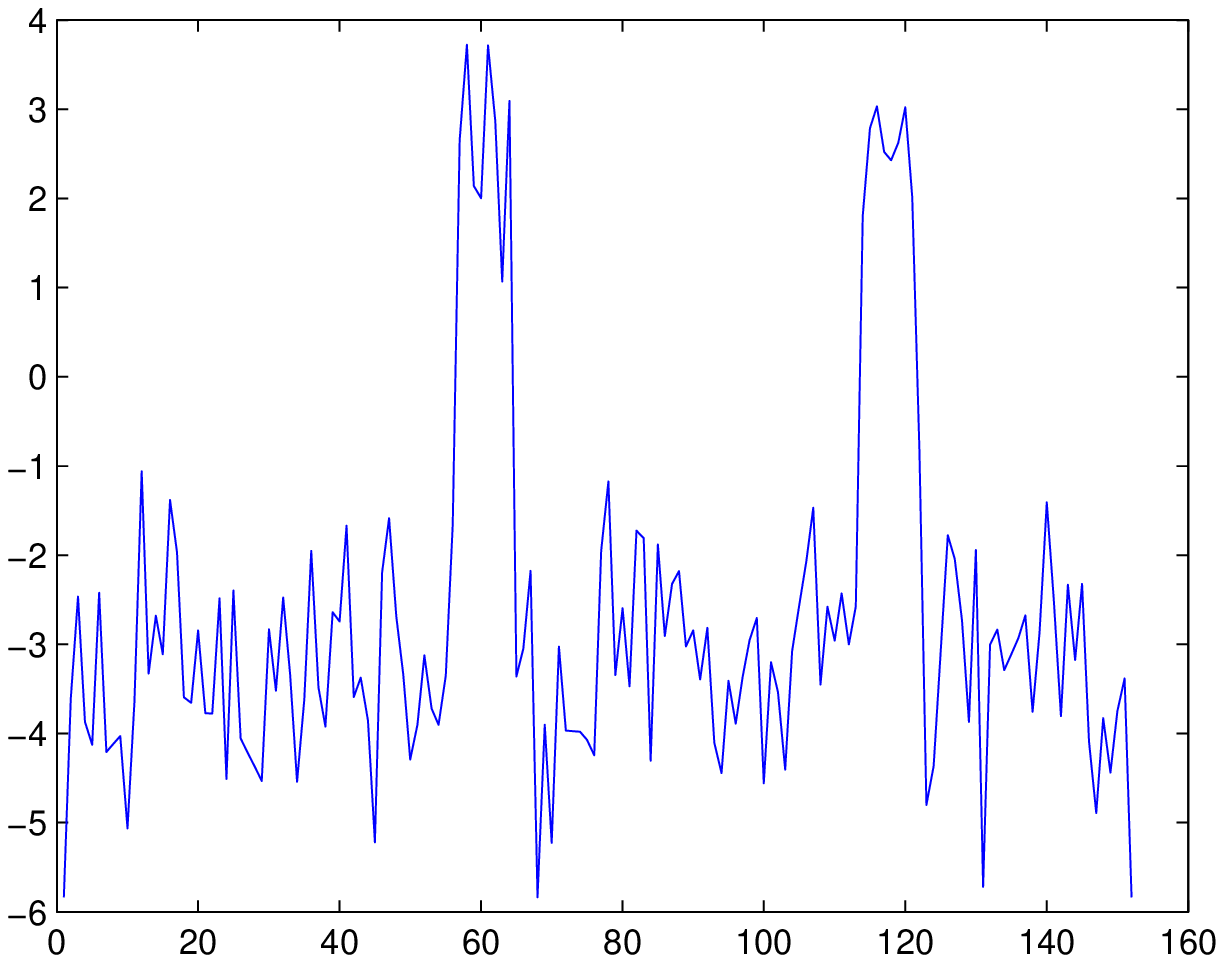}}
  \subfigure[Signal reconstructed with $K = 32$]{\label{fig:s5}\includegraphics[scale=.3]{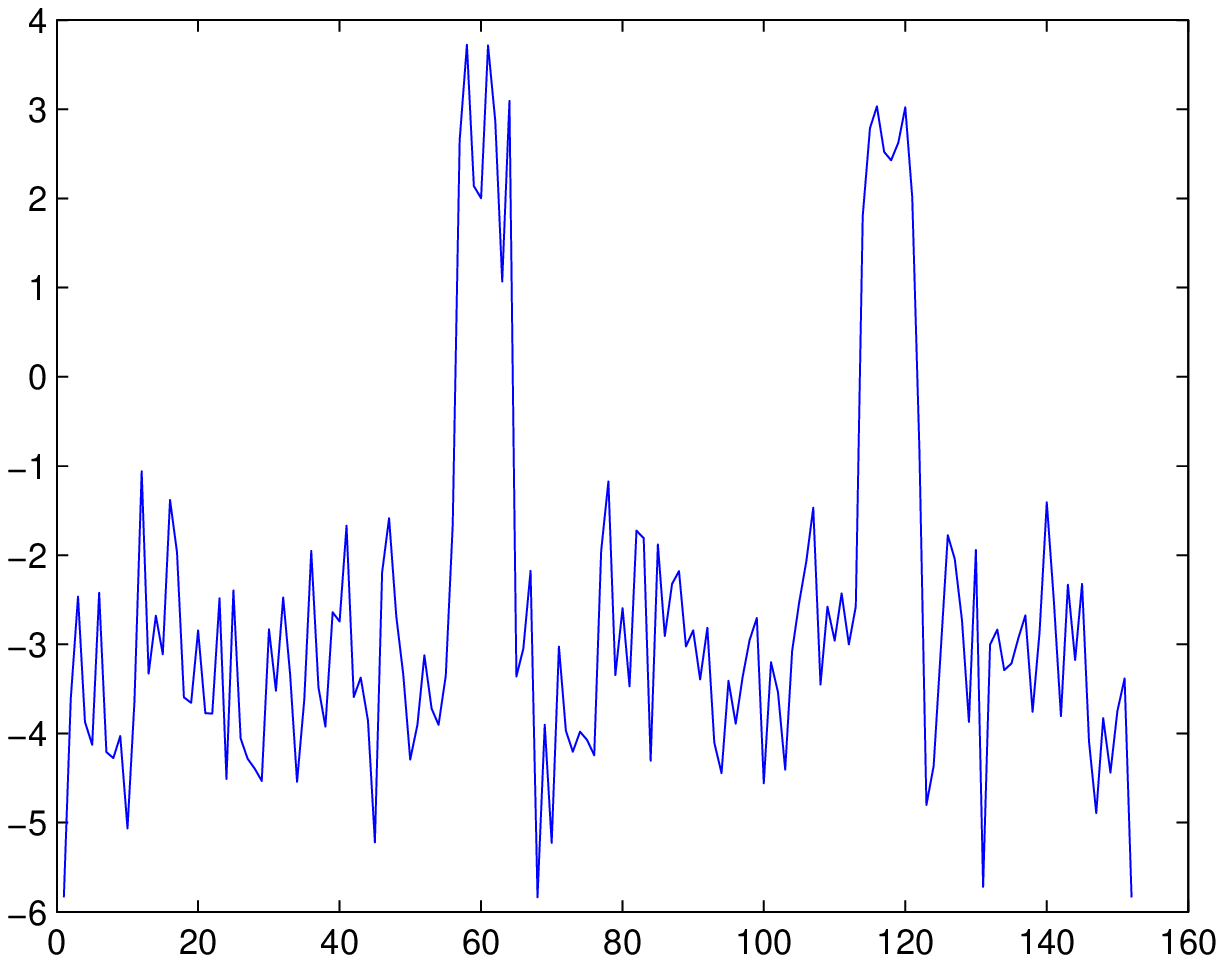}}
  \subfigure[Signal reconstructed with $K = 64$]{\label{fig:s5}\includegraphics[scale=.3]{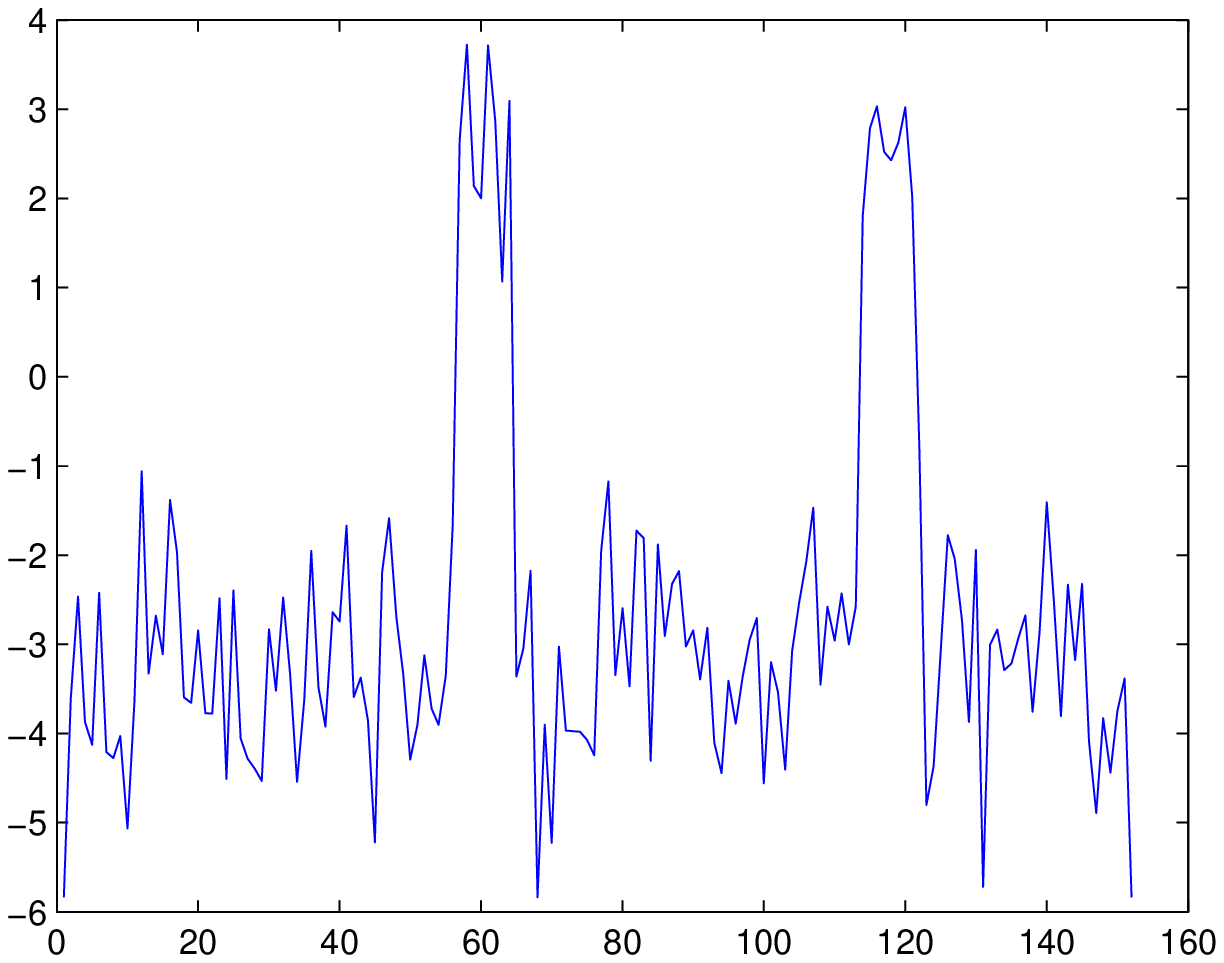}}
  \caption{Degradation of the signal for different values of $K$}
  \label{fig:degradation}
\end{figure}

\begin{table}
\centering
\begin{tabular}{|c|c|c|}
  \hline
  Number of threshold operation $K$ & Kendall Correlation & Length of representation\\
  \hline

  2 & 0.6900 & 4 \\
  4 & 0.2846 & 68 \\
  8 & 0.9420 & 130 \\
  16 & 0.9973 & 280 \\
  32 &  0.9987  & 566 \\
  64 & 0.9999  & 1440 \\ \hline \end{tabular}
\caption{Degradation of the signal for different values of $K$}
\label{tab:degradation}
\end{table}
Note that this transformation cannot be simply related to the theory of sampling and in particular to the Sampling Theorem\cite{Proakis:DSP}, because the non trivial distortion in the spectral components of the original signal that MLA could be introduce.

\begin{theorem}
\emph{(Sampling Theorem \cite{Proakis:DSP})}
If the highest frequency contained in an analog signals , $x_a(t )$ is $F_{max} = B$ and the signal is sampled at a rate $F_s > 2F_{max} = 2B$, then $x_a(t )$
can be exactly recovered from its sample values using the interpolation function:
\begin{equation}
g(t)=\frac{sin (2 \pi B t)}{2 \pi B t}
\end{equation}
\end{theorem}

In other words it does not exist a simple mathematical relation that link the two transformations because they extract different information from the signal, frequency and shape information as stressed before. As a enlighten example, consider two simple but opposite cases: a sinusoidal signal and a rectangular pulse signal. Looking at the figure \ref{fig:sin_comp} and \ref{fig:rect_comp} it is clear that this transformation introduces artifact on the spectrum for the simple sinusoidal signal, that can be represented only by one component with the Fourier Transform, but  it is not present any artifact on the rectangular pulse signal that, in the continuous case, require infinite components to be represented properly in the frequency domain. In other words the number of threshold operations doesn't depend directly on the frequency content of the input signal but only on the quantization levels needed to properly represent it. The quantization levels are obviously proportional to the smallest variation $\epsilon_{min}$ that it is necessary to capture in the signal. If it is necessary to obtain in term of threshold operations an equally spaced ``horizontally sampling'' as in the case of \emph{equally spaced simple MLA} it is possible to use the theorem \ref{th:min_eq_K}.

\begin{figure}[!htb]
\centering
\includegraphics [width=1\textwidth]{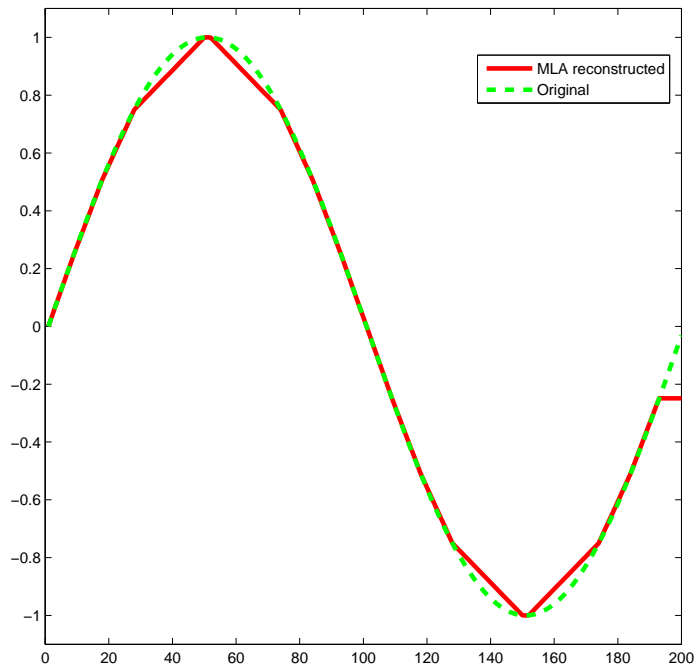}
\caption{MLA reconstruction of the simple sinusoidal signal with $K=8$}
\label{fig:sin_comp}
\end{figure}



\begin{figure}[!htb]
\centering
\includegraphics [width=1\textwidth]{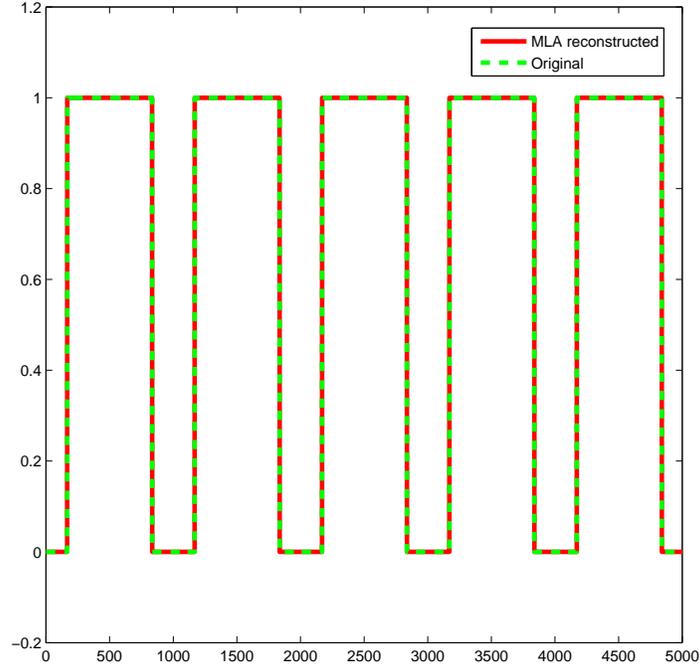}
\caption{MLA reconstruction of the rectangular pulse signal with $K=2$}
\label{fig:rect_comp}
\end{figure}

In some sense the MLA representation is related to the wavelet representation. In fact it is possible to think a signal as composed by scaled and shifted components (in sense of wavelet components) in which the mother wavelet is a single rectangle pulse as depicted in figure \ref{fig:mla_wavelet}. The main difference with wavelet approach is that in MLA transformation the data are represented in a different way and the MLA ``mother'' doesn't need to have mean zero although it has finite duration.

\begin{figure}[!htb]
\centering
\includegraphics [width=1\textwidth]{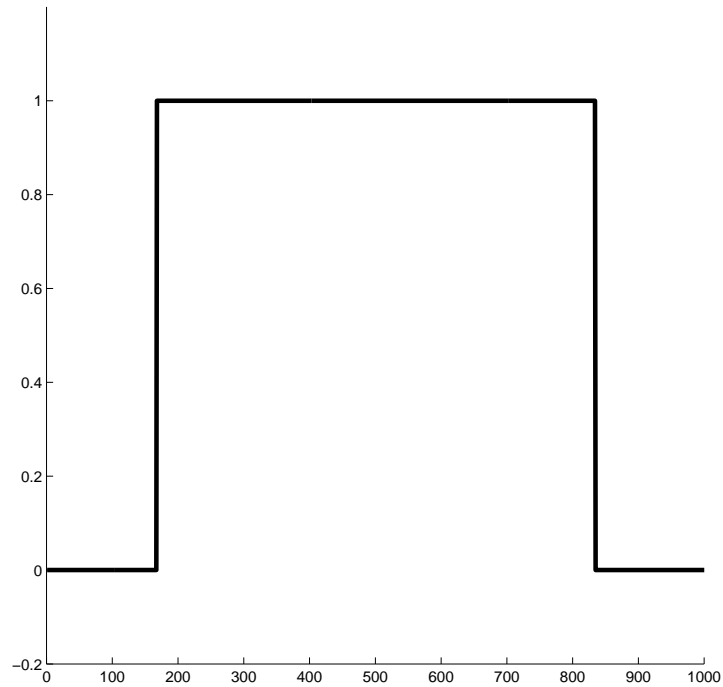}
\caption{MLA ``mother'' function.}
\label{fig:mla_wavelet}
\end{figure}

\section{Choosing the right value for the number of thresholds}\label{choose_K}
The bounds on the values of $K$ given a quantization precision of $\epsilon_{min}$ in the case of $N$ equally spaced thresholds have been previously stated. An interesting question is: is it necessary to use all the levels that the upper bound stated in theorem \ref{th:min_eq_K}? The short answer is no. A practical approach to follow, is to define a similarity measure between the original input signal and the reconstructed signal in order to have an idea on the ``amount'' of information that MLA representation induces. A set of natural similarity functions that can be suitable to this scope belongs to the family of correlation functions. Among the correlation functions, the most known are the \emph{Pearson}, \emph{Spearman} and \emph{Kendall} correlation indices.

\begin{definition}
\emph{Pearson, Spearman, and Kendall correlation}
Given two signal $x$ and $y$ then the correlation indices are defined as:\\

\begin{itemize}
  \item
    \emph{Pearson correlation}
\begin{equation} \label{eq:Pearson}
        r=\frac{\sum_{i=1}^m (x_i-\bar{x})(y_i-\bar{y})}{\sum_{i=1}^m (x_j-\bar{x})^2 \sum_{j=1}^m (y_j-\bar{y})^2}
\end{equation}

  \item
  \emph{Spearman correlation}
\begin{equation}\label{eq:Spearman}
        \rho=1-\frac{6 \sum_{i=1}^m \Delta_i^2}{n(n^2-1)}
\end{equation}

  \item
  \emph{Kendall correlation}
\begin{equation}\label{eq:Kendall}
        \tau=\frac{n_c-n_d}{\frac{1}{2} n(n-1)}
\end{equation}
\end{itemize}

\noindent where $\bar{x}=\frac{1}{m} \sum_i x_i$, $\bar{y}=\frac{1}{m} \sum_i y_i$, $\Delta_i$ is the difference between the ranks of $x_i$ and $y_i$, while $n_c$ and $n_d$ are their number of concording and discording pairs,  respectively.
\end{definition}

In figure \ref{fig:signals} it is possible to see four examples of real world and synthetic signals: an earthquake signal, a gaussian noise signal generated in accordance to the Gaussian distribution equation \ref{eq:gaussian}, a random uniform signal, generated in accordance to the uniform distribution equation \ref{eq:uniform}, and a sinusoidal signal.

\begin{definition}\label{Spearman}
\emph{Uniform Distribution}
The uniform distribution \cite{Feller:Probability} is a distribution that has constant probability over an interval $[a,b]$, and its probability density function $p$ is:

\begin{equation}\label{eq:uniform}
 p(x)= \left \{
 \begin{array}{cc}
   0 & \mbox{for } x<a \\
   \frac {1} {b-a}  & \mbox{for } a \leq x \leq b \\
   0 & \mbox{for } x>b \\
 \end{array}
 \right .
\end{equation}
\end{definition}

\begin{definition}\label{Spearman}
\emph{Normal or Gaussian Distribution}
The Normal or Gaussian distribution \cite{Feller:Probability} is a probability distribution  with probability density function:

\begin{equation}\label{eq:gaussian}
  f(x) = \frac{1}{\sqrt{2\pi}}\; e^{-x^2/2}.
\end{equation}

\end{definition}

Table \ref{tab:varing_K} shows the number of levels required to obtain a correlation value of at least $0.9$ (using the Kendall's correlation, equation \ref{eq:Kendall}) in the case of four examples. It is also important to take into account the length of the signal representation that obviously strongly depends on the number of levels used. The following theorem gives an upper bound on the length of the representation of a signal using $K$ threshold operations.

\begin{figure}[!hb]
  \centering
  \subfigure[Eartquake signal]{\label{fig:scomp}\includegraphics[scale=.25]{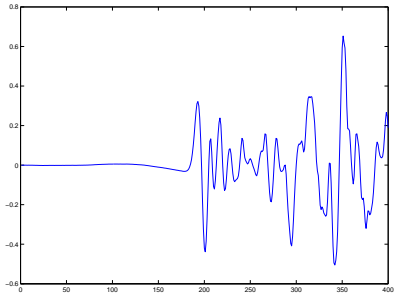}}
  \subfigure[Gaussian noise]{\label{fig:soriginal}\includegraphics[scale=.25]{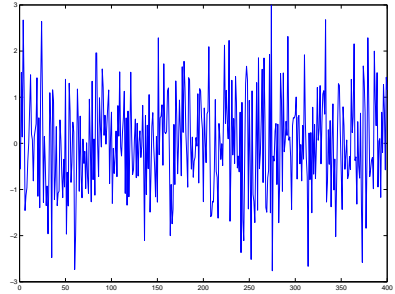}}
  \subfigure[Uniform noise]{\label{fig:smla}\includegraphics[scale=.25]{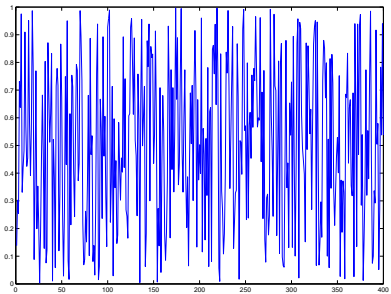}}
  \subfigure[Sinusoidal signal]{\label{fig:smla}\includegraphics[scale=.25]{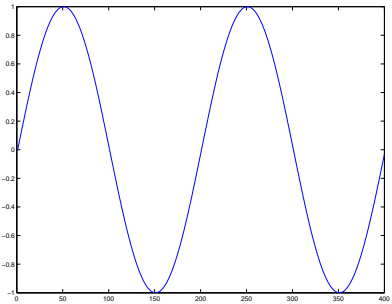}}
  \caption{Different examples of signals (all of length 400) }

  \label{fig:signals}
\end{figure}

\begin{table}
\centering
\begin{tabular}{|c|c|c|c|c|}
  \hline
  signal / $K$ &5 &	10 & 50 & 100 \\
  \hline
  earthquake & 0.3856 & 0.6488 & 0.9399 & 0.9470 \\
  gaussian & 0.9484 & 0.9890 & 0.9994 & 1 \\
  uniform & 0.9916 & 0.9990 & 1 & 1 \\
  sin & 0.9936 & 0.9937 & 0.9950 & 0.9950 \\

  \hline \end{tabular}
\caption{Information loss on the signal for different values of $K$}
\label{tab:varing_K}
\end{table}

\newpage
\begin{theorem}
Given a discrete signal  $f$  of length $L\geq3$  and let $K\geq2$ the number of threshold levels in the equally spaced simple MLA transformation then the upper bound $I_{max}$ on the number of intervals of its representation $\Upsilon(f)$ is:
\begin{equation}
I_{max}(L)= \left \lceil\frac{L}{2} \right  \rceil * (K-1) + 1
\end{equation}

and the real numbers required to represent the intervals are in number of:
\begin{equation}
n_{max}(L)= 2*\left \lceil \frac{L}{2} \right  \rceil * (K-1) +2
\end{equation}

\end{theorem}

\begin{proof}
To avoid confusion, remember that for definition the equally spaced simple MLA adds at the beginning (or to the end) of the signal $f$ a point equal to $min(f)$ if $f(1) \neq min(f)$ (or if $f(L) \neq min(f)$ ) by the \emph{disambiguation} operation.
It is possible to define two kinds of worst case signal, one for $L$ odd (see figure \ref{fig:worst_case} (a)), and one for $L$ even (see figure \ref{fig:worst_case} (b)). The even worst case signal involves always the addiction of a single new point by disambiguation, while two points are added in the case of a worst case odd signal. Moreover, the addition of a new point to the signal, involves the introduction of $K-1$ new intervals as it possible to see in figure \ref{fig:intervals_adding}.
Further it will be  considered a generic threshold operation $\sigma_k$  with $k \neq 1 $ because by definition, the first threshold operation extracts always only one interval independently on the length of the signal.
Note also that, in the case of the best case signal with $L$ odd points, the number of interval is exactly $I_{min}(L)= \left \lfloor \frac{L}{2}  \right  \rfloor * (K-1) + 1$ (see figure \ref{fig:worst_case} (c)).

\begin{figure}[!htb]
\centering
\includegraphics[scale=.52]{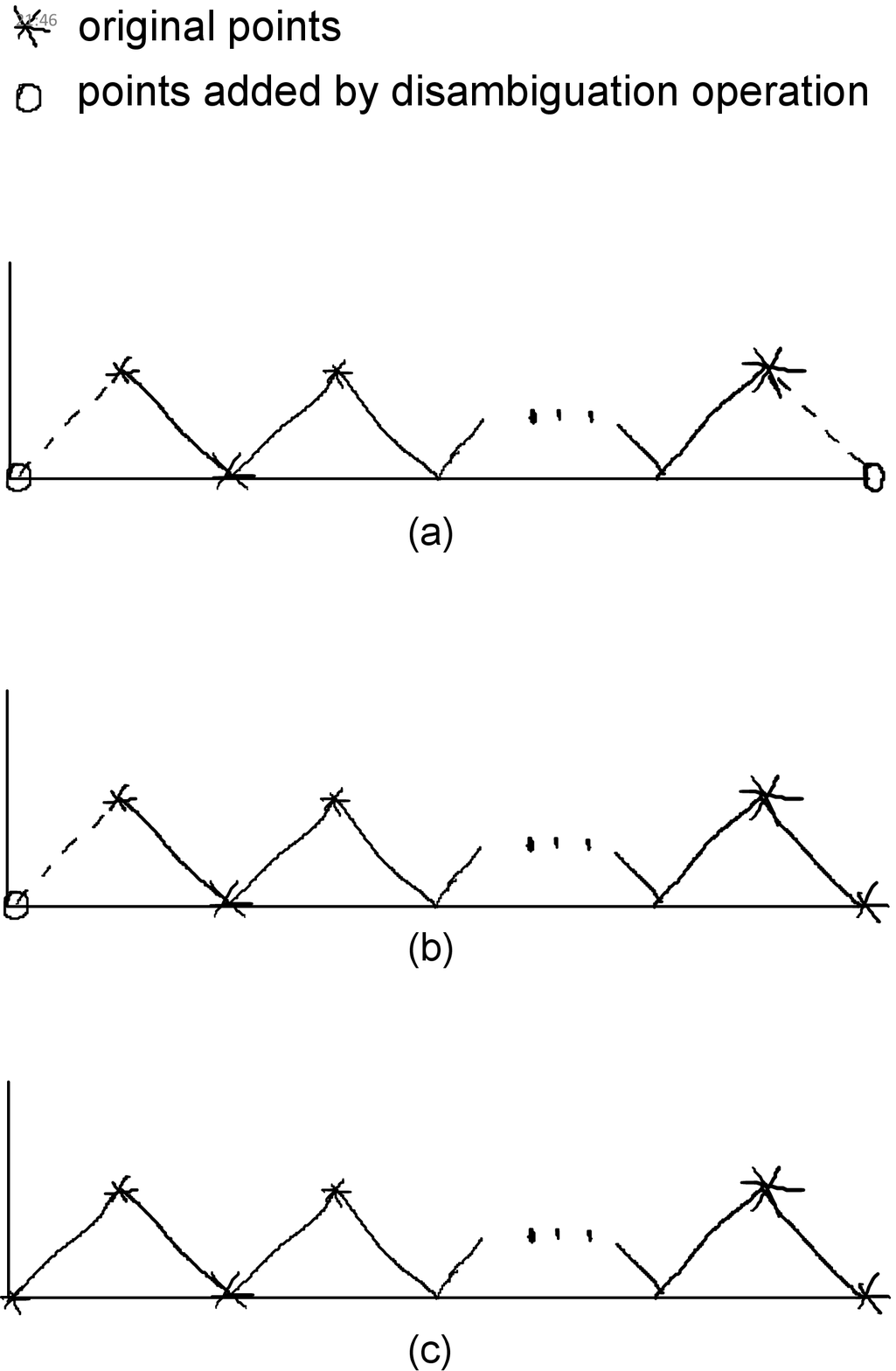}
\caption{(a) Odd worst case,(b) Even best and worst case,(c) Odd best case}
\label{fig:worst_case}
\end{figure}

\begin{figure}[!htb]
\centering
\includegraphics[scale=.52]{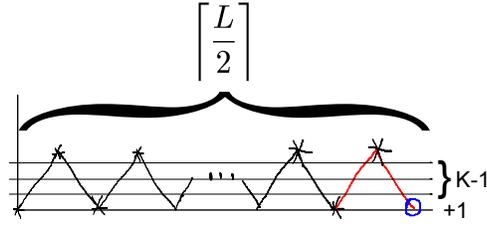}
\caption{Intervals increment: each point added can be add no more than $k-1$ intervals}
\label{fig:intervals_adding}
\end{figure}


\newpage
Let's recall two simple properties of the ceil and floor function:\\

\noindent if $L \in N$ is even then:
\begin{equation}\label{eq:ceil_even}
\left \lceil \frac{L}{2}  \right  \rceil = \left \lceil \frac {L-1}{2} \right  \rceil
\end{equation}\\

\noindent if $L \in N$ is odd then:
\begin{equation}\label{eq:ceil_odd}
\left \lfloor \frac{L}{2}  \right  \rfloor + 1 = \left \lceil \frac {L}{2} \right  \rceil
\end{equation}\\

\noindent Suppose to have a signal of length $L$, consider two cases, $L$ even or odd:\\

\begin{itemize}
  \item \emph{(L) even}: Since $L$ is even, only a new point has to be added. The resulting signal can be seen as the extension of the best case signal with $L-1$ odd points by adding two new points, and applying the induction, and the properties \ref{eq:ceil_even}, \ref{eq:ceil_odd} it results that $I_{max}(L) = I_{min}(L-1)+(K-1)=\left \lfloor \frac{L-1}{2} \right  \rfloor *(K-1)+1+(K-1)= (\left \lfloor \frac{L-1}{2} \right  \rfloor +1 ) * (K-1) + 1= \left \lceil \frac{L-1}{2} \right  \rceil * (K-1) + 1 = \left  \lceil \frac{L}{2} \right  \rceil * (K-1) + 1$.

  \item \emph{(L) odd}: Since $L$ is odd, the worst case signal involve the addiction of two new points. The resulting signal can be seen as the extension of a best case signal with $L$ odd points by adding two new points, and by applying the induction, and the property \ref{eq:ceil_odd} it results that $I_{max}(L)=I_{min}(L)+(K-1)=\left \lfloor \frac{L}{2} \right  \rfloor *(K-1)+1+(K-1)= (\left \lfloor \frac{L}{2} \right  \rfloor +1 ) * (K-1) + 1 = \left \lceil \frac{L}{2} \right  \rceil * (K-1) + 1$.
\end{itemize}
\end{proof}

\begin{lemma}
Given a discrete signal  $f$  of length $L$  and let $K\geq2$ the number of threshold levels in the equally spaced simple MLA then the complexity of this transformation is $O(K*L)$
\end{lemma}

\begin{proof}
Using the previous theorem, it is clear that in the worst case it is possible to obtain  $\left \lceil \frac{L}{2}  \right  \rceil$  intervals for a generic threshold operation. Since the transformation uses in total $K$ threshold operations in the worst case it is possible to obtain  $\left \lceil \frac {L}{2} \right \rceil *K$ interval extractions.
\end{proof}

\section{Usage of the MLA as preprocessing step}
In general it is possible to find two principal problems in which MLA can be successfully used:
\begin{itemize}
  \item given a family of signals and a signal in this family, characterize it in terms of the other signals in the family;
  \item given a signal, discover if it contains interesting substructures in some formal sense.
\end{itemize}

\noindent In more details given a signal $f$ and its MLA representation $\Upsilon(f)$ there are several ways to use it, the most trivial is to use the intervals ``as they are'' in  a feature vectors fashion. It is important to note that they are not real feature vectors since given two signals of equal length not always involve the same representation. In other words it is not possible to have a positional representation of the feature of a signal as in a classic feature vector. For this reason, in order to compare two or more signal using the MLA representation, special distances or more in general dissimilarity functions need to be defined. One way to overcome this problem is to use a set of probability distributions to model the output of threshold operations. It will be shown an example of this approach on chapter $4$ where a randomness test that exploits this idea will be presented. If we need intstead to characterize subparts of a signal, it is necessary to define aggregation rules that reflect our ``interestingness''.  It will be presented this approach in the next chapter where a rule that well characterizes a biological structure (the nucleosomes) will be defined. An extension of this approach will be presented in chapter $5$ where a new structure using a particular intervals aggregation rule, called Tree Interval Representation, will be introduced. It will give also the possibility to define a new kernel function by taking inspiration from the well-known tree kernels that have been successfully used in a completely different context: the processing of natural languages and the text categorization. In particular each of these chapters will be organized in two parts: the first part will show the formal definitions and the second part will present the real problem and the proposed solution, highlighting where the MLA takes place and, if possible, a comparison with the state of the art methodologies.

\chapter{Pattern Discovery and Classification by MLA}
\label{chap:3}

This chapter presents the MLA in the context of Pattern Discovery and Classification; in particular the section \ref{sec:MLA in Pattern Discovery and classification} explains how MLA can be integrated in these contexts. Then in section \ref{sec:biological_problem} a case study is introduced: it regards a particular biological problem, the nucleosome spacing, in which the  MLA was successfully used (see section \ref{sec:MLA in Pattern Discovery and classification}). In addition, in section \ref{sec:First solution Hidden Markov Model}, an alternative approach for this problem based on Hidden Markov Model is presented, while in section \ref{experiment} a comparison of the two methods is presented. Finally, the last section is devoted to the description of a new one-class classifier that was used as new classifier module of the MLA.

\section{MLA in Pattern Discovery and Classification}\label{sec:MLA in Pattern Discovery and classification}
This section explains how it is possible to apply the MLA in the context of pattern discovery. A general schema of pattern discovery
that takes advantage of the MLA is presented in figure \ref{fig:Pattern_discovery_MLA}. The important point here is that MLA  plays the role of the language to express the pattern as it was explained in chapter 1. In particular, given a signal $f$ the patterns correspond to subregions of $f$ that can be found using its interval representation $\Upsilon(f)$ together with an appropriative aggregation rule. In particular as expressed in chapter $2$ it is convenient to use the MLA in order to characterize or discover patterns in term of their shapes. This means that a general criteria to assess if a pattern is interesting into this context, is to check how close a subregion of a signal expressed in term of intervals meets a particular aggregation rule criteria or intervals distribution. In the latter case this means that it is possible to define an expected intervals distribution for a ``background'' that can be used to assesse how interesting a pattern is. This approach, as it will be shown in a case study described in the next section, is particularly useful and natural for signal segmentation.

\begin{figure}[htb!]
    \centering
    \includegraphics [scale=.6]{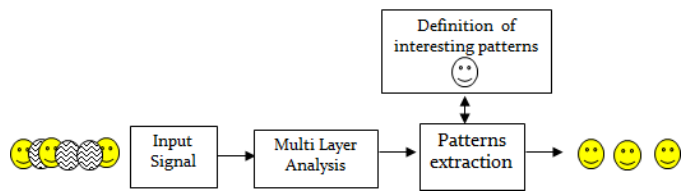}
    \caption {Pattern Discovery by MLA and signal segmentation}
    \label{fig:Pattern_discovery_MLA}
\end{figure}

In the classification problem, since it is necessary to provide an explicit training set (i.e. some examples for each class to discriminate), the MLA can be used as feature extractor, in the sense that each element can be expressed using  MLA as its interval representation, or more in general in a structure built on its interval representation using a particular aggregation rule. Here, an element of a class can be a whole signal or a subpart of a signal maybe extracted with a pattern discovery approach.\\\\
In the next section, the basic biological notions will be provided in order to introduce the MLA in the context of pattern discovery and classification for a particular biological problem: the nucleosome spacing.

\section{Fundamentals of Molecular Biology} \label{sec:Concepts and notions of Biology}
In this section  some concepts and notions of biology  will  be described, in order to introduce the basic terminology that can be  useful  for the comprehension of the matter.

\subsection{DNA}
DNA is a double helix molecule formed by two chains (helices) oriented in  opposite directions, as shown in the figure \ref{fig:dna}. DNA is present in every cell in the body and contains all the genetic information necessary for the body. The major classes of organisms are eukaryotes and prokaryotes. In eukaryotes DNA is contained within the nucleus, separated from the cytoplasm; in prokaryotes, instead, it is contained in cytoplasm. DNA is composed of four distinct types of bases, called nucleotides, that consist of three parts: a phosphate group, a sugar (deoxyribose) and a nitrogenous base (purine or pyrimidine). The four bases that forms the DNA are: adenine (indicated by A), cytosine (indicated by C), thymine (indicated by T) and guanine (indicated by G). The DNA bases are complementary: a C always pairs with a G and an A with a T. The complementarity of the two chains allows to represent a DNA sequence using only one of the two because the other one is complementary and then the information it contains is redundant.

\begin{figure}[!tb]
   \centering
    \includegraphics[width=0.45\textwidth]{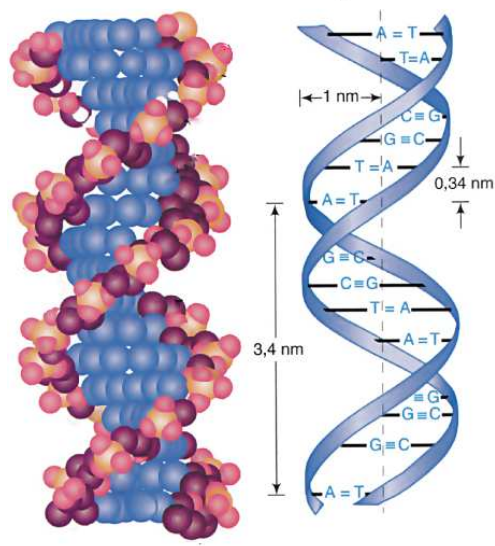}
    \caption{DNA structure}
    \label{fig:dna}
\end{figure}

\subsection{Genes and proteins}
Genes correspond to particular  sub-sequences of  DNA. They belong to the genome of an organism, which can be composed of DNA or RNA;  the genes in particular  direct physical and  behavioral development of the body. Genes also determine the amino acid sequence of proteins, which are  the most involved macromolecules  in biochemical and metabolic processes of the cell. Some other genes do not encode proteins but encode RNA that plays a key role in gene expression. In a cell there are thousands of different proteins, each with a distinct amino acid sequence. In particular each amino acid is encoded by exactly 3 nucleotides as it is possible to see in figure \ref{fig:aminoacids} and there are 20 amino acids in total. In general, a protein is a polymer composed by different combinations of amino acids that bind each other through some interactions that are called peptide bonds. Proteins play a variety of tasks in the cell. In fact, they transmit messages between cells,  turn on and off genes, are essential in muscle contraction, and finally build structures such as hair. Proteins are characterized by a three-dimensional structure articulated on four structural levels, in relation to each other:
\begin{enumerate}
  \item The primary structure is the one that identifies the specific sequence of amino acids from the peptide chain.
  \item The secondary structure corresponds to several configurations such as the spiral shape (or alpha helix),  the planar (or beta sheet), the three intertwined filaments and those belonging to the globular KEMF (keratin, epidermina, myosin, fibrinogen).
  \item The tertiary structure represents the three-dimensional configuration of the polypeptide chain. This configuration is  permitted and maintained by different chemical bonds, including the sulfide bridges and the forces of Van der Waals.
  \item The quaternary structure determines the association of two or more polypeptide units, or of protein  and non-protein units, joined together by weak bonds,such as sulfide bridges, but in a very specific way, such as it occurs in the formation of the enzyme phosphorylase, consisting of four sub-units, or from hemoglobin, which is  the molecule responsible for transporting oxygen in the body.
\end{enumerate}

\begin{figure}[!htb]
   \centering
   \includegraphics[width=0.45\textwidth]{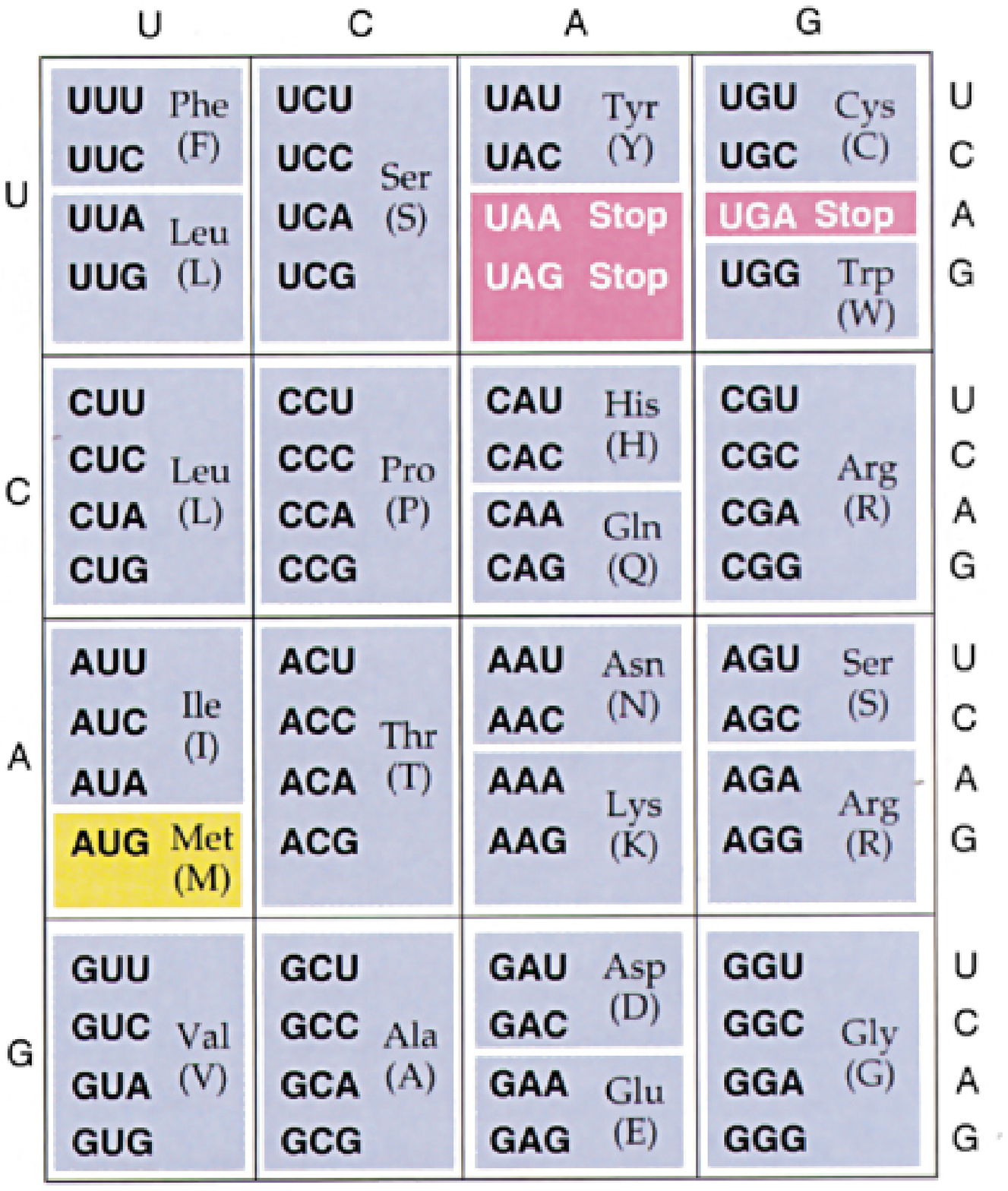}
 \caption{Amino acids alphabet in terms of DNA alphabet}
  \label{fig:aminoacids}
\end{figure}

\subsection{Protein production and  expression level of a gene}
The production of a protein from a gene is called gene expression. To obtain a protein from a gene, the information in  DNA is copied through a process called RNA transcription.  RNA in the form of mRNA acts as a messenger and delivers information from the cell nucleus (where DNA is located) to the cytoplasm. Once in the cytoplasm, the mRNA is translated in its product, the protein, thanks to the usage of the alphabet of amino acids. Then the protein is built starting from the original DNA sequence representing the gene, as it possible to see in figure \ref{fig:dna_rna_protein}. Each cell of an organism contains the same DNA, so the same information; however cells are specialized according to their function. This specialization is because not all genes are expressed at the same time and within the same cell. In fact, gene expression is a controlled dynamic phenomenon so that the processes of a cell are carried out in a controlled way. This phenomenon is regulated by several proteins that bind each other different regions of DNA.
This adjustment may depend on the function that a cell has to make and  it is regulated by both external factors and internal factors produced by the cell.

\begin{figure}[!htb]
   \centering
   \includegraphics[width=0.45\textwidth]{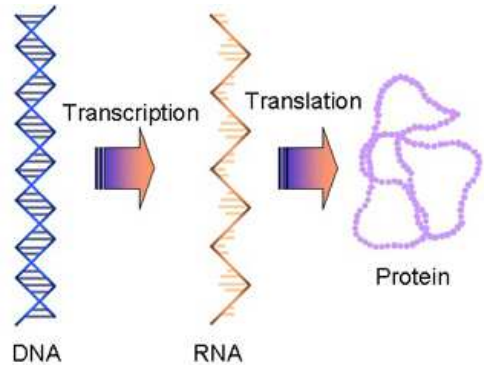}
 \caption{From a genomic sequence to a protein}
  \label{fig:dna_rna_protein}
\end{figure}

\subsection{Nucleosome and chromatin}
As said before, DNA contains all the information of an organism and it is organized in a specific space configuration called chromatin and in particular in chromosomes. More in detail, there are fundamental units called nucleosomes that package DNA into chromatin and there are several levels of space organization from DNA to a chromosome as it is possible to see in figure \ref{fig:chromatin}. The nucleosome, whose discovery dates back to $1974$, is the fundamental unit of chromatin structure and consists of a segment of about $150$bp of DNA associated with a quaternary structure of proteins called histone octamer. The nucleosome has a compact globular shape and plays the role to compact DNA in a eukaryotic cell. In figure \ref{fig:nucleosome} it is possible to see the stylized structure of a nucleosome. Nucleosomes have a diameter of about $11$ nm and are spaced from each other by a stretch of DNA linker varying in length from a few to about $80$ pairs of nucleotides. The resulting structure has the characteristic appearance of a necklace of pearls and is the first level of compaction of chromatin. The formation of nucleosomes in fact converts a molecule of DNA in a strand of chromatin along about a third of  the original length. This structural organization was highlighted after isolating the nucleosomes from chromatin. Several factors can influence the nucleosome organizations \cite{WidomSegal_what} and therefore the chromatin. Recent studies has shown that one of this factor is the sequence specificity that consists in the nucleosomes preference for some sequences: in particular, in vitro studies have shown that nucleosomes have a strong preference for some DNA sequences \cite{Segal_seq_code} and instead ``don't like'' other sequences such as poly (da,dt) tracts \cite{Segal_dadt}. Another important factor is their statistical positioning \cite{Kornberg_statistical}. This theory is based on the concept of barriers, that are regions on the dna in which the nucleosomes cannot stay.  Barriers in particular on average regulates the positions of nucleosomes around them. An important result is that it is possible to derive mathematically the probability  function on the preferences of nucleosome around the barrier. The last point is the set of chromatin remodeler complexes that actively move the nucleosomes across DNA \cite{Rippe_remodelers}.

\begin{figure}[!htb]
   \centering
   \includegraphics[width=.7 \textwidth]{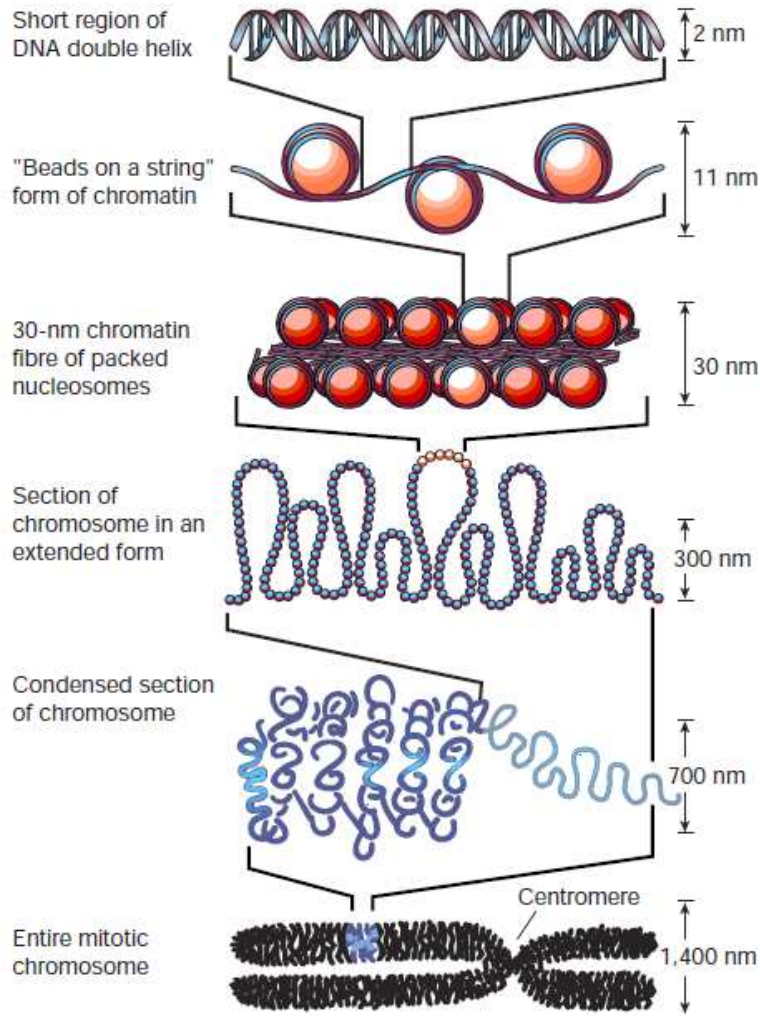}
 \caption{From DNA to chromatin}
  \label{fig:chromatin}
\end{figure}

\begin{figure}[!htb]
   \centering
      \includegraphics[width=.5 \textwidth]{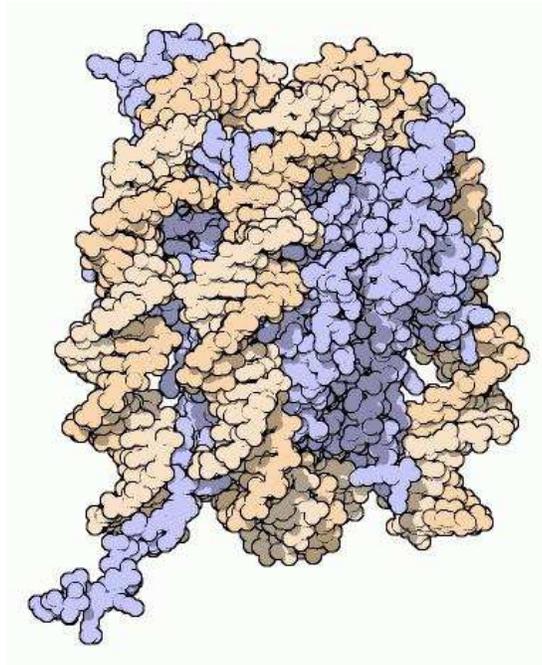}
 \caption{Nucleosome structure: in blue the octamer, in orange the DNA}
  \label{fig:nucleosome}
\end{figure}

\subsection{Microarray}
A DNA microarray (commonly known as gene chip, DNA chip, or biochip) is a collection of microscopic DNA probes attached to a solid surface such as glass, plastic or silicon chip forming an array \cite{Baldi2002}. These arrays are used to examine the expression profile of a gene or to identify the presence of a gene or of a short sequence on thousands (often the entire genome of an organism). Each location corresponds to a specific gene (or a specific sequence) and it does contain multiple copies of a filament with a particular sequence of bases. These DNA strands are anchored to the surface of the substrate, and are used as probes to measure the amount of other DNA molecules (which are also single-stranded) derived from mRNA transcripts and contained in a solution that is deposited on the surface of the microarray. The main approaches used in the manufacturing process of the microarrays are two:
one process is to deposit, with the help of a robot, a solution containing the DNA probes on the surface of the solid support. The probes can be made of a single-stranded  cDNA (complementary DNA obtained by an mRNA transcript having a length of 200-2400 bases) or can be made of pre-chemically synthesized oligonucleotides (short sequences of nucleotides with a length of 50-100 bases). Microarrays made by this process,  are called ``cDNA microarraies'' \cite{Baldi2002}. The other process is to directly synthesize oligonucleotides  on the surface of the microarray(in situ); this operation is carried out mainly with photolithographic techniques (typical of Affymetrix) and inkjet printing \cite{Baldi2002}.

The advantage of using microarrays is the possibility to examine a large amount of data per experiment; for example, it is possible to monitor the expression levels of thousands of genes at a time. In the figure \ref{fig:microarray} it is possible to see the workflow that is usually followed when using the microarray technique:
\begin{itemize}
  \item Preparation and marking of the sample (different samples are labeled with different markers)
  \item Hybridization and alignment
  \item Cleaning
  \item Image acquisition and data analysis
\end{itemize}

\begin{figure}[!htb]
   \centering
   \includegraphics[width=.7 \textwidth]{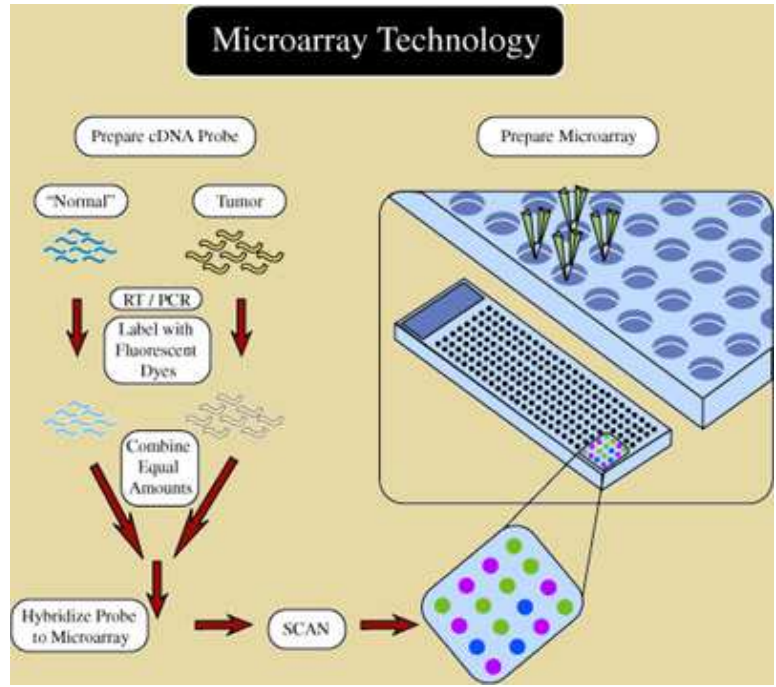}
 \caption{Microarray workflow}
  \label{fig:microarray}
\end{figure}

\section{Case Study: Nucleosome Positioning}\label{sec:biological_problem}
The biological problem under consideration  concerns the positioning of nucleosomes in DNA. This problem is very interesting because the accurate and precise measurement of the  nucleosomes  position  on genomic scale could improve the understanding of the chromatin  structure  and its function.
  Alterations in chromatin and hence in nucleosome organizations can result in a variety of diseases. In fact, the emergence of diseases is thought to be due to the fact that the altered chromosomes condensation leads to  the expression increase of certain genes, causing abnormal production of proteins in the cell. This motivates the use of a methodology capable of determining the position of nucleosomes,  in order to study the implication of nucleosome spacing in the chromatin condensation phenomena. This may be investigated by comparing the positions of nucleosomes in different contexts in which there are different amounts of proteins that remodel chromatin by changing their position. This would figure out which is the molecular basis of chromosome condensation defects or defects in gene expression caused by the partial or total absence of these molecular machines. In fact, it would be possible that the  nuclesome spacing is the basis of this, which would mean that in the absence of such molecular machines, nucleosomes were not spaced properly carrying abnormalities in the cell. So it is very important to understand the processes that modulate the chromatin dynamics and in particular the nucleosome positioning. Their positioning in fact plays a direct role in gene regulation \cite{Luger97}. While the packaging that they provide allows the cell to organize a large and complex genome in the nucleus, they can also block the access of transcription factors and other proteins to DNA \cite{Coro}. For example, under normal conditions the Pho5 promoter in yeast is occupied by well-positioned nucleosomes, preventing the transcription factor Pho4 from binding to its target binding site. When induced by phosphate starvation, the nucleosomes are depleted from the promoter region so that Pho4 can bind to its target DNA binding sequence thus activating the Pho5 gene transcription \cite{HORZ97}.  However, nucleosome binding can sometimes enhance transcription by bringing distant DNA regulatory elements together \cite{STUNKEL97}.  Genome-wide studies have found that transcription activity is inversely proportional to nucleosome depletion in promoter regions in general \cite{BERN04,POKH05,Lee04}. With the help of tiling arrays at $20$bp resolution, Yuan et Al. \cite{YUAN05} have looked at nucleosome occupancy relative to gene regulatory regions on ~4\% of the yeast genome by using an Hidden Markov Model approach HMM. The used microarray-based method allows the identification of nucleosomal and linker DNA sequences on the basis of susceptibility of linker DNA to micrococcal nuclease. This method allows the representation of microarray data as a signal  of green/red ratio values showing nucleosomes as peaks of about 150 base pairs long, surrounded by lower ratio  values corresponding to linker regions.  Consistent with previous studies, Yuan et Al. found that $87\%$ of the transcription factor binding sites \cite{HARB04} are free of nucleosome binding. A substantial improvement over this work has been recently done by Lee et al. \cite{LEE07} where the genome-wide nucleosome positions in yeast have been mapped at $4$bp resolution. A similar approach has also been used to look at differences in nucleosome spacing occurring in the absence of a chromatin remodeler \cite{RANDO07}. A number of other groups have developed analysis methods to detect nucleosomes as well as transcription factor binding sites \cite{BUCK05,JI05,KIM06,ZANG07,JOHNSONS06,KELES06,MIELE08,YASSOUR08}. Compared to transcription factors, it is more challenging to detect nucleosome positions since the majority of a eukaryotic genome is wrapped into nucleosomes.  Another difficulty is that the raw data may contain complex trends that are unrelated to nucleosome binding \cite{YUAN05}. An intuitive method to deconvolve data trend is to define a peak-to-trough difference measure and to detect its local maxima. However, Yuan et Al. \cite{YUAN05} have found that although this method can detect local peaks, it suffers from amplifying observation noise. A similar approach has been adapted in \cite{OZSOLAK07} to map nucleosome positions in human. Although an intrinsic DNA code  for nucleosome positioning has been recently reported \cite{widom},  a significant technological development in genome-wide location of nucleosomes has been made using ``deep sequencing'' approaches \cite{pugh,Barski07,Mikkelsen07,Wold07}, which differs from microarray-based approach in that  the  isolated DNA of interest is mapped to genome via direct DNA sequencing, instead of microarray hybridization. For this new technology, the input data correspond to peaks of DNA fragment counts instead of high hybridization ratio. However, the task of peak detection remains a key problem for the statistical analysis of the input data. Unlike microarray-based approaches, where data collection is constraint to a regular grid, ``deep sequencing'' data are intrinsically base-pair resolution and therefore less statistically stable. One solution to this problem is to first map the data onto a regular grid by binning. However, more sophisticated methods need to be developed to balance the resolution vs variance dilemma. The analysis of stochastic signals aims to both extract \emph{significant} patterns from noisy background and to study their spatial relations (periodicity, long term variation, burst, etc.). The problem becomes more complex whenever the noise background is structured and unknown. Examples of
such kind of data correspond to protein-sequences in the study of folding \cite{DEL93} and the positioning of nucleosomes along chromatin in the study of gene expression \cite{YUAN05}. The analysis carried out in both cases has been based on probabilistic networks \cite{JEN01} (for example, Hidden Markov Models \cite{EPH02}, Bayesian networks). Methods based on probabilistic networks are suitable for the analysis of such kind of signal
data; however, they suffer of high computational complexity and results can be biased from locality that depends on the memory steps they use \cite{YUAN05,DEL93}. In the next section it will be presented an approach that takes advantage of the MLA and its comparison with the proposed method based on HMM. The main advantage of MLA over HMM is its scalability that produce a significant reduction in computational time over the HMM. In this case study in particular it was considered the performances of these two methods to both synthetic and microarray-based nucleosome positioning data and their ability to recover distinct nucleosome configuration. This configurations could be underlie important regulatory roles, highlighting the impact of these methodologies on genome-wide nucleosome positioning studies in higher eukaryotes.

\subsection{The microarray and the signal}\label{microarray}
The following describes the microarray  structure designed and used in the Bauer Center laboratory  for Genomics Research, Harvard University \cite{YUAN05}. As mentioned before, a DNA microarray was used to extract the sequences corresponding to nucleosomes and those corresponding to the linker, in order to identify the nucleosomes on a genomic scale. In particular the microarray data, $\mathbf{S}$, are organized in  $T$ contiguous fragments $S_1,\cdots, S_T$ which represents $DNA$ sub-sequences. In order to obtain the signal on which  subsequent processing are made,  carrying out as follows is needed: Firstly, DNA wrapped in the nucleosome is isolated and labeled with a green fluorescent dye (it is marked the entire genomic DNA of the organism, chromatin is then digested with a particular enzyme that cuts in the linker regions of nucleosomes but leaves intact the DNA around the nucleosome). At the same time the genomic DNA is marked with a red fluorescent dye. At this point there is a  competitive hybridization; if both probes are hybridized in equal proportions, a yellow spot will be obtained, while a red spot if the probe with the red marker is the more hybridized, otherwise a green spot. As a result red or green spots will be obtained as it is possible to see in figure \ref{fig:microarray2}.

\begin{figure}[!htb]
   \centering
   \includegraphics[width=.7 \textwidth]{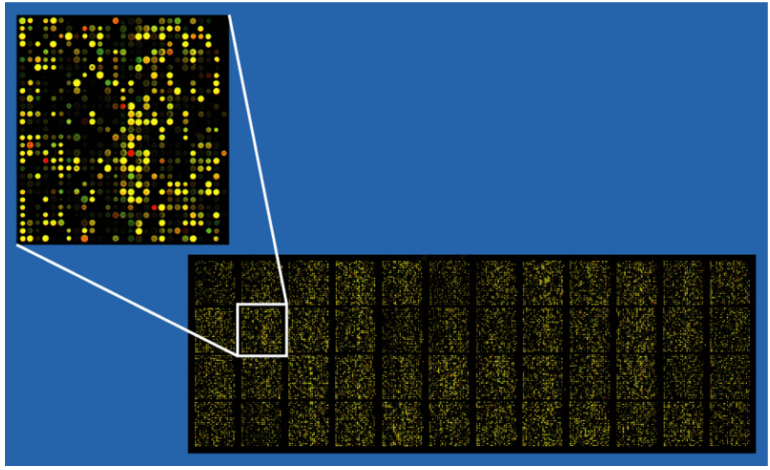}
 \caption{Microarray probes}
  \label{fig:microarray2}
\end{figure}

In particular, in such data, each spot corresponds to a sequence of 50 base pairs. These sequences are  overlapped  of 30 base pairs in order to obtain a final resolution of 20 base pairs. With this resolution a nucleosome, which occupies about 150 base pairs, will correspond to about 6-8 probes in the microarray. These nucleosomes are called  well-positioned  nucleosomes. There is also a class of decentralized nucleosomes, that  can occupy multiple positions due to thermodynamics factors or that can correspond to segments that may come from cells in different states. The next step is to  excite the two dyes with a laser scanner,  using different wavelengths; in this way a separate scanning  of  red and green channels is obtained. To see if the sequences are hybridized or not, their logarithmic ratio has to be considered:
\begin{equation}
S = {\log _2}\left( {\frac{G}{R}} \right)
\end{equation}
This  will give a signal with a pattern which will have peaks in the presence of nucleosomes. An overview of this method and a fragment of this signal is shown in figure \ref{fig:signal}

\begin{figure}[!htb]
   \centering
     \includegraphics[width=0.8\textwidth]{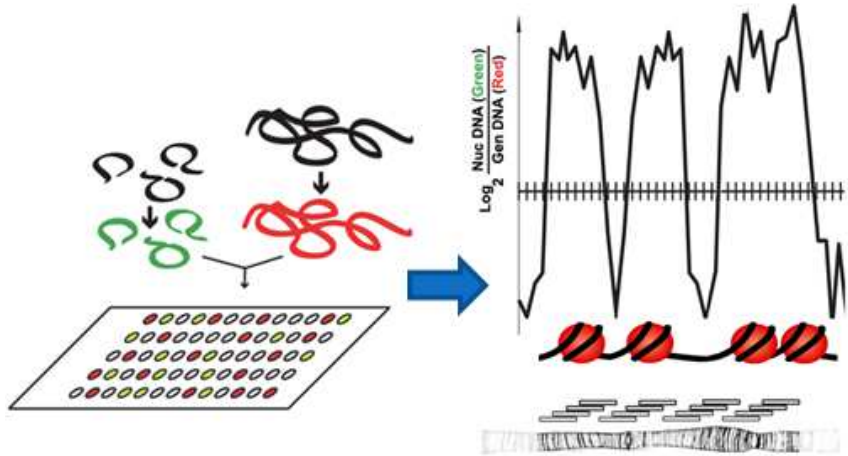}
 \caption{From microarray to one-dimensional signal}
  \label{fig:signal}
\end{figure}

\subsection{Preprocessing}
Before the analysis, the signal coming from the microarray is normalized in order to remove possible measurement errors (bias) and to reduce the influence of cross-hybridization. Normalization is a two-step process:
\begin{itemize}
  \item the mean and variance of each group of spots is taken into account,
  \item the cross-hybridization and the entropy of the signal (base sequence) is taken into account.
\end{itemize}

The cross-hybridization is the hybridization of segments that do not have a perfect match but only a partial one, and consequently do not match and should not be considered. The entropy here is intended the classic definition proposed by Shannon:
\begin{equation}
{E_i} =  - \sum\limits_{k = 1}^{{l_i}} {{p_k}\log {p_k}}
\end{equation}

\noindent Where $p_k$   represents   the probability of emission of the $k-th$ symbol, that is defined in the alphabet of the bases that constitute the DNA (A, T, C, G), and $l_i$ indicates the length of segments in each spot.\\

The first phase of standardization will reduce the bias caused of different groups in which take place the hybridization. In particular this phase uses the following model:
\begin{equation}
{y_{ij}} = {\sigma _j}\left( {{\mu _i} + {\beta _j}} \right) + \varepsilon
\end{equation}

\noindent where  $y_{ij}$ represents the logarithmic ratio of the observed value of $i-th$ probe of the $j-th$ group, $\mu_i$ is the normalized value desired, $\beta_j$ and $\sigma_j$ are respectively the mean and variance of the $j-th$ group and $\varepsilon$ is an instrumental error term, which is assumed to be independent and have zero mean.\\

In the second phase of standardization the objective is at least to reduce the effects of cross-hybridization, as this is considered unavoidable because of the large number of bases considered. In trying to reduce cross-hybridization  two factors are considered:
\begin{itemize}
  \item A specific component that measures the number of small sequences that cross-hybridize with long overlaps with the sequences of the probes;
  \item An unspecified component that measures  the case in which a large number of sequences are weakly cross-hybridized with small overlaps with the sequences of the probes.
\end{itemize}

The first component was modeled by a discrete value $B_i$, which is set to 1 if the sequence of a probe, (which as mentioned before is 50 bases long) corresponds to another sequence of equal length for at least 30 pairs of basis (a partial match, but not negligible), which would introduce an unwanted  positive contribution  to the signal of the logarithmic ratio. Otherwise, the value of $B_i$ is set to 0. The second component was modeled with $E_i$, i.e. the entropy of the $i-th$ sequence present in a probe. The normalized value $v$ of the probe $i$ of the group $j$ is then obtained as:

\begin{equation}
v_i = \mu_i + (w_B \mu_i + q_b)B_i + (w_E \mu_i + q_E) E_i
\end{equation}
where $w_B$ $q_B$ e $w_E$ $q_E$  are the linear coefficients estimated respectively for the first and second component, obtained by linear regression.

\section{First solution: Hidden Markov Model} \label{sec:First solution Hidden Markov Model}
\label{HMM}
In this paragraph a formal definition of HMM will be outlined, and then a model topology designed for the particular biological problem of nucleosome identification will be given.\\

The HMM is a statistical signal modeling technique used in various disciplines such as alignment of gene sequences, acoustic modeling, speech recognition and OCR techniques \cite{Durbin,Rabiner,Giancarlo_HMM}.  In this model, once defined the alphabet of symbols that make up the signal, a set of states are defined, each of one is associated with a particular probability distribution to produce a particular symbol of the alphabet. It also necessary to define the probability of transition from one state to another, and the probability distribution of initial states. In this way this model leads to a weighted graph where the edge weights represent the probability of transiting from one vertex to an adjacent one. The modeling of the signal can then be seen as a visit on this graph, where every time a vertex is visited, a symbol is produced.  A formal definition of HMM will now be given.

\begin{definition}\emph{Hidden Markov Model}\\

Let  $\Sigma$ an alphabet of $M$ symbols.

A HMM is a quintuple:
$\lambda  = \left( {N,M,A,B,\pi } \right)$
where:
\begin{itemize}
  \item $N$ is the number of states of the model indicated by the integers 1,2, $\ldots$, $N$;
  \item $M$ is the number of symbols of  the alphabet that each state can produce or recognize;
   \item $A = \left( {{a_{ij}}} \right)$ is a matrix called \emph{transition matrix }where $a_{ij}$ represent  the probability of transition from the state  $i$ to the state $j$ with  $1 \le i,j \le N$. This matrix must also satisfies the following condition: $\sum \limits_j {{a_{ij}}}  = 1, \quad \forall i$

  \item $B$ is the probability distribution of the observations, where  ${b_j}\left( \sigma  \right)$ represents the probability of recognizing or generating the symbol $\sigma  \in \Sigma $  if you are in the state $j$. In addition, The condition $\sum\limits_{\sigma  \in \Sigma } {{b_j}\left( \sigma  \right) = 1,\quad \forall j} $ needs to be met;
  \item $\pi$ ut is the probability distribution of initial states, where with $\pi_i$ is denoted  the probability of starting from the state $i$. In addition, the condition $\sum\limits_i {{\pi _i} = 1,\quad \forall i} $ needs to be met;
\end{itemize}

\end{definition}
The transition matrix $A$ induces a directed graph where nodes represent states, and arcs are labeled with their corresponding transition probabilities. The term  \emph{hidden} refers to the fact that, given a sequence of symbols that composes the signal you want to model, and set a model, the sequence of states is hidden and not unique, unlike other models such as Markov Chains \cite{Ching_Markov_Chains} for example.

The HMMs can be used, as it will be shown in the following paragraphs, both as generators and as recognizers of signals.

\subsection{HMM as generators}
A HMM can be used to generate a sequence of  $\Sigma^*$. Let $X = {x_1}{x_2} \ldots {x_T} \in {\Sigma ^*}$.
This sequence can be generated by a sequence of states $Q = {q_1}{q_2} \ldots {q_T}$ as follows:

\begin{enumerate}
  \item Set  $i \leftarrow 1$ and choose the state  $q_i$ according to the probability distribution $\pi$ of initial states;
  \item Assuming to be in the state  $q_i$ (having already generated  ${x_1}{x_2} \ldots {x_{i - 1}}$)  produce  in output $x_i$ according to the probability distribution $b_{i}$ ;
  \item If  $i<T$, then  $i \leftarrow i + 1$ and go to the state  ${q_{i + 1}}$ in agreement with $A\left[ {{i},1:N} \right]$ and repeat step $2$ otherwise end.
\end{enumerate}

The probability of observing  $X = {x_1}{x_2} \ldots {x_T}$ and the sequence of states $Q = {q_1}{q_2} \ldots {q_T}$  is:
\begin{equation}
P(X,Q)= \pi_{1} \prod \limits_{i = 1}^T b_{i} (x_i) a_{i {i+1}}
\end{equation}

This probability is often not very useful because it is unknown which sequence of states has produced the string $X$ (since it is possible to have multiple sequences of states that can generate it). Algorithms that solve this problem will be shown later.

\subsection{HMM as recognizers}
A HMM can be used as a probabilistic validator of a sequence of $\Sigma^*$  because  it returns a measure, in terms of mass of the probability of how well a HMM  recognizes or observes $X$. This probability is defined as:

\begin{equation}
\begin{array}{l}
P\left( {X|\lambda } \right) = \prod\limits_{t = 1}^T {\sum\limits_{i = 1}^N {P\left( {{q_t} = i} \right){b_i}\left( {{x_t}} \right)} } \quad \\
\rm{with}\quad \quad P\left( {{q_t} = j} \right) = \left\{ {\begin{array}{*{20}{c}}
{{\pi _j}\quad {\rm{if}}\quad t = 1{\rm{                 }}}\\
{\sum\limits_{i = 1}^N {P\left( {{q_{t - 1}} = i} \right){a_{{i}{j}}}{b_i}\left( {{x_{t - 1}}} \right)} }
\end{array}} \right.
\end{array}
\end{equation}

As mentioned earlier, the HMM  through the transition matrix $A$ induces a multi-parted graph.
This graph can be represented as a matrix with $N$ rows, which correspond to $N$ states of $\lambda$, and for all $t \ge 1$ columns $t$ and $t+1$  form a complete bipartite graph, with arcs directed from vertices in column $t$ to vertices in column  $t+1$ ($1 \le t \le T-1 $).
The recognition consists of superimpose  $X$ over all possible paths of length $T$ in this graph (which is called  \emph{trellis}), starting from the vertices in column 1.
For a given vertex $i$ in column $t$ on a given path, the measure of how well it is possible to recognize the symbol $x_t$ consists of two parts: the probability of being in the state $P( q_t = i)$ and the probability that the state emits the symbol  $x_t$ given by  $b_i( x_t)$.

\subsection{Problems related to HMM}
Given an HMM model $\lambda$, three main issues are considered:

\begin{enumerate}
  \item Given a sequence of observations $X = {x_1}{x_2} \ldots {x_T} \in {\Sigma ^*}$ and a model $\lambda  = \left( {N,M,A,B,\pi } \right)$, calculate the probability of observing the sequence $X$ using the model  $\lambda$ i.e. $P( X|\lambda  )$;
  \item Given a sequence of observations $X = {x_1}{x_2} \ldots {x_T} \in {\Sigma ^*}$ and a model $\lambda  = \left( {N,M,A,B,\pi } \right)$,choose the corresponding sequence of states $Q = {q_1}{q_2} \ldots {q_T}$ that best explains the observations using the model  $\lambda$ and an optimization criterion;
  \item Calculate the values of model parameters  $\lambda  = \left( {N,M,A,B,\pi } \right)$ in order  to maximize $P( X|\lambda  )$.
\end{enumerate}

The first problem is solved efficiently by an algorithm called \emph{forward procedures}, the second by the \emph{Viterbi} algorithm, while the third by the \emph{Baum Welch} algorithm.

\subsection{Forward procedures}
By using this algorithm, is possible to calculate  $P( X|\lambda  )$ in $O( N \times T  \times \delta \max )$ where $\delta \max$  is the maximum degree among all HMM states. This algorithm uses dynamic programming and consider a variable  $\alpha _t(i)$ defined as:
\begin{equation}
{\alpha _t}\left( i \right) = P\left( {{x_1}{x_2} \ldots {x_t},{q_t} = i|\lambda } \right)
\end{equation}
that is the probability that at time $t$, it is possible to observe the partial sequence ${x_1}{x_2} \ldots {x_t}$ and reach the state $i$.
The procedure consists of three phases:
\begin{itemize}
  \item Initialization:
  \begin{equation}
  {\alpha _1}\left( i \right) = {\pi _i}{b_i}\left( {{x_1}} \right)\quad {\rm{with}}\quad 1 \le i \le N
  \end{equation}
  \item Induction:
  \begin{equation}
  \begin{array}{c}
{\alpha _{t + 1}}\left( j \right) = \left( {\sum\limits_{i = 1}^N {{\alpha _t}\left( i \right)} \,{a_{{i}{j}}}} \right){b_j}\left( {{x_{t + 1}}} \right)\quad \\
{\rm{with}}\\
1 \le t \le T - 1,\quad 1 \le j \le N
\end{array}
\end{equation}
  \item Termination:\\
  \begin{equation}
  P\left( {X|\lambda } \right) = \sum\limits_{i = 1}^N {{\alpha _{_T}}\left( i \right)}
  \end{equation}
\end{itemize}

In figure \ref{fig:fp} the single steps that allow to calculate ${\alpha _{t + 1}}\left( j \right)$ are shown.
\begin{figure}[!htb]
   \centering
     \includegraphics[width=0.4\textwidth]{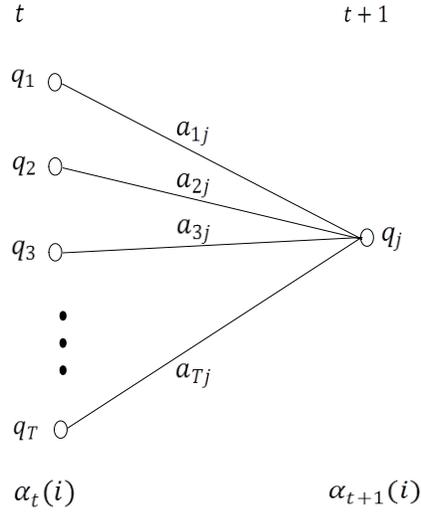}
 \caption{Forward procedure}
  \label{fig:fp}
\end{figure}
The number of possible paths grows exponentially with the length of the sequence, so it is not possible, in many applications, to consider all paths. For this reason a good approximation is to consider only the probability of the most likely path. There is also a variant of this algorithm  that,at the end of computation, calculates the same probability starting from the possible terminal states used to recognize (or generate) the sequence $X$. This variant, which is called the \emph{backward procedures}, as well as the forward procedure, uses a variable  ${\beta _t}\left( i \right)$ defined as:
\begin{equation}
{\beta _t}\left( i \right) = P\left( {{x_{t + 1}}{x_{t + 2}} \ldots {x_T}|{q_t} = i,\lambda } \right)
\end{equation}
that represents the probability at time $t$, to observe a partial sequence from time $t +1$ until the end, being in the state $i$ under the assumption of the model $\lambda$.
In figure \ref{fig:bp} the single steps that allow to calculate ${\beta _t}\left( i \right)$ are shown.

\begin{figure}[!htb]
   \centering
    \includegraphics[width=0.4\textwidth]{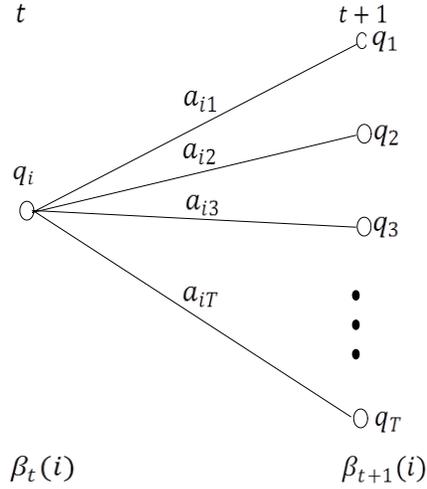}
 \caption{Backward procedure}
  \label{fig:bp}
\end{figure}

\subsection{Viterbi algorithm}
The Viterbri algorithm provides an efficient solution to the second problem of HMM  i.e. computing the optimal sequence of states for the recognition of the sequence $X$ with the model $\lambda$.
The  term ``optimum'' depends on the particular problem taken in exam. In any case, one of the  most used criteria is to find the best sequence of states that generates $X$ maximizing  $P\left( {Q|X,\lambda } \right)$ or equivalently $P\left( {Q,X|\lambda } \right)$ .
The Viterbi algorithm uses dynamic programming and computes:
\begin{itemize}
  \item ${\beta _t}\left( i \right) = \mathop {\max }\limits_{{q_1}{q_2} \ldots {q_{t - 1}}} P\left( {{q_1}{q_2} \ldots {q_{t - 1}},{q_t} = i,{x_1}{x_2} \ldots {x_t}|\lambda } \right)$ i.e. the probability of the  most likely path that takes into account of the first $t$ observations and that ends in state $i$;
  \item ${\gamma _t}\left( i \right)$ that represents the state that leads to the state $i$ at time $t$.
\end{itemize}

The procedure consists of four phases:
\begin{enumerate}
  \item Initialization:

    \begin{equation}
            \begin{array}{*{20}{c}}
            {{\beta _1}\left( i \right) = {\pi _i}{b_i}\left( {{x_1}} \right)\quad }\\
            {{\gamma _1}\left( i \right) = 0{\rm{               }}}
            \end{array}with\quad 1 \le i \le N
    \end{equation}
  \item Induction:
      \begin{equation}
        \begin{array}{*{20}{c}}
    {{\beta _t}\left( j \right) = \mathop {\max }\limits_{1 \le i \le N} \left\{ {{\beta _{t - 1}}\left( i \right){a_{{i}{j}}}} \right\}{b_j}\left( {{x_t}} \right){\rm{      with 2}} \le t \le T{\rm{       }}}\\
    {{\gamma _t}\left( j \right) = \mathop {\arg \max }\limits_{1 \le i \le N} \left\{ {{\beta _{t - 1}}\left( i \right){a_{{i}{j}}}} \right\}{\rm{            with }} 1 \le j \le N{\rm{      }}}
    \end{array}
    \end{equation}

  \item Termination:
     \begin{equation}
            \begin{array}{l}
            P\left( {Q|X,\lambda } \right) = \mathop {\max }\limits_{1 \le i \le N} \left\{ {{\beta _T}\left( i \right)} \right\}\\
            {q_T} = \mathop {\arg \max }\limits_{1 \le i \le N} \left\{ {{\beta _T}\left( i \right)} \right\}
            \end{array}
    \end{equation}

  \item Backtracing:
      \begin{equation}
      {q_t} = \mathop {{\gamma _{t + 1}}\left( {t + 1} \right),\quad t = T - 1, \ldots ,1}\limits_{}
    \end{equation}

\end{enumerate}

This algorithm has a computational cost equivalent to $O\left( {N \times T \times \delta } \right)$ where $\sigma$ represents the maximum degree of the graph induced by the transition matrix of $\lambda$. Again, as in the forward procedure, the number of possible paths grows exponentially with the length of the sequence, making this method not always feasible in the case of large amounts of data.\\

\subsection{Baum Welch algorithm}
The calculation of the values of model  parameters  $\lambda  = \left( {N,M,A,B,\pi } \right)$ that maximize $P\left( {X|\lambda } \right)$, is not an easy task. In fact, there isn't any analytical method that solves the problem by maximizing the probability of observing the sequence: given a finite sequence as a training set, there isn't a perfect way to estimate the parameters of the model. However, it is possible to derive a model $\lambda  = \left( {N,M,A,B,\pi } \right)$ so that  $P(X|\lambda)$ is locally maximized using an iterative procedure.
The best-known iterative procedure that solves this problem is the \emph{Baum Welch algorithm}. To describe how this algorithm works first define this function:
\begin{equation}
{\xi _t}\left( {i,j} \right) = P\left( {{q_t} = i, q_{t + 1} = j|X,\lambda } \right)
\end{equation}

\noindent i.e. the probability of being in state $i$ at time $t$ and in state $j$ at time $t +1$, given the model and the sequence of observations $X$. The sequence of events leading to the conditions required by this variable is shown in the figure \ref{fig:bw}.

\begin{figure}[!htb]
   \centering
    \includegraphics[width=0.65\textwidth]{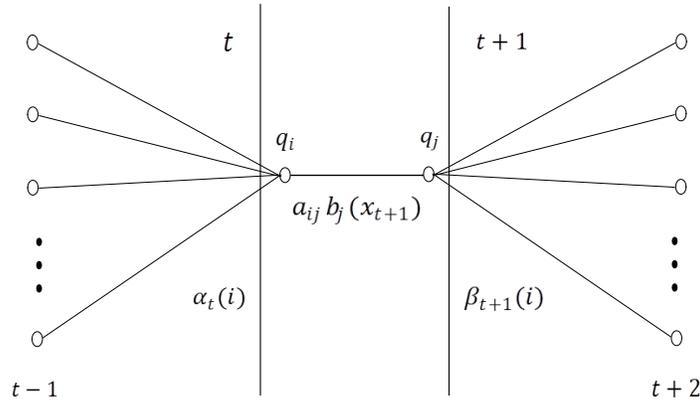}
 \caption{Baum Welch algorithm}
  \label{fig:bw}
\end{figure}

Obviously, it is clear that looking at the definition of the variables used in the procedures of backward and forward, it is possible to rewrite:
\begin{equation}
\begin{array}{l}
{\xi _t}\left( {i,j} \right) = \frac{{{\alpha _t}\left( i \right){a_{{i}{j}}}{b_j}\left( {{x_{t + 1}}} \right){\beta _{t + 1}}\left( j \right)}}{{P\left( {X|\lambda } \right)}} = \\
 = \frac{{{\alpha _t}\left( i \right){a_{{i}{j}}}{b_j}\left( {{x_{t + 1}}} \right){\beta _{t + 1}}\left( j \right)}}{{\sum\limits_{i = 1}^N {\sum\limits_{j = 1}^N {{\alpha _t}\left( i \right){a_{{i}{j}}}{b_j}\left( {{x_{t + 1}}} \right){\beta _{t + 1}}\left( j \right)} } }}
\end{array}
\end{equation}

Where the numerator is simply the probability $P\left( {{q_t} = i,q_{t + 1} = j,X|\lambda } \right)$ .
Previously ${\alpha _t}\left( i \right)$ was defined  as the probability of being in state $i$ at time $t$, by observing the partial sequence ${x_1}{x_2} \ldots {x_t}$. Let's see how  ${\alpha _t}\left( i \right)$  can be defined in terms of ${\xi _t}\left( {i,j} \right)$ :

\begin{equation}
{\alpha _t}\left( i \right) = \sum\limits_{j = 1}^N {{\xi _t}\left( {i,j} \right)}
\end{equation}

Summing over $t$  the functions ${\alpha _t}\left( i \right)$ and ${\xi _t}\left( {i,j} \right)$ it is possible to obtain:

\begin{equation}
\sum \limits_{t = 1}^{T - 1} {{\alpha _t}\left( {i,j} \right)}  = {\mbox{ number of expected transitions from state }}{i}
\end{equation}

\begin{equation}
\sum\limits_{t = 1}^{T - 1} {{\xi _t}\left( {i,j} \right)}  = {\mbox{ expected number of transitions between state }}{i}{\mbox{ and state }}{j}
\end{equation}

Using the defined  formulas will be shown now  the method for estimating parameters for a HMM using the Baum Welch procedure.

Reasonable estimates for the parameters are:
\begin{equation}
\overline {{\pi _i}}  = {\mbox{expected number of times to being in state }}{i}{\mbox{ at time }}\left( {t = 1} \right) = {\alpha _1}\left( i \right)
\end{equation}

\begin{equation}
\overline {{a_{{i}{j}}}}  = \frac{{\sum\limits_{t = 1}^{T - 1} {{\xi _t}\left( {i,j} \right)} }}{{\sum\limits_{t = 1}^{T - 1} {{\alpha _t}\left( {i,j} \right)} }} =
 \frac{{{\mbox{expected number of transitions from state }}{i}{\mbox{ to state }}{j}}}{{{\mbox{expected number of transition from state }}{S_i}}}
\end{equation}

\begin{equation}
\begin{array}{l}
\overline {{b_j}\left( k \right)}  = \frac{{\sum\limits_{t = 1 \wedge {x_t} = {v_k}}^{T - 1} {{\alpha _t}\left( j \right)} }}{{\sum\limits_{t = 1}^{T - 1} {{\alpha _t}\left( j \right)} }} =\\\\
 =\frac{{{\mbox{ expected number of times of being in the state }}j{\mbox{ and observing the simbol }}{v_k}}}{{{\mbox{ expected number of times of being in the state }}j}}
 \end{array}
\end{equation}

these equations can be used in order to develop an iterative process that, starting from a model $\lambda  = \left( {N,M,A,B,\pi } \right)$, allows us to estimate at each step a new model $ \overline \lambda  = \left( {N,M,\overline A ,\overline B ,\overline \pi  } \right)$.

In addition it can be proven that:
\begin{itemize}
  \item The model $\lambda$ represents a critical point of the likelihood function in the case $\overline \lambda   = \lambda$;
  \item The model $\overline \lambda$  is better than the model  $\lambda$, which means that the probability of observing $X$ given the model  $\lambda$ is greater than the probability of observing $X$ given the model  $\lambda$ i.e $P\left( {X|\overline \lambda  } \right) > P\left( {X|\lambda } \right)$  .
\end{itemize}

These two statements tell us that this procedure converges to a critical point. This can be done using  iteratively the model  $\overline \lambda$ instead of  $\lambda$ and repeating the process of parameters estimating, gradually increasing the likelihood of the observations of the training sequence, until a critical point is reached. The end result of this procedure is called the \emph{maximum likelihood estimate} of a HMM. It is important to underline  that this algorithm leads to a local maximum point, and in many real application  the surface to optimize is very complex and has many local maxima.
The formulas to estimate parameters can also be derived directly from the  Blum's auxiliary function $Q\left( {\lambda ,\overline \lambda  } \right)$ in respect to $\overline \lambda  $; this function is defined as:

\begin{equation}
Q\left( {\lambda ,\overline \lambda  } \right) = \sum\limits_Q {P\left( {Q|X,\lambda } \right)\log } \left[ {P\left( {X,Q|\overline \lambda  } \right)} \right]
\end{equation}
It can be proven also that the maximization of the function increases the likelihood:
\begin{equation}
\mathop {\max }\limits_{\overline \lambda  } \left[ {Q\left( {\lambda ,\overline \lambda  } \right)} \right] \Rightarrow P\left( {X|\overline \lambda  } \right) \ge P\left( {X|\lambda } \right)
\end{equation}

\subsection{The proposed HMM for nucleosome positioning}
As mentioned earlier in \cite{YUAN05} the problem of identifying the nucleosome using data from a process of microarray hybridization  and modeling observations with a particular HMM, was addressed. This is because a simple thresholding technique has not sufficient accuracy because of noise and trend in the data. The proposed model for the detection of nucleosomes in chromatin regions is shown in figure \ref{fig:hmm}. In this model, several different states for  different types of nucleosomes with special connections are considered; in particular the states model the sequences of chromosomes corresponding to  a linker (state $L$), well-positioned nucleosomes (states $N_1$, $N_2$, ..., $N_8$) and  delocalized nucleosomes  (states $DN_1$, $DN_2$, ..., $DN_9$).
\begin{figure}[!htb]
   \centering
      \includegraphics[width=.7 \textwidth]{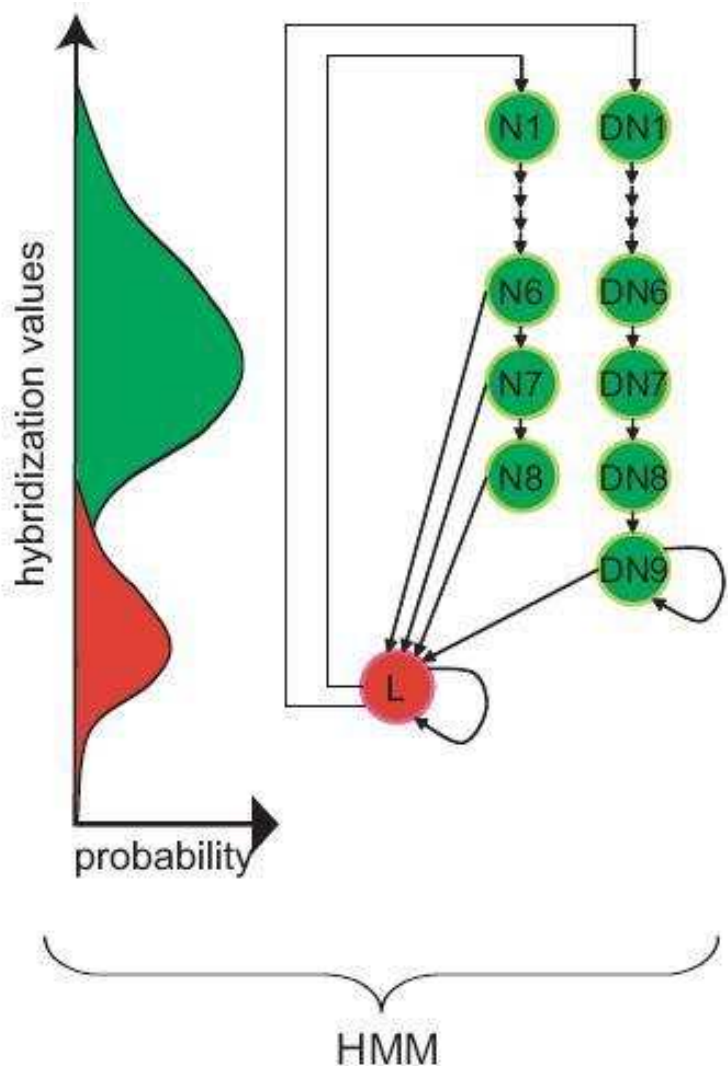}
 \caption{HMM topology for nucleosome positiong}
  \label{fig:hmm}
\end{figure}
The values of the measures that can be observed by each state correspond to the physical values that the system outputs, which in this case represent the logarithmic ratio between the intensity of red and green for each spot of the microarray. The transition matrix that establishes  which are the allowed transitions between states  and their probabilities, is estimated with the Baum Welch algorithm together with the other parameters. In this model there is only one state that represent the class of probes corresponding to linker regions, and this state has a loop in order to model variable length linker regions. The number of states for the class of well-positioned nucleosomes in this model is 8. This choice is justified considering the length of a nucleosome in normal conditions (about 6-8 probe). In this way, the information about the expected length of a nucleosome is encoded in the model. Similarly, it is possible to note that the number of states for the class of delocalized nucleosomes in this model is 9 and the last state has a loop (similar to the state linker)  in order to model the different lengths of nucleosomes regions that cover a number of probe greater than 9.
Finally, a well-positioned nucleosomes in this model have a length between 6 and 8 probes, the  delocalized nucleosomes have a number of probes equal to or greater than 9, and linkers  have a variable length greater or equal to one.

\section {Second solution: MLA} \label{sec:Second solution MLA}
In this section the application of MLA to face the problem of identifying and classifying nucleosomes will be described. The following subsections will show the various steps that allow the classification of the nucleosomes identified trough the MLA and  the construction of a model for well-positioned nucleosomes. Firstly, let's recall that the signal $\mathbf{S}$  is divided into segments in which  probes can  be  not contiguous (due to data referring to different regions of chromosomes, or missing data). In particular $\mathbf{S}$  is organized in  $T$ contiguous fragments $S_1,\cdots, S_T$ which represent $DNA$ sub-sequences.

\subsection{Preprocessing}
In the first stage of processing  a convolution process is applied in order to reduce the noise in the signal.
The smoothing is done for each  probe segment  corresponding to adjacent regions of the signal  i.e each fragment $S_t$, $1 \le t \le T$ of the input signal, $\mathbf{S}$, is smoothed by a convolution operator that perform the weighted average of three consecutive signal values, where the weights are provided by the \emph{kernel window} $w= [\frac{1}{4},\frac{1}{2},\frac{1}{4}]$ \cite{Lyo97}.

\subsection{Creating the model}
The construction of the model represents a phase of training, where it is possible to  learn the shape of the pattern corresponding to the nucleosome considering only the regions that corresponds with high probability to well-positioned nucleosomes. Since well-positioned nucleosomes are shown as peaks of a bell shaped curve, in order to locate the position of a nucleosome, all local maxima of the input signal are automatically extracted from the convolved signal $\mathbf{X}$ of $\mathbf{S}$. Then a subset of maxima are opportunely selected for the model definition. Each convolved fragment $X_t$ is processed in order to find $L(X_t)$ local maxima $M_t^{(l)}$ for $l=1,\cdots,L(X_t)$. The extraction of each sub-fragment for each $M_t^{(l)}$ is performed by assigning all values in a window of radius $os$ centered in $M_t^{(l)}$ to a vector, $F_t^l$ of size $2\times os+1$: $F_t^l(j)=X_t(M_t^{(l)}-os+j-1)$, for $j=1,2,...,2\times os+1$. The selection process extracts the \emph{significant} sub-fragments to be used in the model definition. This is performed by satisfying the following rule:

\begin{equation}
  \begin{cases}
    F_t^l(j+1)-F_t^l(j)>0 & j=1,\cdots,os\\
    F_t^l(j+1)-F_t^l(j)<0 & j=os+1,\cdots,2\times os
  \end{cases}\label{eq0}
\end{equation}

\noindent This condition is equivalent to verify that the signal in that fragment is increasing to the right of the maximum and descending to the left (condition of convexity). If the pattern respects this condition, it will be used for the next phase of construction of the model  of the well-positioned nucleosome. The process continues in a similar way for the other points of relative maximum (if present) in the segment  considered in descending order. After this selection process $G(X_t)$ sub-fragments remain for each $X_t$. The model of the \emph{interesting pattern} is then defined by considering the following average:
\begin{equation}
 \overline{F}(j)=\frac{1}{T} {\sum_{t=1}^T {\frac{1}{G(X_t)} \sum_{k=1}^{G(X_t)}
 F_t^k(j)}}\ \ \ j=1,\cdots,2\times os+1
\end{equation}

That is, for each $j$, the average value of all the sub-fragments satisfying Eq. \ref{eq0}. The model then  will represent the average pattern of a well-positioned nucleosome through its expected shape. Applying this procedure a model shown in figure \ref{fig:model}(a) is carried out averaging the pattern in figure \ref{fig:patterns}(b).

\begin{figure}[!htb]
   \centering
      \includegraphics[width=.7 \textwidth]{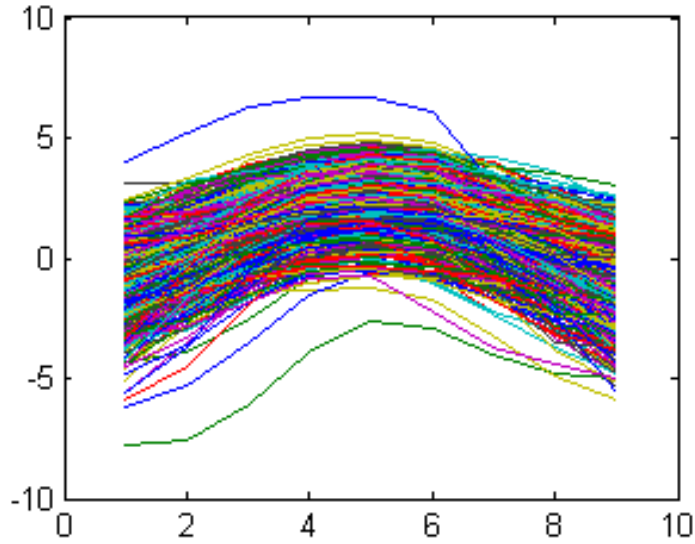}
 \caption{Patterns that meet the condition of convexity}
  \label{fig:patterns}
\end{figure}

\begin{figure}[!htb]
   \centering
      \includegraphics[width=.7 \textwidth]{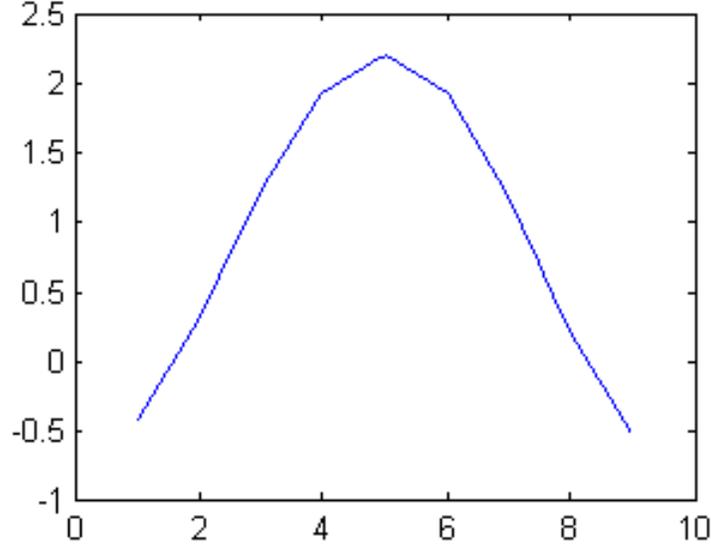}
 \caption{Model of well-positioned nucleosome}
  \label{fig:model}
\end{figure}

\subsection{Interval identification}
 This step is the core of the method i.e. the interval identification obtained by the \emph{Simply Equally spaced MLA} presented in chapter $2$. In particular by considering $K$ threshold levels $t_k$ ($k=1,...,K$) of the convolved signal $\mathbf{X}$, for each $t_k$ a set of intervals $R_k=\left \{ I^1_k,I^2_k,\cdots, I^{n_k}_k \right \}$ is obtained, where, $I^i_k = [b^{i}_{k},e^{i}_{k}]$ and $\mathbf{X}(b^i_k)=\mathbf{X}(e^i_k)=t_k$. This set of intervals as explained in chapter $2$ constitutes the interval representation $\Upsilon(X)$ of the input signal $X$. In Section  \ref{estimation} a calibration procedure to select the proper value of $K$ is described.

\subsection{Aggregation rule and Pattern Definition}
This step is performed by taking into account that bell shaped pattern must be extracted for the classification phase. Such kind of patterns are characterized by sequences of intervals $\left \{ I^1_j,I^2_{j+1},\cdots, I^{n}_{j+l} \right \}$ such that $I^i_{j}\supseteq I^{i+1}_{j+1}$; more formally a pattern $P_i$ is defined using the following aggregation rule:

\begin{equation}
 P_i = \{ I^{i_j}_j, I^{i_{j+1}}_{j+1}, \cdots, I^{i_{j+l}}_{j+l} \ |  \ \forall I^{i_k}_k \ \exists !I \in R_{k+1} :  I=I^{i_{k+1}}_{k+1} \subseteq I^{i_k}_k  \}
\end{equation}

\noindent where, $j$ defines the threshold, $t_j$, of the widest interval of the pattern. From the previous definition it follows that $P_i$ is build by adding an interval $I^{i_{k+1}}_{k+1}$ only if it is the unique in $R_{k+1}$ that is included in $I^{i_k}_k$. Note that, this criterion is inspired by the consideration that a nucleosome is identified by bell shaped fragment of the signal, and the intersection of such fragment with horizontal threshold lines results on a sequence of nested intervals. In figure \ref{fig2} two examples of shapes with the relative patterns are shown.

\begin{figure}[!htb]
\centering
\includegraphics [scale=1.5]{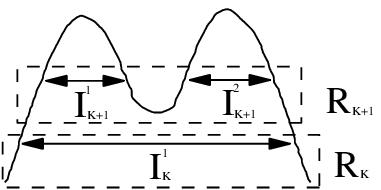}
\includegraphics [scale=1.5]{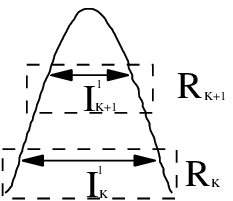}
\caption{\emph{Two different shapes of the input signal:} \emph{(on the left)} Since at threshold level $K+1$ the interval $R_k=\{ I^1_K \}$ has two subset $R_{k+1}=\{I^1_{K+1},I^2_{K+1} \}$, it is possible to set three pattern $P_1 = \{ I^1_K \}$, $P_2 = \{ I^1_{K+1} \}$ and $P_3 = \{ I^2_{K+1} \}$. \emph{(on the right)} In this case, $I^1_{K+1}$ is the unique subset of  $I^1_K$, thus it is possible to set an unique pattern  $P_1 = \{ I^1_K, I^1_{K+1}  \}$ }\label{fig2}
\end{figure}

\subsection{Pattern selection}
In this step the \emph{interesting patterns} $\mathbf{P^{(m)}}$ are selected following the criterium:

\begin{equation}
\mathbf{P^{(m)}} = \{ P_i \ : \ |P_i| > m \}
\end{equation}
\noindent i.e. patterns containing intervals that persists at least for $m$ increasing thresholds. This further selection criterion is related to the height of the shaped bell fragment, in fact a small value of $m$ could represents noise rather than nucleosomes. The value $m$ is said the \emph{minimum number of permanences}; in subsection \ref{estimation} a calibration procedure to  estimate the best value of $m$ is described.

\subsection{Feature extraction}
Each pattern $P_i \in \mathbf{P^{(m)}}$ is identified by $I^{i_j}_j, I^{i_{j+1}}_{j+1}, \cdots, I^{i_{j+l}}_{j+l}$, with $l \ge m$. Straightforwardly, the feature vector of $P_i$ is a $2 \times l$ matrix where each column represents the lower and upper limits of each interval from the lower threshold $j$ to the upper threshold $j+l$. The representation in this multi-dimensional feature space is used to characterize different types of patterns.

\subsection{Dissimilarity function}
A dissimilarity function between patterns  is defined in order to characterize their shape:

\begin{equation}
\delta(P_r,P_s)=(1- \alpha) (A_r-A_s) + \alpha \mathop{\sum}_{i \in I} (a^{r_i}_{i}-a^{s_i}_{i}) \label{eq1}
\end{equation}

\noindent where, $A_r$ and $A_s$ are the surfaces of the two polygons bounded by the set of vertexes
$V=\mathop{\bigcup}_{i \in I} \{ (b^{r_i}_{i}, e^{r_i}_{i}) , (b^{s_i}_{i}, e^{s_i}_{i}) \}$,
$a^{r_i}_{i}=e^{r_i}_{i}-b^{s_i}_{i}$, $a^{s_i}_{i}=e^{s_i}_{i}-b^{s_i}_{i}$, and $\alpha$ is a user parameter ranging in the interval $[0,1]$ to set the weight of the two dissimilarity components.

The first component of this dissimilarity allow us to consider patterns of close dimensions, while the second component has been introduced to include shape information since it can be considered a correlation measure of the two bounding polygons. This dissimilarity can be used by a general classifier in order to distinguish the kind of pattern. An example of input signal and the extracted interesting patterns is given in figure \ref{fig1}.

\begin{figure}[!htb]
\centering
\includegraphics [width=0.8\textwidth]{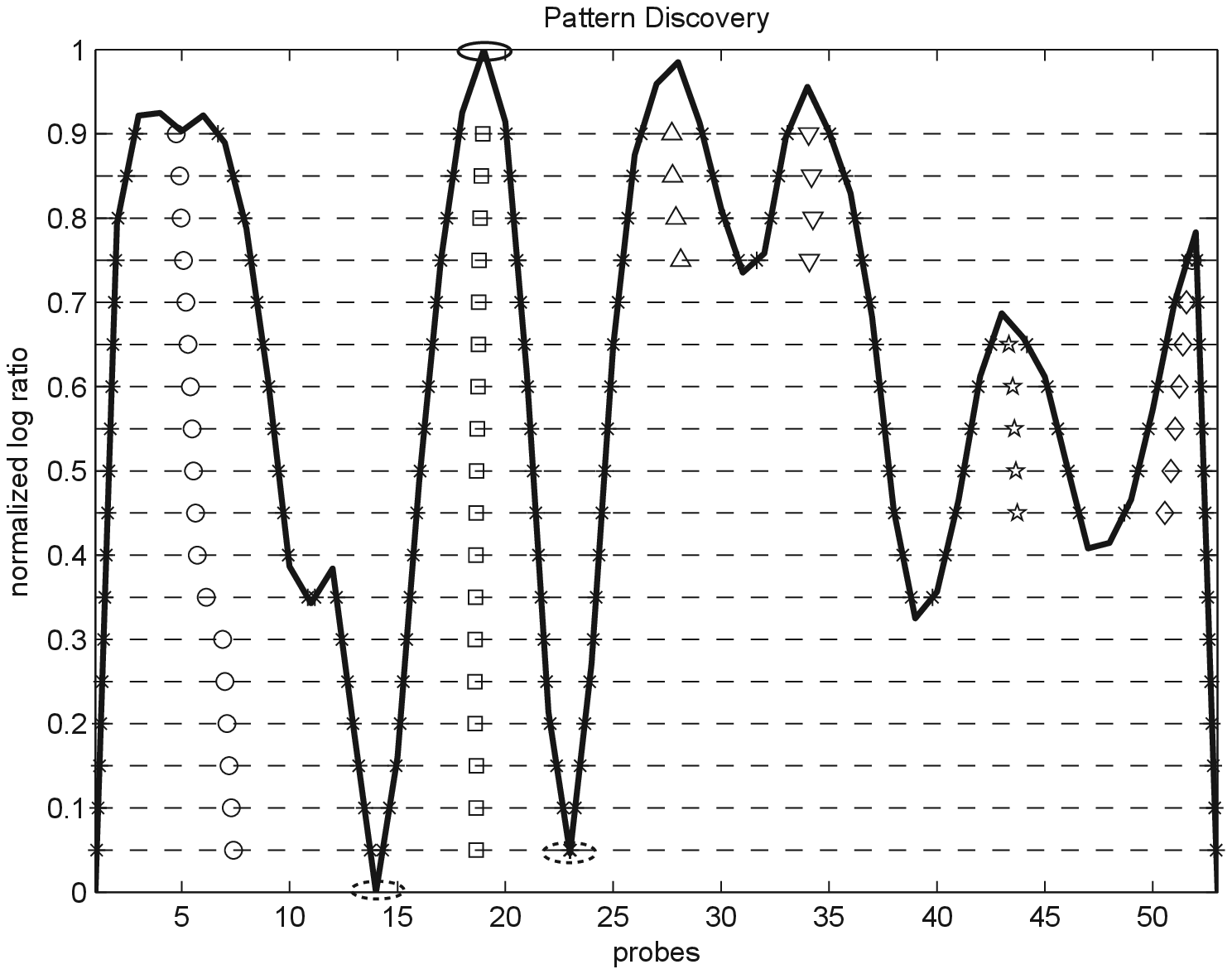}
\includegraphics [width=0.8\textwidth]{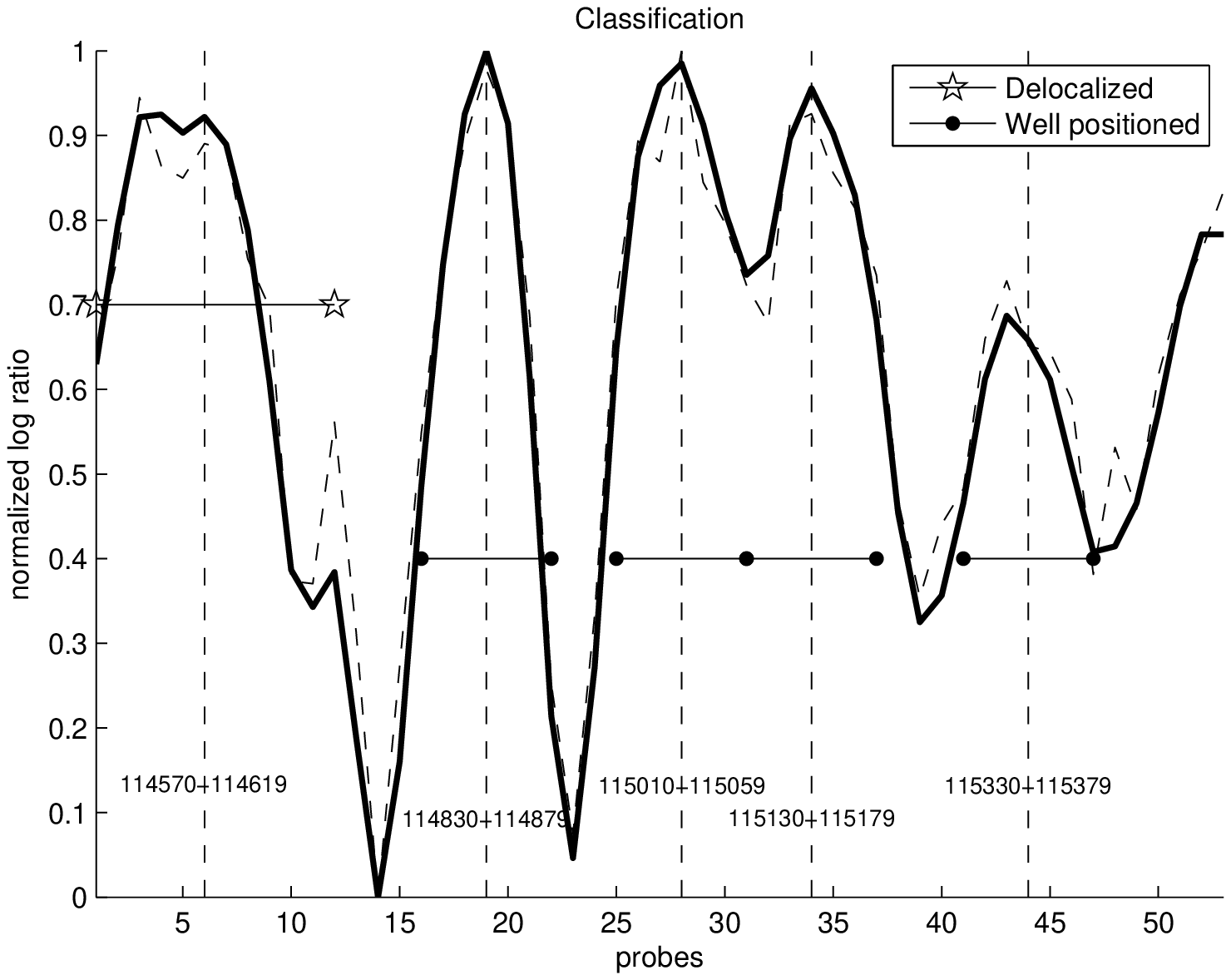}
\caption{ \emph{(a) Input signal, smoothing, pattern identification and extraction:}
A \emph{Saccharomyces cerevisiae} microarray data portion. Each $x$ value represents a spot (probe) on
the microarray and the corresponding $y$ value is the logarithmic ratio of its Green and Red values.
Nucleosomes regions are around the peaks signal (one is marked by black circle), while lower ratio values
show linker regions (marked by dashed circles).
The dashed lines represents the threshold levels, in this example  $6$ patterns are retrieved, identified by rhombus, circle, square,
triangle down, triangle up, star. Each pattern
identifier is replicated for each of its feature values and pointed in
each one of its middle point.
\emph{(b) An example of classification:} In this portion $5$ nucleosome regions are shown together with
its range in base pairs. In particular $1$ out of the $5$ regions is classified as \emph{delocalized} while the remaining \emph{well-positioned.}}\label{fig1}
\end{figure}

\subsection{Nucleosome Classification} \label{bio_goal}
With the MLA, one is able to classify four ``refined nucleosomal states'':  \emph{linkers}, \emph{well-positioned}, \emph{delocalized} and \emph{fused nucleosomes}. (see figure \ref{fig3a}). In the following, the classification rules which allow us to  automatically discriminate such kind of patterns are stated. The classification was conducted in two steps, in the first step the \emph{linker patterns}, the \emph{expected well-positioned patterns} and \emph{expected delocalized patterns} are found. Afterwards, the ranges of the regions representing the expected well-positioned and delocalized nucleosomal patterns are set, defining the \emph{expected regions}. Finally, the classification is performed by testing the intersection of such regions (see figure \ref{fig3b}).

\begin{figure}[!htb]
\centering
\includegraphics [width=0.6\textwidth]{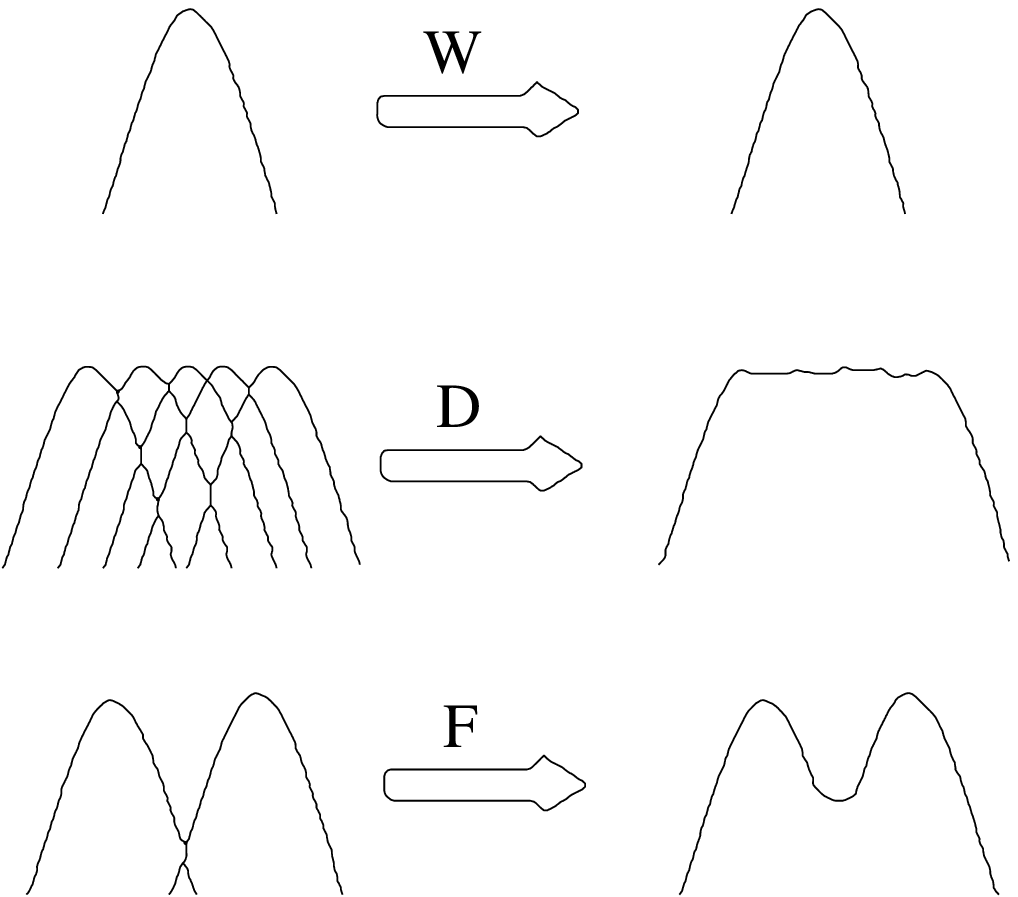}
\caption{\emph{Shapes of the patterns:} The three classes of nucleosomes it is possible to detect with the MLA very
likely reflect different nucleosome mobility existing in vivo at specific chromatin loci. Delocalized
nucleosomes probably represent single nucleosomes or arrays of nucleosomes with high mobility, while fused
nucleosomes  may reflect a single nucleosome that occupies two distinct close positions in different cells. On
the left of the arrows, the particular nucleosome configuration which generates the resulting shape of well-positioned (W), delocalized (D) and fused (F) nucleosome classes are shown.}\label{fig3a}
\end{figure}

\begin{figure}[!htb]
\centering
\includegraphics [width=0.5\textwidth]{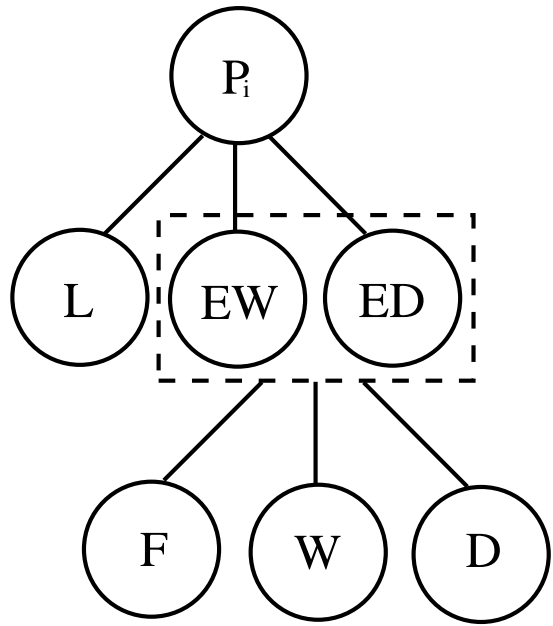}
\caption{\emph{Classification:} The
classification of a generic pattern $P_i$ is performed into two phases. In the first phase the linker ($L$), the
expected well-positioned ($EW$) and the expected delocalized ($ED$) patterns are established by using the
classification rule defined by $c_1$. In the second phase, the expected regions $A_i$ are defined by opportunely
processing $EW$ and $ED$ patterns, and afterwards used by the classification rule $c_2$ in order to finally
classify between well-positioned ($W$), delocalized ($D$) and fused ($F$) nucleosomes.}\label{fig3b}
\end{figure}

\noindent \emph{First phase:}\\
For each interesting pattern $P_i$, the dissimilarity $\delta(P_i,\overline{F})$ is evaluated ($\delta$ is defined in Eq. \ref{eq1}, $\overline{F}$ is the model), the rule to classify $P_i$ is :

\begin{equation}
c_1(P_i)=
  \begin{cases}
    L  & \mbox{ if } \delta(P_i, \overline{F}) \le \phi_1 \\
    EW & \mbox{ if } \phi_1 < \delta(P_i, \overline{F}) \le \phi_2 \\
    ED & \mbox{ otherwise }\\
  \end{cases}
  \label{eq3}
\end{equation}

\noindent where $L$ means \emph{linker pattern}, $EW$ or $ED$  are nucleosomal pattern, and in particular \emph{expected
well-positioned patterns} and \emph{expected delocalized patterns} respectively.

\emph{Second phase:} Afterwards, for  each expected well-positioned nucleosomal pattern $P_i = \{I^{i_j}_j,I^{i_{j+1}}_{j+1}, \cdots,I^{i_{j+l}}_{j+l} \}$ (e.g. $c_1(P_i)=EW$), the \emph{center of the
nucleosomal region} $C_i$ is calculated:

\begin{equation}
C_i = \frac{1}{l} \sum_{k=j}^{j+l} \frac{e^{i}_{k}+b^{i}_{k}}{2}
\end{equation}

\noindent which represents the mean of the first $l$ intervals defining the pattern $P_i$. Conversely,  for each expected delocalized nucleosomal pattern (e.g. $c_1(P_i)=ED$), the \emph{delocalized interval} $[B^{i},E^{i}]$ is defined such that:

\begin{equation}
B^{i} = \frac{1}{l/2} \sum_{k=j}^{j+(l/2)} b^{i}_{k} \hbox{ and } E^{i} = \frac{1}{l/2} \sum_{k=j}^{j+(l/2)} e^{i}_{k}
\end{equation}

\noindent Note that,  $B^{i}$ and $E^{i}$  represent respectively the mean of the first $l/2$ beginning  and ending of each interval belonging to the pattern $P_i$. The \emph{expected regions} is so defined:

\begin{equation}
A_i=
  \begin{cases}
    [C_i(l)-3,C_i(l)+3]  & \mbox{ if } c_1(P_i)=EW\\
    [B^{i},E^{i}] & \mbox{ otherwise }\\
   \end{cases}
\end{equation}

\noindent In particular,  each expected region $A_i$ is, in the case $P_i$ is an expected well-positioned pattern, an interval with beginning $3$ probes before and ending $3$ probes after the center $C_i$, otherwise it is the interval $[B^{i},E^{i}]$. Finally, the classification rule is:

\begin{equation}
c_2(P_i)=
  \begin{cases}
    F \ \ \mbox{if }A_i\cap A_j \neq \emptyset \  j \neq i \mbox{ otherwise}\\
    \left [\begin{array}{ll}
    W &\mbox{if } c_1(P_i)=EW\\
    D &\hbox{if } c_1(P_i)=ED\\
    \end{array} \right.
   \end{cases}
  \label{eq4}
\end{equation}

\noindent where $F$, $W$ and $D$ stands for \emph{fused}, \emph{well-positioned}, \emph{delocalized} nucleosomes respectively (see figure \ref{fig3a}). Informally, the classification rule in Equation \ref{eq4}  assign the \emph{fused} class if the expected nucleosomal regions overlap otherwise confirm the classification of the first phase.

\subsection{Parameter selection by calibration} \label{estimation}
In order to set the proper values of $K$ (number of thresholds), and $m$ (the minimum  number of permanences), a calibration procedure has been used. In particular, such values has been estimated by studying the plots of particular functions able to measure the goodness of several $K$ and $m$.

\subsubsection{Estimation of $m$} The \emph{minimum number of permanences} $m$ has been estimated by using the synthetic signal generator described above. This gives the opportunity to make a massive experimental study on the relation between $K$ and $m$. In particular, $c=10$ copies at different signal to noise ratios $j=1,2,4$ has been generated, resulting in a total of $3 \times 10$ synthetic signals $V_{ij}$. Once fixed a signal to noise
ratio $j$, for each $V_{ij}$ the value of $m$ which maximizes the recognition performances for several thresholds for $k=20, \cdots, 50$ has been found.

\begin{figure}[!htb]
\centering
   \includegraphics[width=1 \textwidth]{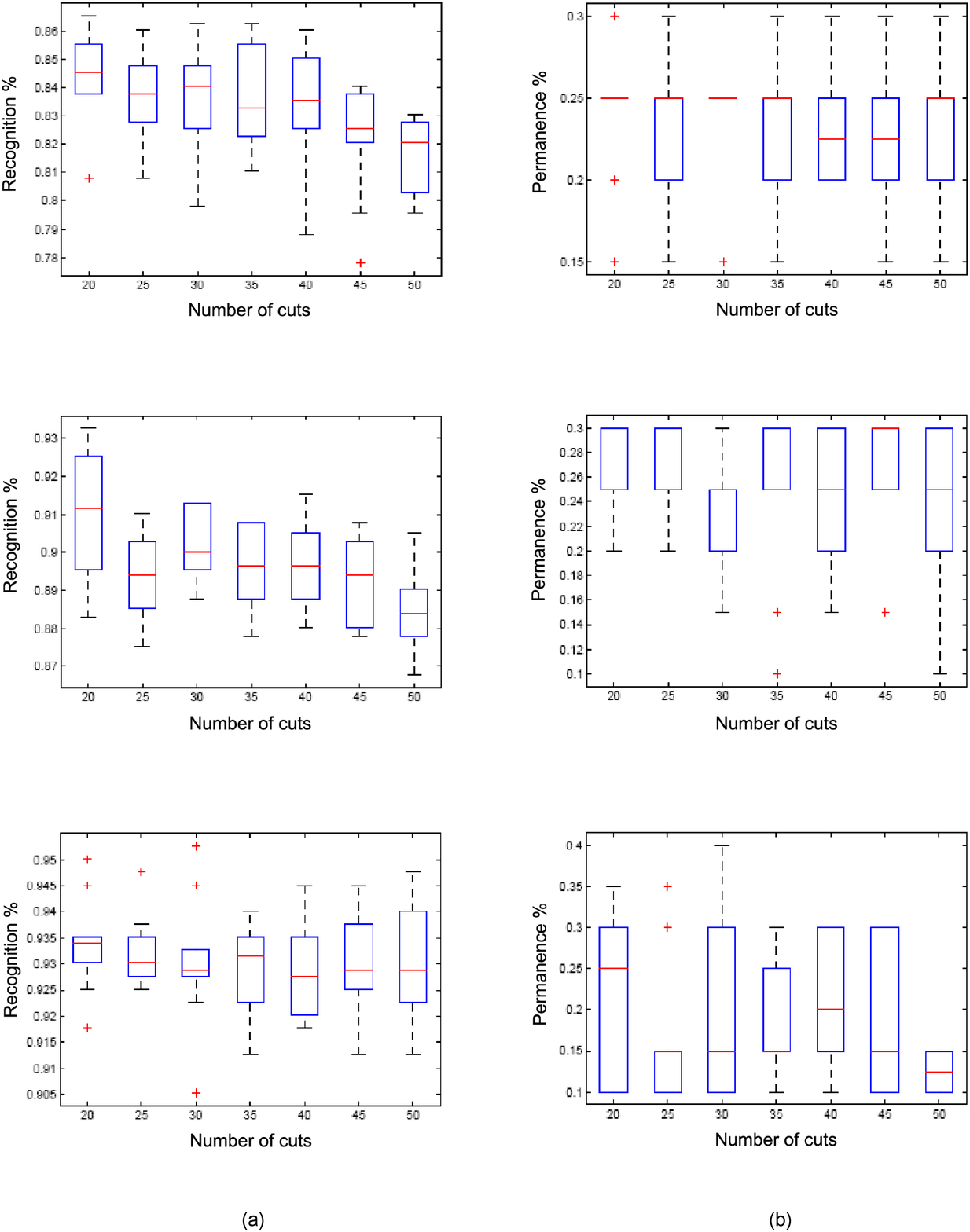}
\caption{\emph{Calibration phase for the choice of $m$}: Recognition performance plots (group a) and percentage of minimum number of permanences plots (group b) for 3 different signal to noise ratios, SNR = 1,2,4 (first, second, third column respectively). The bar in each plot groups the results for 10 experiments occurring at several threshold values (i.e number of cuts).}\label{fig:s1}
\end{figure}

Figure \ref{fig:s1} shows the results performed by considering $c=10$ copies, three signal to noise ratio values $1,2,4$, and $k=20,\cdots,50$ thresholds. In each plot, the $x$ axis represents the number of thresholds $k$ (i.e. number of cuts), the column  bar groups the best recognition and the percentage of minimum number of permanences which causes the best performances on all the $10$ experiments. From this experimental study, it emerges that the use of an high number of thresholds can compromise the recognition process (due to the fact that an high value of $K$ can capture also the noise present in the signal), moreover, the $m$ value seems not dependent from $K$, and the one which causes the best recognition ranges in an interval of $[0.15 \times K , 0.30 \times K]$.

\subsubsection{Estimation of $K$}
The proper value of $K$ is estimated starting from the convolved input signal $X$. Giving a convoluted signal fragment $X_t$ it is resampled it in the $y$ direction resulting in several samples $X_{t}^{(k)}$ for different threshold values $k=1,\cdots,K_{max}$. It is possible to measure the goodness of $k$ by the \emph{average normalized correlation} $\overline{\varrho(k)}$ and the \emph{average missing probes} $\overline{MS(k)}$ so defined:

\begin{equation}
\overline{\varrho(k)}  =\frac{1}{T} \sum^{T}_{t=1} \frac{1+\rho^2(S_t,S_{t}^{(k)})}{2}\\
\end{equation}

\begin{equation}
\overline{MS(k)} =\frac{1}{T} \sum^{T}_{t=1} MS(k,t)
\end{equation}

\noindent In particular $\overline{\varrho(k)}$  measures the average normalized correlation between each resample $X_{t}^{(K)}$ and the generic fragment $X_t$ ($\rho$ is the correlation coefficient), while $\overline{MS(k)}$ the average of the missing probe values $MS(k,t)$ due to the resample of $X_t$ by $k$ thresholds. Finally the value $K$ is selected interactively by looking both at the plots of $\overline{\varrho}$ and $\overline{MS}$,searching for the best compromise of maximum $\overline{\varrho}$ and minimum $\overline{MS}$ (see figure \ref{fig:s2}). In this way the signal obtained has an high correlation with the original signal and a reasonable number of missing samples in order to not capture the noise present in the signal.

\begin{figure}[!htb]
\centering
   \includegraphics[width=0.8 \textwidth]{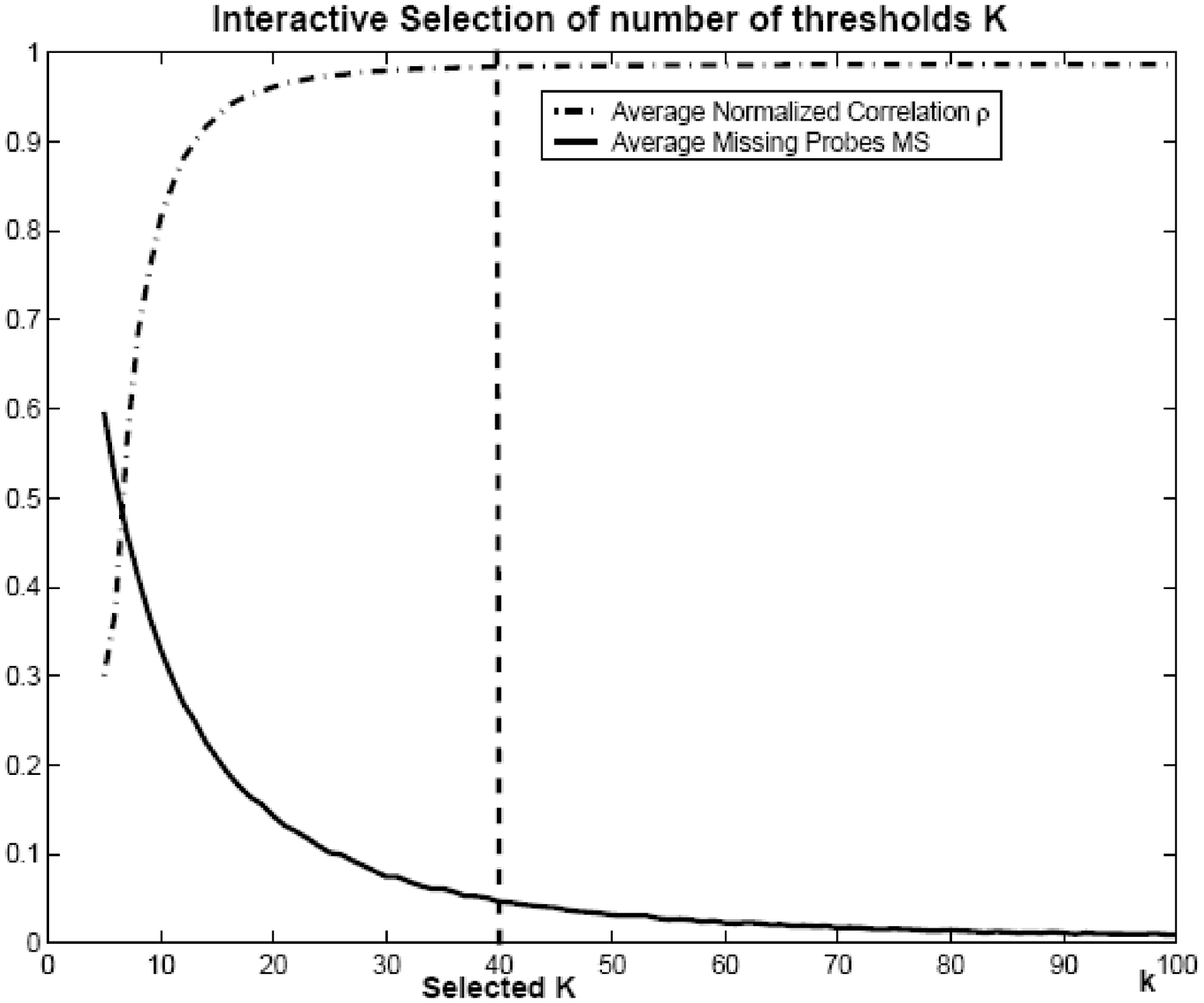}
\caption{ \emph{Calibration phase for the choice of K}: The value for $K$ is selected interactively by looking both at the plots of $\overline{\varrho}$ and $\overline{MS}$} \label{fig:s2}
\end{figure}

\subsection{Synthetic generation of biological signals} \label{simula}
Before validating the $MLA$ approach on biological data, a procedure to generate synthetic signal has been developed in order to assess the feasibility of the method on controlled data. Generated signals emulate the one coming from a tiling microarray where each spot represents a \emph{probe} $i$ of resolution $r$ base pairs overlapping $o$ base pairs with probe $i+1$. In particular, the chromosome is spanned by moving a window (probe) $i$ of width $r$ base pairs from left to right, measuring both the percentage of mononucleosomal DNA $G_i$ (\emph{green channel}) and whole genomic DNA $R_i$ (\emph{red channel}) within such window, respecting also that two consecutive windows (probes) have an overlap of $o$ base pairs. The resulting signal $V(i)$ for each probe $i$ is the logarithmic ratio of the \emph{green channel} $G_i$ to \emph{red channel} $R_i$. Intuitively, nucleosomes presence is related to peaks of $V$ which correspond to higher logarithmic ratio values, while lower ratio values shows nucleosome free regions called \emph{linker regions}. This genomic tiling microarray approach takes inspiration from the work of Yuan et al. \cite{YUAN05} where the authors have used the same methodology on the \emph{Saccharomyces cerevisiae} DNA. Here it is defined a model able to generate such signals characterized by the following parameters:
\begin{itemize}
\item \emph{nn:} The number of nucleosomes to add to the synthetic signal.

\item \emph{nl:} The length of a nucleosome (in real case a nucleosome is $146$ base pairs long)

\item \emph{$\lambda$:} Mean of the Poisson distribution used to model the expected distances between adjacent nucleosomes;

\item \emph{r:} The resolution of a single microarray probe.

\item \emph{o:} The length in base pairs of the overlapping zone between two consecutive probes.

\item \emph{nr:} The number of spotted copies (replicates) of nucleosomal and genomic DNA on each probe of the microarray;

\item \emph{dp:} The percentage of the delocalized nucleosomes over the total number of nucleosomes;

\item \emph{dr:} The range which limits the delocalization of a nucleosome in each copy of \emph{nr}. It is defined in base pairs.

\item \emph{nsv:} The variance of the green channel in each probe, even in absence of nucleosomes due to the cross hybridization. This variance follows a normal distribution with mean $0.1$.

\item \emph{pur:} The percentage of DNA purification, which is the probability that each single DNA fragment of the \emph{nr} copies appears in the microarray hybridization.

\item \emph{ra:} Relative abundance between nucleosomal and genomic DNA.

\item \emph{SNR:} The linear signal to noise ratio of the synthetic signal to generate. Note that the noise is assumed to be gaussian.
\end{itemize}

\noindent Initially, a binary mask signal $M$ is generated by considering as $1$'s all the base pairs representing a nucleosome (the \emph{nucleosomal regions}) and as $0$'s the regions representing linkers (\emph{the linker regions}). Note that, the beginning of each nucleosomal region is established by the Poisson distribution with mean $\lambda$. The mask signal $M$ will be used in order to validate the classification results. The red channel of the microarray (the genomic channel) results from the generation of \emph{nr} replicates $I^R_1, \cdots,I^R_{nr}$  each one starting from an initial nucleosomal region of random size $b \sim  U(0,r)$ (uniformly distributed in the range $[0, r]$), followed by continuous nucleosomic region of \emph{r} base pairs. Conversely, in order to simulate the green channel (the nucleosomic channel) \emph{nr} replicates , $I^G_1, \cdots,I^G_{nr}$ are considered, each one initially equal to $M$ and subsequently modified by perturbing each starting points $x^i_D$ of the nucleosome to consider as delocalized such that $x^i_D=x^i_D+\mu$ with random $\mu \sim U(dr)$. Note that the percentage of nucleosomes to consider as delocalized is established by the parameter \emph{dp}. Afterwards, each nucleosomal region on the generic replicate $I^R_i$ and  $I^G_i$ can be switched off depending on the value of a random variable $\alpha \sim U(0,1)$. Precisely, each nucleosomal region veryfing the test $\alpha<pur$ is considered and set to $1$, otherwise it is not considered and set to $0$. This results in new replicates $T^R_i$ and $T^G_i$. Finally, the generated synthetic signal $V$ for a probe $i$ is so defined:

\begin{equation}
\begin{array}{ll}
V(i)=  & \{ \mathop{log_2}{( \sum_{j=1}^{nr}
{\frac{T^G_j(k)*ra}{T^R_j(k)}} + \epsilon})  | (r-o)i-r+o+1 \le k \le (r-o)i+o \} \label{eq2}
\end{array}
\end{equation}

\noindent where $\epsilon \sim N(0.1,nsv)$. In figure \ref{fig:synth} it is possible to see the steps of this process.

\begin{figure}[!htb]
\centering
\includegraphics[width=.7 \textwidth]{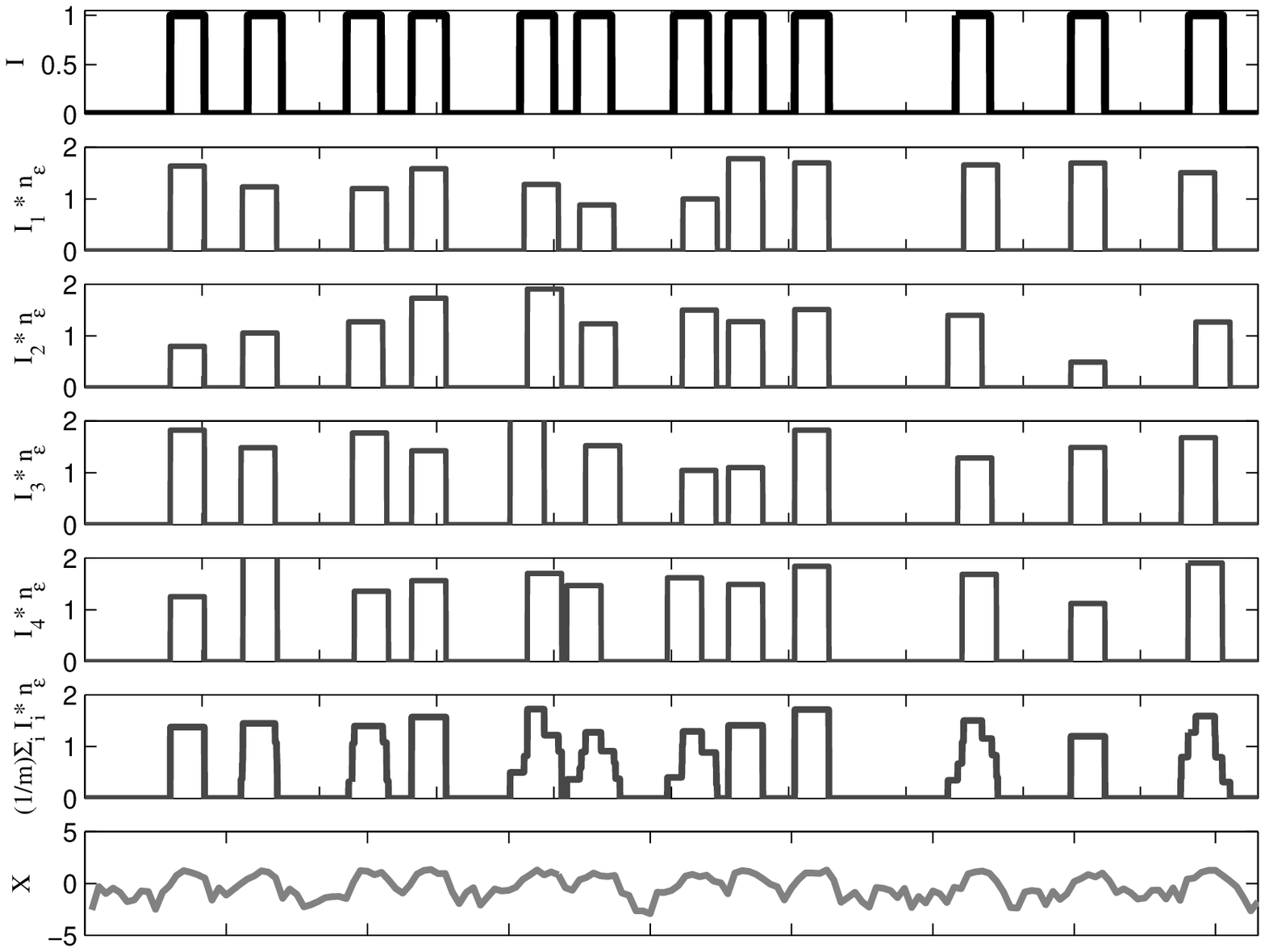}
\caption{{\emph{An example of synthetic signal generation.}
}} \label{fig:synth}
\end{figure}

\section{Results}\label{experiment}
The following experiments have been carried out by measuring the correspondence between nucleosome and linker
regions. In the case of the synthetic signal, the output of the classifier has been compared with a mask $M'$ derived from $M$ while in the case of the real data set it has been compared with the output of the HMM for nucleosome positioning (see section \ref{HMM}) optimally converted into a binary string.\\ In all the experiments, the same value $(\phi_1,\phi_2)=(mean(\delta(F_t^l,\overline{F}))-3std(\delta(F_t^l,\overline{F})), mean(\delta(F_t^l,\overline{F}))+3std(\delta(F_t^l,\overline{F}))$ has been considered, where $F_t^l$ are all the sub-fragments used on the construction of the model $\overline{F}$. Moreover, by biological consideration, the radius $os$ has been set to $os=4$. The performances have been evaluated in terms of \emph{Recognition Accuracy}, $RA$. The $RA$ uses a new mask $M'$ obtained by converting $M$ into probe coordinates such that a
probe value is set to $1$ (e.g. shows a nulceosome portion) if the corresponding base pairs in $M$ include at least a $1$. The real nucleosomal (linker) regions $RNR$ ($RLR$) are represented by $M'$ as contiguous sequence of $1$'s or $0$'s respectively, here a nucleosomal (linker) region $CNR$ ($CLR$) has been classified correctly if there is a match of at least $l=0.7\times L$ contiguous $1$'s ($0$'s) between $CNR$
($CLR$) and the corresponding $RNR$ ($RLR$) in $M'$ where $L$ is the length $RNR$ ($RLR$). The value $0.7$ has been chosen because it represents a $70\%$ of regions overlap very unlikely to be due to chance.

\subsection {MLA vs HMM on Synthetic Nucleosome Positioning data}

For MLA, we have chosen by the calibration phase $K=20$ and $m=5$, the value of $\alpha$ in Eq. \ref{eq1} has been set to $0.5$ to equally balance the two component of the dissimilarity. In particular, $6$ signals of length ranging from $2337$ probes ($70130$ base pairs) to $2361$ probes ($70850$ base pairs) have been generated for the  signal to noise ratio values $1,2,4,6,8,10$.  In Fig.\ref{fig6} the results of the total $RA$ for all the experiments are reported. The confusion matrices of $HMM$ and MLA for all the experiments are reported in the tables \ref{tab2} and \ref{tab3}. In Fig.\ref{fig6} the results of the total $RA$ for all the experiments are summarized. Fig.\ref{fig6} shows that the $HMM$ is slightly more accurate in finding the bounds of the nucleosome regions. The synthetic results can be summarized in an overall $RA$ of $0.96$ for the MLA and $0.98$ for $HMM$.

\begin{figure}[!htb]
\centering
\includegraphics[width=.7 \textwidth]{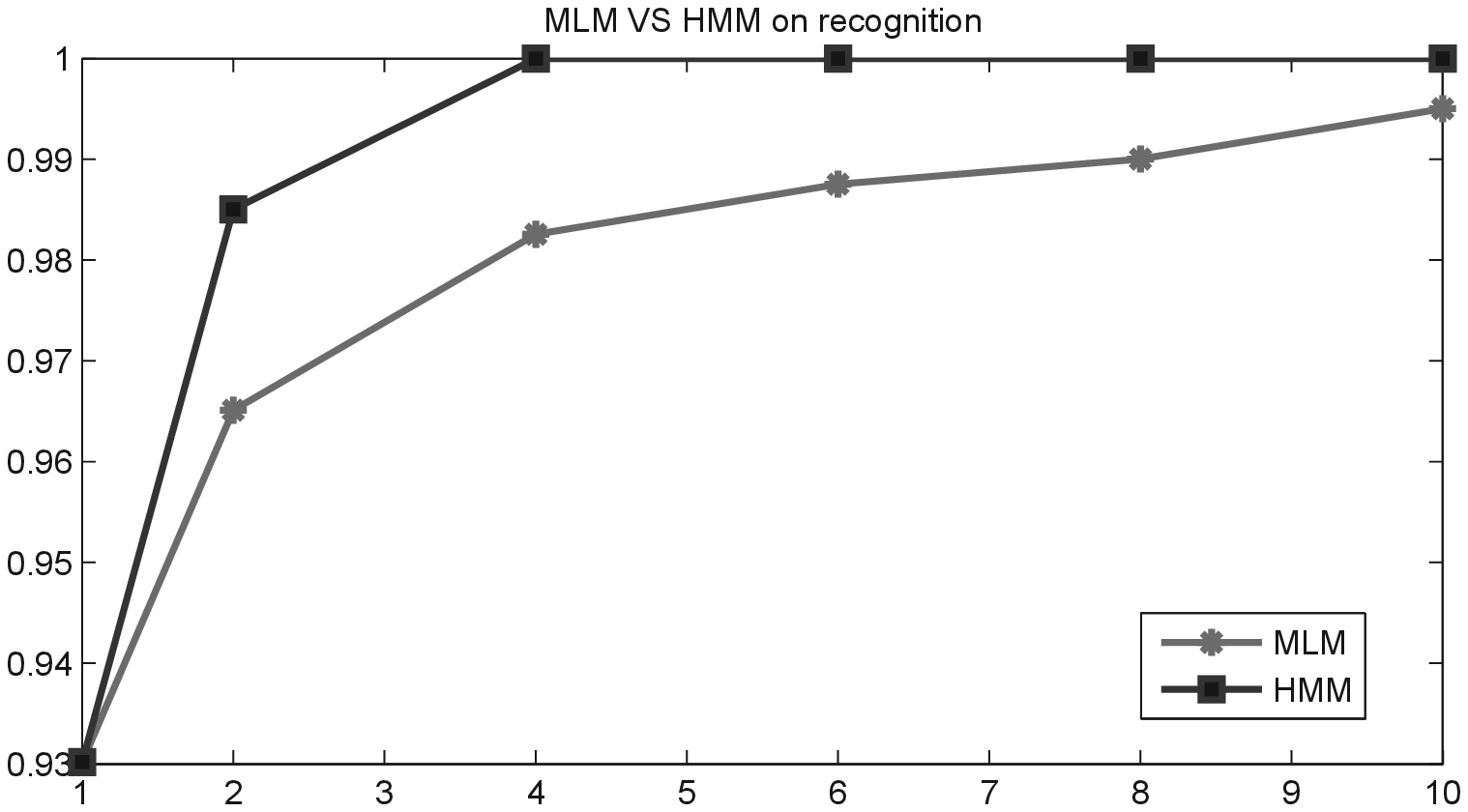}
\caption{\emph{Results on synthetic data:} The Recognition Accuracy
of MLA and $HMM$ on $6$ synthetic signals generated at signal to noise ratios $1,2,4,6,8,10$.}\label{fig6}
\end{figure}

\begin{table}[!htb]
\begin{center}
\begin{tabular}{|c|c|c||c|c|c|}
\hline
$snr= 1$ & $L$  & $N$ & $snr= 2$ & $L$  & $N$\\
\hline
$L$ &  $0,82$ & $0,18$ & $L$ &  $0,96$ & $0,04$\\
$N$ &  $0,03$ & $0,97$ & $N$ &  $0,01$ & $0,99$\\
\hline
\hline
$snr= 4$ & $L$  & $N$ & $snr= 6$ & $L$  & $N$\\
\hline
$L$ &  $1$ & $0$ & $L$ &  $1$ & $0$\\
$N$ &  $0$ & $1$ & $N$ &  $0$ & $1$\\
\hline
\hline
$snr= 8$ & $L$  & $N$ & $snr=10$ & $L$  & $N$\\
\hline
$L$ &  $0.99$ & $0.01$ & $L$ &  $1$ & $0$\\
$N$ &  $0$ & $1$ & $N$ &  $0$ & $1$\\
\hline
\end{tabular}
\caption{Confusion matrices of $HMM$ on $6$ different signal to noise ratios for nucleosome (N) and linker (L) regions.}
\label{tab2}
\end{center}
\end{table}

\begin{table}[!htb]
\begin{center}
\begin{tabular}{|c|c|c||c|c|c|}
\hline
$snr=1$ & $L$  & $N$ & $snr=2$ & $L$  & $N$\\
\hline
$L$ &  $0,81$ & $0,19$ & $L$ &  $0,88$ & $0,12$\\
$N$ &  $0,04$ & $0,96$ & $N$ &  $0$ & $1$\\
\hline
\hline
$snr=4$ & $L$  & $N$ & $snr=6$ & $L$  & $N$\\
\hline
$L$ &  $0,94$ & $0,06$ & $L$ &  $0,96$ & $0,04$\\
$N$ &  $0,01$ & $0,99$ & $N$ &  $0$ & $1$\\
\hline
\hline
$snr=8$ & $L$  & $N$ & $snr=10$ & $L$  & $N$\\
\hline
$L$ &  $0,96$ & $0,04$ & $L$ &  $0,97$ & $0,03$\\
$N$ &  $0$ & $1$ & $N$ &  $0$ & $1$\\
\hline
\end{tabular}
\caption{Confusion matrices of MLA on $6$ different signal to noise ratios for nucleosome (N) and linker (L) regions.}
\label{tab3}
\end{center}
\end{table}

\subsection{MLA vs HMM on real data}
 In this experiment, it has been compared the accordance of the two models on the \emph{Saccharomyces cerevisiae} real data. The input signal representing this data is composed by $215$ contiguous fragments for a total of $24167$ base pairs. In such experiment,  $K=40$, $m=6$ were chosen by the calibration phase ($m=0.15 \times 40$) and $\alpha=0.5$ was considered to equally balance the two components of the dissimilarity (see the definition in Eq. \ref{eq1}). The confusion matrices which show the $RA$ of $HMM$ considering MLA as the truth classification and $RA$ of MLA considering $HMM$ as the truth classification are reported in table \ref{tab4}. The results can be summarized in an overall $RA$ of $0.83$ for the $HMM$ (MLA true) and $0.69$ for MLA ($HMM$ true). In particular, from this studies it is possible to conclude that MLA does not fully agree with $HMM$ on the linkers patterns. Remarkably, comparing MLA and $HMM$  on the data coming from recently developed \emph{deep sequencing approach} ($DS$) \cite{pugh} it is possible to see a better agreement with MLA ($0.58$) rather than with $HMM$ ($0.44$) (table \ref{tab5} and figure \ref{fig:s3}). These analysis indicate that the integration of the $HMM$ and MLA could improve the overall classification.

\begin{table}[!htb]
\begin{center}
\begin{tabular}{|c|ccc||c|ccc|}
\hline
& $M$ & $L$ & $M$ & & $H$ & $M$ & $M$ \\
\hline
 $H$ & & $L$  & $N$ &  $M$ & & $L$  & $N$\\
 $M$ & $L$  &   0.79 &   0.21 & $L$ & $L$  &   0.52 &   0.47\\
 $M$ & $N$  & 0.13 & 0.87 &  $M$ & $N$  & 0.12 & 0.87\\
\hline
\end{tabular}
\caption{Agreement between the $HMM$ and MLA (and viceversa) on the Saccharomyces cerevisiae data set for
Nucleosomes (N) and Linker (L) regions. The table on the left shows the $RA$ results of $HMM$ when considering MLA as the truth classification, while the opposite is shown on the right table.
}
\label{tab4}
\end{center}
\end{table}

\begin{table}[!htb]
\begin{center}
\begin{tabular}{|ccc||ccc|}
\hline
  $M$ & $L$ & $M$ & $H$ & $M$ & $M$\\
\hline
   & $L$  & $N$ & & $L$  & $N$\\
   $L$ & 0.40 & 0.60 & $L$ & 0.40 & 0.60\\
   $N$ & 0.24 & 0.76 & $N$ & 0.53 & 0.46\\
\hline
\end{tabular}
\caption{Confusion matrices of MLA and $HMM$ on deep sequencing approach (DS) data by Pugh et Al. (2007).}
\label{tab5}
\end{center}
\end{table}

\begin{figure}[!htb]
\centering
\includegraphics [width=1 \textwidth]{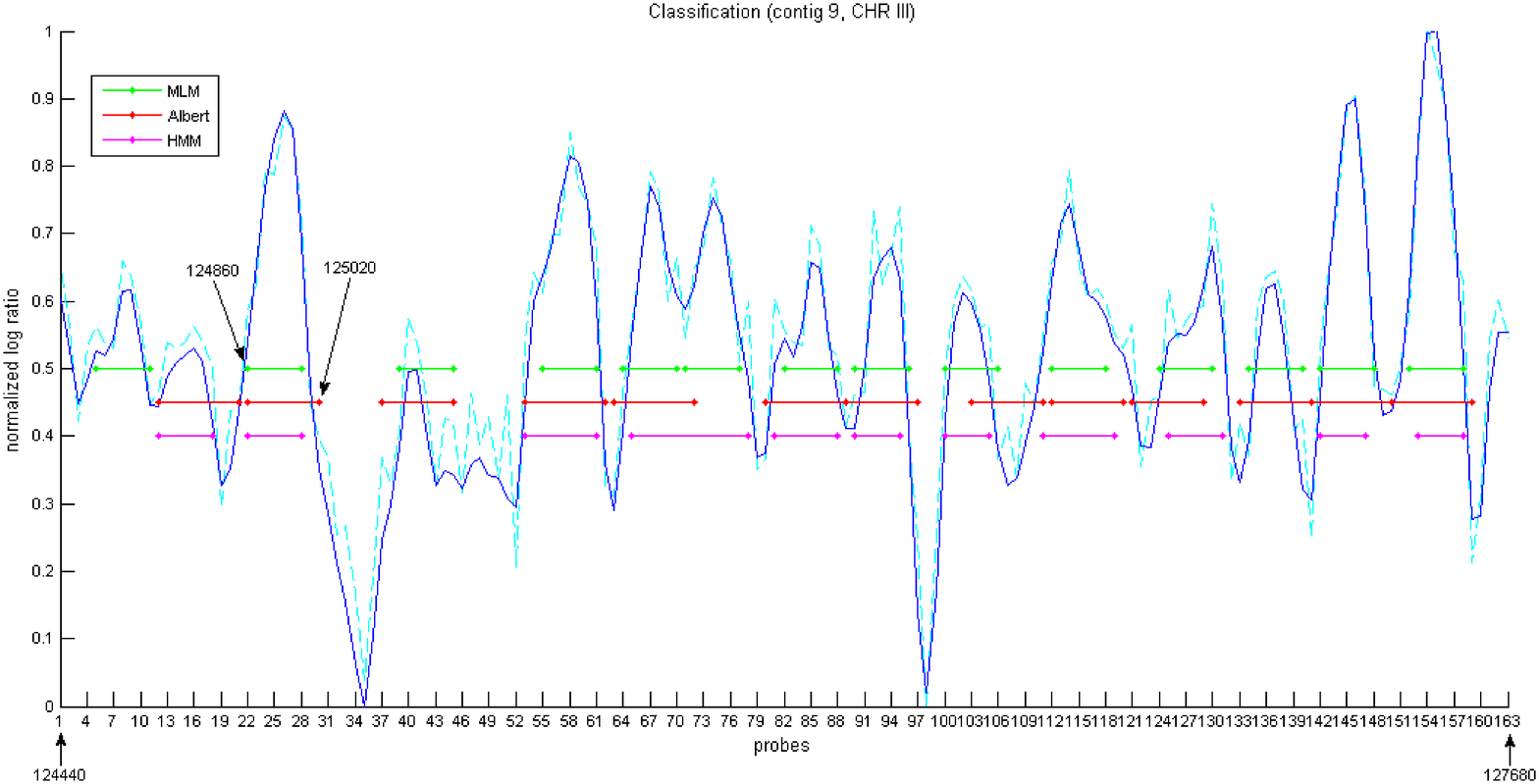}
\caption{A representative sample windows spanning 13 nuclesome where the agreement (disagreement) of the three methods is shown. The red draw represents the classification done by Pugh et Al. (2007) in \cite{pugh} }\label{fig:s3}.
\end{figure}

\subsection{Scalability and computational time of MLA and HMM:}
This point is fundamental because the size of a problem can vary significantly into this application domain, and if our method is not able to scale well it could become totally useless. The computation time of MLA and $HMM$ have been compared on $10$ experiments. In particular, $10$ synthetic signals have been generated, each one with a fixed number of well-positioned nucleosomes ranging from $10$ to $100$ by step of $10$. In figure \ref{fig:s4}, the ratios between the execution time of MLA ($T_m$) and $HMM$ ($T_h$) for each experiment is shown. From this study, it results that, on average, $T_h=1.7 \times 10^{4} \times T_m$.

\begin{figure}[bht]
\centering
\includegraphics [width=0.8 \textwidth]{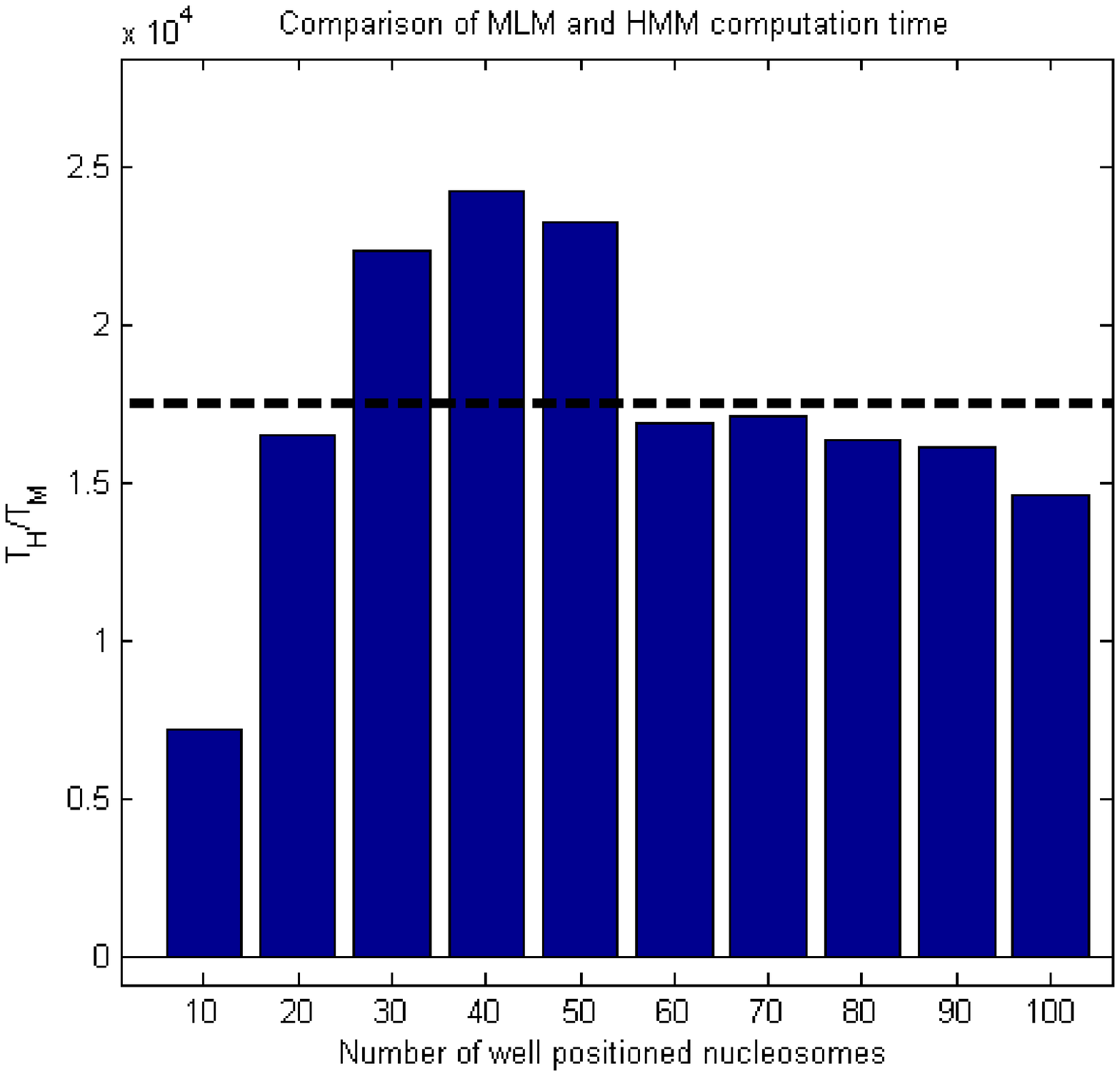}
\caption{\emph{Computation time performances}: The execution time ratio $T_h/T_m$ of the MLA ($T_m$) and HMM ($T_h$) for 10 synthetic signal generated with different number of well-positioned nucleosomes. The dashed line shows the average execution time.}\label{fig:s4}
\end{figure}

\section{One-Class Classifier and MLA}
 One of the key point of the MLA methodology applied on the case of nucleosome positioning, is the classification phase that is applied after the discovery phase. In this section a new classification schema that take advantage of MLA will be presented. As explained in chapter 1 classification algorithms bases the construction of their discriminating function on a training set that contains several examples for each class (or in the particular case of binary classification this means that are necessary both positive and negative examples). However, in many cases either only examples of a single class are available or the classes are very much unbalanced. To address this particular problem one-class classifiers have been introduced in order to discriminate a target class from the rest of the feature space \cite{TAX01}. The approach is based on finding the smallest volume hypersphere (in the feature space) that encloses most of the training data. This approach is mandatory when only examples of the target class are available or the cardinality of the target class is much greater than the other one so that too few training examples of the smallest class are available in order to properly train a classifier. It is important to pinpoint that the nucleosome positioning data considered, involve necessary the use of a one-class scheme, since a training set of only well-positioned nucleosome is available. This section present, a one-class classifier schema, in particular a one-class $KNN$ ($OC-KNN$) in order to distinguish between nucleosome and linkers. The performance of the one-class $KNN$ embedded in the MLA analysis, has been tested on the same kind of data previously described. Results have shown, in both cases, a good recognition rate.

\subsection{One-Class classifiers} \label{repa}
The first algorithms for one-class classification were based on neural networks, such as those of Moya et al. \cite{MOY93,MOY96} and Japowicz et al. \cite{JAP95}. More recently, one-class versions of the support vector machine have been proposed by Scholkopf et al. \cite{SCH01}. The aim is to find a binary function that takes the value +1 in a \emph{small} region capturing most of the data, and -1 elsewhere. Data transformations are applied such that the origin represents outliers, then the maximum margin, separating hyperplane between the data and the origin, is searched.\\
The application of machine learning to classification problems, that depends only on positive examples, is gaining attention in the computational biology community. This section lists some applications of one-class classifiers to biological and biomedical data.\\
In \cite{MAL08} a study using one-class machine learning for microRNA (miRNA) discovery is presented. Authors compare a One-class KNN to two-class approaches using naive Bayes and Support Vector Machines. Using the EBV genome as an external validation of the method they found one-class machine learning to work as well as or better than a two-class approach in identifying true miRNAs as well as predicting new miRNAs.\\
In \cite{NAN05b} a general method for predicting protein-protein interactions is presented. The search of feasible interactions is carried out by a learning system based on experimentally validated protein-protein interactions in the human gastric bacterium Helicobacter pylori. The author shows that the linear combination of discriminant classifier provides a low error rate.\\
In \cite{PEK04} a one-class classification problem is applied to the detection of diseased mucosa in oral cavity. Authors either combine several measures of dissimilarity of an element from a set of target examples in a single one-class classifier or combine several one-class classifiers trained with a given measure of dissimilarity. Results show that both approaches achieve a significant improvement in performance.

\subsection{One-Class $KNN$} \label{KNN}
Here, the one-class classifier named One-class $KNN$ will be described.
A $KNN$ classifier for an $M$ classes problem is based on a training set $T$ for each class $m$, $1 \le m \le M$. The assignment rule for an unclassified element $\mathbf{x}\in X$ is:

\begin{equation}
 j = \mathop{argmax}_{1 \le m \le M} \mid T_{K}^{(m)}(x) \mid
\end{equation}

\noindent where, $T_{K}^{(m)}(x)$ are the training elements of class $m$ in the $K$ nearest neighbors of $x$.

One of the crucial points of the $KNN$ is the choice of the \emph{best} $K$, which is usually obtained
minimizing the misclassification rate in validation data.\\
In the case of a binary classification ($M=2$), one-class training means that in the decision rule can be used examples of only one-class. Here, a one-class training $KNN$ ($OC-KNN$) is proposed and which is a generalization of the classical $KNN$ classifier \cite{jain}. Let $T_p$ be the training set for a generic pattern $p$ representing a \emph{positive} instance, and $\delta$ a dissimilarity function between patterns. Then the membership for an unknown pattern $x$ is:
\begin{equation}
 \chi_{\phi,K}(x)  = \begin{cases}
    $1$ \hbox{ \ \ } \hbox{if $|\{ y \in T_p \hbox{ such that }  \delta(y,x) \le \phi \}|  \ge K$}\\
    $0$ \hbox{ \ \ } \hbox{otherwise}
  \end{cases}
  \label{membership}
\end{equation}
\noindent Informally, the rule says that if there are at least $K$ patterns in $T_p$ dissimilar from $x$ at most $\phi$, then $x$ is supposed to be a
positive pattern, otherwise it is negative.\\
It can be simply proved that the  $OC-KNN$ has some interesting properties:

\begin{proposition} \label{pr1}
Let $D$ a dataset of patterns, $T_p \subseteq D$ the training set for the \emph{positives},  $S_{\phi,K}=\{ x \in D | \chi_{\phi,K}(x)=1\}$ the set with membership $\chi_{\phi,K}$, then:\\

a) $S_{\phi,K'}\subseteq S_{\phi,K}$  $\forall K' \ge K$

b) $S_{\phi,K}\subseteq S_{\phi',K}$  $\forall \phi \le \phi'$
\end{proposition}

The one-class $KNN$ performances depends on the threshold, $\phi$, and the number of neighbors, $K$, that are used in the classification phase.
Both of them can be determined by using a validation procedure applied on the training set of positives $Tp$. In the following, it will be described the procedure used to estimate the best pair $(\phi^*,K^*)$.\\
Let us define the performance function $M$:
\begin{equation}
 M(\phi,K)=\frac{\mid S_{\phi,K} \mid}{\mid T_p \mid}\\
\end{equation}

Note that, in this validation procedure $\forall x \in T_p$ assigned to $S_{\phi,K}$ use the membership $\chi_{\phi,K}(x)$ defined on the training set $T_p-\{x\}$. By using $M$ it is possible to define the functions $P$ and $Q$
\begin{equation}
P(\phi)=\sum_{k \in \{ K_m, K_M \}} M(\phi,k) \mbox{ and } Q(k)=\sum_{\phi \in \{ \phi_m, \phi_M \}} M(\phi,k) \\
\end{equation}
\noindent where $\{ \phi_m, \phi_M \}$ and $\{ K_m, K_M \}$ are sets of increasing values of thresholds and number of neighbors respectively.
By applying the proposition \ref{pr1}, it results that the function $M$ increases while the threshold $\phi$ increases, and decreases while the neighbors $K$ increases. In figure \ref{mplot} a $3d$ plot of the function $M$ relative to the classification of nucleosome and linker regions on the Saccharomyces cerevisiae data set is shown. Assigning the values, $\phi_m = \mathop{min}_{x,y \in T_p} \delta(x,y)$ and $\phi_M = max_{x,y \in T_p} \delta(x,y)$, $K_m=1$, $K_M=|T_p|$, the pair $(\phi^*,K^*)$ to choose is:

\begin{equation}
\phi^*=\mathop{min} \{ {\phi \mid P(\phi)=max\{P(\phi)}\}\}\\
\end{equation}

\begin{equation}
K^*=\mathop{max} \{ {K \mid Q(K) \ne 0} \}
\end{equation}

Informally, such estimation methodology selects the smallest threshold $\phi^*$ which causes the best performances on the validation data, most independently from the values of $K$. Moreover, the value $K^*$ is chosen to be the largest one causing performances different from zero. In this way it is possible to obtain a good compromise between the generalization ability of the classifier and its precision, in fact the best value of $\phi$ takes in account of several values of $K$ and the value of $K$ chosen should guarantee a good generalization ability. In figure \ref{estimate} an image representation of $M$ shows also the chosen $(\phi^*,K^*)$ concerning the classification of nulceosome and linker regions on the Saccharomyces cerevisiae data set. A fuzzy extension version of the $OC-KNN$, has been recently tested on two public data-sets \cite{dj-wilf07}, studying also the gain in classification performances when combining several one-class classifiers defined by different dissimilarity functions.

\begin{figure}[!htb]
\begin{center}
\subfigure[]{\epsfxsize=2.35 in \epsfbox{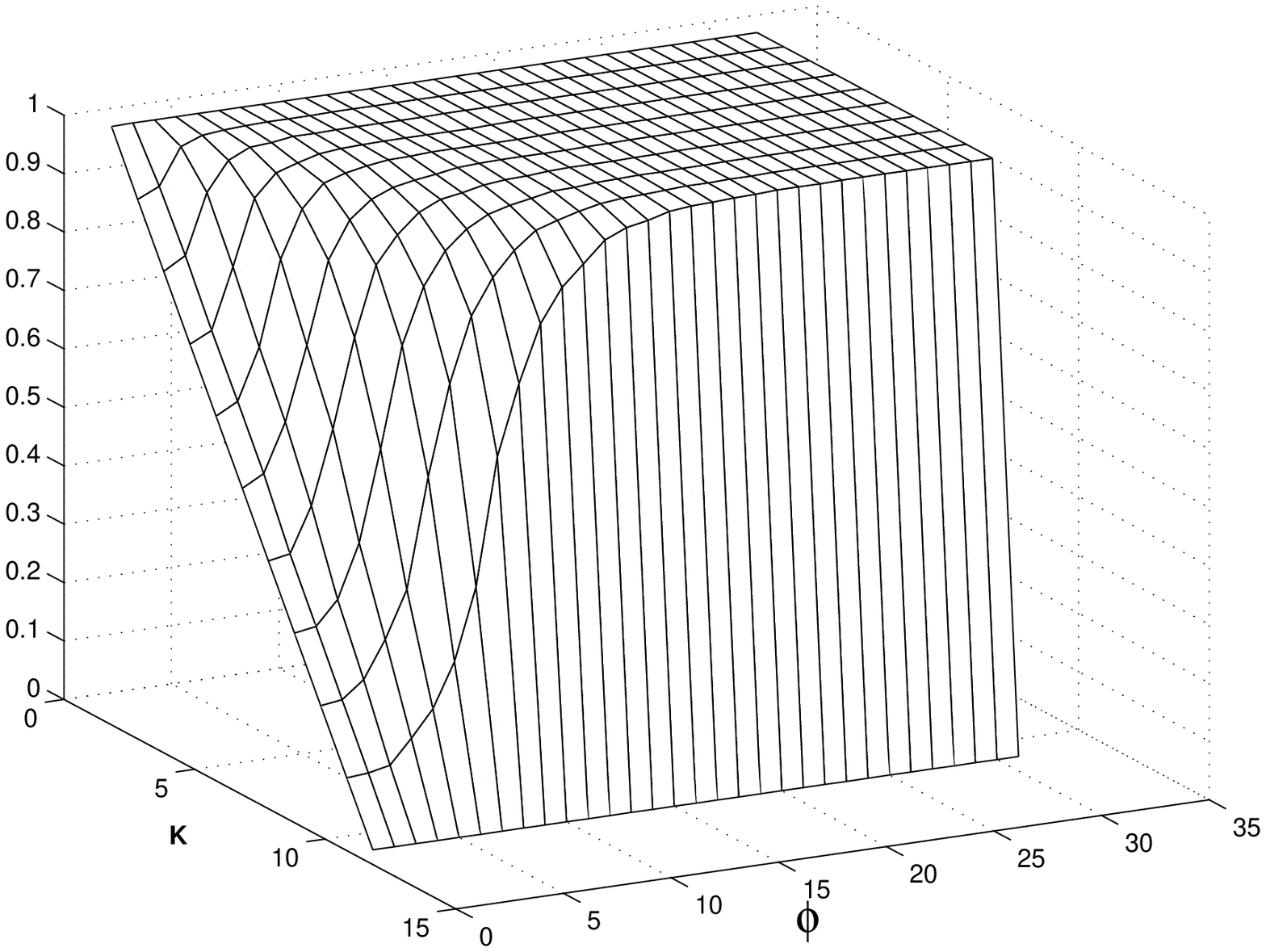}\label{mplot}}
\subfigure[]{\epsfxsize=2.35 in \epsfbox{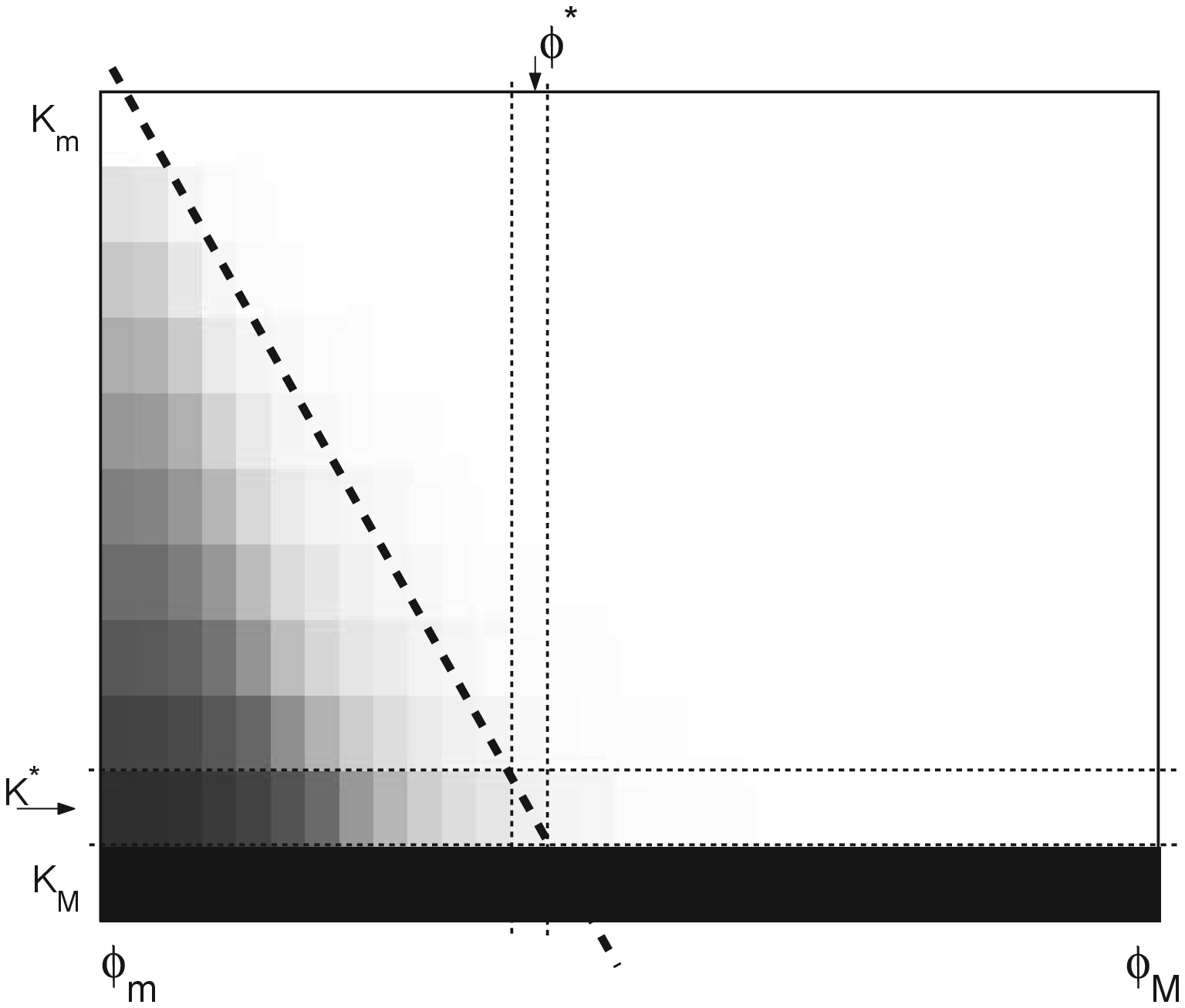}\label{estimate}}
\end{center}
\caption{Two different representations of $M$, on the left (a) a 3d plot, on the right (b) an image representation showing the values of $M$ using grayscale ($0$ is black, $1$ is white). In this latter figure, there are also the chosen pair $(\phi^*,K^*)$}
\end{figure}

\subsection{Results on synthetic data}
\label{one-class-experiments}
Also in this case, the performances have been evaluated in terms of \emph{Recognition
Accuracy}, $RA$ (see section \ref{experiment} for details). The synthetic experiments allows to test the robustness of the $OC-KNN$ to signal noise. All parameters used in the generation of synthetic data have been inspired by biological considerations and are $nn=200$, $nl=250$, $\lambda=200$, $r=50$, $o=20$, $nr=100$, $dp=0$, $dr=0$, $pur=0.8$, $nsv=0.01$,
$SNR=\{1,2,4,6,8,10\}$ and $ra=4$, resulting in $6$ synthetic signals at different $SNR$. The training set $Tp$ is represented by all $WPN$'s that fit better the conditions in Eq. \ref{eq0} with $os=4$, because, by biological consideration, it is known that a nucleosome is around $150$ base pairs which corresponds to $8$ probes. Thus, the training set $T_p$ and consequently its size $TL$, are automatically selected by the MLA depending on the generated input signal, resulting that, for the specific experiments reported here, $TL=\{63,98,127,142,145,147 \}$ for $SNR=\{1,2,4,6,8,10\}$ respectively. The optimal parameters for the MLA are derived by a calibration phase described in \cite{co2007} and have resulted $H=20$ and $m=5$. Here and in the next section $H$ represents the number of threshold operations of MLA analysis in order to avoid ambiguities with the $K$ of OC-KNN that represents the number of neighbors. The performances have been evaluated measuring the correspondence between the classified $WPN$ or $LN$ regions and the ones imposed in the
generated signal. The parameters ($\phi^*,\ K^*$) of the $OC-KNN$ has been chosen by the validation procedure described in section \ref{KNN} for each $SNR=\{1,2,4,6,8,10\}$.
Figure \ref{optim1} reports the best \emph{Accuracy} and \emph{FPR} values versus $SNR$, showing also, for each $SNR$ signal, the ($\phi^*,\ K^*$) causing such values. From this study, it results that the average accuracy and $FPR$ over the $6$ experiments is $94\%$ and $9\%$ respectively.

\begin{figure}[!htb]
\begin{center}
{\epsfxsize=3.35 in \epsfbox{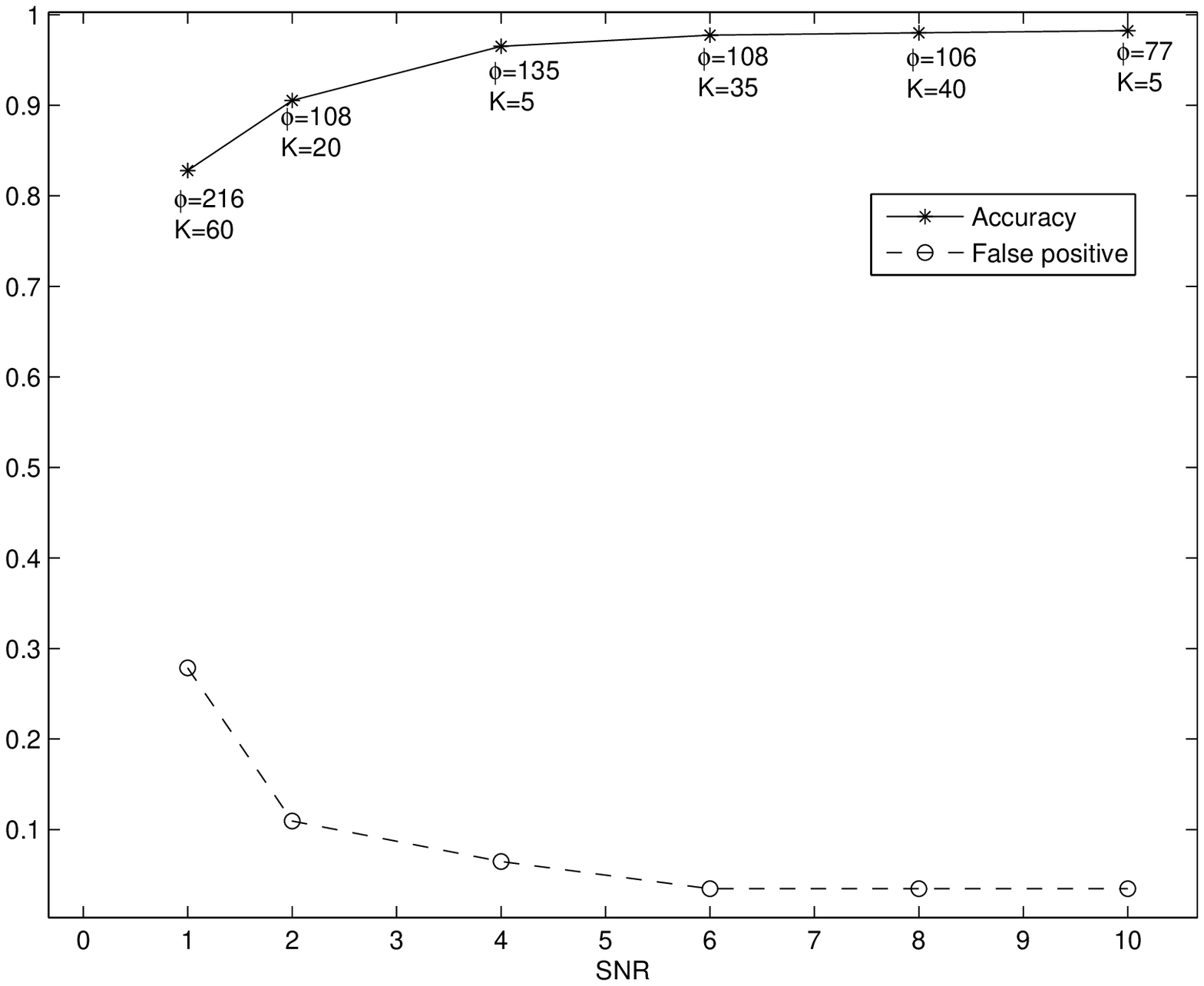}}
\end{center}
\caption{Best Accuracy and $FPR$ values versus SNR. The couples $(\phi,K)$ causing such results are also reported.} \label{optim1}
\end{figure}

\subsection{Results on real data:}
In this experiment, it has been again compared the accordance of the Hidden Markov model ($HMM$) for nucleosome positioning on the \emph{Saccharomyces cerevisiae} real data. The training set $Tp$ has been decided in the same way as above. In such experiment, $H=40$, $m=6$ were chosen by a calibration phase ($m=0.15 \times 40$) that is fully described in \cite{co2007}. The confusion matrices, which show the $RA$ of $HMM$ considering MLA as the truth classification and $RA$ of MLA considering $HMM$ as the truth classification, are reported in table \ref{tabocc}. The results can be summarized in an overall $RA$ of $(0.76)$ for the $HMM$ (MLA true) and $0.65$ for MLA ($HMM$ true).

In particular, from this studies it is possible to conclude that MLA does not fully agree with $HMM$ on the nucleosome patterns as in the previous case, in addition seems comparing  the tables \ref{tabocc} and \ref{tab4}, that this classifier doesn't introduce any significant improvement than the one used in Section \ref{bio_goal}.

\begin{table}
\centering
\begin{tabular}{|c|ccc||c|ccc|}
\hline
&$M$ & $L$ & $M$ & & $H$ & $M$ & $M$\\
\hline
$H$ &     & $L$  & $N$  &$M$ &      & $L$  & $N$\\
$M$ & $L$ & 0.66 & 0.33 &$L$ & $L$  & 0.65 & 0.34\\
$M$ & $N$ & 0.14 & 0.85 &$M$ & $N$  & 0.34 & 0.65\\
\hline
\end{tabular}
\caption{Agreement between the $HMM$ and MLA (and viceversa) on the Saccharomyces cerevisiae data set for
Nucleosomes (N) and Linker (L) regions. The table on the left shows the $RA$ results of $HMM$ when considering MLA as the truth classification, while the opposite is shown on the right table}.
\label{tabocc}
\end{table} 
\chapter{Test of Randomness by MLA}
\label{chap:4}

This chapter presents a new nonparametric test of randomness of a set of one-dimensional signals that take advantage of MLA preprocessing step. In particular, this procedure is based on the probability density function of the symmetrized Kullback Leibler distance, estimated via a Monte Carlo simulation on the intervals lengths obtained by MLA. The main advantage of this new approach is that it allows to perform an exploratory analysis in order to verify directly the presence of several structures in an input signal. In particular this test differs from the other approaches because it exploits shape features that are rare in a random signal.

\section{Test of Randomness}
Given a signal or a sequence of symbols, it is first necessary to define the meaning of ``random''.
In fact the term \emph{randomness} has several meanings as used in several different fields.
A good literature survey about randomness tests can be found here \cite{URL:SurveyRandom}.
In the statistic literature, the concept of randomness is somewhat related to a sequence of random variables.
The non randomness could be suggested by any tendency of the observation to exhibit regularities in the sequence of observations.
For example, if an observation in a sequence is influenced by the previous observations or, more in general, if the observed value in a sequence is influenced by its position, the process is not truly random. More formally, a generic sequence is said random in statistical context if the process that has generated it, produces independent and identically distributed observations or i.i.d..
In some context, it is typical that the observations are not truly random in rigorous statistical sense i.e. i.i.d, but although the sequence are not formally random, it could be of interest to measure, fixed a certain degree of confidence, how close to random it is.
The application of these approaches are manifold: for example a test of randomness can be useful in the case of exploratory analysis in order to verify the possible presence of structures in an input signal; in the context of cryptography to assess the performance of a good pseudo-random generator (because it is a fundamental building block in a lot of algorithms) or  can be used to test the strength of a password \cite{Hollander:nonpara,Gibbons03:nonpara}.

\subsection{State of the art}
This section does not pretend to be a detailed revision of all the methodologies known in literature;
 the main ideas and their references will be presented instead. In particular in statistic literature, there are several approaches to test if a sequence is random, exploiting the ``non randomness'' in different ways:
\begin{itemize}
  \item test based on runs
  \item test based on entropy estimator
  \item test based on ranking
  \item test based on goodness of fitting to a given distribution
\end{itemize}

It will be shown that the test of randomness that uses the $MLA$ as preprocessing step belongs to the last class.

\subsection{Test based on runs}
These tests are based all on the central concept of \emph{run} given in the following definition:

\begin{definition}
Given an ordered sequence of one or more symbols, a \emph{run} is defined to be a succession of one or more type of symbols which are followed and preceded by different symbols or no symbol at all.
\end{definition}

Once the runs in the signal are identified, the measure of randomness could depend on their number, lengths or both. That's why in a real random sequence is very unusual to have too few or too many runs or runs of considerable length. So these information can be used as statistical criteria to assess if a signal is truly random. Common approaches to define runs starting from a signal are to dichotomize it (e.g. considering its sign for each observation), comparing the amplitude of consecutive points within respect to a focal point (e.g. its  mean or its  median) or looking for trends. More information about these approaches can be found here \cite{Gibbons03:nonpara}.

\subsection{Test based on entropy estimator}
These tests are based on the entropy of a signal or related features. In general the entropy
is a measure of the uncertainty associated with a random variable \cite{CT06:Information_Theory}:

\begin{definition}
Let $X$ a discrete random variable with alphabet $\Sigma$ and probability mass function $p(x)=Pr\{X=x\}$, $x \in \Sigma$. The \emph{entropy} $H(X)$ of a discrete random variable $X$ is defined by:

\begin{equation}
H(X) \equiv H(p) = - \sum_{x \in \Sigma} p(x)log_2 p(x)
\end{equation}

\end{definition}

For example if we consider the sign test \cite{Gibbons03:nonpara} (a particular run test) or a binary vector, it should be expected that the sequence of signs (or bits) are i.i.d. and this obviously follows from the fact that the positive and negative signs are equiprobable i.e. $P(s(i) \geq 0)= P(s(i) < 0)$. If this assumption is not true, it is easy to prove that the entropy will be strictly less than $1$. In general these tests use this null hypothesis:

\begin{equation}
H_0:H(p)=1
\end{equation}

Usually, given a signal $f$ these tests start approximating the probability distribution for $f$ and then
calculating its entropy. Further details can be found in \cite{Gao:Entropy_estimator} and \cite{Wegenkittl:Entropy_Chain} .

\subsection{Test based on ranking: Wilcoxon rank sum test}
These tests are based on the concept of ranking, where for ranking is meant a sorting of the observation in non-crescent or non-descdendent order. A very popular test that falls in this category and that can be used to evaluate the randomness of a signal is the Wilcoxon rank sum test.
Given two vectors of observations $X$ and $Y$ also of different lengths, test the null hypothesis that data in the vectors are independent samples from identical continuous distributions with equal medians, against the alternative that they do not have equal medians \cite{Hollander:nonpara}. More formally:\\

\noindent Given $N=m+n$ observations $X_1,\ldots,X_m$ and $Y_1,\ldots,Y_n$, the assumed model is:

\begin{eqnarray}
  X_i = e_i  \qquad i=1,\ldots,m \\
  Y_j = e_{m+j}+\Delta \qquad j=1,\ldots,n
\end{eqnarray}

\noindent where $e_{m+1},\ldots, e_{m+n}$ are unobservable random variables, and $\Delta$ is the shift between the samples. Here we suppose that the $N$ observations are mutually independent and each $e$ come from the same continuous population.

\noindent The test consist in evaluating the null hypothesis:

\begin{equation}
H_0:\Delta=0
\end{equation}

\noindent The first step is to sort the $N$ observations in increasing order and let $R_j$ denote the rank of $Y_j$ in this ordering. Then the statistic $W$ is calculated using this equation:

\begin{equation}
W= \sum _{j=1} ^{n} R_j
\end{equation}

\noindent For a one side test of $H_0$ versus the alternative $H_1:\Delta>0$, at $\alpha$ level of significance:
\begin{displaymath}
\begin{array}{lr}
    \mbox{reject } H_0    & \mbox{if } W \geq w(\alpha,m,n)\\
    \mbox{accept } H_0    & \mbox{if } W < w(\alpha,m,n)\\
\end{array}
\end{displaymath}

\noindent where the constant $w(\alpha,m,n)$ satisfies $P_0[W \geq w(\alpha,m,n)]=\alpha$

\noindent Let $R^{(1)}<, \ldots ,<R^{(n)}$ the ordered $Y$ ranks in the joint ranking of $X$ and $Y$ then the null distribution for $W= \sum _{j=1} ^{n} R_j = \sum _{j=1} ^{n} R^{(j)} $ can be obtained considering that under the hypothesis $H_0$ all possible ${N\choose n}$ assignments for $[ R^{(1)}, \ldots ,R^{(n)}]$ have probability $1 / {N\choose n}$ in this way it is possible to derive the null distribution without specifying the underling distributions of the $e's$.

\subsection{Test based on goodness of fit: Kolmogorov-Smirnov goodness of fit Test}
These tests start from a statistical model try to assess how well some observations fit the model. A very popular test that falls in this category and that can be used to evaluate if two samples are drawn from the same distribution is the Kolmogorov-Smirnov goodness of fit Test \cite{Senoglu:Goodneed-of-fit}. This distribution free test is used to check if one sample comes from a particular distribution or if two samples come from the same distribution. This test is based on the comparison between the empirical cumulative distribution function and the theoretical cumulative distribution function. More formally:\\

\noindent Let $X$ a random variable with cumulative function $F(x)$, given another cumulative function $F_N(x)$ this test check the hypothesis:

\begin{equation}
H_0: F(x) = F_N(x), \forall x
\end{equation}

\noindent Let $D$ the max absolute value of the difference between the two cumulative distribution, i.e.
\begin{equation}
    D = \underset{-\infty < x < + \infty}{sup} |F_N(x)- F(x)|
\end{equation}

\noindent where $F(x)$ is the theoretical cumulative function and $F_N(x)$ is the cumulative distribution observed. Let $x_1, x_2, ..., x_N$ a random sample, $F_N(X)$ is obtained as:

\begin{equation}
F_N(x)=
\begin{cases}
0 & \text{if $x\leq x_1$,}\\
\frac {k}{n} &\text{if $x_k \leq x \leq x_{k+1}$}\\
1 & \text{if $x\geq x_N$.}
\end{cases}
\end{equation}

\noindent $F_N(x)$ is a good estimator of $F(x)$, in fact it can be proven that $\underset{n \rightarrow \infty}{F_N(x)}=F(x)$.

\noindent At this point considering the observed value of $D$, and considering the theoretical distribution of $D$, once fixed a confidence level $\alpha$ it's possible to calculate $D_\alpha$, then choose to reject or not the hypothesis $H_0$ using the condition:

\begin{displaymath}
\begin{array}{lr}
    \mbox{reject } H_0    & \mbox{if } D \geq D_\alpha\\
    \mbox{accept } H_0    & \mbox{if } D < D_\alpha\\
\end{array}
\end{displaymath}

\section{MLA Test of Randomness}
As it was shown in the previous Chapters, the $MLA$ is strongly related to the class of methods successfully used in the analysis of very noisy data which, by using several views of the input data-set are especially able to recover statistical properties of a signal. Here a test of randomness, based on the distance of the interval
lengths p.d.f's detected by the Multi-Layer Analysis ($MLA$) will be presented. Such p.d.f's are estimated for each
cut-set and the hypothesis test is performed against random signals generated via a Monte Carlo simulation. At
this end the symmetrized Kullback-Leibler measure has been used to estimate the distribution distances.

\subsection{Monte Carlo simulation}
The  Monte Carlo methods \cite{Metropolis:montecarlo}  are a class of computational algorithms that perform their computation using a random process to simulate or sample the possible space of solutions. They are used in the case when a deterministic approach are inapplicable for example due to the complexity of the problem. A typical scenario is the use of these methods to randomly sample a large number of states of a complex system so to use those states to model the behavior of the the whole system. The Montecarlo Method is used in several different contexts, but shares the same general approach depicted in Figure \ref{alg:Montecarlo}.
In the MLA test of randomness, a Montecarlo Method is used to model the random case in term of  Kullback Leibler distance applied on the interval representation obtained by MLA on random signals, and will be shown in the following sections:

\begin{figure}[!hbt]
\newlength{\mylength}
\[
\setlength{\fboxsep}{11pt}
\setlength{\mylength}{\linewidth}
\addtolength{\mylength}{-2\fboxsep}
\addtolength{\mylength}{-2\fboxrule}
\ovalbox{
\parbox{\mylength}{
\setlength{\abovedisplayskip}{0pt}
\setlength{\belowdisplayskip}{0pt}

\begin{pseudocode}{Monte Carlo Method} {P}
\BEGIN
\,\,\,\,\,1. \mbox{ Define the space of inputs or solutions } S\\
\,\,\,\,\,2. \mbox{ Random sampling from } S  \mbox{ using a particular probability distribution } P\\
\,\,\,\,\,3. \mbox{ Use the sample of the previous step to perform a deterministic computation }\\
\,\,\,\,\,4. \mbox{ Aggregate the results of the previous step to produce the final result } R\\
\END\\
\RETURN{R}
\end{pseudocode}
}
}
\]
\caption{The general schema of \codice{Montecarlo Method}.}\label{alg:Montecarlo}
\end{figure}

\subsection{Hypothesis test} \label{hypothesis}
In order to detect the presence of structures in the signal an hypothesis test based on the
expected probability distribution function (p.d.f.) of the segments length is proposed. The
null hypothesis ($H_0$) represents a \emph{random signal} and it is accepted if the p.d.f. of the segment lengths,
$p_1$, is compatible with a random signal distribution, $p_0$; the hypothesis $H_1$ represents a \emph{structured
signal} and it is accepted if the p.d.f. of the segment lengths is not compatible with a
random signal, $p_0$. It follows that we need to measure the similarity (dissimilarity) of two p.d.f.'s and set
a confidence level $\alpha$ to perform the decision.

The symmetric Kullback-Leibler measure, $SKL$, has been considered to evaluate the dissimilarity of the two
distributions $p_0$, and $p_1$ \cite{Johnson01symmetrizingthe}:
\[
SKL(p_0,p_1)=\frac{KL(p_0,p_1)+KL(p_1,p_0)}{2}
\]

\noindent where, $KL$ is the no-symmetric Kullback-Leibler measure. In the continuous case, p.d.f.'s are defined
in a dominion $I\subseteq \mathbb{R}$ and the $KL$ measure is defined as:
\[
KL(p,q)=\int_I p(x)log \frac{p(x)}{q(x)}dx
\]
In the discrete case $I\subseteq \mathbb{N}$ and the $KL$ become:
\[
KL(p,q)=\sum_{i\in I} p_i log \frac{p_i}{q_i}
\]
In order to perform the hypothesis test it is necessary to know the p.d.f. of the $SKL$ in the case of a random signal.
The derivation of analytical form of this p.d.f. is usually an hard problem that has been solved by a Monte
Carlo simulation. For example in \cite{Ikuo:KL_Normal_test} a goodness-of-fit test for normality is introduced; it is based on Kullback-Leibler information and a Monte Carlo simulation is performed to derive and estimate the p.d.f.'s. In
\cite{Senoglu:Goodneed-of-fit} an extension of the previous test is described for s-normal, exponential, and uniform distributions and also in this work a Monte Carlo simulation has been used to estimate the p.d.f. of the measure $KL$.

\subsection{Probability density functions estimation} \label{dens}
In this section, the simulation performed to estimate the p.d.f.'s of both the
intervals length, $IL_k$ ($PIL_k$), and the $SKL_k$ ($PSKL_k$), at a given threshold $t_k$
will be outlined. Here, $SKL_k$ is the distance between the p.d.f.'s of two interval length.

To estimate the p.d.f. of $IL_k$, $RS_n$,  $n=1,...,N$ signals of length $l$ have been generated, according to
a normal distribution with $\widehat{\mu}$ and $\widehat{\sigma}$ estimated from an input signal $S$ of length
$L$. Each signal, $RS_n$, is then used to evaluate experimentally $PIL^{(n)}_k$ ($n=1,2,...,N$).

In the simulation,  for each threshold, $t_k$, it is then possible to derive the experimental distributions of the $IL_k$ in $R_k$. Therefore $k=1,2,...,K$ normalized p.d.fs are obtained . $PIL^{(n)}_k$ with $nb$ bins. Figures \ref{p-a4}, \ref{p-a5}, \ref{p-a6}, show examples of $PIL_k$ for a simulation using $l=20000$, $L=200000$, $N=1000$, $K=9$, $nb=100$.

The estimation of the p.d.f. of $SKL_k$ and $PSKL_k$, is carried out by computing the $SKL_k$ between the
pairs $\left(PIL^{(m)}_k,\  PIL^{(n)}_k \right )$, with $m\neq n$. In this simulation it was drawn the evaluation of
$PSKL_k$ from a sample of $\frac{N \times (N-1)}{2}$ elements by using a density estimation with Gaussian kernel. Figures \ref{p-b4}, \ref{p-b5}, \ref{p-b6} shows examples of $PSKL_k$ for a simulation using $l=20000$, $L=200000$, $N=1000$, $K=9$, $nb=100$.

\begin{figure}[!htb]
\centering
\subfigure[] {\includegraphics[width=.7 \textwidth]{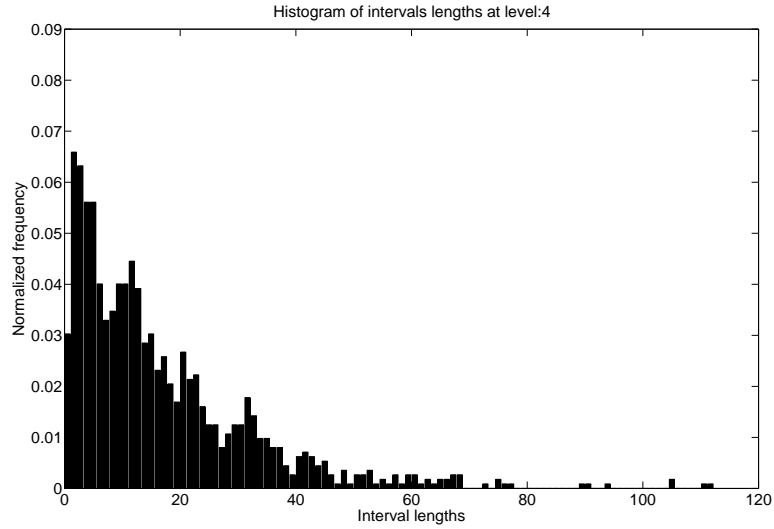} \label{p-a4}}
\subfigure[] {\includegraphics[width=.7 \textwidth]{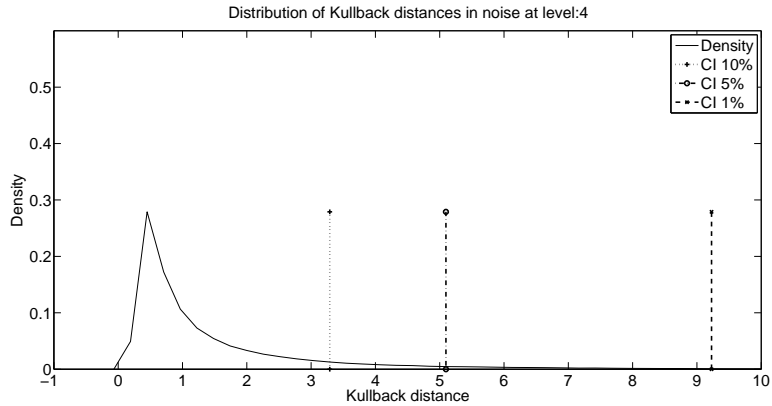} \label{p-b4}}
\caption{Examples of $PIL_k$ (a), and $PSKL_k$ (b) for $k=4$}

\end{figure}

\begin{figure}[!htb]
\centering
\subfigure[] {\includegraphics[width=.7 \textwidth]{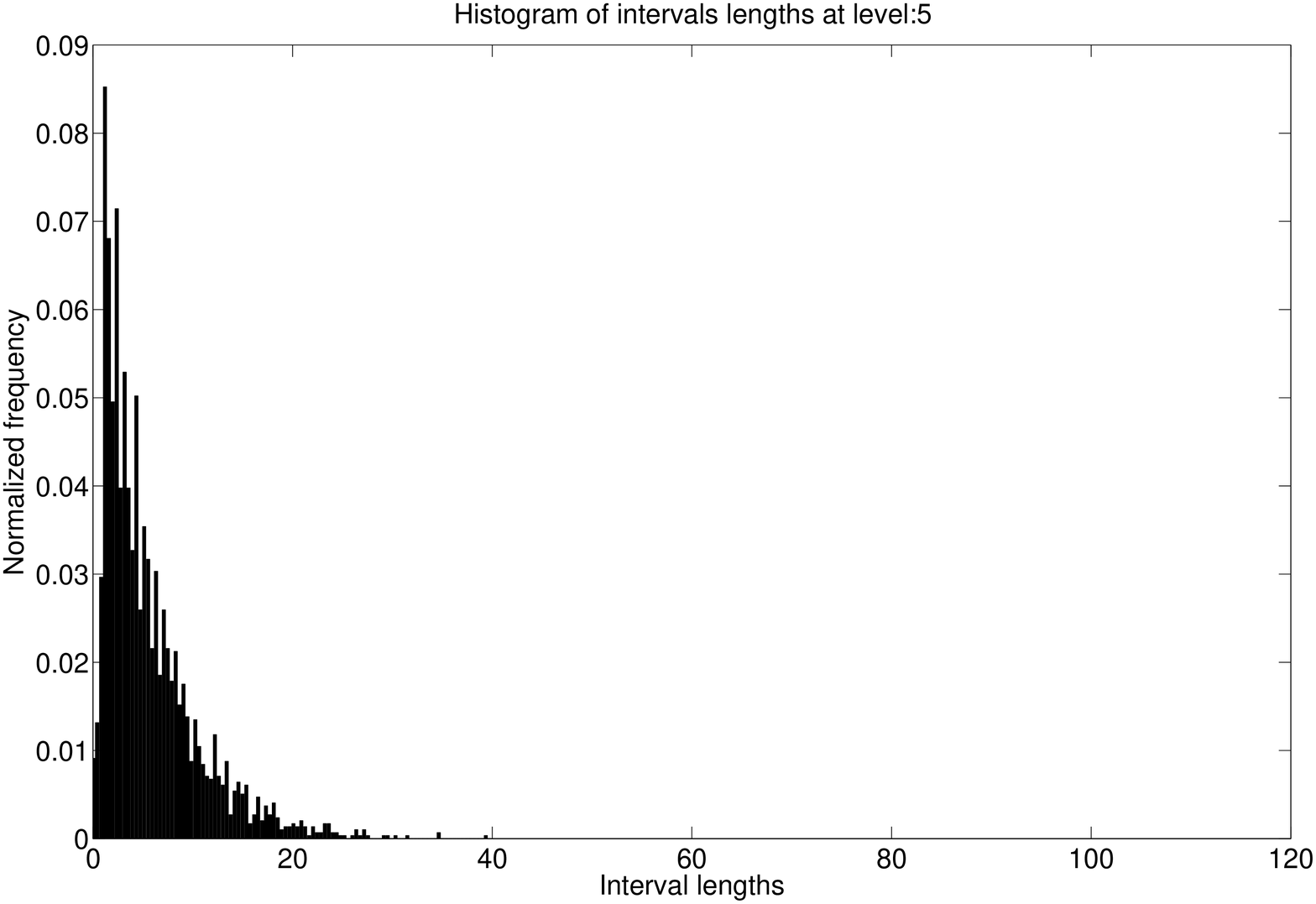} \label{p-a5}}
\subfigure[] {\includegraphics[width=.7 \textwidth]{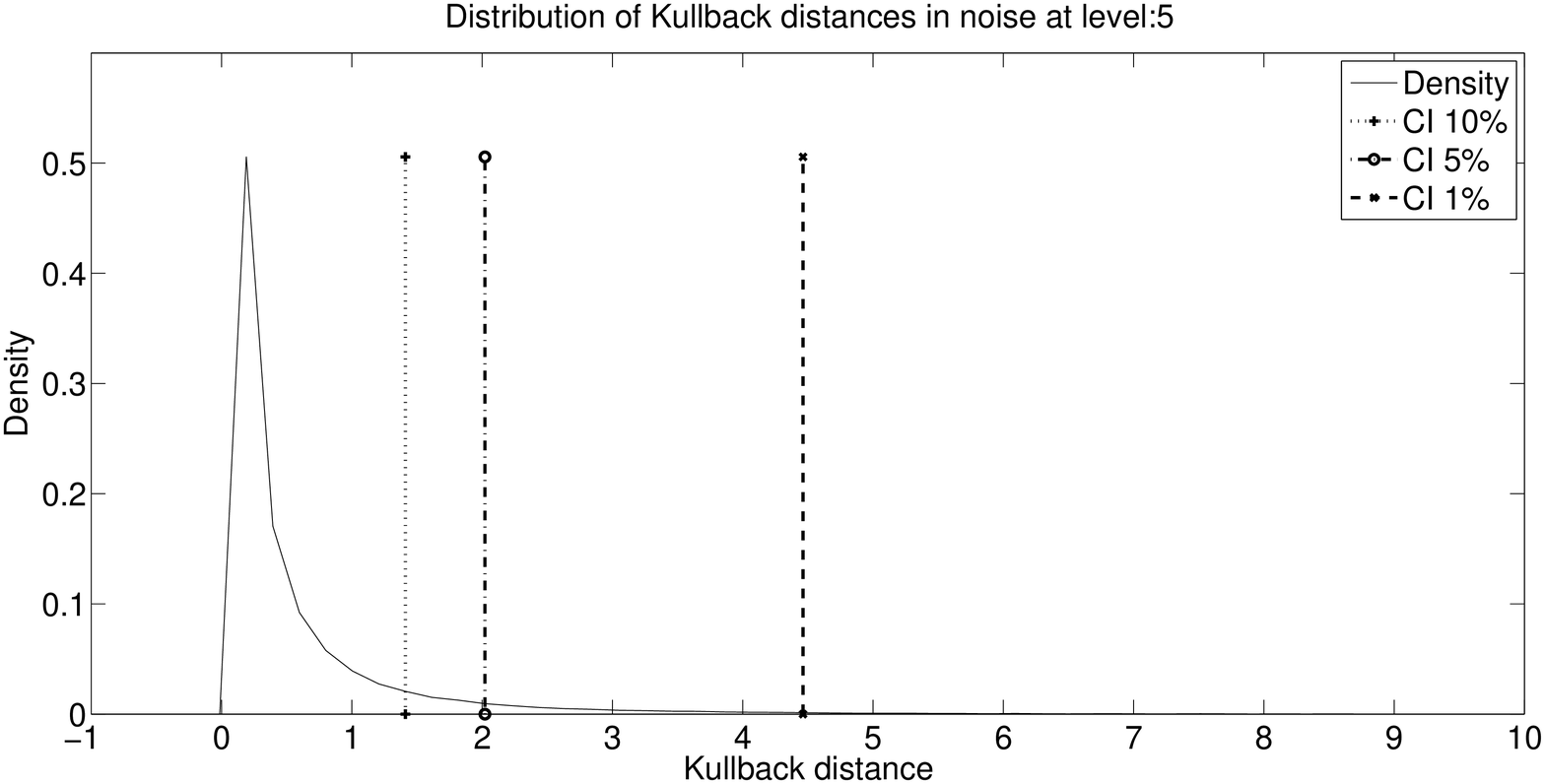} \label{p-b5}}
\caption{Examples of $PIL_k$ (a), and $PSKL_k$ (b) for $k=5$}
\end{figure}

\begin{figure}[!htb]
\centering
\subfigure[] {\includegraphics[width=.7 \textwidth]{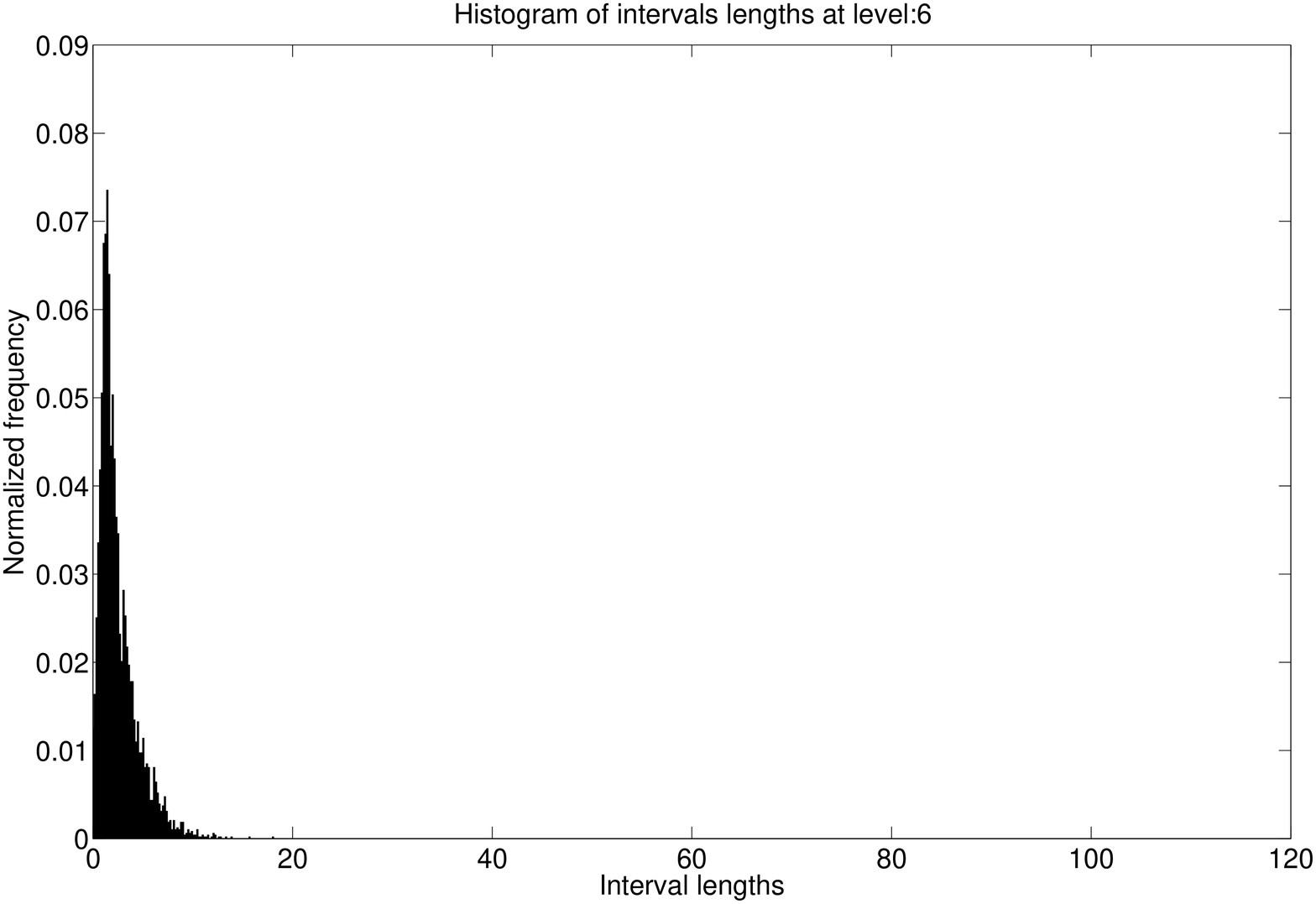} \label{p-a6}}
\subfigure[] {\includegraphics[width=.7 \textwidth]{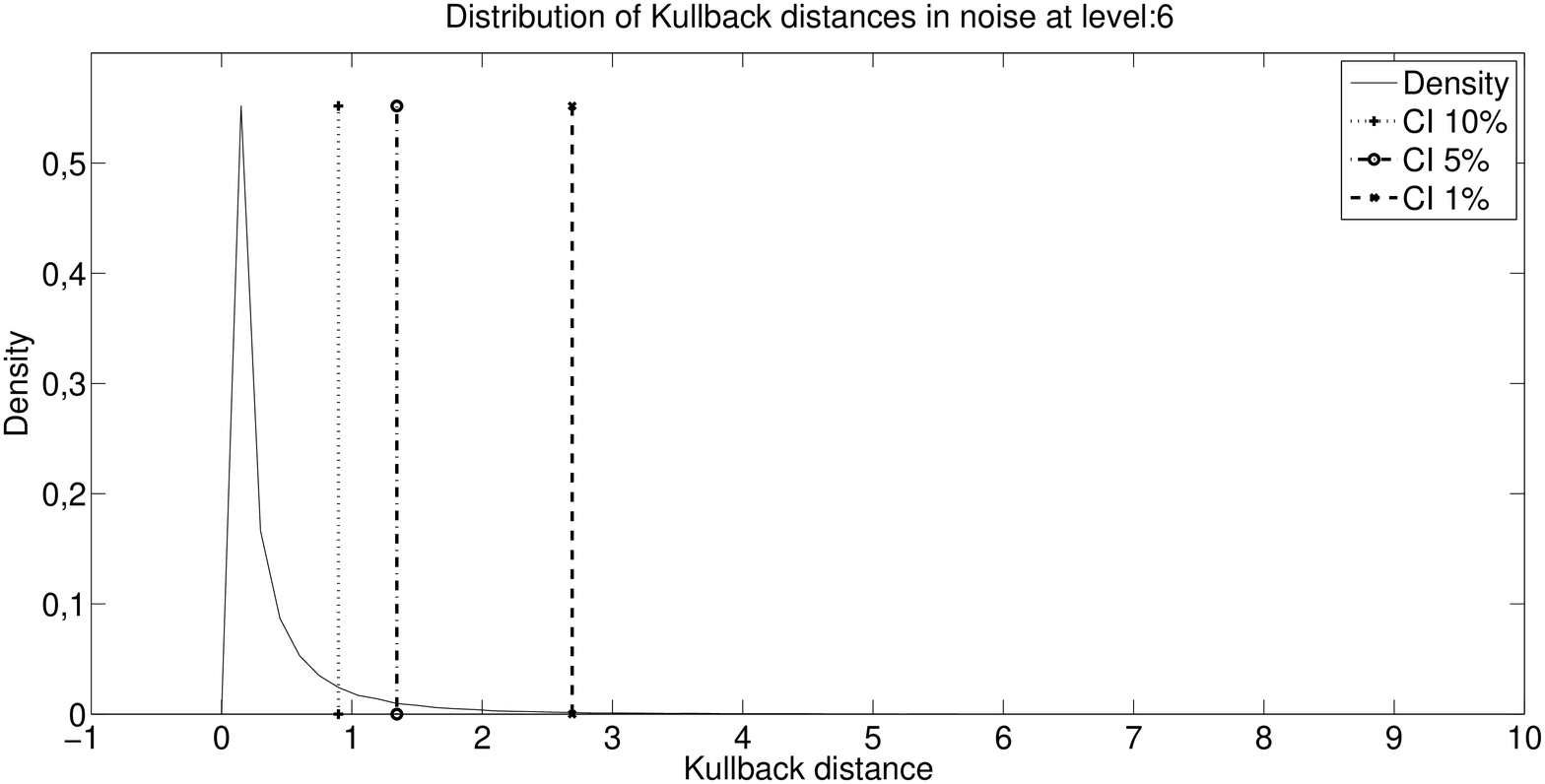} \label{p-b6}}
\caption{Examples of $PIL_k$ (a), and $PSKL_k$ (b) for $k=6$}
\end{figure}

\section{Experimental Setup}
Here, the evaluation of the test will be presented. In particular, the test has been carried out on simulated and real data respecting a particular tiled microarray approach able to reveal nucleosome positioning information on DNA \cite{YUAN05} and presented in detail in chapter 3. Here unlike to the case study presented in chapter 3 (in which the problem was to infer the nucleosome positions) the problem is to investigate if the shapes correspondent to nucleosome binding sequences have some specific features. Results indicate that such statistical test may indicate the presence of structures in real and simulated biological signals, showing also its robustness to data noise and its superiority to the Wilcoxon rank sum test. In Fig.s \ref{input-a},\ref{input-b},\ref{input-c} three examples of input signals with signal to noise ratio $SNR=1,1.5,10$, are given.
This allows to control the accuracy of the proposed test of randomness and perform the calibration of the methodology.
The same test has been applied to the  data used in the simulation phase.

\begin{figure*}[!htb]
\begin{center}
\subfigure[] {\includegraphics[width=.55 \textwidth]{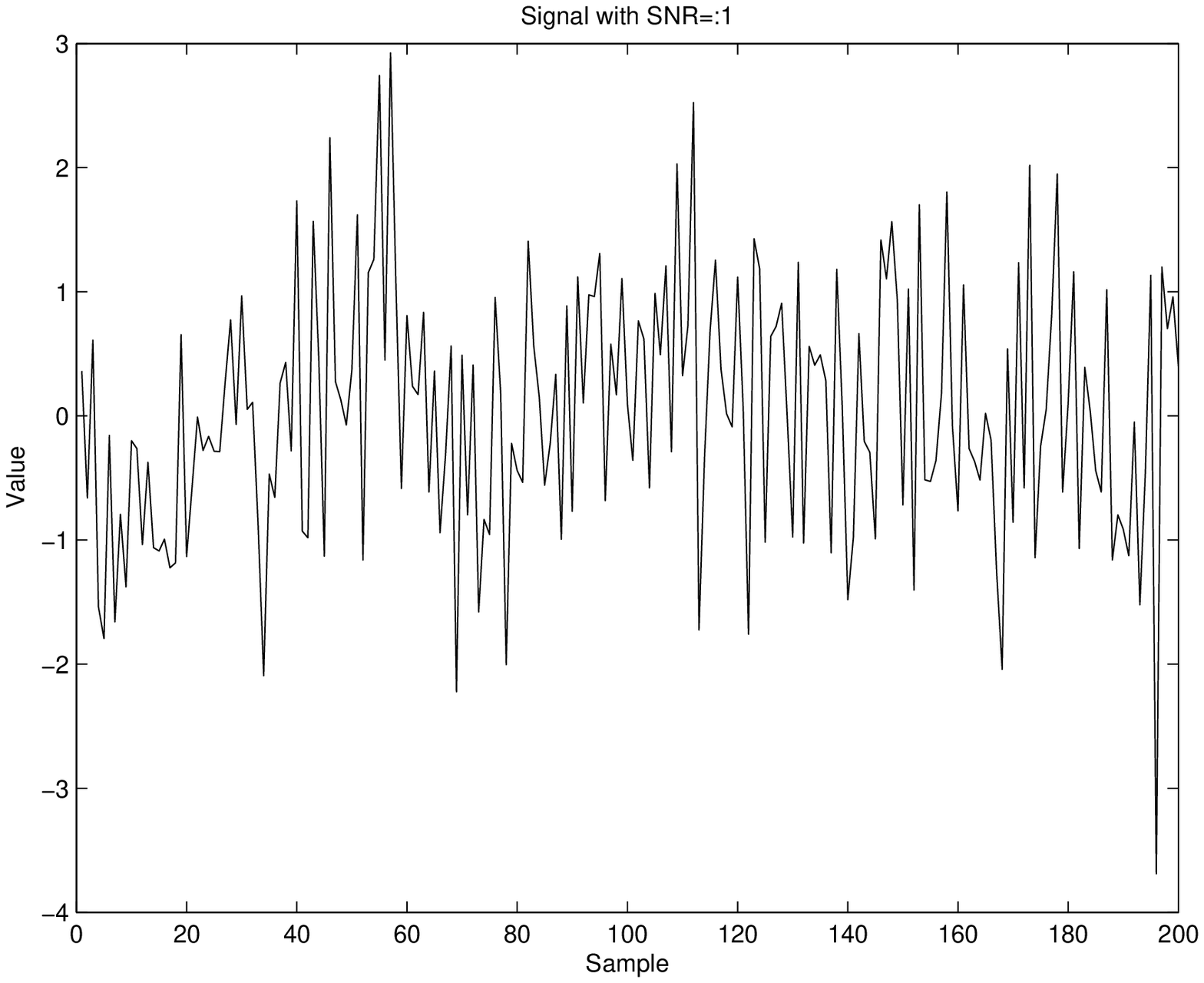}\label{input-a}}
\subfigure[] {\includegraphics[width=.55 \textwidth]{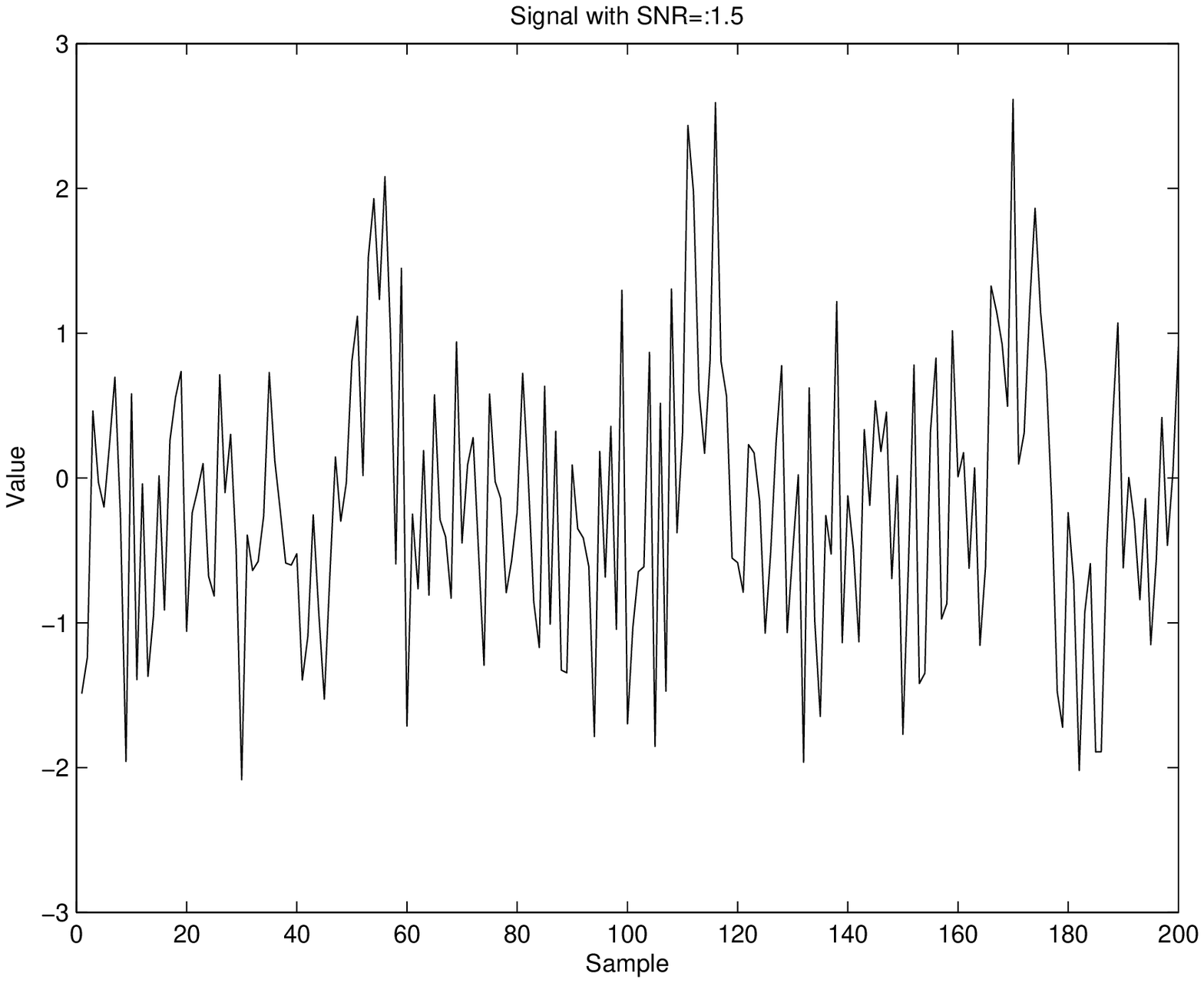}\label{input-b}}
\subfigure[] {\includegraphics[width=.55 \textwidth]{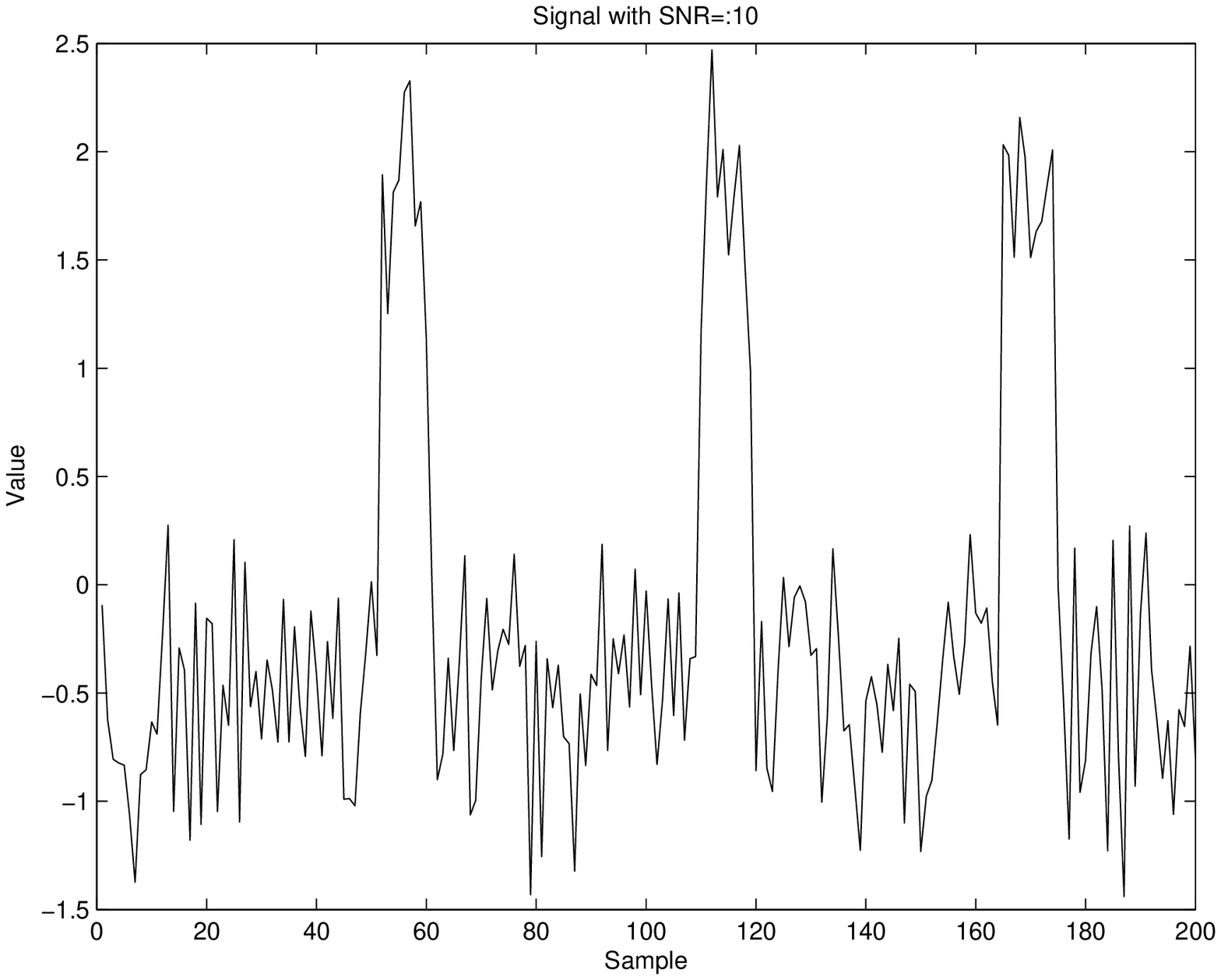} \label{input-c}}
\end{center}
\caption{Examples of input signals: (a) input signal $SNR=1$; (b) input signal  $SNR=1.5$; (b) input signal
$SNR=10$.}\label{input}
\end{figure*}

\subsection{Assessment on synthetic data}
The input signals used to evaluate the test, are synthetically generated following the procedure described in \cite{dig08} and in Chapter 3 and represent signals which emulate the nucleosome positioning data.

\begin{figure}[!htb]
\begin{center}
\subfigure[] {\includegraphics[width=.8 \textwidth]{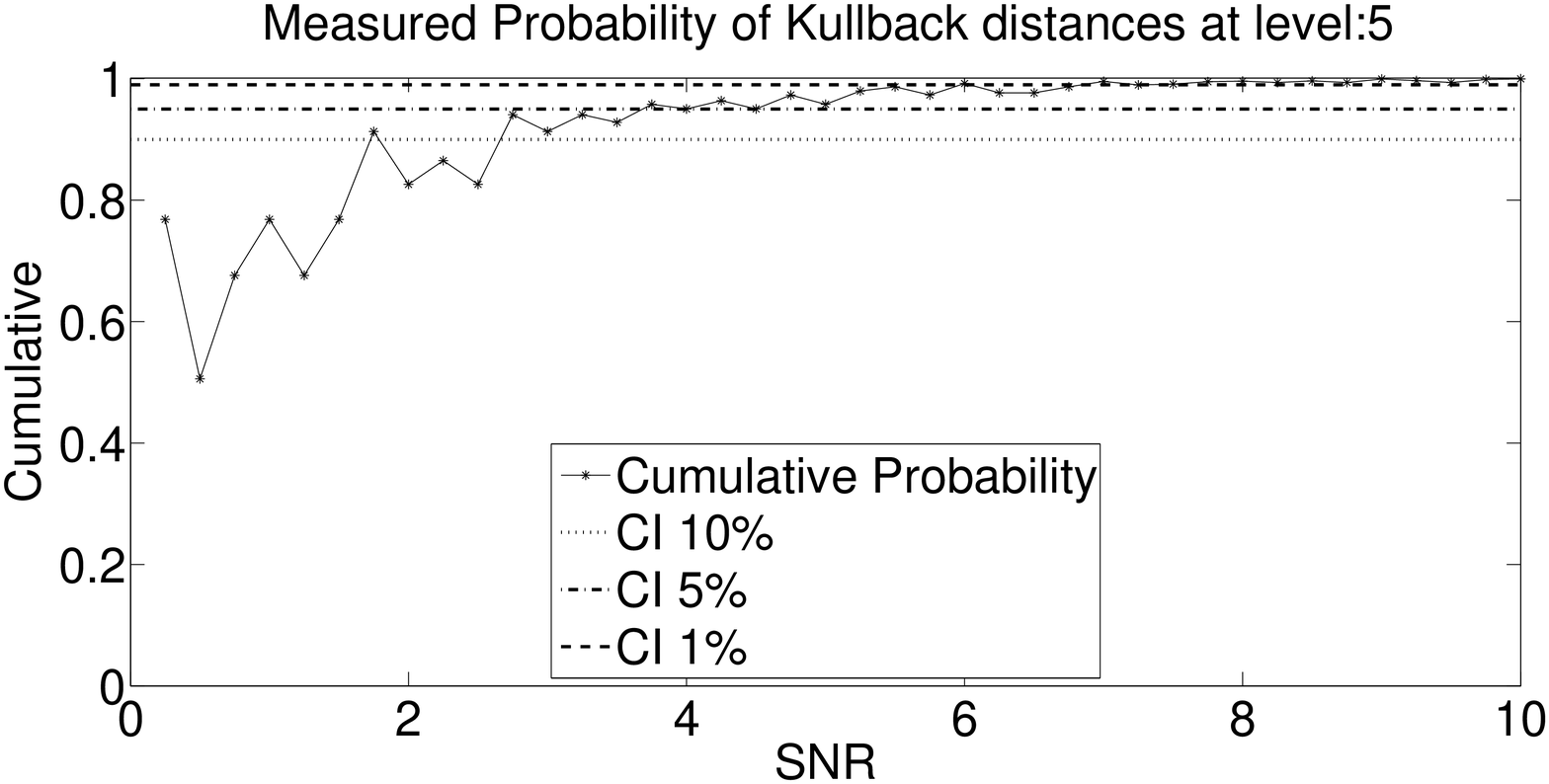}}
\subfigure[] {\includegraphics[width=.8 \textwidth]{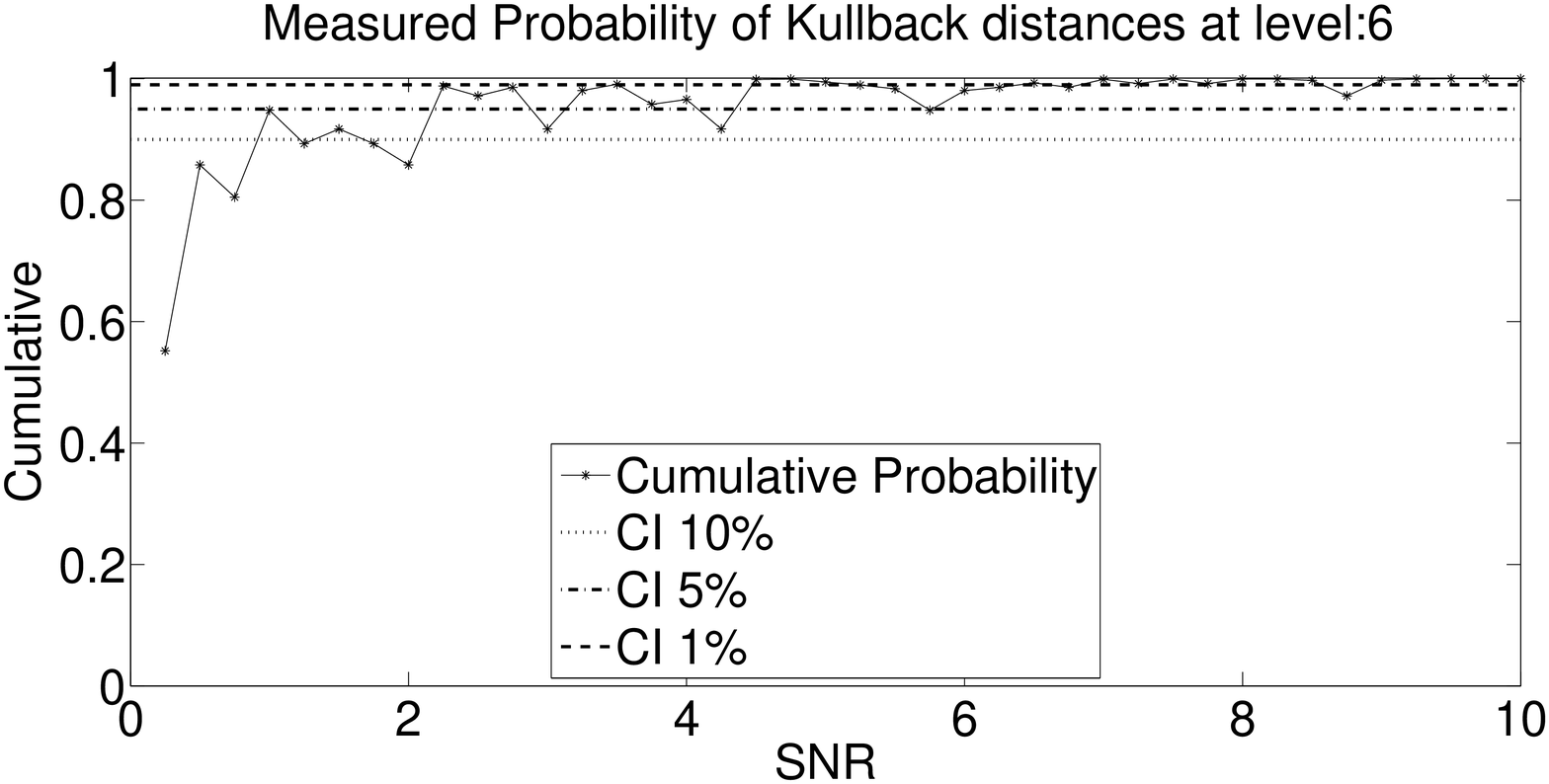}}
\end{center}
\caption{Examples of hypothesis test at different $SNR$ and thresholds.}\label{confi1}
\end{figure}

\begin{figure}[!htb]
\begin{center}
\subfigure[] {\includegraphics[width=.8 \textwidth]{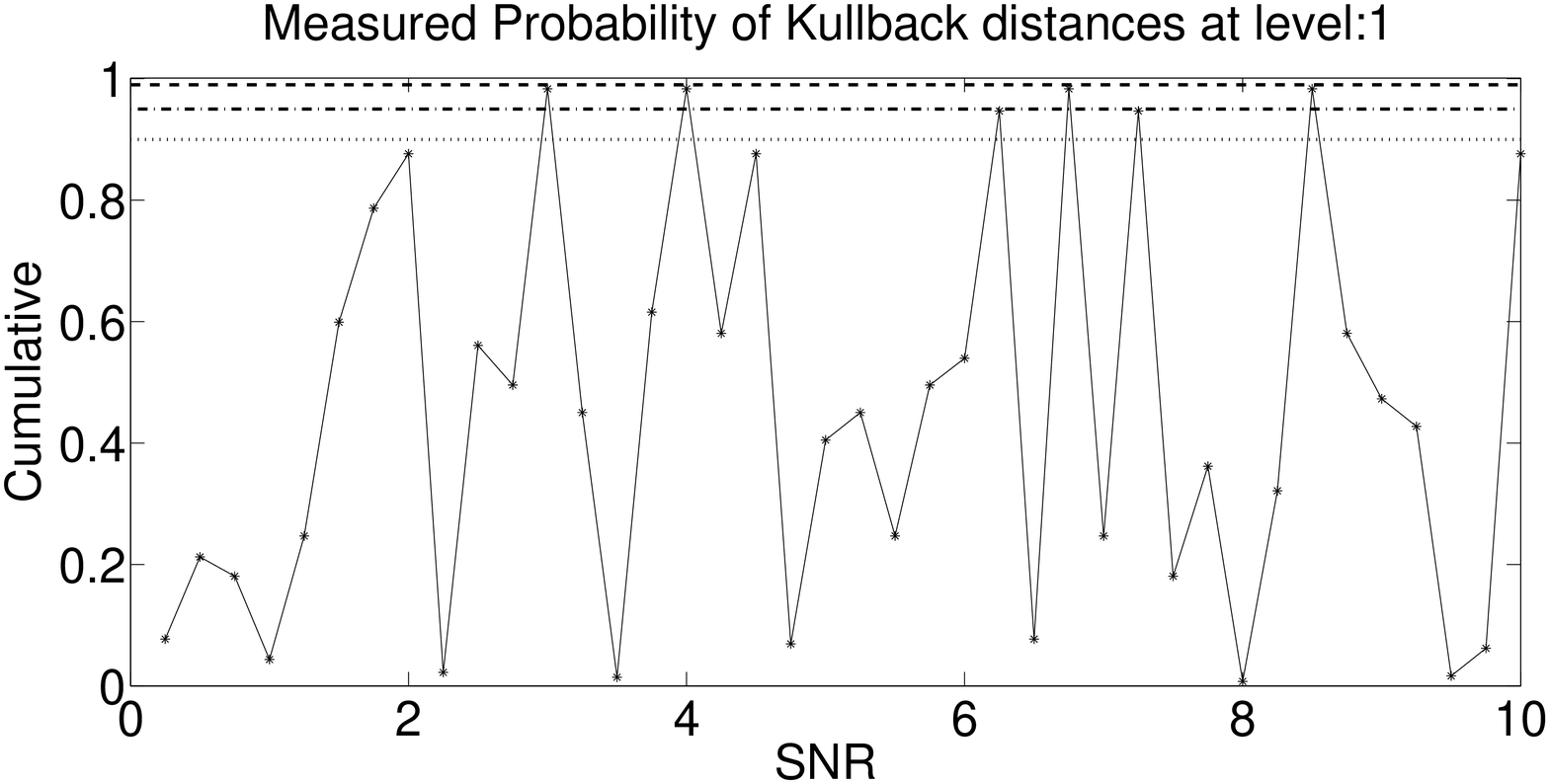}}
\subfigure[] {\includegraphics[width=.8 \textwidth]{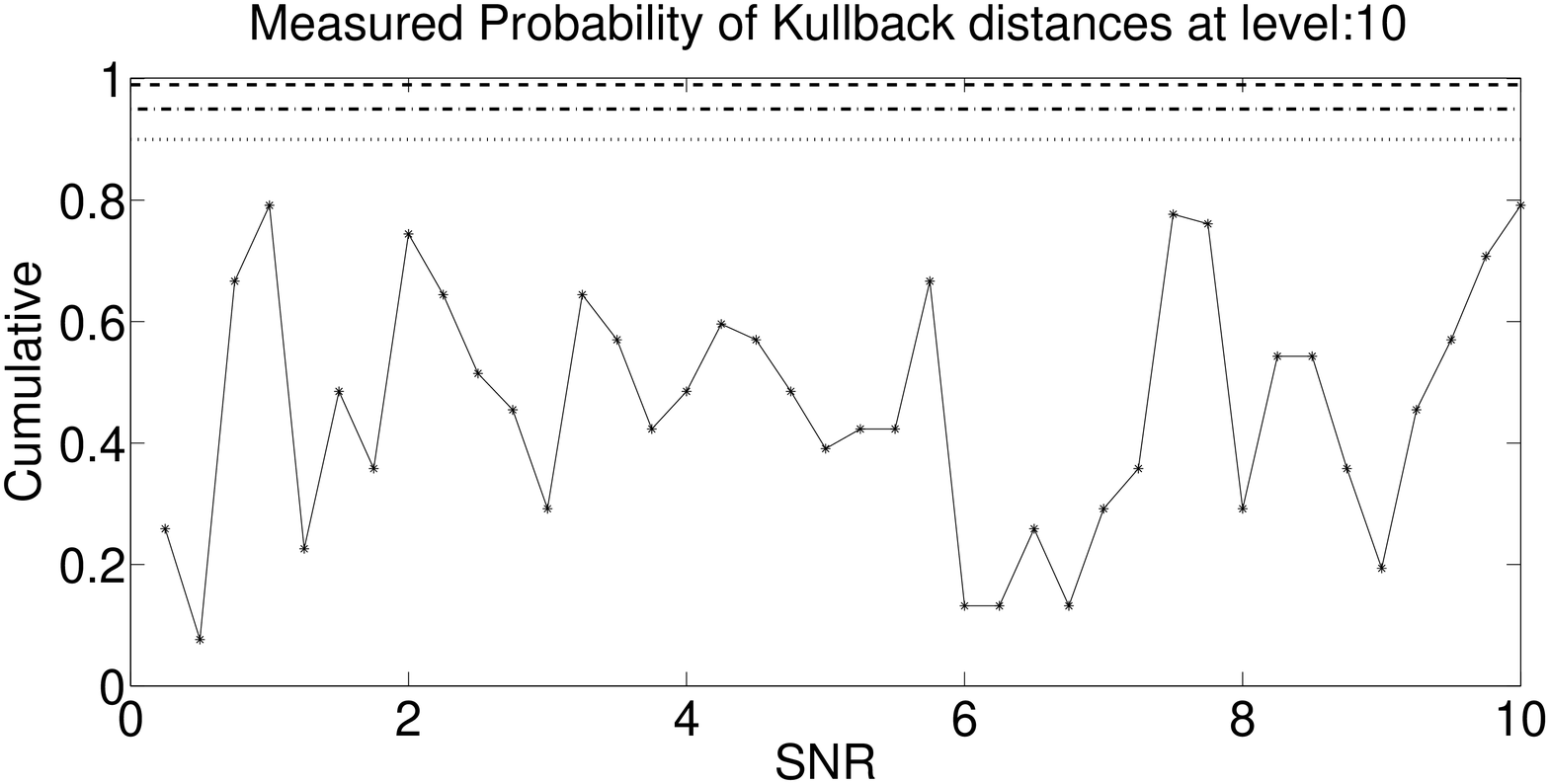}}

\end{center}
\caption{Examples of hypothesis test at different $SNR$ and thresholds.}\label{confi2}
\end{figure}

\begin{figure}[!htb]
\begin{center}
\includegraphics[width=.8 \textwidth]{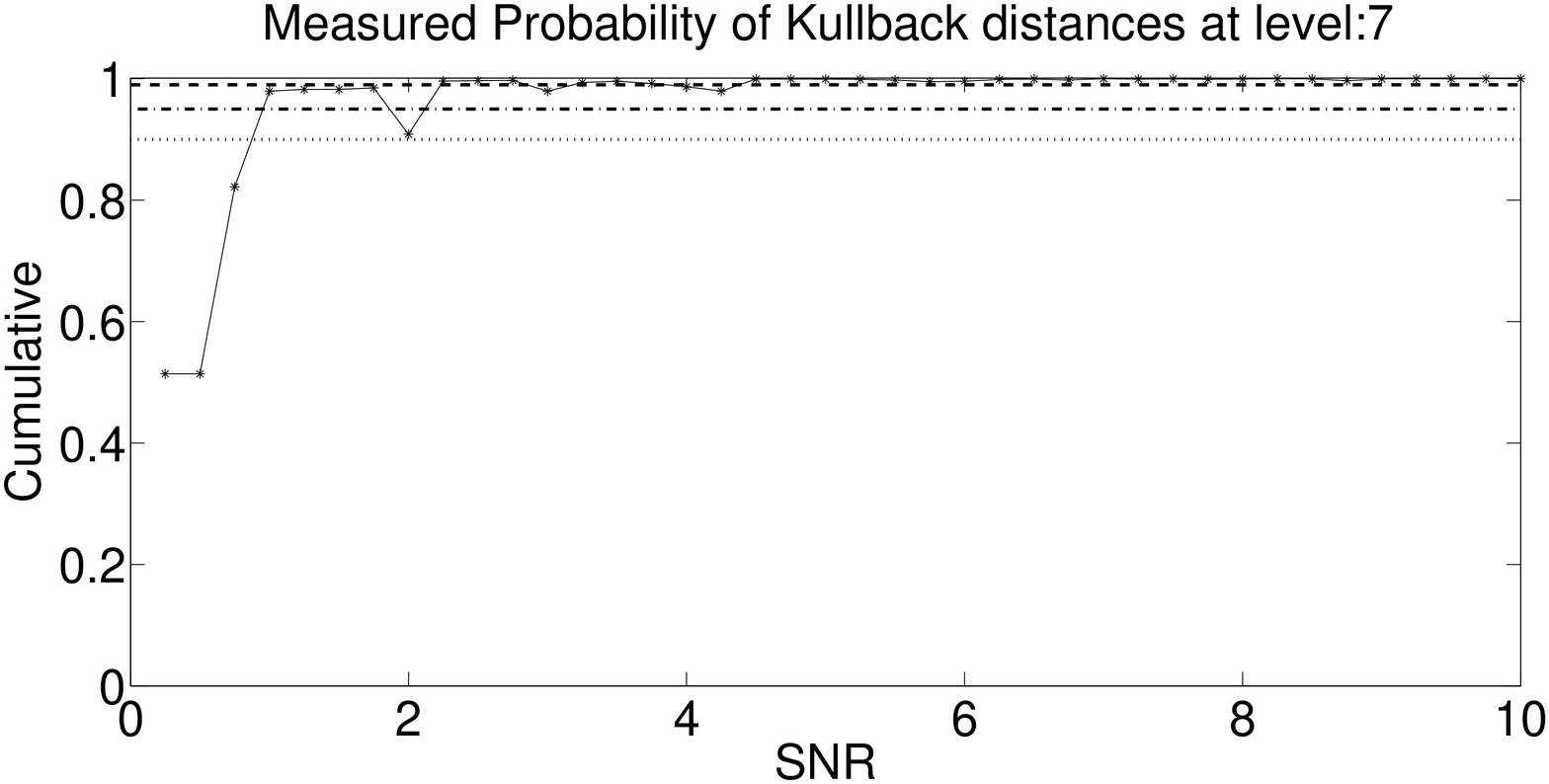}
\includegraphics[width=.8 \textwidth]{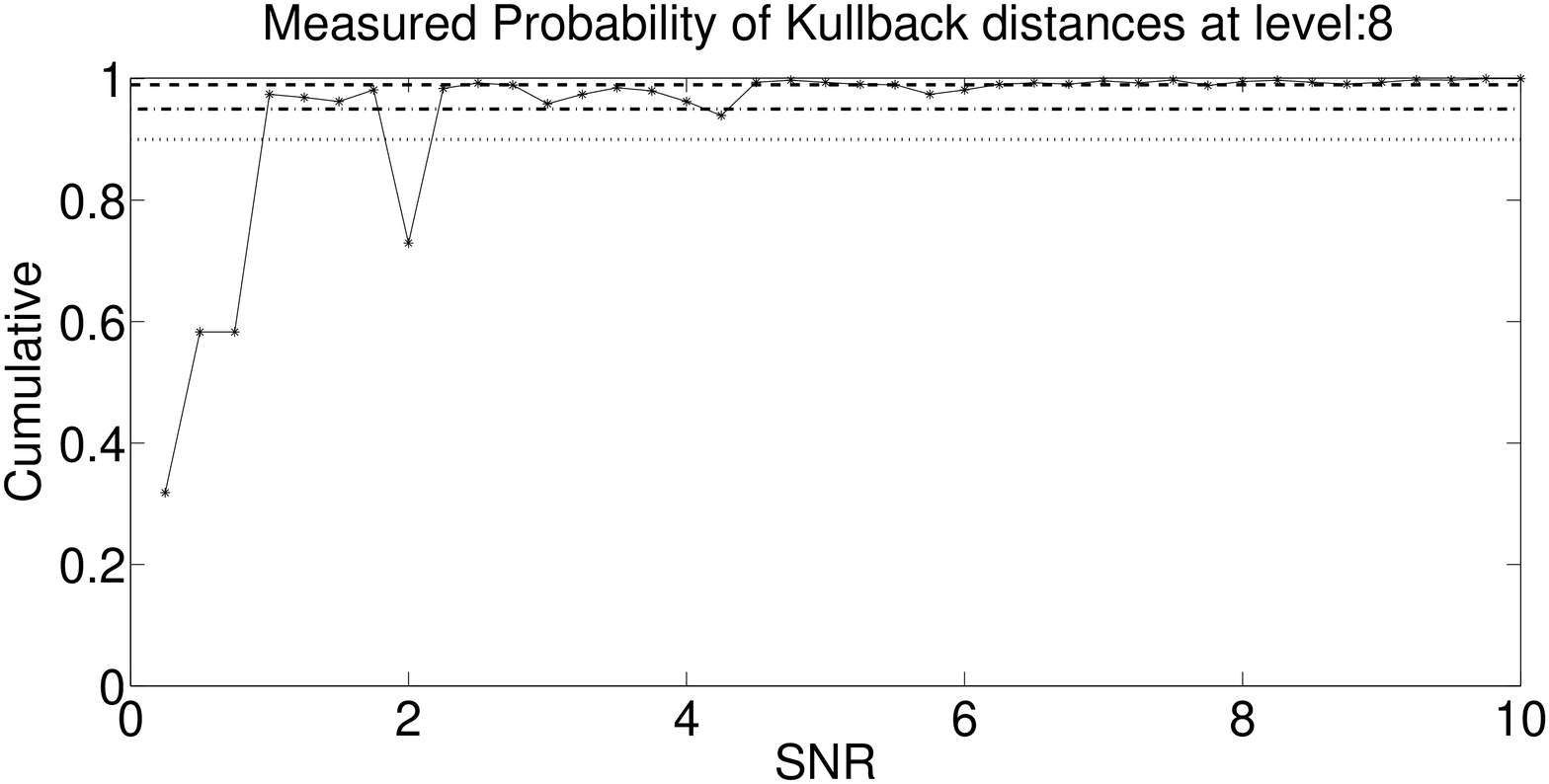}
\end{center}
\caption{Examples of hypothesis test at different $SNR$ and thresholds.}\label{confi3}
\end{figure}

In the following, an assessment of the proposed hypothesis test to guess the presence of structures in
an unknown input signal is performed. In this sense these can be considered part of an exploratory data analysis procedure.
This experiment has been carried out generating $40$ synthetic test signals of length $L=200000$ (base pairs), with signal to noise ratio ranging from $SNR=0$ to $SNR=10$ by steps of $0.25$ and $N$ random samples of length $l=20000$ (base pairs).
The simulation used to estimate the p.d.f. of $SKL$, has been done using the synthetic signals of length $L$ (base pairs),
$N=1000$ random samples of length $l$ (base pairs), $K=9$ thresholds and $nb=100$ bins.  The number of bins has been
set as a good compromise among different sample size at different thresholds.
For each test signal $S$ its $SKL_k$ from a random sample drawn from the $RS_n$
samples is computed and used to verify the test of hypothesis on the $PSKL_k$. In particular for each test signal $S$ recalling that
$R_n$, as defined above, is of the same length of $S$ and generated according to a normal distribution with $\widehat{\mu}$ and $\widehat{\sigma}$ estimated from $S$.
In figures \ref{confi1},  \ref{confi2},  \ref{confi3}  some results of the test are provided for increasing $SNR$, for confidence level
$\alpha=99\%,\ 95\%,\ 90\%$ and at different thresholds. In the abscissa  the $SNR$ is represented, while in the
ordinate the probability that the symmetrized Kullback-Leibler distance falls in the interval $[0,SKL_k]$. If the
ordinate value is greater than the confidence $\alpha$ the random test is rejected.
From previous results, it can be seen that the test is not reliable for lower and the higher thresholds while it
is quite sensitive for intermediated thresholds. For example, for $t_k=5,6,7,8$ and $\alpha\geq 90\%$ the random
hypothesis is rejected for $SNR\geq 3.0,\ 1.5,\ 1.25, 1.5$ respectively. Intuitively, this can be explained
because the number of intersections is low for higher and lower threshold values.

\subsection{Assessment on real data}
\label{real}
The test of randomness has been applied to real biological data derived from a tiled microarray approach able to reveal nucleosome positioning information on the Saccharomyces cerevisiae DNA \cite{YUAN05}.  The input microarray data, $\mathbb{S}$, are organized in  $T$ contiguous fragments $S_1,\cdots, S_T$ which represents $DNA$ sub-sequences. This dataset is explained in detail in Chapter 3.

In the experiment we set $K=10$, $nb=100$, for each signal fragment $S_i$ the corresponding intervals $INT_{ik}$ are extracted for each threshold $t_k$.

Finally, the set of intervals $INT_k= \bigcup_{i=1}^T INT_{ik}$ are used to compute the interval distribution length $PIL_k$. Then, the $SKL_k$ from a random sample $R$ drawn from the $RS_n$ samples is computed and used to verify the test of hypothesis on the $PSKL_k$.
Note that, in this experiment, the length of the real signal and of the random sample is $20000$ base pairs. Figures \ref{pr-a4}, \ref{pr-a5}, \ref{pr-a6} show $PIL_k$ for $k=4,5,6$.

The experiment indicates that the hypothesis test is rejected at confidence level $95\%$ for $k=5$, while for $k=6,7,8,9$ is rejected at a confidence level $\geq 99\%$. In figures \ref{pr-b4}, \ref{pr-b5}, \ref{pr-b6} are shown the result of the test of randomness for $k=4,5,6$. Moreover, the test of randomness is quite unstable for $k\leq 4$ and $k=10$; this property highlights that the central part of the signal contains the majority of the useful information for the test of randomness (see Figure \ref{central}).

\begin{figure}[!htb]
\centering
\subfigure[] {\includegraphics[width=.8 \textwidth]{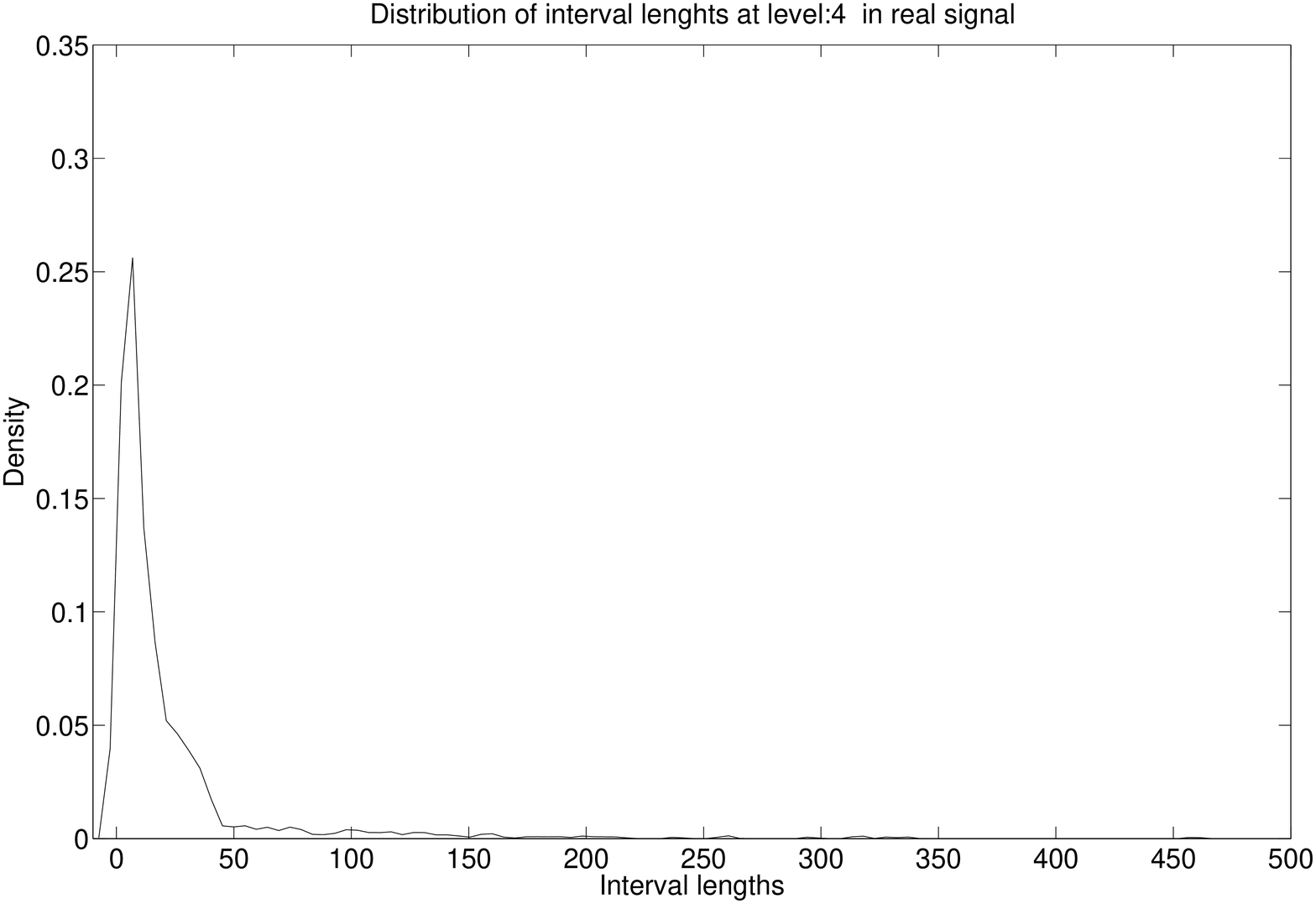} \label{pr-a4}}
\subfigure[] {\includegraphics[width=.8 \textwidth]{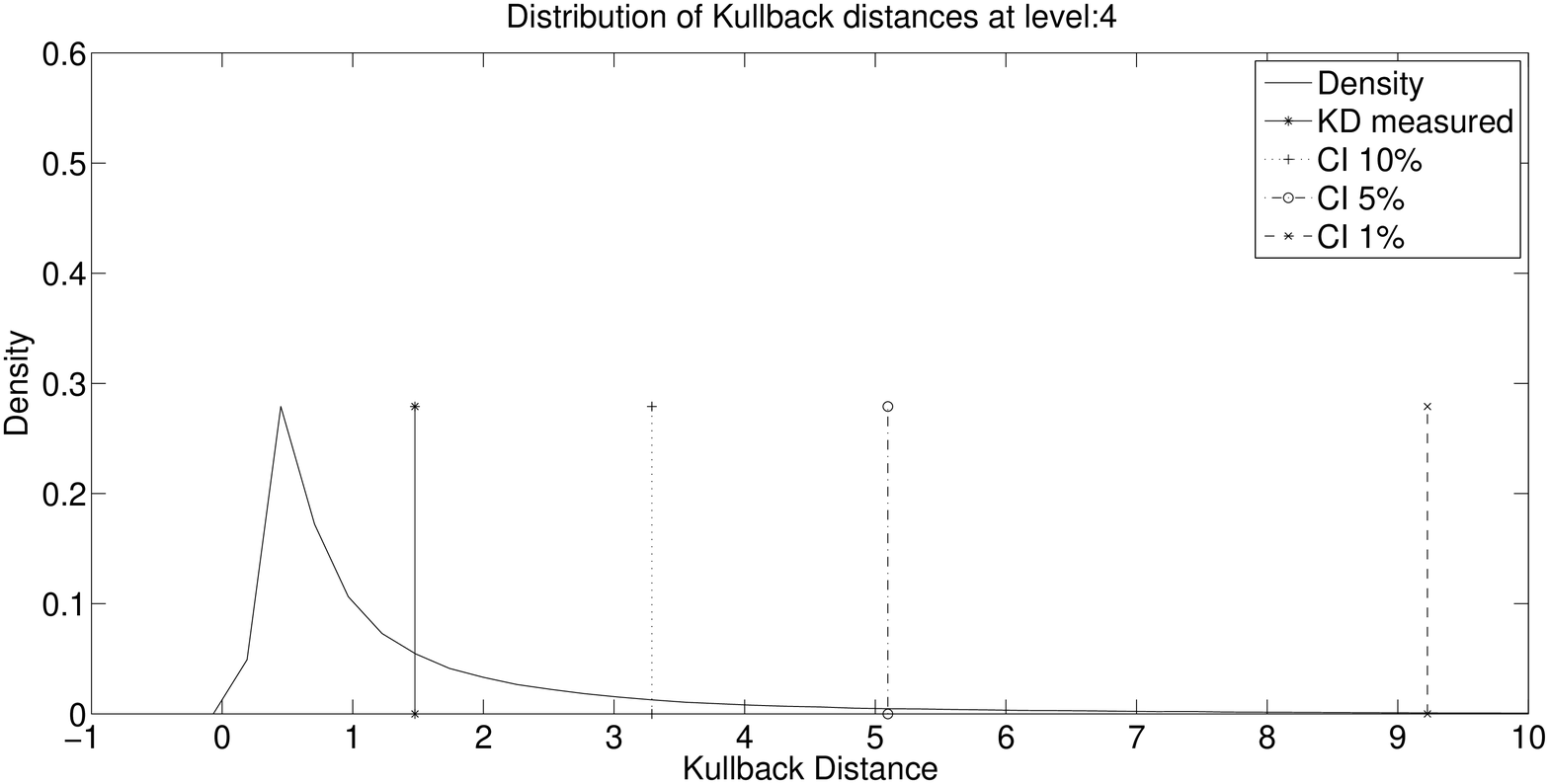} \label{pr-b4}}
\caption{$PIL_k$ (a) and $PSKL_k$ and hypothesis test results (b) of the real signal for $k=4$.}
\end{figure}

\begin{figure}[!htb]
\centering
\subfigure[] {\includegraphics[width=.8 \textwidth]{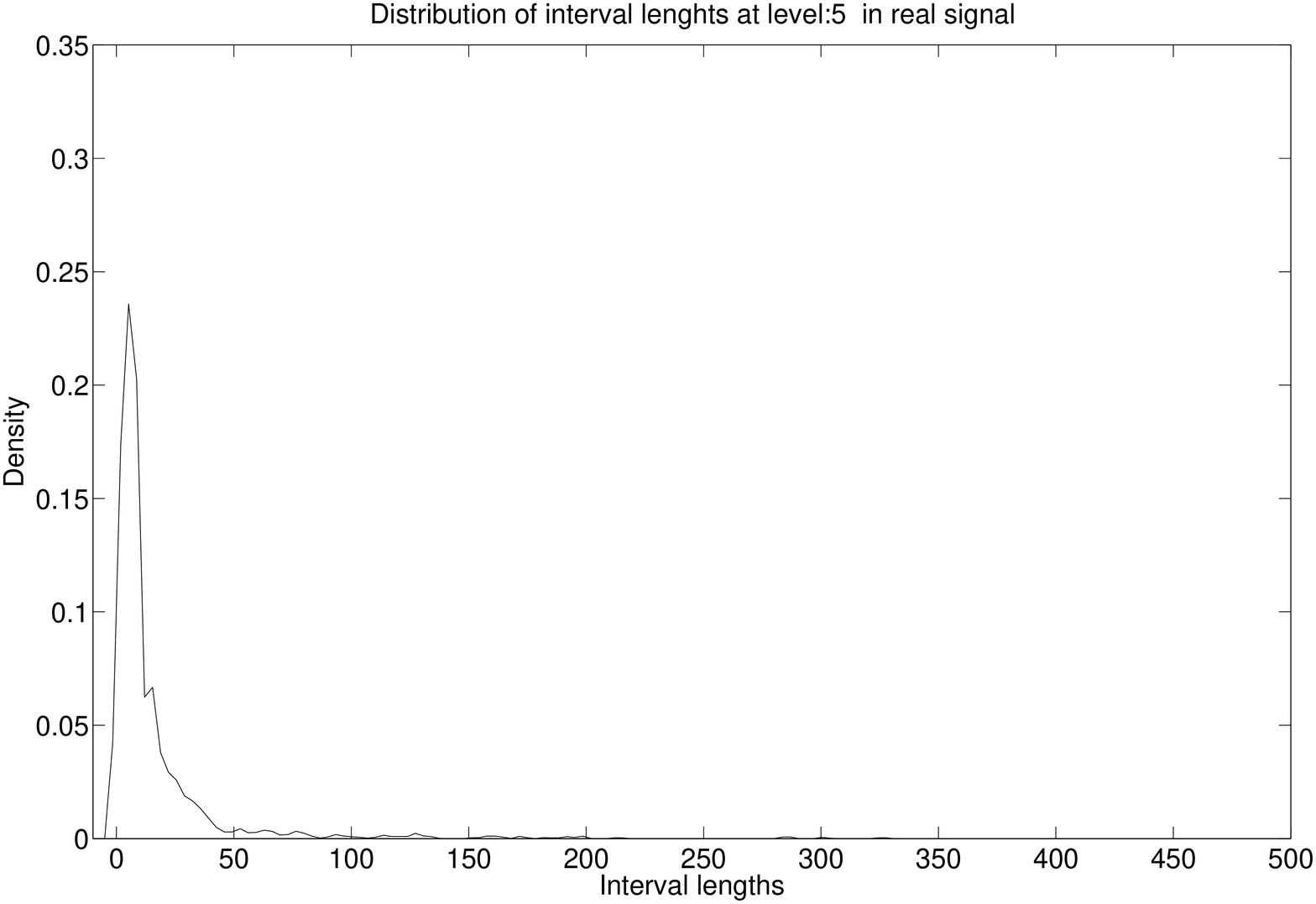} \label{pr-a5}}
\subfigure[] {\includegraphics[width=.8 \textwidth]{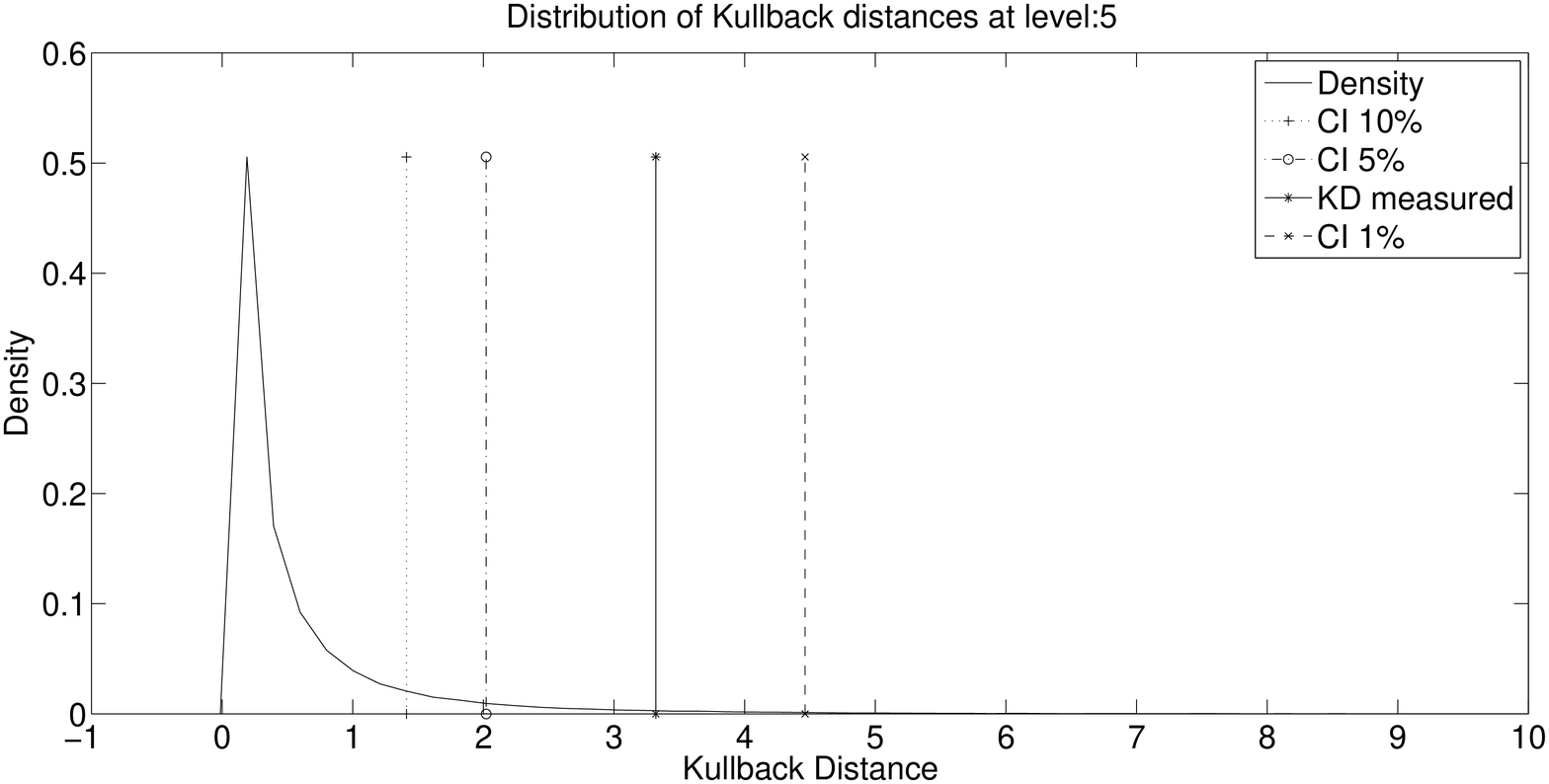} \label{pr-b5}}
\caption{$PIL_k$ (a) and $PSKL_k$ and hypothesis test results (b) of the real signal for $k=5$.}
\end{figure}

\begin{figure}[!htb]
\centering
\subfigure[] {\includegraphics[width=.8 \textwidth]{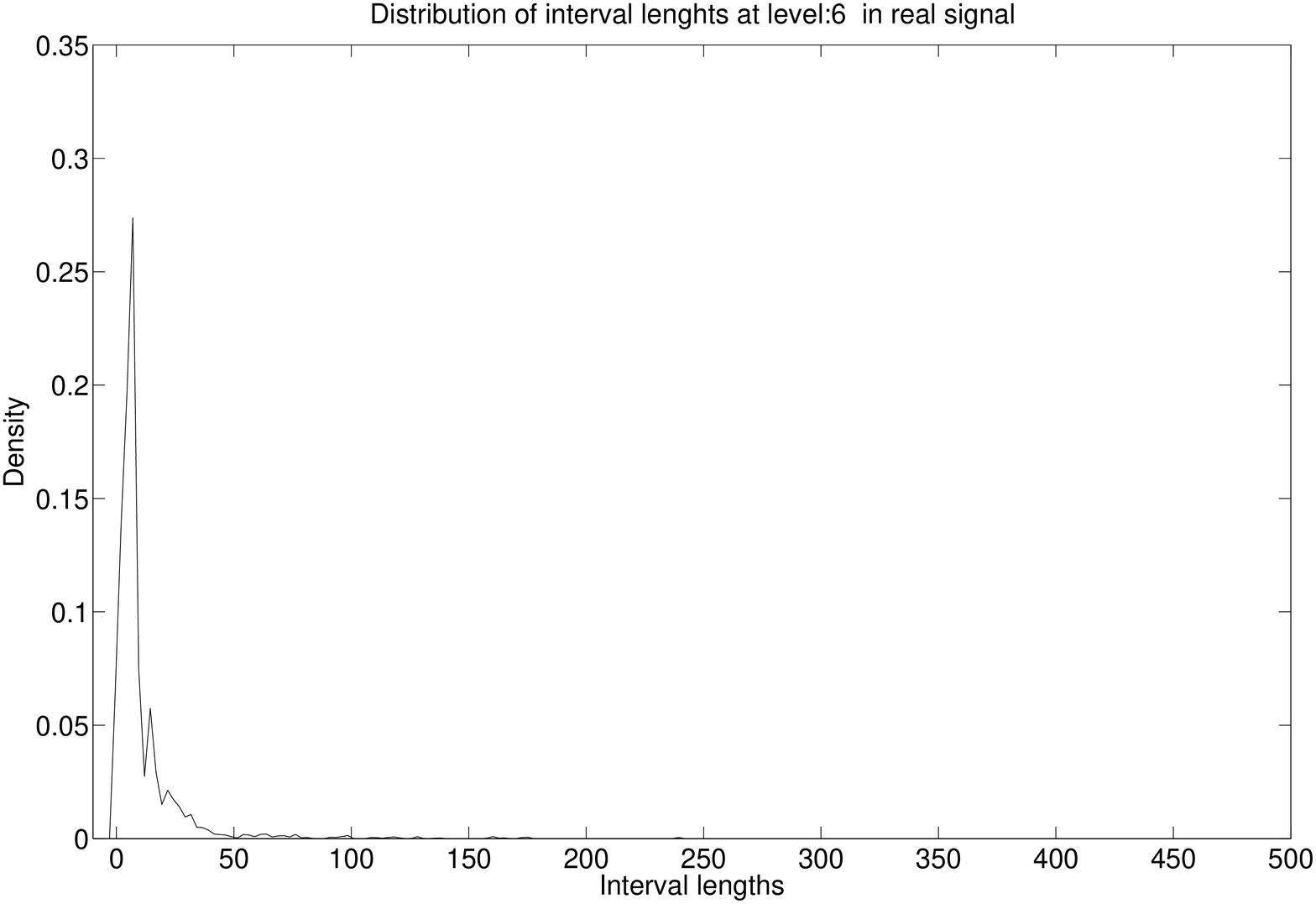} \label{pr-a6}}
\subfigure[] {\includegraphics[width=.8 \textwidth]{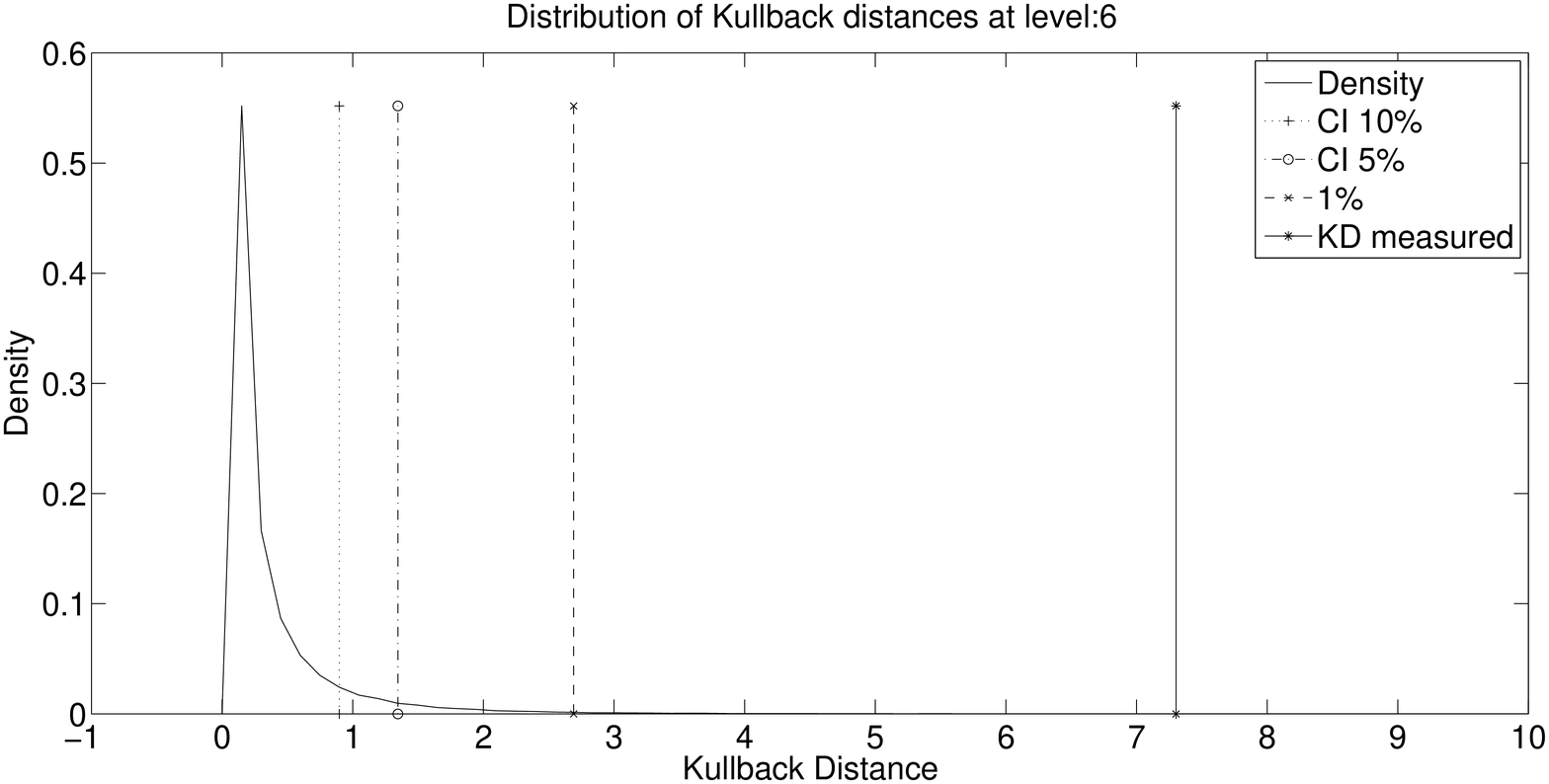} \label{pr-b6}}
\caption{$PIL_k$ (a) and $PSKL_k$ and hypothesis test results (b) of the real signal for $k=6$.}
\end{figure}

\begin{figure}[!htb]
\begin{center}
\includegraphics[width=.9 \textwidth]{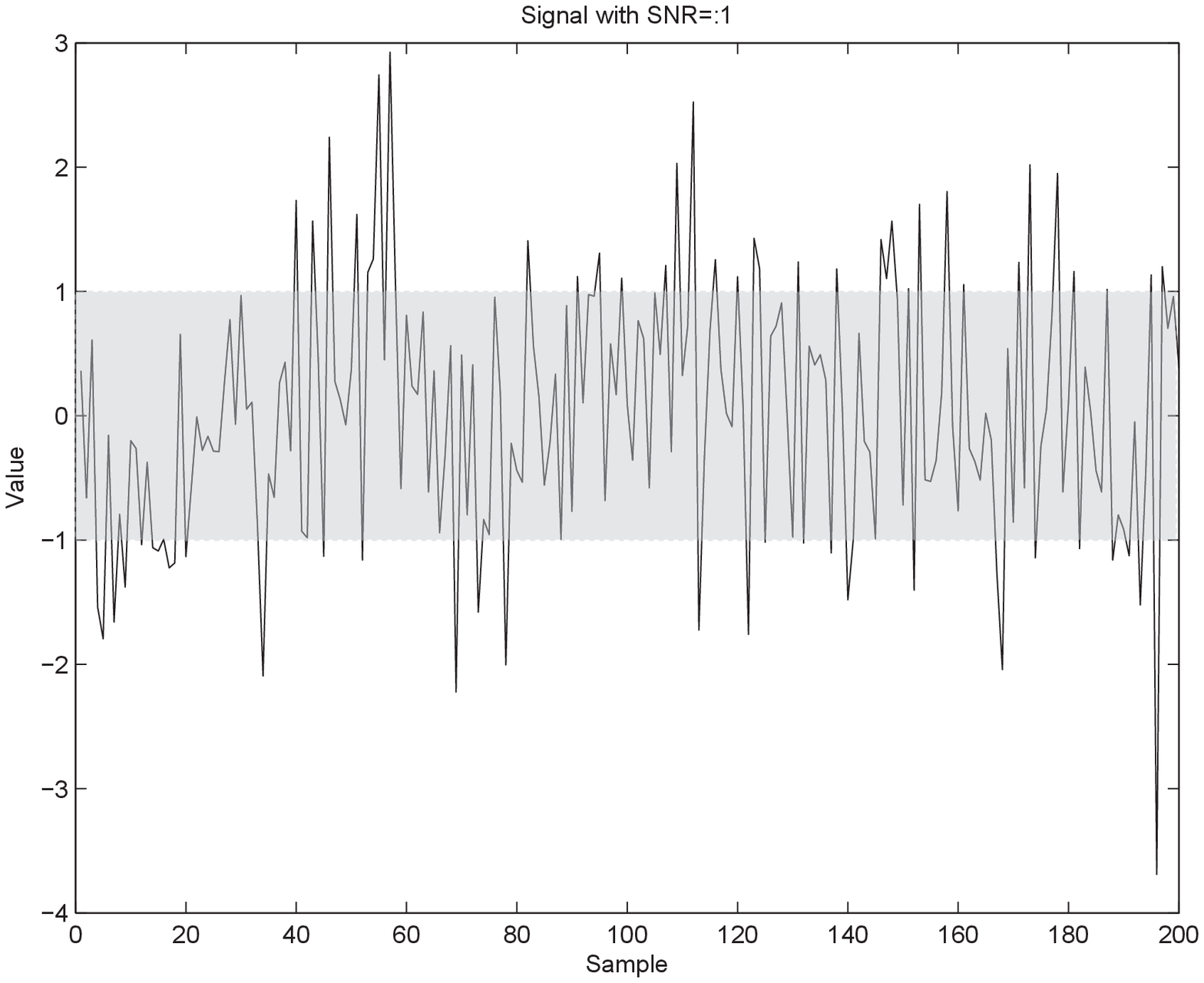}
\end{center}
\caption{The gray strep indicates the useful part of the input signal in order to perform the test of randomness.} \label{central}
\end{figure}

\subsection{Comparison with Wilcoxon rank sum test}
In this section a comparison of the  results of MLA test with the Wilcoxon rank sum test, both on synthetic and real data is presented. Both hypothesis tests, can be applied when no assumption about sample distribution can be made, condition which falls in this case.

Firstly, it was  verified if each synthetically generated $S$ (a total amount of $40$ signal) and $N=100$ random samples drawn from the $RS_n$ are significantly different by using a Wilcoxon rank sum test.  Figure \ref{ranksum-a} shows the results that can be summarized affirming that $S$ and a generic random signal $R$ are at least $90\%$ significantly different starting from $SNR=1.25$. This reveal that the Wilcoxon test and our test have quite the same predictive power when considering intermediate threshold levels of the $MLA$ (k=6,7,8,9) .

In the case of a real signal $S$, the Wilcoxon rank sum test has rejected the hypothesis of randomness on $S$ only $3$ times over $N=100$ tests (see figure \ref{ranksum-b}). This makes the Wilcoxon rank sum test not reliable for such kind of data, while our test, as already shown in section \ref{real}, confirms his predictive power on intermediate thresholds.\\\\

In this chapter several tests have been introduced in order to check the randomness of a set of one dimensional signals and a new test of randomness based on the MLA preprocess has been also presented. It makes uses of the Symmetrized Kullback-Leibler distance, and it has been shown to be useful in the case of exploratory analysis in order to verify the possible presence of structures in an input signal. Finally, it is able to guess structures in the case of real and simulated data for nucleosome positioning with low $SNR$ (1.5), while a simple Wilcoxon rank sum test has not shown enough reliability on the same kind of data.

\begin{figure}[!htb]
\centering
\subfigure[] {\includegraphics[width=.8 \textwidth]{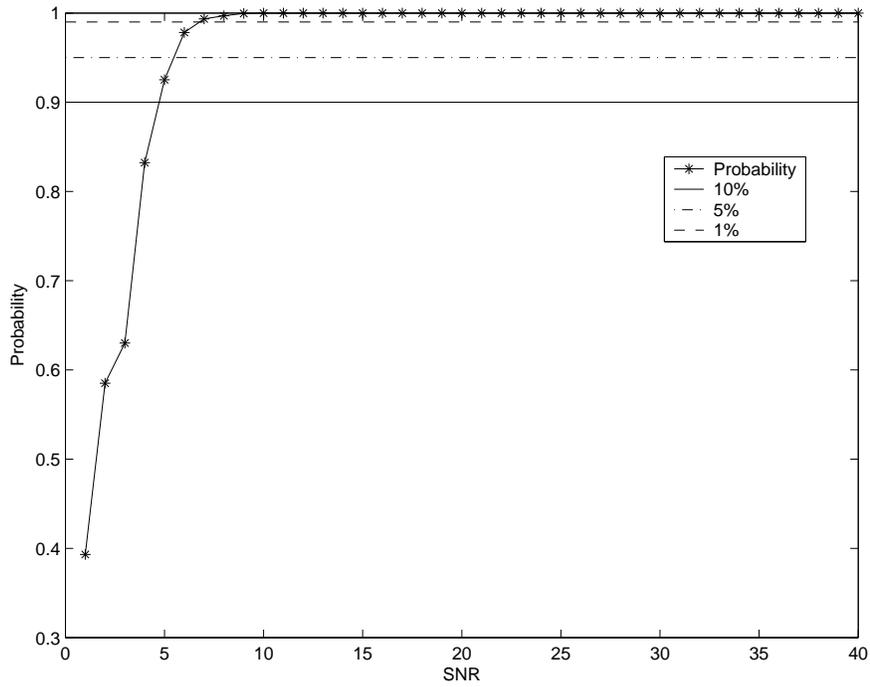} \label{ranksum-a}}
\subfigure[] {\includegraphics[width=.8 \textwidth]{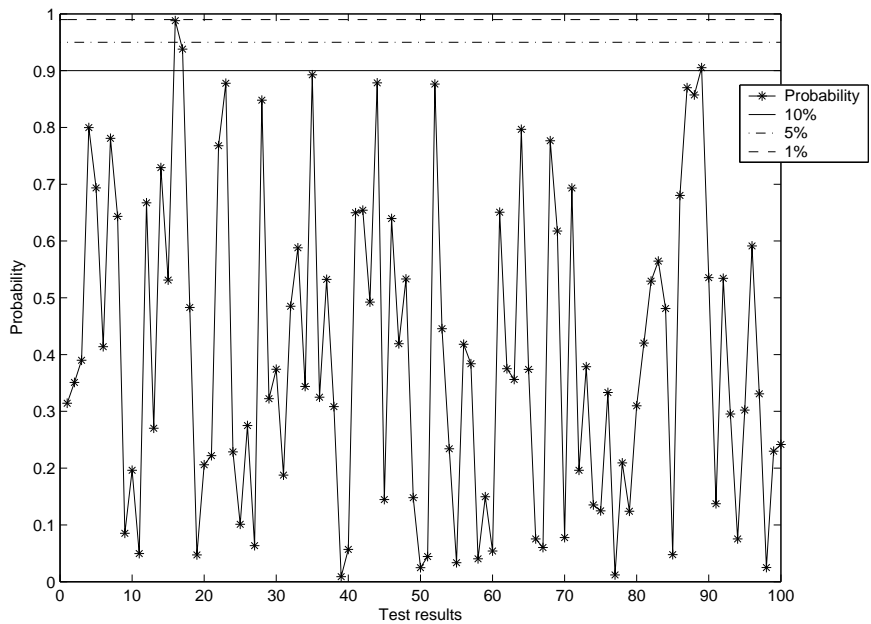} \label{ranksum-b}}
\caption{Mann-Whitney rank sum test results for different signal to noise ratio (a) and for the real signal (b).}
\end{figure}

\chapter{MLA and Kernel methods}
\label{chap:5}
This chapter presents how the MLA can help on designing  new kernel functions that explicitly take into account the shape information contained in a one-dimensional signal. In the following, the main idea of kernel methods will be presented, giving more details to a particular subclass of kernel functions applicable to structured data, and in particular trees. The MLA is used to define a mapping from the set of one-dimensional signal to the set of trees. For this reason the main advantage of defining a kernel function based on MLA is that it is possible to incorporate shape information directly in a kernel function encoded as a tree.

\section{Kernel methods}

Kernel methods are a class of algorithms used in the context of pattern analysis. Although initially they were developed in the context of classification, with the well known Support Vector Machine (SVM) method first introduced by Vapnik \cite{Vapnik:statistical_learning}, the kernel approach has shown to be applicable to several key problems in data analysis (Principal Component Analysis, Clustering, Regression, Ranking, Correlation). In this sense, nowadays, it is usually referred to the kernel methods, as a general framework applicable to all kinds of data \cite{Taylor_2004_kernel}. In fact recently, kernel methods were developed to deal with data without an explicit vector representation such as complex objects or structured data (string, tree, graph, etc.).

\subsection{Main ideas of kernel methods}
The main advantage of kernel methods came from their modularity: all these methods consist of two parts: a kernel function, and an algorithm to analyze the data after the kernel mapping, as shown in figure \ref{fig:general_schema}. In particular, the kernel embeds the input space into a new vector space where the algorithm used to analyze the data could have better performance than the same algorithm applied on the original input space (see figure \ref{fig:kernel_mapping}). The kernel functions represent the spatial relation between pair of data elements, using an inner product in the new space without explicitly map such data. In this way it is also possible to use infinite dimensional space without encoding the data explicitly with new coordinate vectors. Moreover, in many case the computation of the inner product could be more efficient than explicitly map each point into the new vector space and computing for example the pairwise distances. This imply that it is not necessary to know the exact coordinates of the points in the vector space but only their pairwise inner product. In other words the dimensionality of the new vector space does not affect the computation time. This propriety is  usually called ``kernel-trick'', and can be summarized saying that, to perform data analysis with a kernel, it is not necessary to know explicitly the vector space where the data will be projected in. An important point of the kernel functions is that the mapping could catch non-linear relation present in the data linear in the new space.  This permits to take advantage of the large class of well understood methodologies that search linear relation in the data. In this way the choice of a particular kernel function is related to the vector space where the data points will be implicity projected. A deeper coverage of the theory and application of kernel methods can be found in the book by Taylor and Cristianini \cite{Cristianini_Kernel}.
Now, it will be given the formal definition of kernels and some of their properties will be count.

\begin{figure}[!htb]
\centering
\includegraphics [scale=.8]{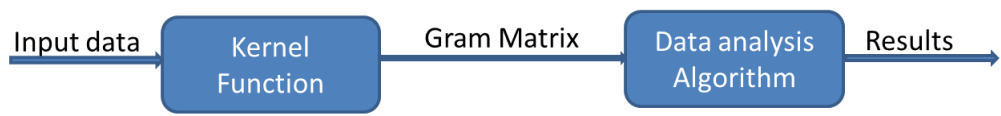}
\caption {General schema of kernel methods}
\label{fig:general_schema}
\end{figure}

\begin{figure}[!htb]
\centering
\includegraphics [width=0.9\textwidth]{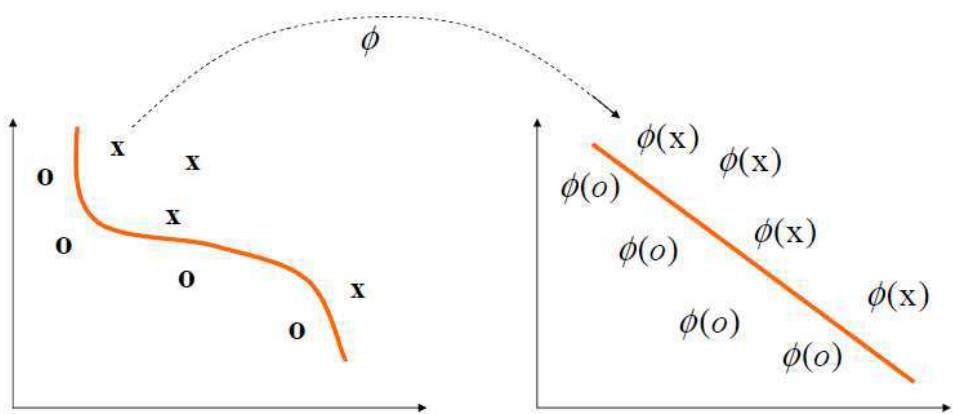}
\caption {Kernel mapping}
\label{fig:kernel_mapping}
\end{figure}

\subsection{Formal definition and properties of kernels}
\begin{definition} \emph{Kernel Function}\\
Given a set $X\neq \emptyset$, and a mapping function from $X$ to a features vector space $F$ i.e. $\phi(x):X \rightarrow F$ a kernel is a function $K: X \times X \rightarrow \mathbb{R}$ that for all $x,y\in X$ is:

\begin{equation}
k(x,y)= \langle \phi(x), \phi(y) \rangle
\end{equation}

where $\langle \cdot , \cdot \rangle$ denotes the euclidean inner product on $F$

\noindent It is clear that the function $K$ is symmetric i.e.:
\begin{equation}
k(x,y)= \langle \phi(x), \phi(y) \rangle =\langle \phi(y), \phi(x) \rangle = k(y,x)
\end{equation}

\end{definition}

An important theorem that provide a characterization of the class of kernel function is the here stated Mercer's theorem\cite{Cristianini_SVM}:

\begin{theorem}
\emph{(Mercer's Theorem)}
Let $X$ a compact subset of $\mathbb{R}^n$. Suppose $K$ is a continue symmetric function such that the integral operator $T_K: L_2(x) \rightarrow L_2(x),$

\begin{equation}
(T_Kf)(\cdot)=\int_X K(\cdot,x) f(x)dx
\end{equation}

is positive, that is:

\begin{equation}
 \int_{X \times X} K(x,z) f(x)f(z)dx dz \geq 0
\end{equation}

\noindent for all $f \in L_2(x)$. Then it is possible to expand $K(x,z)$ in a uniformly convergent series (on $X \times X$)
in terms of $T_k$'s eigen-function $\phi_j \in L_2(X)$, normalized in such a way that $\| \phi_j \|_{L_2}$, and positive associated eigenvalues $\lambda_j \geq 0$.

\begin{equation}
K(x,z) = \sum _{j=1}^{\infty} \lambda_j\phi_j(x)\phi_j(z)
\end{equation}
\end{theorem}

A special case of this theorem is the following, that characterizes the Kernel function on Finite spaces.

\begin{theorem}\label{th:Mercer_simple}
Let $X$ a finite input space with $K(x,z)$ a symmetric function on $X$. Then  $K(x,z)$ is a kernel function if and only if the matrix:
\begin{equation}
\mathbf{K}= (K(x_i,x_j))_{i,j=1}^n
\end{equation}

\noindent is positive semi-definite (has non negative eigenvalue) i.e:

\begin{equation}
\sum _{i=1} ^{n} \sum _{j=1} ^{n} c_i c_j K(x_i,x_j) \geq 0
\end{equation}

\noindent with $n>0$, $x_1,\ldots,x_n \in X$ and  $c_i,c_j \in \mathbb{R}$.
\end{theorem}

\begin{proof}
Since the matrix $(K(x_i,x_j))_{i,j=1}^n$ is symmetric, there exists an orthogonal matrix $\mathbf{V}$ such that:
$K=\mathbf{V} \Lambda \mathbf{V'}$ where $\Lambda$ is the diagonal matrix containing the eigenvalues $\lambda_t$ of $\textbf{K}$, and the columns of $\mathbf{V}$ are the corresponding eigenvectors $v_t =(v_{ti })_{i=1}^{n}$. By hypothesis, the eigenvalues of $\textbf{K}$ are non-negative, so it is possible to define the mapping $\phi$:

\begin{equation}
\phi: x_i \mapsto (\sqrt{\lambda_t}v_{ti} )_{i=1}^{n}
\end{equation}

And express the inner product as:
\begin{equation}
\langle \phi(x_i), \phi(x_j) \rangle = \sum _{i=1}^{n}\lambda_t v_{ti} v_{tj} = (\mathbf{V} \Lambda \mathbf{V'})_{ij}= K(x,y)
\end{equation}

And this proves that $K$ is a kernel function that calculate the inner product in the vector space given by the mapping function $\phi$.
Note that the condition of positive semi-definiteness is necessary, since if it exists at least a negative eigenvalue $\lambda_s$  with corresponding eigenvector $\mathbf{v}_s$, the point:
\begin{equation}
\mathbf{z}= \sum _{i=1}^{n}\mathbf{v}_{si} \phi(x_i) = \sqrt{\Lambda} \mathbf{V'} \mathbf{v}_s
\end{equation}
would have a norm squared less than $0$ in that space that is impossible:
\begin{equation}
\parallel \mathbf{z} \parallel^2 = \langle \mathbf {z},\mathbf {z} \rangle=  \mathbf{v'}_s \mathbf{V} \sqrt{\Lambda} \sqrt{\Lambda} \mathbf{V'} \mathbf{v}_s = \mathbf{v'}_s \mathbf{V} \Lambda \mathbf{V'} \mathbf{v}_s= \mathbf{v'}_s \mathbf{K} \mathbf{v}_s = \lambda_s <0
\end{equation}

\end{proof}

\subsection{Kernels and distances}
A simple property of the inner product, is that it naturally induces a norm:

\begin{equation}
\| x \|_2 = \sqrt{\langle x,x \rangle}
\end{equation}

\noindent and thus a metric or distance:

\begin{equation}
d(x,z) = \| x-z \|_2
\end{equation}

\noindent It follows immediately, that a generic kernel function also induces a distance:

\begin{definition} \emph{Distance induced by a kernel function}\\
Given a kernel function k, and consider the Gram's matrix $G_{ij}=k(x_i,x_j)= \langle \phi(x_i), \phi(x_j) \rangle$, it is possible to obtain a pairwise distance matrix  $D_{ij}$ from $G$
using the following relation:

\begin{equation}\label{eq:distance}
D_{ij}= \sqrt{  \parallel  {\phi(x_i) - \phi(x_j) }\parallel ^2 } = \sqrt{ k(x_i,x_i)+ k(x_j,x_j) -2k(x_i,x_j)}
\end{equation}

\noindent As an example, let us consider the euclidean distance:

\begin{definition}
\emph{Euclidean Distance}\\
Given two signals $\vec{x}$ and $\vec{y}$ their Euclidean Distance is defined as:\\
\begin{equation}\label{eq:Euclidean}
d_n(\vec{x},\vec{y})=\sqrt{\sum_{i=1}^m (x_i - y_i)^2}
\end{equation}
\noindent where $\vec{x}=(x_1,\dots,x_m)$, $\vec{y}=(y_1,\dots,y_m)$.
\end{definition}

It is straightforward that the euclidean distance is induced by the linear kernel $K(x,y) = xy'$.

\end{definition}

\section{Kernel methods for tree}
All the classes of kernel function in this category are based on the concept of tree i.e. the input data are represented in a tree structure. One assumes that the reader is familiar with the general concepts of graph theory, in particular with the definition of tree structure. For an appropriate background, the reader is referred to standard literature  \cite{Bondy}. As stressed in the introduction, it is possible to define kernel function even  when the input data doesn't have an explicit vector representation. This is the case of structured data and in particular in the case of  tree structure. More in general, there exists a class of kernel function called \emph{Convolution Kernel} and firstly introduced by Hausler \cite{Haussler_1999_ConvolutionKernels} and later extended  by Shin and Kuboyama \cite{Shin08_generalization} \cite{2010Kuboyama} that decompose a data object into simpler parts and then define a kernel function in terms of such parts.

\subsection{Convolution kernel}
This class of kernels are particular devoted for problem involving the processing of structured data like string, trees, graph. In fact it provides a way to extract real-valued features and thus to map these data into a vector space $\mathbb{R}$ (finite case) or in the Hilbert space of all square summable sequences (infinite case).  The main idea of this approach is that in some case, it is easier to compare two objects in terms of their simpler parts or features. As the other kernels, it is not necessary to explicit map an input data in the feature space, the only requirement is the calculation of the inner product between two input data in the feature space. The name \emph{convolution} came from the fact that the value of the kernel is obtained from a sum of products of other kernels, similar to the idea of convolution between function.

\begin{definition} \emph{Convolution Kernel}\\
Let  $x\in X$ a structured data,  $X_1, \ldots X_D$ non-empty separable metric spaces and $\overrightarrow{x}=(x_1, \ldots ,x_D)$ the subparts of $x$ (for example in a string a subpart could be a substring) with each $x_d \in X_d$ with $1 \leq d \leq D$. Consider the relation $R: X_1 \times \ldots \times X_D \times X$ where $R(\overrightarrow{x},x)$ is true if and only if  $x_1,\ldots,x_D$ are the subparts of $x$. Let $R^{-1}(x)=\{\overrightarrow{x} : R( \overrightarrow{x} , x) \}$ and $R$ is said finite if $R^{-1}(x)$ is finite for all $x \in X$.
Given two element $x,y \in X$ their decomposition $ \overrightarrow{x}=(x_1, \ldots ,x_D), \overrightarrow{y}=(y_1, \ldots ,y_D)  $ in  $X_1, \ldots X_D$, suppose that for each $X_d$
with $1 \leq d \leq D$ exists a kernel $K_d$, then the Convolution Kernel is defined as:

\begin{equation}
K(x,y)= \sum _{ \overrightarrow{x} \in R^{-1}(x), \overrightarrow{y} \in R^{-1}(y)} \prod_{d=1}^{D} K_d(x_d,y_d)
\end{equation}

The proof that $K$ is a valid kernel can be found in the original paper \cite{Haussler_1999_ConvolutionKernels}.
\end{definition}

\subsection{Tree kernels}
In the last years a variety of convolution kernel has been proposed for different kind of structured data, such as string, tree and graph \cite{Gartner_Structured_data_2003}, \cite{Gärtner_Graph_Kernel_2008},\cite{Camastra2008}. Here, only the main idea on kernels for trees will be presented, the interested reader can found a good characterization of tree kernels in the phd thesis by Kuboyama \cite{2010Kuboyama}. Tree kernels \cite{CollinsandDuffy2001} can be applied to ordered trees and  they compute the similarity between trees considering their common subtrees.
There are several kind of tree kernels but all of them share the same idea of decomposing, in the convolution kernel framework, a tree in different kind of subtree (for example simple subtree or co-rooted subtree). As an example, let us consider a particular convolution kernel: let  $x \in X$ a rooted and ordered tree and $X_1, \ldots X_D$  the set of all $D$-degree ordered and rooted trees. In this case the relation $R$ defined before is: $R(\overrightarrow{x},x) \Leftrightarrow  x_1,\ldots,x_D$ are the $D$ subtrees of the tree $x$.
in the following, one tree kernel used in context of Natural Language Parsing that exploit this idea and that has inspired several works on tree kernel (and also the MLA tree kernel) will be defined.

\begin{definition} \emph{Collins and Duffy Tree Kernel}\cite{CollinsandDuffy2001} \label{Collins_tree_kernel} \\
Given a tree $T$, and considering the enumerable set of all possible trees $\mathbf{T}=\{T_1,T_2, \ldots ,T_n\}$, $T$ can be represented by an n-dimensional vector where the $i$'th component contains the number of occurrences of the $i$'th tree $T_i$ of $\mathbf{T}$ in $T$. This mapping is done considering the function $h_i(T)$ that count the number of occurrences of $T_i$ in $T$. In this way it is possible to represent a tree $T$ as $\textbf{h}(T)=(h_1(T),h_2(T),\ldots,h_n(T))$. Note that the number $n$ could be huge because the number of subtree of a given tree $T$ is exponential on its size.
The kernel is then defined as:

\begin{equation}
K(T_1,T_2)= \textbf{h}(T_1) \cdot \textbf{h}(T_2) = \sum _{i} h_i(T_1)h_i(T_2) = 
\end{equation}

\begin{equation}
= \sum_{n_1 \in N_1} \sum_{n_2 \in N_2} \sum_{i} I_i(n1)I_i(n2) =  \sum_{n_1 \in N_1} \sum_{n_2 \in N_2} C(n_1,n_2)
\end{equation}

\noindent where $N_1$ is the number of node in $T_1$, $N_2$ is the number of node in $T_2$, $I_i(n)$ is an indicator function defined as:

\begin{equation}
 I_i(n)= \left \{
 \begin{array}{cc}
   1 & \mbox{if the subtree $T_i$ is seen rooted at node n}\\
   0 & \mbox{otherwise }\\
 \end{array}
 \right .
\end{equation}

\noindent and $C(n_1,n_2) =  \sum_{i} I_i(n_1)I_i(n_2) $

This kernel can computed in polynomial time, expressing  $C(n_1,n_2)$ with the following recursive definition:

\begin{itemize}
  \item if the productions at $n_1$ and $n_2$ are different:     $C(n_1,n_2)=0$
  \item if the productions at $n_1$ and $n_2$ are the same and $n_1$ and $n_2$ are pre-terminal nodes:     $C(n_1,n_2)=1$
  \item else if the productions at $n_1$ and $n_2$ are the same and $n_1$ and $n_2$ are not pre-terminal nodes:     \\
  \begin{equation}
   C(n_1,n_2)= \prod_{j=1}^{nc(n_1)}  (1+ C(ch(n_1,j),ch(n_2,j)))
  \end{equation}
  \noindent where $nc(n_1)$ is the number of children of $n_1$ in the tree (note that $nc(n_1)=nc(n_2)$
  because the productions are the same) and $ch(n_k,i)$ is the $i$'th son of node $n_k$ in a tree.

\end{itemize}

In the original paper some variant of this kernel is proposed to take into account some issues:
\begin{itemize}
  \item The value of kernels $K(T_1,T_2)$ depends strongly on the size of the trees $T_1$ and $T_2$. A possible
  solution is to use a new normalized kernel defined as:
  \begin{equation}
  \label{issues1}
  K'(T_1,T_2) = \frac{K(T_1,T_2)}{\sqrt{K(T_1,T_1) K(T_2,T_2)}}
  \end{equation}
  \noindent Note that K is still a kernel function because still satisfies the theorem \ref{th:Mercer_simple}.

  \item Since the number of subtree increases with size or depth,  it is necessary to scale the importance
   of each subtree taking in account their sizes:
  \begin{equation}
  \label{issues2}
  C(n_1,n_2)= \lambda \mbox{ and } C(n_1,n_2)= \lambda \prod_{j=1}^{nc(n_1)}  (1+ C(ch(n_1,j),ch(n_2,j))) \mbox{ with } 0 \leq \lambda \leq 1
  \end{equation}
  \noindent This correspond to the kernel:
    \begin{equation}
    \label{issues3}
  K(T_1,T_2)=  \sum _{i} \lambda ^{size_i} h_i(T_1)h_i(T_2)
  \end{equation}

  \noindent In order to obtain this result the parameter $0 \le \lambda \leq 1$ was introduced. In this way the kernel downweight the contributions of tree fragments exponentially with their size.

\end{itemize}
\end{definition}

\section{MLA Kernels}

\subsection{MLA Tree Kernel}
The MLA Tree Kernel is based on the MLA and in particular it is obtained using (1) the MLA on an input signal, (2) a particular aggregation rule that produce a tree from intervals and (3) a modified tree kernel adapted to the nature of the class of trees produced by the first two steps. A schematic view of the MLA Tree Kernel inserted on the whole process of Kernel Methods is depicted in figure \ref{fig:general_schema_mla_kernel}.

\begin{figure}[!htb]
\centering
\includegraphics [scale=.8]{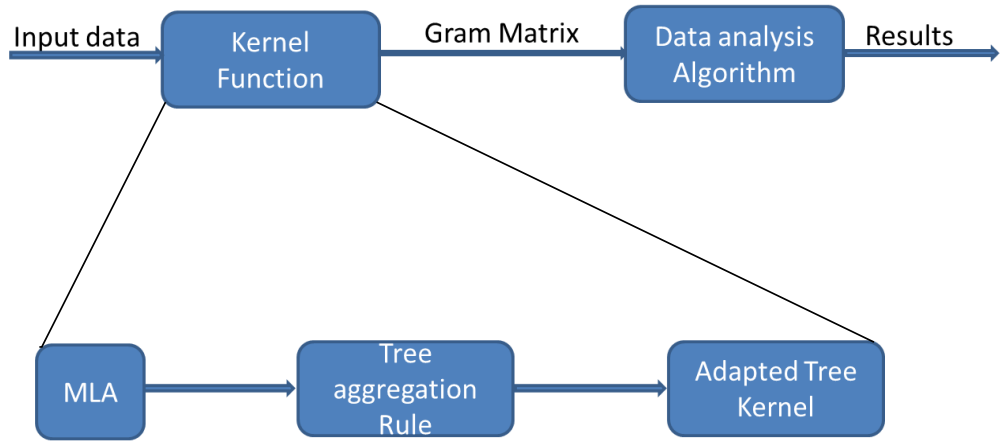}
\caption {General Schema of MLA Tree Kernel}
\label{fig:general_schema_mla_kernel}
\end{figure}

\subsubsection{From signal to tree}
\begin{definition} \emph{MLA tree aggregation rule}\\
Given a signal $f$ defined in $[a,b]$ and $K$ threshold operations $\sigma_k$ ($k=1,...,K$) after the application of Equally spaced simple MLA where the condition on each sigma is:

\begin{displaymath}
\sigma(x,\phi)=
\left\{\begin{array}{lr}
    f(x)    & \mbox{if } f(x) \leq \phi\\
    \phi       & \mbox{otherwise}\\
\end{array}
\right.
\end{displaymath}

\noindent it is possible to obtain the interval representation $\Upsilon(f)$ of $f$, recalling that $ \Upsilon(f) =\left \{ I_1,I_2,\cdots, I_K \right \}$
with $I_k=\left \{ i^1_k,i^2_k,\cdots, i^{n_k}_k \right \}$ the set of intervals corresponding to  $\sigma_k$.
To obtain a tree from the signal $f$ it is necessary to use its interval representation $ \Upsilon(f) $ using a particular aggregation rule on intervals.
It is necessary first to introduce a relation $R:I_k \times I_{k+1} $ with $I_k$ and $I_{k+1} \in \Upsilon(f)$. Given two intervals $i^s_k$ and $i^t_{k+1}$
they are in relation and it will be indicated as $R(i^s_k,i^t_{k+1})$ if and only if $i^t_{k+1}\subseteq i^t_k$.\\
Now, let us define the undirected tree $T=(V,E)$ such as:
\begin{equation}
\label{treedef}
V = I_0 \cup \bigcup_{i=1}^{K} {I_i} \mbox{ with } I_0=\{ r= [a,b] \}
\end{equation}

\noindent and \\
\begin{equation}
\label{treedef2}
 E=\{(i_1,i_2) \mbox{ with } i_1,i_2 \in V:R(i_1,i_2)\}.
\end{equation}

\end{definition}

\noindent In this way it is possible to define a labeled and rooted tree $T$ with root $r$ and in which each node encode the correspondent interval. The depth of the tree is exactly $K+1$ as it is necessary to  add the node $r$ that represents the interval $[a,b]$ where $f$ is defined.  It is possible to see an illustrative picture of the process in figure \ref{fig:general_schema_signal_to_tree}

\begin{figure}[!htb]
\centering
\includegraphics [scale=.8]{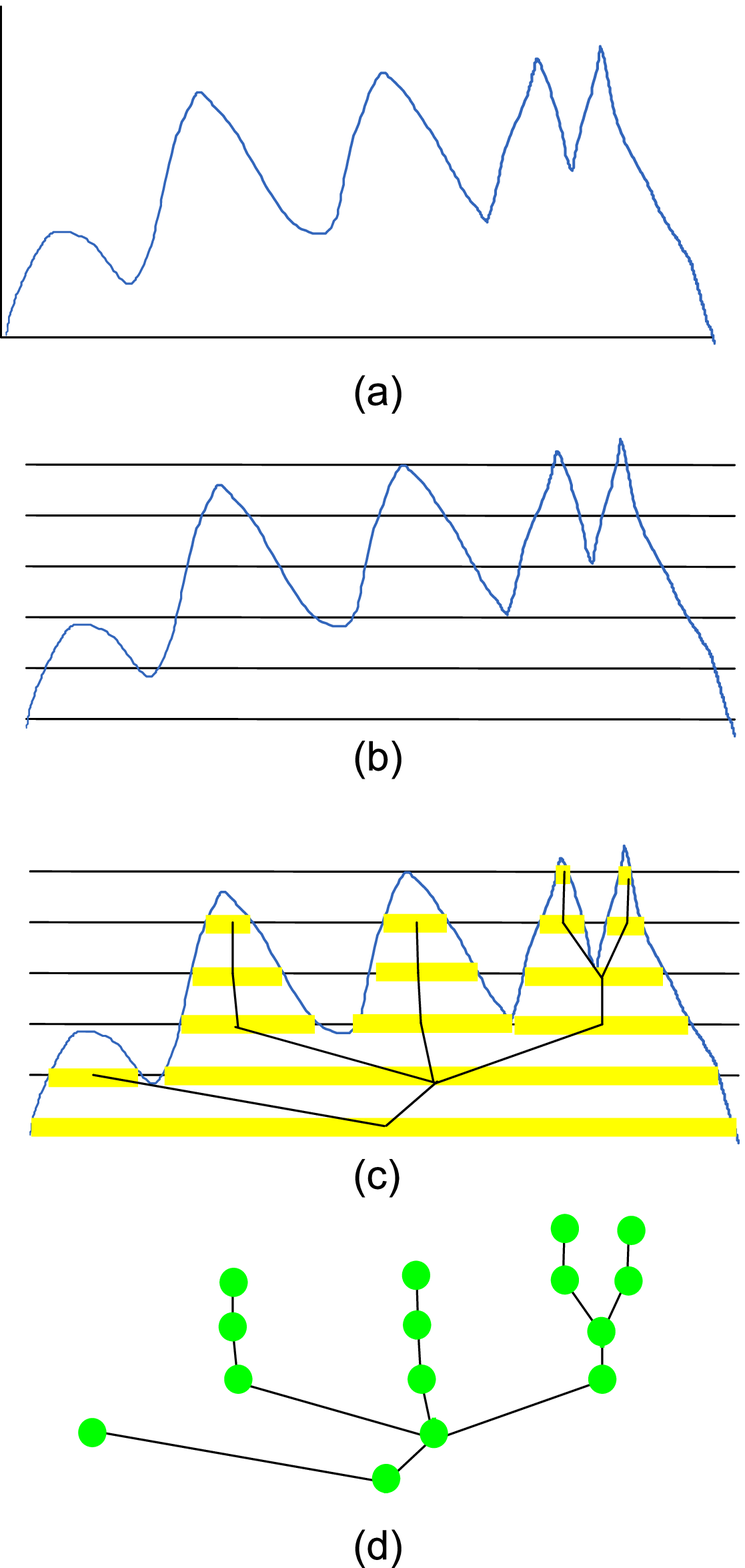}
\caption {General Schema of Kernel Methods}
\label{fig:general_schema_signal_to_tree}
\end{figure}

\subsubsection{Proposed Tree Kernel }
This kernel is defined starting from the tree $T$ previously defined in \ref{treedef} and \ref{treedef2}. The
idea behind this kernel is similar to the tree kernel proposed by Collins and Duffy introduced in section \ref{Collins_tree_kernel}.
In their original work they have used the tree kernel to characterize parse trees, here it is shown how adapt their approach to the set of tree obtained by MLA and representing the class
of one-dimensional signals defined for some interval $[a,b]$. The main idea of this kernel is to compare two signals using their tree representation. In the original kernel of
Collins and Duffy each node represent a production rule or a terminal symbol for some formal languages, here the nodes represent intervals.

\begin{definition} \emph{MLA Tree Kernel}\\
Using the same convention of tree kernel presented in \ref{Collins_tree_kernel}, the \emph{MLA tree kernel} is defined as:
\begin{equation}
K(T_1,T_2)= \textbf{h}(T_1) \cdot \textbf{h}(T_2) =  \sum_{n_1 \in N_1} \sum_{n_2 \in N_2} C(n_1,n_2,\delta)
\end{equation}

\noindent where $n_1$ and $n_2$ for simplicity of expression represent also the interval lengths associated to the nodes $n_1$ and $n_2$,  $\delta \in \mathbb{R}$ with  $0 <\delta<(b-a)$, and $C(n_1,n_2,\delta)$ recursively defined as:

\begin{itemize}
  \item if $n_1$ is a leaf and $n_2$ is not a leaf or viceversa then $C(n_1,n_2,\delta)=0$
  \item if $|n_1-n_2| > \delta $ and the intervals are pre-terminals (both fathers of a leaf) then $C(n_1,n_2,\delta)=0$ ($n_1$ and $n_2$ are considered different).
  \item if $|n_1-n_2| \leq \delta$ and the interval $n_1$ and $n_2$ are two leafs then $C(n_1,n_2,\delta)=1$ ($n_1$ and $n_2$ are considered equal).
  \item else if $|n_1-n_2| \leq \delta$  and the intervals $n_1$ and $n_2$ are not both fathers of a leaf then:     \\
  \begin{equation}
   C(n_1,n_2,\delta)= \prod_{j=1}^{nc(n_1)}  (1+ C(ch(n_1,j),ch(n_2,j),\delta))
  \end{equation}
  \
\end{itemize}

Note that this kernel suffers of the same issues as the Collins and Duffy tree kernel, for this reason it could be useful to consider the variant proposed in \ref{issues1}, \ref{issues2}, \ref{issues3}.Note also that here the node $n_1$ and $n_2$
\end{definition}

\subsection{MLA Convolution Kernel}

This kernel is defined starting from the interval representation of a signal trough the Equally Spaced MLA defined in Chapter $2$. In particular given $2$ signal $x,y$ and let $\Upsilon(x)=\{ Ix_1,Ix_2,\cdots, Ix_K \}$ and $\Upsilon(y)=\{ Iy_1,Iy_2,\cdots, Iy_K  \}$ their intervals representation with $K$ threshold operations.

\begin{definition} \emph{MLA Convolution Kernel}\\
Let $I$ a generic set of intervals from some interval representation of a signal of length $L$ and let define $B_I$ a signal of length $L$ with:

\begin{equation}
B_I(j)=
\begin{cases}
1 & \text{if $ \exists $ an interval $[a,b]\in I$ such that $j\in [a,b]$}\\
0 & \text{otherwise}
\end{cases}
\end{equation}
\noindent with $1 \leq j \leq L$. In this way to a generic interval representation it is possible to associate a set of binary string.

Finally the kernel is defined as:
\begin{equation}
S(x,y)=  \sum \limits _{k=1+hnp} ^{K-hnp+1} {\frac{1}{np} \left [ \left ( \sum \limits _{j=k-hnp+1} ^{k+hnp-1} B_{Ix_j}   \right ) \left( \sum \limits _{j=k-hnp+1} ^{k+hnp-1} B_{Iy_j} \right ) \right ]}
\end{equation}

\noindent where $0 \le \gamma \leq 1$ and $np= \left | \gamma*K \right |$ and $hnp=\frac{np}{2}$.
\end{definition}

This kernel function can be seen as a local correlation between correspondent internal portions of the signals and in which the size of the portion is controlled by the parameter $\gamma$.

\section{Support Vector Machines}
Support Vector Machines (SVM) are learning systems that use an hypothesis space of linear functions in an high dimensional space, trained with a learning algorithm for optimization motivated from statistical learning theory \cite{Cristianini_SVM}. SVM are binary classifiers; in particular the discriminative function of the SVM represent
a linear decision boundary also called margin. More formally, a SVM constructs an hyperplane in a high (eventually infinite) dimensional space, using the implicit projection of the kernel functions, in order to obtain a good separation between positive and negative points. In particular SVM  consider the hyperplane that has the largest distance to the nearest training data points of any class since in general the larger the margin the lower the generalization error of the classifier. In figure \ref{fig:margin} it is possible to see the concept of margin and the hyperplane (a straight line in 2 dimensions).
The interested reader can found a good survey of SVM classifiers in \cite{Cristianini_Kernel}.

\begin{figure}[!htb]
\centering
\includegraphics [width=0.6\textwidth]{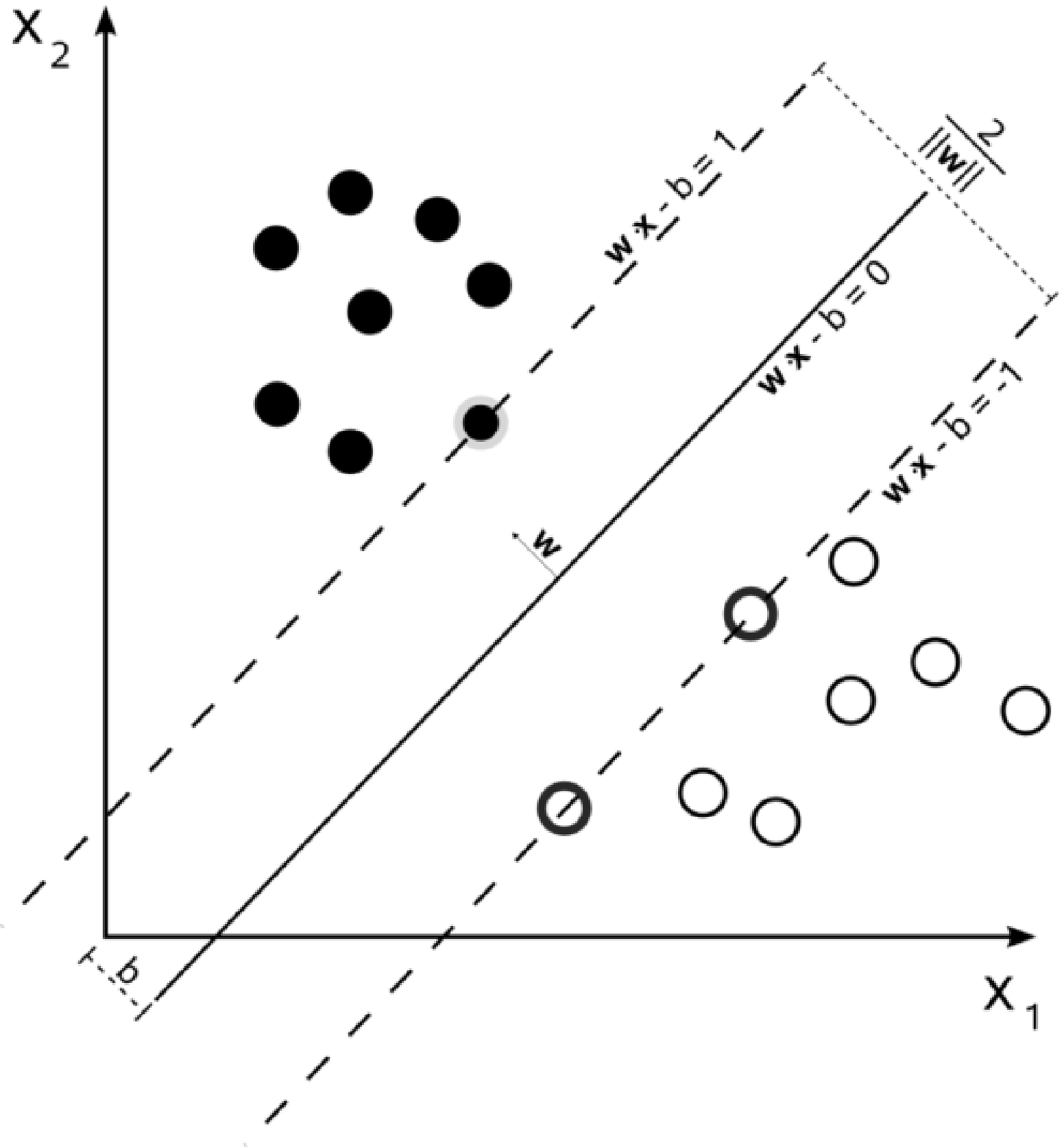}
\caption {SVM margin and the separation hyperplane}
\label{fig:margin}
\end{figure}

\section{Experimental Setup}
In this section  three experiments that use the MLA Tree kernel will be presented, in particular the first two involve a classification, while the third is related to clustering.

\subsection{Synthetic data: discrimination power of MLA Tree Kernel on basic functions}
To validate MLA Tree Kernel, three basic signals that can be characterized in term of shape in time domain, has been considered (see figure \ref{fig:syn}):
\begin{itemize}
  \item sinusoid signal
  \item rectangular pulse signal
  \item sawtooth signal
\end{itemize}

As training set $S$, $N$ signals have been generated with an increasing linear SNR noise value ranging from  $0.1$ to $1$, for each of the three categories. In this way, one dispose of a training set with $3 \times N$ elements and with $3$ classes. Analogously a Test Set $T$ disjointed from $S$ was taken into account, with the same cardinality i.e. $3 \times N$. To validate the performances, a Support Vector Machine with different kernel functions has been considered: linear, polynomial, RBF, sigmoid and MLA Tree. The results obtained with $N=50$ and with different kernels are shown in table \ref{tab:syn}. As it it is possible to see all the kernels obtain very good performances although in the case of very noisy signal the MLA Tree Kernel can still recover the shape information leading to a slightly better result. This makes the MLA tree kernel more robust to noise than the other kernels.

\begin{table}[t]
\centering
\begin{tabular}{|l||c||c|}
\hline
Kernel Function & Correctly Classified & Accuracy \\
\hline
MLA Tree &   150/150  & 100\% \\
Linear & 143/150 & 95\%\\
Polinomial(2) & 131/150 & 87\%\\
RBF &  130/150& 87\%\\
Sigmoid & 141/150 & 94\%\\
\hline
\end{tabular}
\caption{Classification accuracy on basic functions dataset.}
\label{tab:syn}
\end{table}

\begin{figure}[!htb]
\centering
\includegraphics [width=0.8\textwidth]{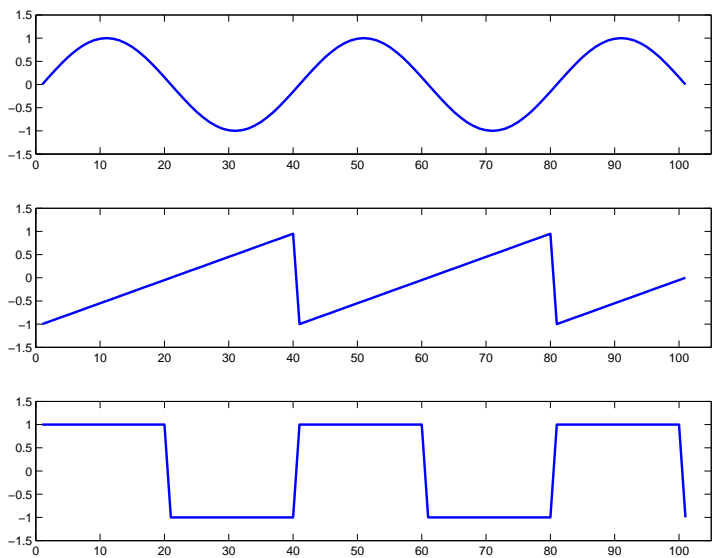}
\caption {Basic function}
\label{fig:syn}
\end{figure}

\begin{figure}[!htb]
\centering
\includegraphics [width=0.8\textwidth]{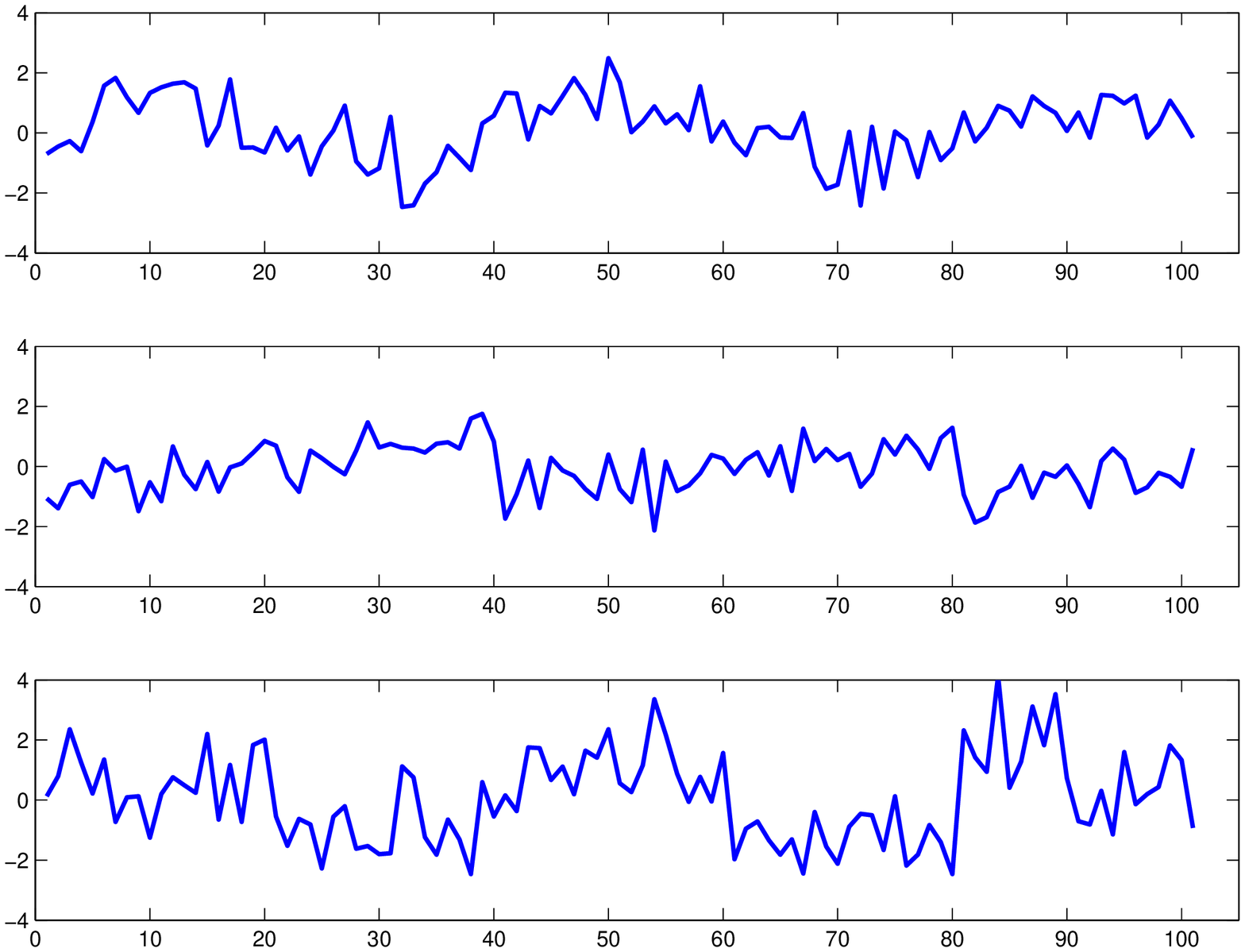}
\caption {Basic function plus noise}
\label{fig:syn}
\end{figure}

\subsection{Synthetic data: MLA Tree Kernel on waveform dataset}
In this experiment the dataset from \cite{Breiman} was considered. It contains $5000$ instances divided in $3$ classes of waves of $21$ attributes, all of which include gaussian noise with mean $0$ and variance $1$.  In particular, each class is generated from a combination of $2$ of $3$ ``base'' waves. The best accuracy that has been obtained processing this dataset has been reached by the Optimal Bayes classifier, with a value of $86\%$. Here the dataset was split in two balanced parts (training and test sets) of $1500$ elements equally distributed into the tree classes for evaluating the performances of MLA Tree kernel with a SVM classifier. In particular as in the previous experiment, linear, polynomial, RBF, sigmoid and MLA Tree kernels functions have been used. In the table \ref{tab:waveform} results are shown. As it it is possible to see all the kernels obtain very good performances.

\begin{table}[t]
\centering
\begin{tabular}{|l||c||c|}
\hline
Kernel Function & Correctly Classified & Accuracy \\
\hline
MLA Tree &   1364/1500  & 91\% \\
Linear & 1286/1500 & 86\%\\
Polinomial(2) & 1187/1500 & 80\%\\
RBF &  1286/1500& 86\%\\
Sigmoid & 795/1500 & 53\%\\
\hline
\end{tabular}
\caption{Classification accuracy on waveforms dataset.}
\label{tab:waveform}
\end{table}

\subsection{Assessment of induced distance of MLA Convolution Kernel for clustering of seismic signal}
The dataset taken in exam for this experiment consists of $n$ undersea explosion of an array of bombs at different distanced from a ship. This dataset was builded in order to have a well characterized set of signals to use as a benchmark for problems involving geological signals. In particular, the ship record for each explosion at time $t_i$ a signal $s_i$ that express the variation on pressure level. The explosions take place at regular intervals of $300$ seconds and each signal is sampled at $100$hz. A particularity of this dataset, as it is possible to see in figure \ref{fig:ship_bomb}, is that close temporal explosions occurs at similar distances from the ship. This means that given a signal  $s_i$, with high probability the most similar signal in term of shape is the signal  $s_{i+d}$ with $d$ close to $1$ or $-1$ i.e. a signal recorded in proximity of instant $t_i$. This property allows to test in a natural way the performances of a similarity or dissimilarity function comparing the "order" that it induces on the set of signals. In particular let $s_1, \ldots ,s_n $ the set of signals recorded  at starting time $t_1, \ldots, t_n$ respectively, and the natural order of the signals can be represented by the permutation $P=(1,2, \ldots, n)$. Given a generic distance $d$, let $D$ the $n \times n$ distance matrix containing all the pairwise distances between the signals i.e. $D_{i,j}=d(s_i,s_j)$ with $1 \leq i,j \leq n$. A measure of goodness of distance, can be defined by the distance optimality function so defined:

\begin{figure}[!htb]
\centering
\includegraphics [scale=1]{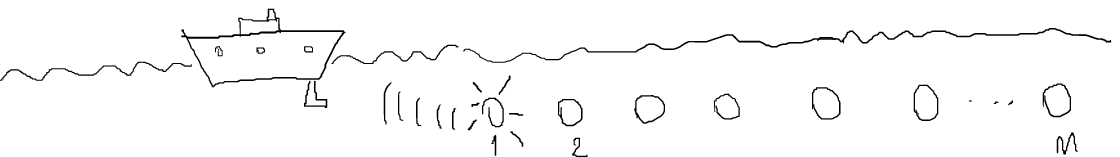}
\caption {Schema of the experiment}
\label{fig:ship_bomb}
\end{figure}

\begin{definition}\emph{Distance Optimality}\\
Given a distance $d$ and a dataset $S$ of size $N$, and let $D$ the pairwise distance matrix with $D_{i,j}=d(s_i,s_j)$, $s_i,s_j \in S$ and $1 \leq i,j \leq n$, the distance optimality of $d$ is  defined as:

\begin{equation}
    do= \sum _{i=1}^{n} \frac{|i-j-1|}{n-2} \mbox{  with    } j=\underset{1 \leq k \leq n, k \neq i}{\operatorname{argmin}} D_{i,k}
\end{equation}

\end{definition}

What is expected, in the case of a good distance measure, is a $do \approx 0$. It was assessed the performances of the distance induced by MLA Tree Kernel (using the equation \ref{eq:distance} on its Gram's matrix) and compared its results with two common distances i.e. Euclidean distance and Spearman correlation distance by the distance optimality function. Note that the used Spearman correlation distance is defined as $1-r$ where $r$ is the Spearman correlation index defined in \ref{chap:2} by equation \ref{eq:Pearson}. The results of this analysis are shown on table \ref{tab:do}. As it is possible to see, the induced distance from MLA Convolution Kernel can exploit better the natural similarity between signals than the other classic measures.\\\\

\begin{table}[t]
\centering
\begin{tabular}{|l||c|}
\hline
Distance & Distance Optimality \\
\hline
MLA Convolution &  0.2369 \\
Euclidean  & 0.3889\\
Pearson Correlation & 0.2813\\
\hline
\end{tabular}
\caption{Distance optimality on geological signals}
\label{tab:do}
\end{table}

This chapter has shown how the data extracted by MLA can be optimally organized in a tree of intervals, encoding the shape properties of a signal, using a particular aggregation rule. It was shown also an example of kernel trees properly adapted to be used with this tree representation induced by MLA. In addition another convolution kernel and based on local correlations was introduced. The first results are encouraging although it is necessary to do a more systematic study on the class of kernel functions that can be induced by the proposed aggregation rule on the interval representation and also on their properties and extensions. The major suggestion of the study carried out in this chapter is the connection between the class of algorithms on trees and graph and the class of digital signal processing technique. In fact, the MLA transformation can be useful to search for relation between operation on trees and graph and signal manipulation in time or frequency domain.

\chapter{Conclusions and Future Directions}
\label{chap:conclusion}

This thesis has  introduced a new methodology called Multi Layer Analysis (MLA), and its use on several contexts such as Pattern Discovery, Classification, Clustering and also Test of Randomness. In chapter 3, 4, and 5 several application domains related to these problems have been faced with the MLA approach.  In some sense, the use of MLA can be considered as a general boosting step to improve classic algorithms in the fields of classification or clustering. The main idea behind MLA is the transformation from the space of one-dimensional signals into a new space called the space of intervals in which a more detailed analysis could be performed.\\

In particular, in chapter 3 it has been shown that, by using particular aggregation rules on such space, it is possible to characterize different signal shapes; this allows to approach some key problems in biology i.e. the nucleosome spacing problem.\\

Moreover, in chapter 5 it has been proposed another aggregation rule that is capable to represent a one-dimensional signal in terms of a tree of intervals, and thus permits to express or characterize any kind of shape. This point has strong implications since it establishes a connection between the class of algorithms that process one-dimensional signal such as digital signal processing techniques, and algorithms on trees and graphs. This result is really important because it makes possible the application of particular transformations on a one-dimensional signal, modifying its tree representation and viceversa. In this sense further investigation in this direction will be performed.\\

The final consideration is that MLA can be fruitfully applied on problems that involve the processing of one-dimensional signals, such as Geology, Biomedicine, Biology and other disciplines. In some cases MLA on such problems have comparable or sometimes superior performances to other methodologies currently applied for the same purposes. Further investigation on MLA properties and its extension to multidimensional data will be investigated.\\

\bibliographystyle{plain}
\bibliography{Thesis}
\end{document}